\documentclass[11pt,a4paper]{article}
\usepackage[T1]{fontenc}
\usepackage[margin=23mm,headheight=15pt]{geometry}
\usepackage{amsmath,amssymb,amsthm,mathtools}
\usepackage{graphicx,booktabs,tabularx,array,longtable,caption,alphalph}
\usepackage[numbers,sort&compress]{natbib}
\usepackage{xcolor,microtype,enumitem,fancyhdr}
\definecolor{inkblue}{HTML}{173D59}
\usepackage[breaklinks,colorlinks,linkcolor=inkblue,citecolor=inkblue,urlcolor=inkblue]{hyperref}
\usepackage[capitalize,nameinlink,noabbrev]{cleveref}
\hypersetup{pdftitle={Appearance Compensation in Volume Rendering},pdfauthor={Jianghong Zhou},pdfsubject={Density identifiability, physical ambiguity, appearance capacity and finite measurements},bookmarksopen=true,bookmarksnumbered=true}
\newtheorem{theorem}{Theorem}[section]
\newtheorem{lemma}[theorem]{Lemma}
\newtheorem{corollary}[theorem]{Corollary}
\newtheorem{proposition}[theorem]{Proposition}
\theoremstyle{definition}\newtheorem{assumption}[theorem]{Assumption}
\theoremstyle{remark}
\newcolumntype{Y}{>{\raggedright\arraybackslash}X}
\newcommand{\R}{\mathbb R}
\newcommand{\Sph}{\mathbb S}
\newcommand{\cR}{\mathcal R}
\newcommand{\cH}{\mathcal H}
\newcommand{\cV}{\mathcal V}
\newcommand{\cA}{\mathcal A}
\newcommand{\eps}{\varepsilon}
\newcommand{\dd}{\,\mathrm d}
\newcommand{\norm}[1]{\left\|#1\right\|}
\newcommand{\avg}[1]{\left\langle#1\right\rangle}
\newcommand{\supp}{\operatorname{supp}}
\newcommand{\sym}{\operatorname{sym}}
\newcommand{\Id}{\operatorname{Id}}

\newcommand{\Div}{\operatorname{div}}
\newcommand{\F}{\mathcal F}
\newcommand{\T}{\mathsf T}
\newcommand{\one}{\boldsymbol 1}
\numberwithin{equation}{section}
\graphicspath{{figures/}}
\makeatletter
\renewcommand{\@pnumwidth}{2.7em}
\renewcommand{\@tocrmarg}{3.7em}
\renewcommand*{\l@section}{\@dottedtocline{1}{0em}{2.8em}}
\renewcommand*{\l@subsection}{\@dottedtocline{2}{2.8em}{3.5em}}
\makeatother
\setlist{topsep=5pt,itemsep=4pt,parsep=0pt}
\begin{document}
\raggedbottom
\pagenumbering{roman}
\hypersetup{pageanchor=false}
\begin{titlepage}
\centering
\vspace*{15mm}
{\sffamily\small\color{inkblue} RESEARCH MONOGRAPH / PREPRINT}\par
\vspace{13mm}
{\Huge\bfseries Appearance Compensation\\[5pt] in Volume Rendering\par}
\vspace{15mm}
{\Large Jianghong Zhou\par}
\vspace{8mm}
{October 10, 2026\par}
\vspace{13mm}
\begin{minipage}{.93\textwidth}
\small
\begin{abstract}
Can a volume that reproduces every external RGB view still have an incorrect density? We study this question in emission--absorption rendering with an independently parameterized, physically bounded effective source. Appearance compensates a density change by replacing the radiance that it removes from the rest of each ray. We characterize the density information surviving this compensation, while retaining a common spatial appearance field across observations.

Our first result is an optimal density-only gain after finite-degree spherical-harmonic source refitting. A coupled tensor gauge and physical localized packets identify the sharp limiting nonlinear inverse coefficient as its reciprocal. Our second result realizes positive gain at every fixed finite degree for a residual set of isotropic spatial sources at one fixed density. At the same references, growing angular capacity hides a fixed nonzero density error while retaining physicality and any prescribed finite set of smoothness bounds. Local identifiability and the absence of a capacity-uniform recovery modulus therefore coexist. Our third result constructs an exact physical SH2 ambiguity: central density doubles and relative density $L^2$ error exceeds $14.52\%$, yet every external RGB ray is unchanged.

We connect these continuum results to shared finite representations, continuous pixel measurements and density-sensitive acquisition. Complete-class certificates separate SH0, SH1 and SH2 fits at one learned density. Finite-camera results expose the distinct role of spatial capacity, and transmission measurements separate otherwise identical RGB fields. Controlled training tests and captured-smoke profiles supply mechanism evidence with explicit observation scopes. Together the results make appearance capacity, spatial resolution and measurement design part of the conditions for recovering density, rather than properties that can be inferred from rendering accuracy alone.
\end{abstract}
\end{minipage}
\vfill
{\small\color{inkblue} Complete statements, proof ideas, technical proofs and experimental evidence}\par
\end{titlepage}
\hypersetup{pageanchor=true}
\pdfbookmark[0]{Contents}{contents}
\tableofcontents
\clearpage
\pagenumbering{arabic}
\part{Research question, contributions and model}

\section{Introduction}\label{m:intro}
Does agreement from every viewpoint determine a volume's density? The question matters when a reconstructed field is reused. Compositing queries transmission, and volumetric measurement queries where absorption occurs. These operations depend on density even when the captured RGB is reproduced exactly.

NeRF \cite{mildenhall2020} and 3D Gaussian splatting \cite{kerbl2023} fit geometry and directional appearance jointly. NeRF++ \cite{zhang2020} analyzes shape--radiance ambiguity and relates successful reconstruction to a preference for smooth view dependence. This motivates a quantitative question: \emph{how much density evidence survives the appearance model's best explanation?}

In emission--absorption rendering on a black background, added density $h\geq0$ blocks the attenuated radiance $U$ arriving from behind each point. Increasing the effective source $j$ by $hU$ replaces precisely that radiance:
\begin{equation}\label{m:exactcomp}
 \cR(\sigma+h,j+hU)=\cR(\sigma,j).
\end{equation}
The identity retains the changed attenuation. Its transport proof and physical bounds are given in \cref{lem:transport}. Whether a finite-degree appearance class can supply this compensation is the central issue: density evidence remains where its angular functions cannot imitate the blocked radiance.

Three results establish how appearance capacity changes what RGB can determine about density.
\begin{itemize}
\item \textbf{Low degree already permits exact ambiguity.} Bounded SH2 appearance makes two volumes with doubled central density agree on every external RGB ray. Additional views, pixel averaging and image-only losses preserve this equality.
\item \textbf{Recovery and hiding coexist at one scene.} The sharp compensated gain $\kappa_L$ determines local nonlinear conditioning. A topologically generic source family admits localized recovery at each fixed degree, with known exterior and degree-dependent neighborhoods, while increasing degree hides a fixed density error. The limiting inverse coefficient grows faster than every polynomial.
\item \textbf{Training learns compensation; attenuation reveals it.} Appearance interventions and complete-class refits expose compensation at the same learned density. Captured-smoke profiles show how isotropic refitting flattens density-scale sensitivity while opacity changes. Switching the background with source fixed separates the exact pair through transmission.
\end{itemize}
Together, these findings make appearance capacity part of the conditions for density recovery. They motivate evaluating representations by the physical quantities preserved under refitting, and designing observations for the intended reuse.

\begin{figure}[t]\centering
\includegraphics[width=\textwidth]{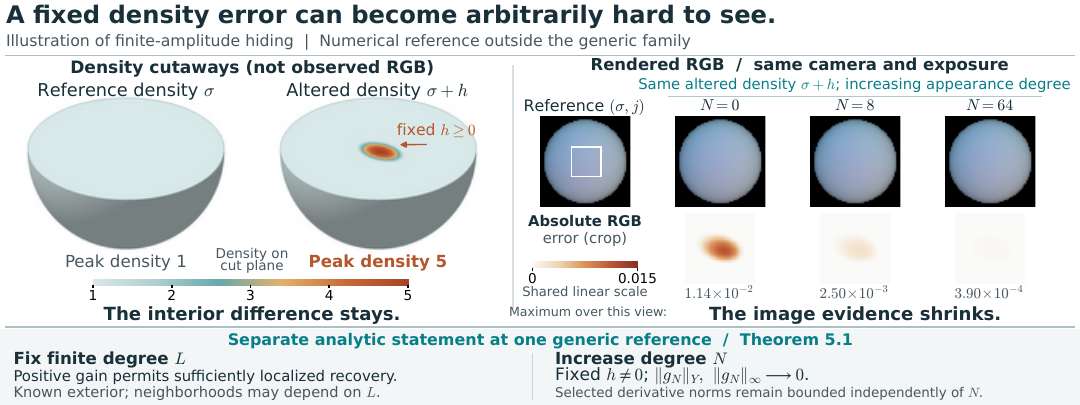}
\caption{\textbf{Appearance compensation at fixed incorrect density.} The density addition is unchanged across the appearance sequence; even $N=0$ includes a compensating source. Error crops share one linear scale and report sampled maxima. This is an affine-reference illustration, separate from the residual source class in \cref{m:coexistence}. The latter is established analytically by the source-realization argument.}\label{m:teaser}
\end{figure}

\subsection{What the theory changes about reconstruction}
Image formation maps a density and an appearance field to a collection of images. Density recovery asks whether this map distinguishes densities after appearance is allowed to change. These are different requirements: agreement of images controls a physical field only through the allowed explanations of those images. The central scientific object is therefore the best residual that remains after a \emph{shared} appearance field has absorbed the effect of a density change.

Three distinctions organize the analysis. First, an opaque surface location is a first-intersection problem, whereas a volume density controls both emission weight and transmission along an entire ray. Second, positive local gain at one finite angular degree does not supply a bound uniform over increasing degrees. Third, complete-ray observations and a finite camera array leave different freedoms to the spatial appearance field. Each distinction has a theorem-level consequence in this paper.

The positive and negative results answer complementary questions at one reference. The positive result identifies when every sufficiently localized density change leaves a quantitatively detectable signal after all permitted source changes. The negative result fixes a nonzero density change and approximates its exact compensation with increasing angular degree. The order of these quantifiers is the reason the conclusions coexist. A separate SH2 construction shows that exact all-view ambiguity can already occur at a low finite degree in a structured physical scene.

For downstream use, this moves the evaluation target from images alone to the quantity being reused. Density-dependent transmission, compositing and material measurement need evidence that constrains density across image-compatible explanations. An angular regularizer or a training schedule may prefer one explanation; the recovery theorems instead characterize information contained in the declared observations and representation.

\subsection{Contributions and the specific mathematical advances}
\begin{enumerate}[label=\textbf{\arabic*.}]
\item \textbf{Optimal density conditioning under shared angular refitting.} We derive a weighted great-circle residual that is sharp over the full spatial source class, and prove its role as the limiting nonlinear inverse coefficient. The substantive passage is from a frozen density-only projection to an actual localized nonlinear problem with source gauge freedom and physical constraints (\cref{m:thmfrozen,m:thmnonlinear,m:thmfragility}).
\item \textbf{A same-reference separation of identifiability and accuracy.} We realize generic positive gain by perturbing actual isotropic spatial sources at fixed density, and combine it with physical finite-amplitude hiding at those same references. This establishes a single-scene coexistence statement and super-polynomial growth of the limiting inverse coefficient (\cref{m:coexistence,m:superpoly}).
\item \textbf{Exact finite-degree physical ambiguity.} A triangular attenuated Abel--Volterra construction produces genuine SH2 alternatives. A continuous physical enclosure certifies the exact solution at doubled central density, rather than only a sampled fit or an infinitesimal null direction (\cref{m:thmseparation,m:twofold}).
\item \textbf{Shared-representation and finite-observation consequences.} We make the spatial sharing and observation operator explicit, derive query-dependent physical image classes, and give full-class residual gaps and continuous-renderer recovery/ambiguity certificates at fixed reference. We also state when smooth finite models transfer the continuum guarantee to positive pixel averages, and when spatial freedom instead preserves exact finite-data ambiguity (\cref{long:finite,rs:theorem,long:pixeltransfer,long:gridcomp}).
\item \textbf{Evidence tied to what is being identified.} Fixed-density refit certificates isolate angular capacity; interventions identify learned compensation; captured-smoke profiles demonstrate isotropic density-scale compensation; paired background measurements identify transmission under source preservation. The experiments distinguish these claims from generalization or convergence claims (\cref{m:numerics,long:evidence,m:discussion}).
\end{enumerate}
The contributions are the stated recovery results, physical constructions and model-specific certified consequences. Their proofs use established transport calculus, tensor analysis, approximation, transversality, convex duality and statistical testing. Specific uses are cited where they enter; no new general-purpose version of those tools is needed for the claims.

\section{A map of the claims and their evidence}\label{long:claims}
This section makes the paper's assertions explicit before the technical development. \textbf{T} denotes an analytic theorem or corollary, \textbf{C} a theorem with a computer-assisted enclosure for the stated instance, and \textbf{E} an empirical finding. The type concerns the evidence supporting the statement; a numerical illustration is never used in place of a continuum proof. The source and density classes in the linked statements are part of each claim.

\small
\begin{longtable}{@{}p{.06\textwidth}p{.06\textwidth}p{.60\textwidth}>{\raggedright\arraybackslash}p{.19\textwidth}@{}}
\toprule ID & Type & Claim and information supplied & Statement / proof\\\midrule\endfirsthead
\toprule ID & Type & Claim and information supplied & Statement / proof\\\midrule\endhead
\bottomrule\endfoot
\multicolumn{4}{@{}l}{\textbf{Core continuum theory}}\\[3pt]
C1&T&The compensated weighted great-circle gain is the optimal frozen density-only lower bound, including in any prescribed nonempty support ball.&\cref{m:thmfrozen,thm:frozen}\\[6pt]
C2&T&Positive gain gives local nonlinear density recovery with fixed exterior, small support, physical margins and the complete permitted finite-degree source class.&\cref{m:thmnonlinear,thm:nonlinear}\\[6pt]
C3&T&Physical two-scale packets make $1/\kappa_L$ a necessary inverse and asymptotic deterministic minimax coefficient. Shrinking admissible neighborhoods attains it.&\cref{m:thmfragility,sec:fragility,cap:S}\\[6pt]
C4&T&At fixed smooth density, a residual class of physical isotropic spatial sources has positive gain at every point in a fixed interior compact set and every finite degree.&\cref{m:coexistence,cap:G}\\[6pt]
C5&T&At the same references, a fixed nonzero density addition is hidden in both $Y$ and maximum RGB by physical finite-degree sources as degree grows, with finite smoothness controls.&\cref{m:coexistence,cap:R,cap:synthesis}\\[6pt]
C6&T&The optimal local inverse coefficients grow faster than every polynomial for these smooth references, with fixed-degree localization taken first.&\cref{m:superpoly,cap:S}\\[6pt]
\multicolumn{4}{@{}l}{\textbf{Physical constructions and consequences}}\\[3pt]
C7&T&A quadratic reference has exact, genuinely SH2, arbitrarily local physical partners with different densities and the same complete oriented-ray RGB.&\cref{m:thmseparation,sec:abel,app:physical}\\[6pt]
C8&C&An exact partner doubles central density and has relative density $L^2$ separation in $[.1452219,.1459197]$, with continuous source and absorption margins.&\cref{m:twofold,app:finitecertificate}\\[6pt]
C9&T/C&Image averaging and fixed image-only processing preserve exact equality. Further physical families and higher-optical-depth instances extend the construction.&\cref{m:rgbfiber,app:extensions}\\[6pt]
C10&C&Specified same-norm calibrations quantify the conditioning loss: $23.12$ to $82.92$ at the quadratic reference, and more than $16.13\times$ in an asymmetric reference.&\cref{app:calibration,app:radius,app:circlecertificate}\\[6pt]
\multicolumn{4}{@{}l}{\textbf{Shared finite representations and measurements}}\\[3pt]
C11&T&Finite shared-query image classes depend on query geometry. Antipodal SH1 queries can erase information retained by repeated directions; physical interpolation has explicit resource conditions.&\cref{q:Q1,q:antipodal,q:twolayer,q:pairdegree,q:circledegree}\\[6pt]
C12&C&At one learned density and fixed finite renderer, optimal physical SH0--SH1 and SH1--SH2 maximum-RGB residual intervals are strictly separated.&\cref{m:querytable,lr:dual,app:learnedrefit}\\[6pt]
C13&C&One fixed 27-node continuous-ray model has local full-class SH1 density stability and exact local SH3 partners at the same reference and all 384 rays.&\cref{rs:continuouslocal,rs:capacitytransitionproof}\\[6pt]
C14&C&In that model, a finite-amplitude alternative exceeds $.10$ density $L^2$ separation; separate pairwise and prior/noise certificates specify what a fit can resolve.&\cref{rs:coupledwitness,rs:continuouspairwise,rs:priornoise}\\[6pt]
C15&T&Relative sampling preserves a full-appearance density gain with all bounded source changes retained, including source kernels and changing-rank directions.&\cref{rs:theorem,rs:sharedproof,long:pixeltransfer}\\[6pt]
C16&T&Fixed finite rays or suitable continuous pixels admit exact local SH0 compensation with enough spatial freedom; a sufficiently fine but fixed physical color grid can inherit it.&\cref{rs:spatialfreedom,rs:pixelspatial,long:gridcomp}\\[6pt]
C17&T&A specified full-ray covering proof has a camera-count obstruction. A one-camera, two-parameter physical counterexample shows why that obstruction is not a universal inverse-problem lower bound.&\cref{long:cameracost}\\[6pt]
C18&T/C&Source-preserving background pairs identify transmission, distinguish the exact SH2 pair, and give a necessary/sufficient cost bracket for a declared finite-model reference-exclusion target.&\cref{bg:footprint,ac:costtheorem}\\[6pt]
C19&C&Under an explicitly assumed independent Gaussian readout law, physical pairs yield finite-sample density-risk lower bounds; these quantify statistical ambiguity for the specified measurements.&\cref{rs:gaussianpairs,rs:gaussianproof}\\[6pt]
\multicolumn{4}{@{}l}{\textbf{Experimental findings and interpretation}}\\[3pt]
C20&E&Training on controlled volumes learns appearance that masks density errors. Matched true-density controls distinguish the true-color intervention from directional ablation alone.&\cref{app:trainedcompensation,app:densityoracle}\\[6pt]
C21&E&Captured-smoke SH0 refits flatten scale sensitivity: test PSNR varies by at most $.33$ dB for sampled scales $1.5,2,3$, while predicted maximum opacity change grows from $.15$ to $.38$.&\cref{smoke:scaleprofile,smoke:grayreference}\\[6pt]
C22&E&Native-field refits show local prediction and modest renderer-specific compensation. Matched smoke and neural-field controls do not support a reproducible additional angular-compensation effect on real scenes.&\cref{smoke:angularreplication,app:officialrefit,app:sharedshdiagnostic}\\[6pt]
\end{longtable}
\normalsize

\subsection{Which components carry the theoretical contribution?}
The leading claims are C1--C8. C1--C3 identify the optimal density quantity and connect it to the actual nonlinear inverse problem. C4--C6 place positive local recovery and finite-amplitude hiding at a common physically realizable reference, with explicit order of quantifiers. C7--C8 establish exact low-degree physical instances. These are the central theorem-level statements of the work.

C9--C19 develop the consequences for measurement, representation and reuse. Some are direct corollaries of the core theory; others are application-specific constructions or certified finite models. Their interest lies in the precise physical or shared-representation statement, not in claiming invention of projection, interpolation, Sard's theorem, matrix perturbation, or a two-point testing inequality. C20--C22 test specified manifestations of compensation and delimit empirical transfer.

\subsection{How to read the monograph}
For the main theoretical argument, read \cref{m:model,m:criterion,m:capacitytheorem,m:separation}. Each central theorem has a proof sketch there and a link to its full technical proof. For finite representations and practical observation scopes, read \cref{long:finite,long:measurement}. For the evidence, read \cref{m:numerics,long:evidence}; the supporting appendices retain complete primary outcomes and the controls needed to interpret them. The full contents and the claim identifiers provide complementary routes through the paper.

\section{Related work}\label{m:related}
\paragraph{Appearance and geometry in learned fields.}
NeRF++ \cite{zhang2020} analyzes shape--radiance ambiguity and appearance priors. Ref-NeRF \cite{verbin2022} treats reflective-scene artifacts through structured appearance and normal regularization. VDN-NeRF \cite{zhu2023} studies the directional-capacity/geometry tradeoff, including the distortions caused by insufficient capacity. Fang et al.\ \cite{fang2023} estimate SH color from density and images to regularize geometry. PINF \cite{chu2022} uses opacity penalties and progressive capacity for fluid reconstruction; BayesRays \cite{goli2024} estimates learned-field uncertainty. These studies establish appearance modeling as a central part of geometric reconstruction.

\paragraph{Identifiability in rendering and tomography.}
Wang--Miller \cite{wang2026} relate texture parallax, surface misalignment and angular fitting capacity in a first-hit surface model. Joint attenuation/source identification has a broader inverse-problem history: Stefanov \cite[Sections~3, 5 and 7]{stefanov2014} analyzes planar emission tomography, deriving joint transport identities, conditional local recovery and radial nonuniqueness. Holman--Richardson \cite{holman2023} give recovery conditions for piecewise-constant attenuation, while Courdurier et al.\ \cite{courdurier2015} study additional single-scattering data. These results provide the setting for asking how geometry information depends on the admissible appearance class.

\paragraph{Finite angular sources and density information.}
Attenuated tensor tomography characterizes finite angular sources and potential kernels at fixed optical parameters \cite{psu2012,assylbekov2012,monard2018,balmonard2019}; tensor decomposition and extension estimates support stability analysis \cite{sharafutdinov1994,pz2016}. Our central object is the optimal \emph{density-only} information remaining after shared finite-degree source refitting. The resulting theory combines a sharp nonlinear inverse coefficient, physically realizable source genericity at fixed density, and finite-amplitude capacity hiding at the same reference. The bounded SH2 construction gives an exact finite-degree endpoint. Proofs develop the source-gauge, realization and physical-margin arguments around this question.

\subsection{Relation of the contributions to established results}
Joint source/attenuation nonuniqueness and shape--radiance ambiguity are established phenomena. Stefanov's planar identification analysis supplies transport and weighted-system background \cite{stefanov2014}; attenuated tensor tomography supplies the source-gauge setting \cite{balmonard2019,sharafutdinov1994}. The question answered here is the density signal remaining after shared, finite-degree directional sources are allowed to vary in a three-dimensional volume. The sharp density coefficient, physical source-realization argument and same-reference quantifiers specify that contribution.

The first-hit surface theory of Wang and Miller \cite{wang2026} studies parallax-induced angular demands for misaligned opaque surfaces. A volume changes the attenuation weight continuously along the ray. This paper keeps that changed weight in the nonlinear construction and in the stability proof, with source/color bounds checked on the resulting fields. Consequently, its exact SH2 pair, tensor-gauged local recovery and physically generic coexistence statements concern a different observation operator and unknown.

General finite-measurement inverse theory provides stability-transfer principles under specified priors and hypotheses \cite{alberti2021finite}. Our finite-model sections retain appearance nonuniqueness and the density-only target, and state the pixel geometry and constants needed to apply such reasoning. Conversely, refining a shared SH0 spatial color class after fixing the measurements can restore exact data ambiguity. Both statements make the model/measurement order essential.

The resulting theory connects the amount of surviving density information to realizable physical alternatives. Claims C1--C8 supply the sharp coefficient, establish generic fixed-degree recovery and increasing-capacity hiding at the same reference, and construct an exact low-degree endpoint. This combination turns the appearance--geometry ambiguity into quantitative statements about recoverability, conditioning and physical reuse.

\section{The volume problem}\label{m:model}
Let $\Omega=B_1\subset\R^3$. All external \emph{oriented} lines are observed, with black background and an independently specified effective source. A line is $x=p+td$, $p\perp d$, with endpoints $\pm\ell(p)$ and $\ell(p)=\sqrt{1-|p|^2}$. The three-channel data are
\begin{equation}\label{m:render}
 \cR(\sigma,j)(p,d)=\int_{-\ell}^{\ell}T(p+td,d)j(p+td,d)\dd t,
\end{equation}
where $T$ is entry-to-point exponential attenuation. The scalar density is shared across channels, and
\begin{equation}\label{m:physical}
 \sigma\geq0,\quad0\leq j_k\leq\sigma,\quad
 j_k(x,\cdot)\in\cH_{\leq L}(\Sph^2).
\end{equation}
Color is $c=j/\sigma$ where $\sigma>0$; in vacuum $j=0$ and color is irrelevant. The independent unknowns encode a learned rendering representation: density controls attenuation, while an appearance function supplies directional radiance. At positive density, freely refitting $j$ under \eqref{m:physical} is equivalent to refitting bounded color $c\in[0,1]^3$. SH degree specifies the angular freedom; the finite experiments additionally impose shared spatial coefficients.

\paragraph{Why directional sources?}
NeRF represents directional radiance as a learned output \cite{mildenhall2020}; SH voxels make its angular capacity explicit. Directional effective sources also arise from single scattering under fixed incident illumination \cite{pbr2023}. Here \emph{physical} means nonnegative attenuation and bounded color, as in \eqref{m:physical}. This representation-level class isolates what appearance freedom can remove from density evidence (\cref{app:effectivephysical}).

For the positive theorem, spatial coefficients are $C^6$ up to the boundary of $\overline\Omega$, hence uniformly bounded there. Unknown changes have compact support in a small ball $B_r(x_0)$; the exterior is fixed. Near that ball the reference has strict positive margins for $\sigma$, $j_k$, and $\sigma-j_k$. A fixed jump to exterior vacuum is permitted. These hypotheses provide physical perturbation directions and smooth interior transport weights.

The data norm, using unitary Fourier transforms in the two offset coordinates, is
\begin{equation}\label{m:Y}
 \norm g_Y^2=\frac1{4\pi^2}\int_{\Sph^2}\int_{d^\perp}
 |\xi|\,|\F_{2,p}g(\xi,d)|^2\dd\xi\,\mathrm dS(d).
\end{equation}
Line integration smooths spatial fluctuations by half a derivative. The weight $|\xi|$ compensates for that smoothing: $Y$ is the natural offset-Sobolev ray norm that compares density and data at the same spatial scale \cite{assylbekovstefanov2020}. Inverse coefficients measure density $L^2$ error per unit $Y$ error. Pixel RMS and maximum RGB are reported separately, with their own units.

Write $U(x,d)$ for the attenuated radiance from $x$ to ray exit. Transport identities \cite[Sections~3 and 5]{stefanov2014} give
\begin{align}
 D\cR_*[h,q]&=\int T_*(q-U_*h)\dd t,\label{m:linear}\\
 \cR_2-\cR_1&=\int T_1[(j_2-j_1)-U_2(\sigma_2-\sigma_1)]\dd t.
 \label{m:exact}
\end{align}
\Cref{lem:transport} proves both identities with the changed attenuation retained exactly. Source coordinates make the physical constraint meaningful even in vacuum. The factor $U_*h$ is the radiance hidden by a density increase, and $q$ is the emission available to replace it.

An admissible local class consists of these physical $C^6$ fields, with prescribed exterior, compactly supported changes, and bounded rescaled coefficient norms. Spatial bandwidth is unrestricted. The radius-$.01$ class in \cref{m:separation} specifies these bounds numerically; \cref{sec:nonlinear} gives the general rescaling.

\subsection{Notation and the three observation regimes}
\begin{center}\small
\begin{tabularx}{\textwidth}{@{}p{.19\textwidth}Y@{}}\toprule
Symbol & Meaning and role\\\midrule
$\sigma,j,c$ & Scalar attenuation density, RGB effective source, and color $c=j/\sigma$ at positive density. The source bound is $0\leq j_k\leq\sigma$.\\
$T,U$ & Entry-to-point transmission and point-to-exit attenuated remaining radiance. $U$ is determined by the reference volume.\\
$L,\cH_{\le L}$ & Maximum angular degree and its spherical-harmonic source class. This leaves the spatial source free unless a finite model is stated.\\
$x_0,C_n$ & The localization point and the great circle of directions perpendicular to a spatial frequency normal $n$.\\
$E_L,\kappa_L$ & Residual after optimal weighted circle refitting, and its minimum over all normals.\\
$Y,\|\cdot\|_2$ & Offset half-derivative norm of complete ray data, and spatial density $L^2$ norm.\\
$\lambda_L,C_E$ & Joint gauged gain and full operator perturbation cost. The absorption test compares these two quantities.\\
$s,c_{\rm fin}$ & Finite density and appearance coordinates when a finite spatial representation is specified. These are not camera indices.\\\bottomrule
\end{tabularx}\end{center}

\textbf{Complete-ray theory} observes every oriented exterior line and uses $Y$ or explicitly stated maximum RGB. \textbf{Certified finite models} specify a spatial basis and either an exact continuous ray integral or an exact finite compositor. \textbf{Captured-image experiments} report held-out camera errors in their declared encoding and resolution. These regimes answer different questions. A coefficient in $Y$ is not inserted into a pixel-noise formula without the transfer hypotheses in \cref{long:pixeltransfer}.

\clearpage
\part{Continuum density information and physical ambiguity}
\section{Density information after angular compensation}\label{m:criterion}
Freeze $T_0(d)=T_*(x_0,d)$, $U_0(d)=U_*(x_0,d)$, and define $C_n=\{d\in\Sph^2:d\cdot n=0\}$. Circle averages are normalized by $1/(2\pi)$. The restrictions of $\cH_{\leq L}(\Sph^2)$ are exactly the circle polynomials $\cV_L(C_n)$. Define
\begin{align}
 E_L(x_0,n)&=\sum_k\min_{p\in\cV_L(C_n)}
 \avg{T_0^2|U_{0,k}-p|^2}_n,\label{m:E}\\
 \kappa_L(x_0)^2&=\min_{n\in\Sph^2}E_L(x_0,n).\label{m:kappa}
\end{align}

\begin{theorem}[Exact frozen criterion]\label{m:thmfrozen}
For $A_0(h,q)=\int T_0(q-U_0h)\dd t$ and compactly supported coefficients of source degree at most $L$,
\begin{equation}\label{m:sharp}
 \norm{A_0(h,q)}_Y\geq\kappa_L(x_0)\norm h_2.
\end{equation}
This constant is optimal even if all changes must be supported in any one fixed nonempty open ball.
\end{theorem}
\paragraph{Why a great circle?}
A spatial plane wave has frequency $\xi$. Integrating it along $d$ cancels its oscillations unless $\xi\cdot d=0$: only directions tangent to the wavefront remain. On the sphere these directions form the great circle $C_{\xi/|\xi|}$. SH polynomials restrict to degree-$L$ trigonometric polynomials on that circle. Thus the best source refit removes precisely those modes.

\paragraph{Proof mechanism.}
Fourier slice and coarea express the squared data norm as the integral of
$\avg{T_0^2|\widehat q-U_0\widehat h|^2}_{\xi^\perp}$ over spatial frequency. Projecting out each circle polynomial proves the bound. For sharpness, lift an optimal circle polynomial to $P(d)$ and choose real compactly supported packets $h_N=\phi\cos(Nn_0\cdot x)$, $q_N=P h_N$. Their frequency normals approach a minimizing $\pm n_0$. The full proof is \cref{thm:frozen}.

\paragraph{Explicit computation.}
For $e_a=e^{ia\theta}$, $|a|\leq L$, form the weighted Gram matrix $G_{ab}=\avg{T_0^2\bar e_a e_b}$ and $b_{k,a}=\avg{T_0^2\bar e_aU_{0,k}}$. Weighted projection (\cref{eq:gram}) gives
\begin{equation}\label{m:gram}
 E_L=\sum_k\left[\avg{T_0^2U_{0,k}^2}-b_k^*G^{-1}b_k\right].
\end{equation}
When $T_0=t_0$ is constant this is $t_0^2$ times the circle Fourier energy above order $L$. For nonconstant $T_0$, taking the unweighted tail of $T_0U_0$ is incorrect: the nuisance space is $T_0\cV_L$. A QR residual is preferable to subtracting nearly equal quantities in \eqref{m:gram}.

We abbreviate $\kappa_L(x_0)$ to $\kappa_L$ only while the base point is fixed. For one fixed reference, nested source spaces imply
\begin{equation}\label{m:ladder}
 \kappa_{L+1}\leq\kappa_L,\qquad\kappa_L\longrightarrow0.
\end{equation}
Thus increasing capacity can only remove density evidence at the same scene. With a uniform circle derivative bound $S_s$, the remaining signal obeys $\kappa_L\leq t_+S_s/(L+1)^s$ (\cref{cor:ladder,eq:tailrate}).

\subsection{A sufficient nonlinear theorem}
For $L\geq1$, write a source using trace-inclusive symmetric tensors of ranks $L$ and $L-1$. The attenuation couples their potential parts. A zero-boundary potential preserves source data; choosing the correct coupled gauge makes both positive-rank representatives interior solenoidal. A Euclidean ball extension estimate controls the boundary contribution of their zero extensions. These steps make the estimate uniform over the full spatial source class: invisible source directions must be controlled before perturbing the transport weights.

Let $\lambda_L>0$ be a conservative gain for the full gauged density/source unknown, and let $C_E$ bound the full operator's departure from its frozen reference. The gauge condition is $\sqrt{3L}\,r\norm{\sigma_1}_\infty/\pi<1$. \Cref{sec:nonlinear,app:ball} define the rescaled Fourier norm and extension constants. These are finite at each fixed rank and numerically explicit here through rank one; the general sufficient radius is qualitative in rank, whereas \eqref{m:gram} computes $\kappa_L$ for every $L$.

\begin{theorem}[Local nonlinear density recovery]\label{m:thmnonlinear}
Under \cref{m:model}'s field assumptions, if $\kappa_L(x_0)>0$, then for every fixed finite $L$ there is a sufficiently small unknown-support ball and a nonempty relative $C^6$ neighborhood in which density is unique. More explicitly, the gauge smallness condition and $C_E<\lambda_L$ imply
\begin{equation}\label{m:nonlinearbound}
 \norm{\sigma_2-\sigma_1}_2
 \leq\frac{\norm{\cR_2-\cR_1}_Y}{\kappa_L(1-C_E/\lambda_L)}.
\end{equation}
\end{theorem}
The argument first bounds the entire gauged perturbation by $\norm z\leq\norm{A_0z}/\lambda_L$, then applies the sharp density estimate. The relative perturbation cost $C_E/\lambda_L$ vanishes as the support and reference neighborhood shrink, revealing $1/\kappa_L$ as the limiting inverse coefficient (\cref{thm:nonlinear,cap:S}).

\subsection{A necessary bound for the actual problem}
Freeze neither the packet envelope nor its oscillation count. Let
$h_{r,N}=r^{-3/2}\phi((x-x_0)/r)\cos(Nn_0\cdot(x-x_0)/r)$, normalized in $L^2$, and set $q_{r,N}=P(d)h_{r,N}$. The true derivative obeys
\begin{equation}\label{m:packet}
 \norm{D\cR_*[h_{r,N},q_{r,N}]}_Y
 \leq\sqrt{\kappa_L^2+C_{\rm wp}/N^2}+C_{\rm sp}r.
\end{equation}
The two error terms control finite frequency and spatial support (\cref{eq:twoscales}); taking $r\to0$, $N\to\infty$ gives the matching limit $\kappa_L$. Choosing amplitude last preserves the strict physical box and $C^6$ neighborhood.

\begin{theorem}[Nonlinear fragility]\label{m:thmfragility}
Any local density inverse in the admissible infinite-dimensional class has $Y$-to-$L^2$ Lipschitz coefficient at least $1/\kappa_L(x_0)$. The same coefficient lower-bounds the asymptotic deterministic minimax error divided by noise radius. If $\kappa_L=0$, no finite Lipschitz coefficient exists.
\end{theorem}
Thus physical shrinking packets transfer the frozen optimum to nonlinear recovery (\cref{sec:fragility}).

\subsection{Why symbol positivity alone is insufficient}
The great-circle residual determines a necessary local density signal, but the source itself has potential directions that integrate to zero. A positive density component of the principal symbol therefore cannot be promoted directly to a full nonlinear inverse. The proof first represents finite angular functions by trace-inclusive tensors, uses a coupled potential gauge, and controls extension through the support boundary. The resulting joint estimate absorbs weight variation before the sharp density estimate is applied.

This is why the denominator in \cref{m:thmnonlinear} contains $1-C_E/\lambda_L$, even though the numerator is governed by the sharper density constant $\kappa_L$. Replacing the joint absorption condition by a comparison with $\kappa_L$ would remove a necessary proof step. On the other side, physical localized packets show that the remaining $1/\kappa_L$ cannot be improved in the shrinking-neighborhood limit. The proof addresses all permitted shared source changes, rather than testing only a chosen compensation envelope.

\section{Local identifiability without uniform stability}\label{m:capacitytheorem}
The preceding estimates distinguish two limits: shrinking a local recovery class at fixed degree, and increasing appearance degree within a fixed smoothness neighborhood. The following result places both at the \emph{same} reference.

Fix smooth $\sigma\geq0$ on $\overline\Omega$ and compact sets $K,D\subset\Omega$ with $K\Subset\operatorname{int}D$, with $\operatorname{int}K\ne\varnothing$. Let $j_b$ be a smooth physical isotropic RGB source and let
\begin{align}
 X_D&=\{v\in C^\infty(\overline\Omega;\R^3):\supp v\subset D\},\label{m:XD}\\
 \mathcal A_\mu&=\{j_b+v:v\in X_D,\ \mu<j_k<\sigma-\mu\text{ on }D\}\label{m:Amu}
\end{align}
be nonempty, for $\mu>0$. Genericity is topological, not probabilistic: it means dense residual in the $C^\infty$ Fr\'echet topology of this fixed-support source class. The density stays fixed while isotropic spatial sources vary.

For $\Omega_i=(e_i\times d)\cdot\nabla_d$, use the mixed seminorm
\begin{equation}\label{m:mixed}
 \norm f_{m,s}=\max_{k,|\alpha|\leq m,|w|\leq s}
 \sup_{\overline\Omega\times\Sph^2}|D_x^\alpha\Omega_w f_k|,
\end{equation}
where $w$ is an ordered word, including the empty word. At fixed $L$, let $C_{\mathrm{opt},L}(r,\eps)$ be the supremum of density $L^2$/data $Y$ ratios over physical pairs in the full coefficient $C^6$ $\eps$-neighborhood with changes supported in $\overline{B_r(x_0)}$ and known reference exterior. These classes include arbitrarily small multiples of every smooth supported degree-$L$ direction.

\begin{theorem}[Same-reference capacity separation]\label{m:coexistence}
There is a dense residual $\mathcal G\subset\mathcal A_\mu$ such that each reference $(\sigma,j)$ with $j\in\mathcal G$ has both properties:
\begin{enumerate}
\item For every fixed finite $L$ and $x_0\in K$, $\kappa_L(x_0)>0$ and density admits sufficiently small-support local nonlinear recovery with known exterior. Along shrinking neighborhoods with uniform $C_E\to0$ and $\overline{B_{2r}(x_0)}\subset\operatorname{int}D$,
\begin{equation}\label{m:sharpC}
 C_L^\sharp(x_0):=\lim_i C_{\mathrm{opt},L}(r_i,\eps_i)
 =\kappa_L(x_0)^{-1}.
\end{equation}
\item For any fixed nonzero $h\geq0$ smooth and compactly supported in $\operatorname{int}K$, physical sources $j_N\in\cH_{\leq N}$ at density $\sigma+h$, with $j_N=j$ outside $K$, satisfy
\begin{equation}\label{m:hiding}
 \begin{gathered}
 g_N=\cR(\sigma+h,j_N)-\cR(\sigma,j),\\
 \norm{g_N}_Y+\norm{g_N}_\infty\longrightarrow0.
 \end{gathered}
\end{equation}
The density distance stays $\norm h_2>0$, and every fixed finite list of derivative norms has a common bound independent of $N$. Moreover, inside \emph{any} prescribed neighborhood defined by finitely many such seminorms, one can first choose a sufficiently small, fixed $h=a h_0\ne0$ and then obtain \eqref{m:hiding} as $N$ increases.
\end{enumerate}
Consequently no density recovery modulus tending to zero can hold independently of angular degree in any such fixed neighborhood.

Both conclusions also hold in the scalar source class rendered as $j=\mathbf a f$, with fixed $\mathbf a\in(0,1]^3$, $\max_k a_k=1$, throughout the volume. Genericity uses the scalar analogue of $\mathcal A_\mu$; $C_{\mathrm{opt},L}$ uses the restricted local class and every $j_N$ has these same proportions.
\end{theorem}
\begin{proof}[Proof mechanism]
Use sine/cosine moments at frequencies $L+1,L+2,L+3$ of one source channel: six real components over the five-dimensional $(x,n)$ base. A spatial shell controls all six. A finite cover, Sard and Baire give generic positive gain at every finite degree. The coupled gauge in \cref{m:thmnonlinear} and physical packets in \cref{m:thmfragility} give \eqref{m:sharpC}, including within fixed channel proportions.

For the negative result, apply the positive spherical operator \cite{yang2011}
\begin{equation}\label{m:positiveoperator}
 P_N=\cV_N U:=\int_{\Sph^2}(N+1)
 \left(\frac{1+d\cdot e}{2}\right)^N U(x,e)\frac{\mathrm dS(e)}{4\pi},
\end{equation}
and set $j_N=j+hP_N$. Its unit-mass kernel preserves $0\leq P_N\leq1$ and commutes with spatial and rotation derivatives. The exact residual is $X_{T_{\sigma+h}(P_N-U)}h$. Uniform spatial $C^4$ convergence controls $Y$; uniform convergence also controls maximum RGB. Scaling $h=a h_0$ once fits any finite list of neighborhood constraints. The complete arguments are in \cref{cap:G,cap:S,cap:R,cap:synthesis}.
\end{proof}

\begin{corollary}[Super-polynomial local conditioning]\label{m:superpoly}
At a reference in \cref{m:coexistence}, fix $x_0$. For every $p>0$,
\begin{equation}\label{m:superpolyformula}
 \frac{C_L^\sharp(x_0)}{(L+1)^p}\longrightarrow\infty.
\end{equation}
\end{corollary}
\begin{proof}
Smoothness makes $S_s<\infty$ for each fixed integer $s$. Positivity and the angular-tail bound give $C_L^\sharp\geq(L+1)^s/(t_+S_s)$. Choose $s>p$.
\end{proof}
First shrink the recovery neighborhood at fixed $L$, then compare degrees. Genericity uses the smooth-source topology, recovery uses coefficient $C^6$ neighborhoods, and hiding preserves prescribed finite mixed-derivative bounds (\cref{app:capacityproofs}). We next show exact cancellation at finite degree in a quadratic scene.

\begin{figure*}[t]
\centering\includegraphics[width=\textwidth]{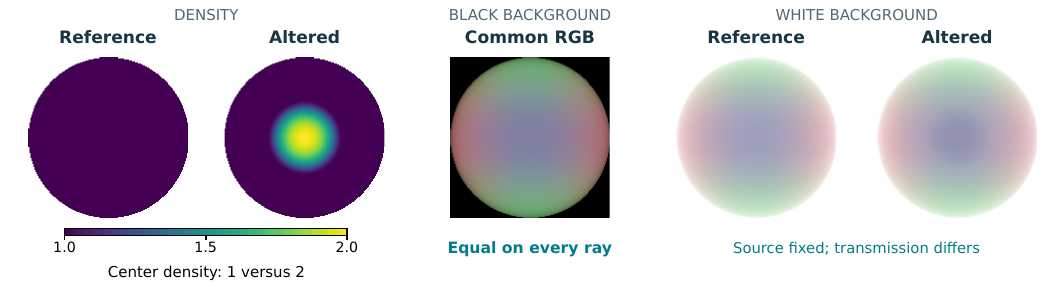}
\caption{\textbf{Same RGB, different response to a background change.} Central density doubles while SH2 appearance preserves every black-background ray (\cref{m:twofold}). White-background composites reveal the changed transmission with each source held fixed. Density and RGB panels share their respective scales.}\label{m:certifiedpair}
\end{figure*}

\subsection{The coexistence statement in quantifier order}
The reference is selected once from the residual physical source class. After fixing a degree $L$, one may choose a sufficiently small support and coefficient neighborhood in which density recovery holds. The allowed radius and the inverse constant may depend on $L$. For hiding, a finite list of smoothness requirements is fixed first, then one nonzero density addition is chosen, and only the angular degree grows. The addition does not shrink along the hiding sequence.

This order has two consequences. First, the generic theorem is stronger than finding a different favorable reference for each degree. Second, it does not turn generic local recovery into a noise guarantee uniform over model capacity. The super-polynomial coefficient growth in \cref{m:superpoly} concerns the limiting local coefficients, whereas the fixed-amplitude sequence in \cref{m:hiding} directly rules out a degree-independent modulus of recovery on the stated neighborhood.

\subsection{Where the source-genericity proof gains its force}
The genericity proof perturbs an actual spatial source. At a fixed point and circle, a source supported on an interior spherical shell can prescribe six real moments of the remaining radiance. Local circle frames form a rank-six bundle over the five-dimensional point/normal base. A finite cover supplies one finite-dimensional perturbation family that is onto in every fiber; Sard's theorem removes the common zeros, and a countable Baire intersection handles every finite degree. Strict physical margins permit these signed source directions at small amplitude. Thus the source can remain isotropic while its \emph{transported} angular radiance becomes sufficiently rich. The construction and topology are proved in \cref{cap:G}.

\section{Exact ambiguity at a quadratic reference}\label{m:separation}
The generic theorem establishes capacity-dependent hiding across a source family. The following separate construction establishes its strongest finite-degree endpoint: SH2 already suffices for exact, physical all-ray ambiguity. Use the isotropic reference
\begin{equation}\label{m:reference}
 \sigma_*=1,\qquad j_{*,k}=\tfrac14+\tfrac12x_k^2.
\end{equation}
At its center $T_0=e^{-1}$ and $U_{0,k}=A+Bd_k^2$, where $A=(1-e^{-1})/4$, $B=(2-5e^{-1})/2$. Circle variances give
\begin{equation}\label{m:numbers}
 \kappa_0(0)=\kappa_1(0)=e^{-1}B/\sqrt6\simeq.01206016,
\end{equation}
while $\kappa_L(0)=0$ for $L\geq2$ (\cref{eq:quadraticgain}). Fixed-color density gain has inverse $23.12$ in the same norm, versus $82.92$ after SH1 refitting: a factor $3.59$ deterioration. An independently specified radius-$.01$ Fourier-prior class has optimal inverse coefficient between $82.91765$ and $263.057$, with an outward-interval certificate (\cref{app:calibration,app:radius}). In the asymmetric scene, certified frozen inverse gains satisfy $C_{\rm fixed}<52.316$ and $C_{\rm SH3}>843.909$, a deterioration exceeding $16.13\times$ in the same norms (\cref{app:circlecertificate}). 

\begin{theorem}[Exact local SH2 alternatives]\label{m:thmseparation}
For every $0<r\leq.01$ and $0<\eps\leq1/6000$, the reference \eqref{m:reference} has a distinct physical density $1+\eps\psi(4|x|^2/r^2)$ and a genuinely degree-two source, with unchanged exterior and identical RGB on all oriented lines. The source is in $[.2,.8]$. At fixed $r$, alternatives approach the reference in every finite smooth norm as $\eps\to0$.
\end{theorem}
Here $\psi(z)=e^{-z/(1-z)}$ for $z<1$ and zero otherwise. Four radial profiles multiplying $1,x_k^2,x_kd_k,d_k^2$ solve a triangular Abel--Volterra system with the changed attenuation retained exactly (\cref{sec:abel,app:physical}).

\begin{theorem}[Certified twofold central density]\label{m:twofold}
At \eqref{m:reference}, the density $\sigma_2(x)=1+\psi(4|x|^2)$ has a smooth SH2 source, unchanged outside $B_{.5}$, with identical RGB on every oriented line and
\begin{equation}\label{m:twofoldbox}
 j_{2,k}>.2062,\qquad \sigma_2-j_{2,k}\geq\tfrac14.
\end{equation}
Its relative density $L^2$ separation lies in $[.1452219,.1459197]$ and $\sigma_2(0)=2$.
\end{theorem}
A segmented interval residual bound transfers the physical box from an analytic proxy to the exact Volterra solution (\cref{app:finitecertificate}). \Cref{app:extensions} extends the construction to other optical thicknesses and common nonradial perturbations.

\begin{corollary}[RGB-only objectives cannot separate the pair]\label{m:rgbfiber}
Every exterior camera subset, fixed pixel average and common image postprocessing preserves the equality. RGB-only objectives with fixed auxiliary inputs are constant on the pair; image-only density estimation has worst-case error at least half its separation.
\end{corollary}
Adding views or resolution cannot distinguish this pair under the same effective-source and background conditions (\cref{app:extensions}).

\paragraph{Why minimize over every circle?}
For the quadratic reference at $x\ne0$, the circle with normal $n=x/|x|$ has $x\cdot d=0$ and constant exit length. Its backlight is degree two, so $\kappa_2(x)=0$ throughout the interior even when the full spherical function has higher harmonics (\cref{sec:abel}). This tangent circle explains both the all-normal minimum in \eqref{m:kappa} and the exceptional structure of this quadratic reference. The generic references in \cref{m:coexistence} have positive gain at every finite degree.

\begin{figure*}[t]
\centering\includegraphics[width=\textwidth]{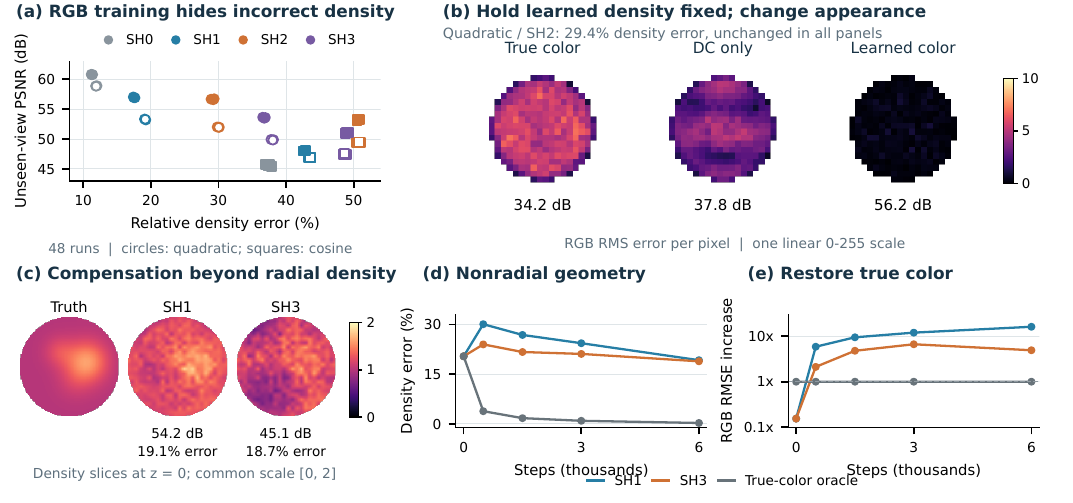}
\caption{\textbf{RGB optimization learns appearance that conceals density error.} (a) All $48$ random starts: open/filled markers use $100/400$ training views. (b) One prescribed SH2 checkpoint and unseen camera; density is identical in the three error maps, true-color replacement exposes density compensation, and DC-only color measures angular reliance (\cref{app:densityoracle}). The same density is certified in \cref{m:querytable}. (c) Nonradial density slices at the first seed, with unseen-view PSNR and 3D density error. (d,e) Nonradial trajectories: means and full ranges of two completed seeds per arm. The true-color oracle supplies appearance during training. All image errors use independent test rays.}\label{m:trainedcompensation}
\end{figure*}

\clearpage
\part{Shared representations, finite data and acquisition}
\section{What changes when observations are finite?}\label{long:finite}
Complete rays constrain a common spatial source through many intersecting lines. A finite camera array leaves fewer constraints, and a shared representation couples those constraints in a model-dependent way. The appropriate question is whether a density change can be explained by \emph{all} permitted shared appearance coefficients, with the camera and spatial model held fixed.

\subsection{Query geometry controls the appearance image class}
The finite-query theory in \cref{q:Q1,q:antipodal} characterizes the entire rendered image class. For two queries, repeated directions can impose a common response that retains a density signal, whereas antipodal directions admit SH1 physical separation of the two directional values. This is a statement about a specified shared-query operator, not a prescription to point ordinary cameras in arbitrary directions. The two-layer calculation in \cref{q:twolayer} retains the shared coefficients and attenuation weights, making the distinction exact.

Uniform resources matter. Interpolating arbitrary values at finitely many directions is not enough unless the resulting source obeys its physical range and the degree bound. The regular-circle construction in \cref{q:circledegree} uses positive interpolation to preserve those bounds. The finite-class certificate in \cref{q:certificate} then quantifies the residual left by insufficient angular resources at one fixed density.

\subsection{A continuous-renderer transition at one fixed reference}
For one 27-node trilinear density/source representation and all 384 prescribed continuous rays, the following two statements hold simultaneously:
\begin{itemize}
\item Over the full physical SH1 source class, density changes with nodal norm at most $10^{-4}$ satisfy a raw maximum-RGB gain of at least $3.3116\times10^{-4}$ per unit spatial density $L^2$ (\cref{eq:continuouslocalgain}).
\item Every nodal density change of norm at most $10^{-8}$ has an exactly matching physical SH3 source (\cref{rs:capacitytransitionproof}).
\end{itemize}
The first statement bounds every permissible refit. The second solves the source response at the changed density using a certified right inverse, and then checks physical margins. It does not infer an exact nonlinear partner from a numerically small singular value. Both are continuous-integration statements for the same scene and observations. They identify a capacity-dependent change of local identifiability, while leaving the intermediate SH2 case unspecified.

The same reference also has a finite-amplitude alternative with density separation exceeding $.10$ that passes a stated one-gray SH1 fit test while SH0 rejects it (\cref{rs:coupledwitness,rs:coupledcontinuous}). The small local inverse neighborhood and this distant accepted alternative coexist without contradiction. Pairwise bounds in \cref{rs:continuouspairwise} state what two admissible reconstructions can disagree about at a much smaller noise scale.

\subsection{Spatial capacity can hide finite data even at SH0}
Angular restriction by itself does not remove all appearance freedom. For finitely many distinct rays, localized spatial source functions can control their responses; suitable continuous pixel geometries also admit an exact source right inverse (\cref{rs:spatialfreedom,rs:pixelspatial}). The source remains one shared field. Under strict physical margins, this right inverse can compensate sufficiently small density changes exactly on the declared finite measurement set.

\Cref{long:gridcomp} transfers this observation to physical trilinear \emph{colors}, retaining $j=\sigma c$. Once the measurement design is fixed, a sufficiently fine grid is selected and then held fixed. Every sufficiently small allowed density change has a same-grid isotropic color correction. The proof uses the exact color response at the changed density; it does not confuse a trilinear effective source with a trilinear color.

The two orders of design now become explicit. Fixing a model first and choosing informative measurements can permit recovery. Fixing finitely many measurements first and enlarging the spatial color class can recover exact ambiguity. A declaration of finite dimension is insufficient without specifying which model and which observations are fixed.

\section{From continuum stability to pixels and noise}\label{long:measurement}
\subsection{A relative bound that retains all appearance changes}
Write a finite rendering map as $F(s,c)=b(s)+B(s)c$, with bounded allowed appearance coefficients, and fix $(s_0,c_0)$. Its exact difference splits into
\begin{equation}\label{long:split}
\begin{aligned}
g&=f+v,\\
f&=B(s_0)(c-c_0),\\
v&=b(s)-b(s_0)+[B(s)-B(s_0)]c.
\end{aligned}
\end{equation}
Only $f$ belongs to the reference finite image space. The density-dependent remainder $v$ contains source directions that become visible as density changes. Finite-dimensional norm equivalence controls derivatives of $f$ by its image amplitude; density differentiation controls $v$ by density distance. This retains arbitrarily large permitted source-only changes and the reference source kernel. The complete estimates are in \cref{rs:theorem,long:pixeltransfer}.

For a smooth finite model lying inside the local $Y$ recovery class, let $H$ be density $L^2$ distance, $R$ maximum ray RGB, and let the constants in \cref{long:pixeltransfer} be $C_Y,A_p,A_q,D_p,D_q$. That appendix proves
\begin{equation}\label{long:pixelmain}
\begin{gathered}
R\ge\gamma H,\\
e_{\rm pix}\ge[\gamma-(\delta+\rho)(A_q\gamma+D_q)]H,\\
\gamma=\frac{2}{TD_p+\sqrt{T^2D_p^2+4TA_p}},\qquad T=3a^2C_Y^2.
\end{gathered}
\end{equation}
Here $a$ bounds offset support, $\delta$ is coverage in the declared line metric, and $\rho$ bounds positive pixel footprints. When the bracket is positive, two fits to the same data have density separation bounded linearly by their summed fit and certified forward-error budgets. The statement has no additive sampling floor. The transfer constants depend on the fixed smooth model; $\gamma$ is a maximum-RGB gain, not the sharp $Y$ coefficient $\kappa_L$.

\subsection{Pixel density, camera coverage and statistical repetition}
Small pixels do not imply small $\delta$: each pinhole center samples a two-dimensional family in the four-dimensional line space. For the unit-ball metric used in \cref{long:pixeltransfer}, $M$ exterior camera centers providing $\delta<1$ coverage must satisfy
\begin{equation}\label{long:cameramain}
\delta^2\ge\frac{4(M-1)}{M^2},\qquad M\ge2.
\end{equation}
The proof in \cref{long:cameracost} uses only central lines and spherical-cap area. If the nonzero reference source image is offset-supported within radius $b$, its derivative/amplitude constant obeys $A_q\ge1/b$. Thus the sufficient pixel certificate requires $\delta+\rho<b$. These are necessary costs of that sufficient covering strategy. A smooth two-parameter physical example, also proved there, recovers density amplitude and unknown color from two rays of a single camera; no universal camera lower bound is asserted.

Measurement repetition answers a different question. Under the explicit independent Gaussian law in \cref{rs:gaussianpairs}, certified physical mean-image pairs give necessary repetition counts for a stated density-error risk. Under bounded deterministic errors, \cref{rs:priornoise} instead shows that one-gray data can leave the entire small density prior unresolved. A positive noiseless derivative gain therefore does not by itself guarantee useful accuracy at a chosen noise level.

\subsection{Measurements that remove the source nuisance}
With the effective source preserved across a background change, black/white readings differ by transmission. The continuous physical SH2 pair has a certified gap under this intervention: the central-ray gap exceeds $1/32$, and a specified common finite footprint has gap greater than $1/56$. A per-reading error below $1/224$ separates the pair, including $1/255$ (\cref{bg:footprint}). This is the source-preserving matting principle \cite{smithblinn1996}, applied to a pair that RGB alone cannot separate.

A separate finite-model acquisition certificate makes cost explicit. For its declared 27-node model, source class, candidate ray pool and reference-exclusion objective, the number of extra scalar readings needed to exclude every density more than $.10$ away obeys
\begin{equation}\label{long:costmain}
6\le C_{\rm ref}(.10)\le152.
\end{equation}
The lower bound supplies a whole compatible physical family for every design below the threshold; the upper bound controls every compatible density for a specified 76-pair design. Both sides use the same reference, noise budget and target (\cref{ac:costtheorem}). The interval is a certified cost bracket, not an exact optimal camera design.

\clearpage
\part{Evidence and scientific implications}
\section{Appearance compensation in learned fields}\label{m:numerics}
We ask whether RGB training learns the compensation predicted by the theory. The experiments first expose it by changing appearance at fixed learned density, then certify the best achievable fit as angular capacity changes. Nonradial fields and density-derived transmission/depth measurements test its consequences beyond the quadratic construction.

\subsection{Appearance compensation learned from RGB}\label{m:trained}
We train $48$ dense $32^3$ Plenoxels-style fields \cite{plenoxels2022}: two isotropic-color scenes, $100/400$ views, SH0--SH3, three seeds each. Each uses $1{,}500$ Adam updates in $B_{.8}$ with known exterior. Evaluation uses $7{,}936$ rays from $32$ unseen directions and $8{,}192$ geometry points in $B_{.75}$. \Cref{m:trainedcompensation}a includes every run; \cref{app:training} specifies the physical color constraints and settings.

Across $36$ SH1--SH3 runs, test PSNR is $46.85$--$57.04$ dB with $17.4$--$51.0\%$ density error. At each fixed learned density, restoring true color raises RGB RMSE $4.63$--$22.77\times$. Six matched color-only fits at exact density separate this effect: true-color replacement leaves below $.00011$ gray levels of error, versus $2.46$--$12.56$ at their learned-density counterparts (\cref{app:densityoracle}). Directional removal also raises error at true density, measuring angular reliance.

For quadratic/$400$-view/SH2, the three-seed mean is $56.66$ dB with $29.2\%$ density error; true-color replacement gives $34.24$ dB. The prescribed first checkpoint and camera appear in \cref{m:trainedcompensation}b. \Cref{m:querytable} certifies complete physical refits at this same density; \cref{app:trainedcompensation} reports all runs.

\subsection{The mechanism beyond radial density}\label{m:nonradialtraining}
A nonradial $24^3$ volume has three displaced density bumps and spatially varying isotropic color. With $100$ training and $32$ test views, we compare SH1, SH3 and a density-only oracle supplied with true color: two $6{,}000$-step seeds and a separate three-seed $30{,}000$-step cohort per arm (\cref{app:nonradialtraining,app:gpulong}).

At $6{,}000$ steps, the directional runs attain $45.09$--$54.26$ dB with $18.70$--$19.30\%$ density error; true-color replacement raises RMSE $4.89$--$16.32\times$ (\cref{m:trainedcompensation}c--e). At $30{,}000$ steps, SH1 density improves to $7.95$--$8.13\%$ and appearance compensation persists in both degrees: replacement raises RMSE $4.29$--$12.44\times$. The $800$-ray training-monitor MSE is $1.50$--$1.69\times10^{-6}$ for SH1 and $1.58$--$1.71\times10^{-5}$ for SH3; these are different optimization endpoints (\cref{app:gpulong}).

\begin{figure*}[t]
\centering\includegraphics[width=\textwidth]{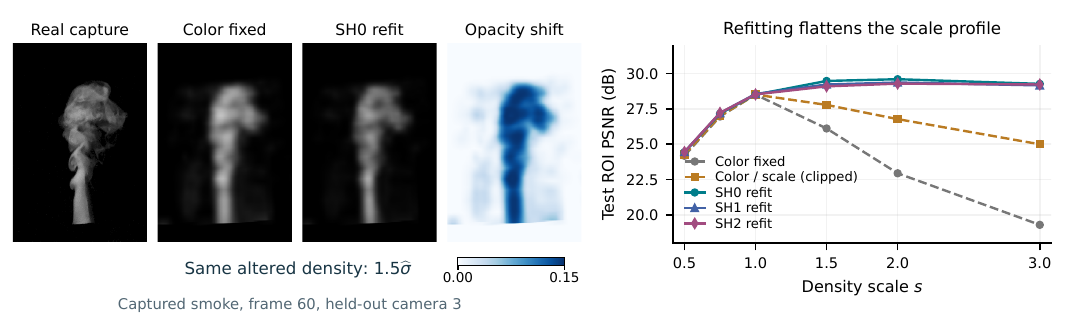}
\caption{\textbf{Color refitting flattens density-scale sensitivity in captured smoke.} Left: held-out capture, fixed/refitted SH0 at $1.5\widehat\sigma$, and predicted opacity change. Right: the complete scale sweep at the same time; fixed color is the SH0 fit at $s=1$. The unit-color reference sets the brightness ceiling responsible for the low-scale decline (\cref{smoke:ceiling}). Shared crops and intensity scales; $384\times216$ evaluation rays. All three times and controls: \cref{smoke:scaleprofile}.}\label{m:capturedsmoke}
\end{figure*}

\paragraph{Captured smoke: image fit versus density scale.}
Three-camera ScalarFlow unit-color fits~\cite{eckert2019,chu2022} define frozen densities scanned at $s\in\{.5,.75,1,1.5,2,3\}$. At all three times, $s=1.5,2,3$ gives SH0 test-PSNR span $\leq.33$ dB, fixed-color loss $6.53$--$6.83$ dB, and peak predicted opacity change $.15\to.38$ from $s=1$ (\cref{smoke:scaleprofile}). At $s=.5$, the ceiling $c\leq1\Rightarrow R\leq1-T$ forces $93.8$--$95.3\%$ of SH0 MSE (\cref{smoke:ceiling}). Interior-color references ($c=.75$) give $.69$--$1.09$ dB SH0 spans across $s=.75,1,1.5$ (\cref{smoke:grayreference}). This measures isotropic compensation; the following same-density certificates isolate angular capacity.

\subsection{Certifying compensation at the learned density}\label{m:querytest}
Freeze the $32^3$ density in \cref{m:trainedcompensation}b and bound the optimal physical SH0--SH2 fit on $179{,}200$ training rays. The $256$-step renderer is linear in color: feasible fits and dual witnesses bound its optimal maximum-RGB residual, retaining all spatial sharing (\cref{lr:dual,app:learnedrefit}).

\begin{table}[t]
\caption{\textbf{Certified capacity gaps at one learned density.} The $32^3$ density in \cref{m:trainedcompensation}b stays fixed. Bounds apply to the optimal maximum-RGB error over each complete physical appearance class on all training rays, in linear gray levels ($1/255$).}\label{m:querytable}
\centering\small
\begin{tabularx}{\linewidth}{@{}Ycc@{}}\toprule
Appearance class & Lower & Upper\\\midrule
All physical SH0 & $1.1228$ & $1.8820$\\
All physical SH1 & $0.4418$ & $0.9888$\\
All physical SH2 & $0$ & $0.4093$\\\bottomrule
\end{tabularx}
\end{table}

\Cref{m:querytable} certifies two strict gaps: $1.1228>0.9888$ separates SH0 from SH1, and $0.4418>0.4093$ separates SH1 from SH2. On these fitting rays, one-gray tolerance rejects every physical SH0 fit but accepts SH1. Density and spatial sharing stay fixed: the gaps isolate angular capacity independently of solver convergence. Untouched test PSNR is $47.10/52.90/60.55$ dB for the three refits; the original MSE-trained SH2 field gives $56.73$ dB (\cref{app:learnedrefit}).

An independent $27$-node Q1 construction pairs certified local SH1 recovery with a finite density error exceeding $.10$ in spatial $L^2$, under continuous ray integration. At one shared observation with one-gray noise, SH0 rejects that alternative density while SH1 accepts it (\cref{rs:coupledwitness,rs:continuouslocal}).

At this same reference and all $384$ continuous rays, every nodal density change $\|h\|_\infty\le10^{-8}$ has an exactly matching physical SH3 source (\cref{eq:capacitytransition,rs:capacitytransitionproof}). Thus angular capacity changes local identifiability with the scene and observations fixed.

Under the density prior and a small RGB neighborhood, any two physical SH1 fits satisfy an explicit density stability bound (\cref{eq:continuouspairwise,rs:continuouspairwiseproof}).

One-gray noise leaves the full $10^{-4}$ density prior compatible under SH0: its minimax error equals the prior radius (\cref{rs:priornoise}).

Spatial capacity also matters: near a reference with strict physical margins, unrestricted smooth SH0 sources exactly compensate sufficiently small density changes on any fixed finite ray set (\cref{rs:spatialfreedom}). Continuous pixels admit the same ambiguity under the conditions of \cref{rs:pixelspatial}.

\subsection{Controlled evidence for the capacity results}\label{m:capacityevidence}
\paragraph{One fixed density error.}
At the affine reference, fix $\|h\|_2=.319774$. Maximum RGB discrepancy is $34.91$ gray levels without source compensation, $3.577$ at $N=0$, and $1.896/1.539/1.296$ at $N=2/3/4$. Refinement to $N=64$ reduces the $N=0$ residual another $28.7\times$ to $.125$ levels, and $Y$ by $31.2\times$, with common derivative bounds (\cref{app:positiveR}). Isotropic compensation and angular refinement provide distinct reductions.

\paragraph{Directional packet consistency.}\label{m:packettest}
Prescribed nonradial packets test $\sqrt{E_L(x_0,n)}$ with $q=h p_L$, at fixed envelope and normal. At radius $.02$, SH1 errors are $1.32$--$2.33\%$ and SH3 errors $13.86$--$31.00\%$; at $.08$, SH3 errors reach $59.68$--$72.41\%$ (\cref{app:weightedpacket}).

\section{What the experiments establish}\label{long:evidence}
The strongest experiment for isolating angular capacity fixes density, spatial sharing, observations and physical bounds, then compares complete appearance classes. The certified intervals in \cref{m:querytable} supply that comparison. Training interventions answer whether an optimizer has learned compensation. Real images answer whether related behavior is visible in captured data. These roles do not require every experiment to instantiate every theorem hypothesis.

\begin{center}\small
\begin{tabularx}{\textwidth}{@{}p{.23\textwidth}YY@{}}\toprule
Evidence & Established result & Interpretation\\\midrule
Analytic physical pair & Every exterior ray agrees; density changes finitely & A mathematical impossibility for RGB-only separation under the stated rendering model\\
Same-density class certificates & Strict SH0--SH1 and SH1--SH2 residual gaps & Angular capacity alone changes feasible fitting accuracy in this fixed finite model\\
Synthetic interventions and density oracle & True-color replacement exposes errors at the learned density & Learned appearance compensates density; directional ablation alone is not density-specific\\
Captured-smoke scale profiles & SH0 refitting preserves similar held-out PSNR while predicted opacity changes & Empirical isotropic compensation around a fitted reference, without density ground truth\\
Matched real-smoke angular tests & Four predefined aggregate endpoints do not pass & Ordinary angular fitting is insufficient evidence of a reproducible density-edit-specific gain\\
Native neural-field controls & Local prediction and modest refit gains; shared angular corrections need not generalize & The theoretical existence of compensation does not prescribe optimizer or cross-camera behavior\\\bottomrule
\end{tabularx}\end{center}

\subsection{The matched neural-field diagnosis}
Two matched controls use a photo-trained Nerfacto Fern reference and target its renderings. Replacing its color head by coarse shared SH colors gives poor unchanged-density fits, around 16 dB at SH2. Ordinary angular improvement then mixes density compensation with improved approximation of the reference's texture. A second control preserves the native texture exactly at zero correction and adds physically bounded shared angular modulation. Across its two grids and two density edits, angular terms reduce fitting-ray RMSE by $5.8$--$12.4\%$ but increase held-out RMSE by $1.4$--$6.8\%$. The density-scale comparisons pass their final refinement test; the ramp comparisons do not.

Native R/D/H color-head refits still improve the fixed-renderer full-image RMSE by $7.6$--$12.3\%$. These parameter classes are not SH degrees. The numerical record therefore supports modest compensation and a transfer limitation in the matched angular control, rather than a blanket failure of neural radiance fields. The complete 36-control scope and all 21 native-head outcomes are in \cref{app:officialrefit,app:sharedshdiagnostic}.

\subsection{Physical meaning of the captured-image evidence}
The captured-smoke experiment compares multiple fitted emission--absorption explanations of the same images. Its opacity differences are predictions of these explanations, not independently measured opacity. The scale-profile evidence is primarily isotropic. A brighter or dimmer reference color changes which density scales can compensate while retaining $c\le1$, which is why the interior-color control is included. Matched cross-camera angular tests further separate ordinary SH fitting gains from the gain specifically added after a density edit. This interpretation preserves the empirical value of the scale profiles while keeping the stronger angular-capacity claim with the theoretical constructions and certified same-density experiments that establish it.

\section{Three consequences for reconstruction}\label{m:discussion}
\paragraph{1. High test PSNR does not certify density.}
The SH2 pair is indistinguishable under complete noiseless RGB coverage; trained fields exhibit compensation on unseen views. Reuse for editing, prediction or measurement therefore requires evidence about the queried density-dependent quantity, beyond image fidelity.

\paragraph{2. Appearance capacity is an information budget.}
Fix the reference, rays, image weights and geometry parameterization. Let $A$ and $B_L$ map geometry and appearance increments to weighted images, with nested appearance ranges and inactive constraints. Variable projection \cite{golubpereyra2003} gives the profiled Gauss--Newton curvature
\begin{equation}\label{m:profilecurvature}
 G_L=A^*(I-P_{B_L})A,\qquad G_{L+1}\preceq G_L,
\end{equation}
where $P_{B_L}$ projects onto the appearance image range (\cref{cap:reduced,cap:psd}). New directions reduce curvature strictly when they explain previously uncompensated geometry signal; low-to-high schedules restrict these directions. \Cref{m:ladder} gives the continuous counterpart. Sharing also matters: an exact two-ray model retains density evidence under repeated directions but loses it under antipodal SH1 queries (\cref{q:twolayer}). \Cref{app:nativepilot} checks the native-readout prediction on held-out rays. At fixed resolution, \cref{rs:theorem} transfers a positive full-query density bound to finite samples while retaining all bounded appearance refits, including changing-rank cases.

\paragraph{3. Transmission measurements reveal hidden information.}
Keeping effective source unchanged during a black/white background switch gives $F_1-F_0=T$, as in triangulation matting \cite{smithblinn1996}. This directly measures attenuation: \cref{bg:footprint} separates the exact pair with finite pixels and one-gray-level reading error; \cref{app:acquisition} gives density-error certificates at specified spatial resolution.

\subsection{Implications for modeling, evaluation and acquisition}
\textbf{Modeling.} Angular degree and spatial resolution jointly determine which density changes appearance can explain. Selecting a compact representation is therefore an inverse-problem assumption as well as a computational choice. The relevant capacity is the appearance image space induced by shared coefficients and actual queries, rather than parameter count alone.

\textbf{Evaluation.} A reconstruction intended for compositing should be evaluated on transmission-sensitive observations; a reconstruction intended for volumetric measurement needs evidence about density in its declared physical model. RGB test error remains a valuable measure of view synthesis, but the exact pair shows why it cannot certify density even with complete coverage. A useful diagnostic should compare image-compatible fields or hold density fixed while certifying the entire allowed appearance class.

\textbf{Acquisition.} More observations help only if they constrain the density signal surviving appearance refitting. Spatial and angular priors, pixel footprints, camera coverage, forward-integration error and noise must be specified together. Source-preserving transmission measurements provide one explicit way of changing that signal. The mathematical lesson is an observable-specific design principle, not a universal recommendation for one optimizer, SH schedule or camera count.

\subsection{Scope of the conclusions}\label{m:scope}
Continuum recovery assumes complete rays, known exterior, smooth fields and degree-dependent local neighborhoods. Genericity is topological; the affine illustration is a separate reference. Effective sources, fixed microscopic scattering and sorted Gaussian splatting impose different constraints. Packet tests fix the source envelope and normal; $\kappa_L$ measures reference sensitivity rather than calibrated density error. Synthetic training uses isotropic truth and finite budgets; captured smoke supplies no independent density truth. Matched smoke and neural-field controls do not establish a reproducible additional angular-compensation effect (\cref{smoke:angularreplication,app:sharedshdiagnostic}). Certificates concern their stated grids and rays. \Cref{app:limitations,app:officialrefit,app:capturedsmoke} details these scopes. Our contribution is the recovery theory and physical constructions, using established analytical tools.

\subsection{Conclusion}
RGB constrains density through the appearance class used to explain it. Exact SH2 ambiguity defeats complete-view certification, while same-reference recovery and hiding show that local identifiability alone gives no capacity-uniform accuracy guarantee. The fixed-density certificates isolate angular compensation; captured smoke demonstrates isotropic compensation. These results motivate evaluating physical quantities across RGB-compatible fits and designing density-sensitive measurements for field reuse. The statements identify the physical information a rendering representation preserves, and provide a basis for evaluating and acquiring the information required by its intended reuse.

\clearpage
\part{Technical proofs and complete supporting evidence}
\appendix
\renewcommand{\thesection}{\AlphAlph{\value{section}}}

\section{Model, admissible fields, and data topology}\label{sec:model}
Let $\Omega=B_1\subset\R^3$. An oriented line is $x=p+td$, with $d\in\Sph^2$ and $p\in d^\perp$. For $|p|<1$, put $\ell(p)=\sqrt{1-|p|^2}$. With zero background, the data are
\begin{equation}\label{eq:forward}
 \cR(\sigma,j)(p,d)=\int_{-\ell(p)}^{\ell(p)}
 \exp\left(-\int_{-\ell(p)}^t\sigma(p+sd)\dd s\right)j(p+td,d)\dd t.
\end{equation}
There are three source channels and one shared scalar density. Assume
\begin{equation}\label{eq:physical}
 \sigma\geq0,\qquad 0\leq j_k(x,d)\leq\sigma(x),\qquad
 j_k(x,\cdot)\in\cH_{\leq L}(\Sph^2).
\end{equation}
Where $\sigma>0$, the color is $c_k=j_k/\sigma$ and has the same angular degree. In vacuum, $j=0$ and color is immaterial; no division is used there.

For an interior point, let $\tau_+(x,d)$ be its exit distance and define
\begin{align}
 T(x,d)&=\exp\left(-\int_{\text{entry}}^x\sigma\dd t\right),\label{eq:T}\\
 U(x,d)&=\int_0^{\tau_+(x,d)}
 \exp\left(-\int_0^s\sigma(x+ud)\dd u\right)j(x+sd,d)\dd s.\label{eq:U}
\end{align}
Here $U$ is the radiance from the ray portion beyond $x$. It is not an externally controlled illumination. It satisfies $(d\cdot\nabla-\sigma)U=-j$ and has zero exit value.

\begin{assumption}[Local physical class]\label{ass:class}
The reference density and the finitely many spatial source coefficients are $C^6$ on $\overline\Omega$. A fixed jump to exterior vacuum is allowed. Near an interior point $x_0$, there are strict margins $\sigma_*\geq s_*>0$, $j_{*,k}\geq\mu>0$, and $\sigma_*-j_{*,k}\geq\mu$, uniformly in direction. Only globally $C^6$ coefficient perturbations supported in $\overline{B_r(x_0)}\Subset\Omega$ are unknown. The exterior is fixed. Candidate fields satisfy \eqref{eq:physical} and lie in a relative $C^6$ neighborhood of the reference.
\end{assumption}
The up-to-boundary assumption is part of the field class, not merely pointwise interior smoothness: coefficient derivatives through order six are bounded on the closed ball. Chord lengths are at most two, giving finite optical thickness. On each compact interior ball the chord exit map has bounded derivatives, so differentiation of the transport integrals yields interior $C^4$ regularity of $T,U$, uniformly in direction. Fixed vacuum or saturated colors outside the unknown region are permitted. The margin inside the region is needed for two-sided physical perturbations.

All Fourier transforms are unitary: $\F_m f(\xi)=(2\pi)^{-m/2}\int e^{-ix\cdot\xi}f(x)\dd x$. For a data difference define
\begin{equation}\label{eq:Y}
 \norm{g}_Y^2=\frac1{4\pi^2}\int_{\Sph^2}\int_{d^\perp}
 |\xi|\,|\F_{2,p}g(\xi,d)|^2\dd\xi\,\mathrm dS(d).
\end{equation}
RGB components use the Euclidean squared norm. If the absolute reference data are not in $Y$, observations are understood in the affine space $\cR(\sigma_*,j_*)+Y$. Smooth interior-supported differences belong to this space. The half-derivative normalization is consistent with sharp ray-transform stability theory \cite{assylbekovstefanov2020}; numerical values below depend on this normalization.

\begin{lemma}[Transport identities]\label{lem:transport}
For admissible fields, with $h=\sigma_2-\sigma_1$ and $q=j_2-j_1$,
\begin{align}
 D\cR_*[h,q]&=\int T_*(q-U_*h)\dd t,\label{eq:derivative}\\
 \cR_2-\cR_1&=\int T_1(q-U_2h)\dd t.\label{eq:exactdifference}
\end{align}
\end{lemma}
\begin{proof}
These are the classical joint-identification identities of Stefanov \cite[Proposition 3.1 and \S5, (5.1)--(5.4)]{stefanov2014}, with ray orientation reversed. In our exit-value convention, subtraction gives $(d\cdot\nabla-\sigma_1)(U_2-U_1)=hU_2-q$ with zero exit value. Its integrating factor proves \eqref{eq:exactdifference}; differentiation gives \eqref{eq:derivative}.
\end{proof}

\section{The exact frozen density gain}\label{sec:frozen}
Fix $x_0$. Abbreviate $T_0(d)=T_*(x_0,d)$ and $U_0(d)=U_*(x_0,d)$. Let $t_- =\min T_0>0$ and $t_+=\max T_0$. For a normal $n$, write $\avg{f}_n=(2\pi)^{-1}\int_{C_n}f\dd\theta$. The restriction of $\cH_{\leq L}(\Sph^2)$ to $C_n$ is exactly the space $\cV_L(C_n)$ of trigonometric polynomials of degree at most $L$. Indeed, in a circle frame $e_1,e_2$, the real and imaginary parts of $(d\cdot e_1+i d\cdot e_2)^m$ lift each frequency $m\leq L$.

Define
\begin{align}
 E_L(n)&=\sum_{k=1}^3\min_{p\in\cV_L(C_n)}
 \avg{T_0^2|U_{0,k}-p|^2}_n, & \kappa_L^2&=\min_n E_L(n),\label{eq:gain}\\
 A_0(h,q)&=\int T_0(d)[q(y,d)-U_0(d)h(y)]\dd t.\label{eq:frozenop}
\end{align}
The minimum exists by compactness and continuity. The source coefficients may be complex after spatial Fourier transformation.

\begin{theorem}[Sharp frozen density estimate]\label{thm:frozen}
For compactly supported $L^2$ coefficients with source degree at most $L$,
\begin{equation}\label{eq:frozenbound}
 \norm{A_0(h,q)}_Y\geq\kappa_L\norm{h}_2.
\end{equation}
On any fixed nonempty open ball, $\kappa_L$ is the largest constant valid for all such supported coefficients. Hence $\kappa_L>0$ is equivalent to a uniform frozen density estimate after arbitrary degree-$L$ source refitting.
\end{theorem}
\begin{proof}
The Fourier slice relation and coarea give the exact identity
\begin{equation}\label{eq:slice}
 \norm{A_0(h,q)}_Y^2=\int_{\R^3}
 \avg{T_0^2|\widehat q(\xi,d)-U_0(d)\widehat h(\xi)|^2}_{\xi^\perp}\dd\xi.
\end{equation}
Specifically the slice factor is $\sqrt{2\pi}$, and
$\int_{\Sph^2}\delta(\xi\cdot d)f(d)\,\mathrm dS(d)=(2\pi/|\xi|)\avg f_{\xi^\perp}$.
At each nonzero frequency, minimize over the permitted circle polynomial. The resulting lower bound is $E_L(\xi/|\xi|)|\widehat h(\xi)|^2$. Plancherel proves \eqref{eq:frozenbound}.

For sharpness choose a minimizing normal $n_0$, its optimal real circle polynomials $p_k$, and real spherical lifts $P_k\in\cH_{\leq L}$. For a nonzero real $\phi\in C_c^\infty$ in the ball, set
$h_N(y)=\phi(y)\cos(Nn_0\cdot y)$ and $q_{N,k}(y,d)=P_k(d)h_N(y)$.
The two Fourier packets concentrate in directions $\pm n_0$. The residual circle energy of $T_0(P-U_0)$ is continuous and has the same value $\kappa_L^2$ at both normals. The cross term tends to zero, so the data-to-density ratio tends to $\kappa_L$. Multiplying by an arbitrarily small $N$-dependent factor preserves that ratio and makes any prescribed finite smooth norm small.
\end{proof}

\subsection{An explicit weighted formula and the angular ladder}\label{sec:formula}
Let $e_a(\theta)=e^{ia\theta}$, $-L\leq a\leq L$, and $w=T_0^2$. Form
\begin{equation}\label{eq:gram}
 G_{ab}=\avg{w\overline e_a e_b}_n,\qquad
 b_{k,a}=\avg{w\overline e_a U_{0,k}}_n.
\end{equation}
Then $G$ is positive definite, with condition number at most $(t_+/t_-)^2$, and
\begin{equation}\label{eq:schur}
 E_L(n)=\sum_k\left[\avg{w|U_{0,k}|^2}_n-b_k^*G^{-1}b_k\right].
\end{equation}
This is a weighted least-squares Schur complement. In computation, a QR-based residual avoids subtracting nearly equal positive numbers.

If $T_0=t_0$ is constant and $u_{k,m}(n)$ are the circle Fourier coefficients of $U_{0,k}$, Parseval gives
\begin{equation}\label{eq:tail}
 \kappa_L^2=t_0^2\min_n\sum_k\sum_{|m|>L}|u_{k,m}(n)|^2.
\end{equation}
For varying $T_0$, the nuisance space is $T_0\cV_L$, not $\cV_L$. Thus the ordinary Fourier tail of $T_0U_0$ is generally incorrect.

\begin{corollary}[One fixed reference]\label{cor:ladder}
For the same $T_0,U_0$, $\kappa_{L+1}\leq\kappa_L$ and $\kappa_L\to0$. If
$S_s^2=\sup_n\sum_k\avg{|\partial_\theta^sU_{0,k}|^2}_n<\infty$ for an integer $s\geq1$, then
\begin{equation}\label{eq:tailrate}
 \kappa_L\leq\frac{t_+S_s}{(L+1)^s}.
\end{equation}
\end{corollary}
\begin{proof}
The approximation spaces are nested. Approximation on any one fixed circle gives convergence to zero. The ordinary Fourier truncation of $U_0$ is an admissible competitor in \eqref{eq:gain}; its error, multiplied by $t_+^2$, is bounded using $|m|\geq L+1$ and Parseval.
\end{proof}
The inequalities need not be strict. For a comparison of actual inverse problems, the reference must belong to every compared angular class. A different reference for each $L$ does not establish a fixed-scene trend.

\subsection{The $L=0$ connection to classical identification}
For isotropic sources, \eqref{eq:schur} reduces to
\begin{equation}\label{eq:variance}
 E_0(n)=\sum_k\left(\avg{T_0^2U_{0,k}^2}_n
 -\frac{\avg{T_0^2U_{0,k}}_n^2}{\avg{T_0^2}_n}\right).
\end{equation}
In a planar scalar problem the frequency-normal set consists of two opposite directions. With weights $w_+,w_->0$ and radiances $U_+,U_-$, its normalized counterpart is
\begin{equation}\label{eq:twod}
 \min_a\frac{w_+(U_+-a)^2+w_-(U_--a)^2}{2}
 =\frac{w_+w_-}{2(w_++w_-)}(U_+-U_-)^2.
\end{equation}
This is the same opposite-direction degeneracy appearing in Stefanov's characteristic analysis \cite[Sections 4--5]{stefanov2014}. Three-dimensional frequency normals supply a circle instead of two directions, and finite-degree sources replace constants by $\cV_L$. This is an extension of the classical mechanism. A vanishing leading gain does not exclude weaker estimates with derivative loss or support-dependent injectivity.

\section{From the criterion to local nonlinear stability}\label{sec:nonlinear}
A density-only frozen estimate cannot absorb arbitrary perturbations of the source weight unless a source representative is controlled as well. This section supplies that control without assuming the desired density injectivity.

Translate $x_0$ to the origin and set $y=x/r$. Let $h(y)=(\sigma_2-\sigma_1)(ry)$, $q(y,d)=(j_2-j_1)(ry,d)$, and
$g(p,d)=[\cR_2-\cR_1](rp,d)/r$. All differences are zero extended from $B_1$.

\subsection{A trace-aware tensor gauge}
For $L\geq1$ there is a unique representation on $|d|=1$,
\begin{equation}\label{eq:tensorrep}
 q(y,d)=F_L(y)[d^L]+F_{L-1}(y)[d^{L-1}],
\end{equation}
where symmetric tensors may have traces. Each tensor accounts for one parity of the spherical-harmonic expansion; multiplication by the metric lifts lower degrees by two. The dimensions sum to $(L+1)^2$. For $L=0$ retain only $F_0$. Norms on tensors are Euclidean Frobenius norms and symmetrization is orthogonal averaging.

Write $D$ for the symmetric gradient, $i s=\sym(g_{\mathrm E}\otimes s)$, and $\alpha(y)=r\sigma_1(ry)$. For potentials $P=(p,s)$ of ranks $L-1,L-2$, define
\begin{equation}\label{eq:gauge}
 G_\alpha(p,s)=(Dp-\alpha i s,\ Ds-\alpha p),
 \qquad G_0(p,s)=(Dp,Ds).
\end{equation}
For $L=1$, $p$ is scalar and the conventions are explicitly
\begin{equation}\label{eq:gaugeL1}
 G_\alpha p=(\nabla p,-\alpha p),\quad G_0p=(\nabla p,0),\quad
 q^s=(F_1-\nabla p)\cdot d+(F_0+\alpha p).
\end{equation}
The scalar attenuation term is retained. For $L=0$ there is no potential gauge: $q=F_0$ is unchanged and no $\gamma_L$ condition is imposed. Evaluated on the sphere, the general gauge is $(d\cdot\nabla_y-\alpha)P(y,d)$. A zero-boundary potential has zero weighted ray integral because $d\cdot\nabla_yT_1(ry,d)=-\alpha T_1(ry,d)$. Attenuation therefore couples adjacent ranks; independent projections of each SH coefficient would not preserve the data.

\begin{lemma}[Small-optical-scale gauge]\label{lem:gauge}
If $\gamma_L=\sqrt{3L}\norm\alpha_\infty/\pi<1$, each $L^2$ tensor pair has a representative $F^s=F-G_\alpha P$, with $P\in H_0^1(B_1)$, that is divergence free in both positive ranks inside the ball and has exactly the same weighted source data. Rank zero has no divergence condition.
\end{lemma}
\begin{proof}
Solve
\begin{equation}\label{eq:weakgauge}
 \avg{G_\alpha P,G_0Q}_{B_1}=\avg{F,G_0Q}_{B_1}
 \quad\text{for all }Q\in H_0^1(B_1).
\end{equation}
For a symmetric rank-$(m-1)$ tensor $p$,
$|\sym(\xi\otimes p)|^2=(|\xi|^2|p|^2+(m-1)|\iota_\xi p|^2)/m$.
Fourier transformation after zero extension, followed by the Dirichlet Poincar\'e inequality on the unit three-ball, gives
$\norm P_2\leq\sqrt L\norm{G_0P}_2/\pi$.
Also $\norm{i s}\leq\sqrt3\norm s$. Hence
$\norm{(G_\alpha-G_0)P}_2\leq\gamma_L\norm{G_0P}_2$.
The bilinear form is coercive in the equivalent $H_0^1$ norm $\norm{G_0P}_2$, so Lax--Milgram applies. Equation \eqref{eq:weakgauge} says $G_0^*F^s=0$. Approximation by smooth zero-boundary potentials proves the weighted line identity. The representative is an auxiliary difference, and need not satisfy the physical box.
\end{proof}

\subsection{The boundary term and a joint estimate}
Let $P_\xi=\Id-\xi\xi^\T/|\xi|^2$ and $\Pi_m(\xi)=P_\xi^{\otimes m}$ on symmetric tensors; $\Pi_0=\Id$.
\begin{lemma}[Ball projection and circle frame]\label{lem:ball}
For every finite $m$ there is $C_m<\infty$, with $C_0=1$, such that an interior-solenoidal $F\in L^2(B_1;S^m)$ satisfies
\begin{equation}\label{eq:ballproj}
 \norm{\Pi_m\widetilde F}_2\geq C_m^{-1}\norm F_{L^2(B_1)},
\end{equation}
where $\widetilde F$ is its zero extension. On transverse tensors,
\begin{equation}\label{eq:frame}
 \avg{|F[d^m]|^2}_n\geq2^{-m}|F|^2.
\end{equation}
The full circle-frame eigenvalue list is $2^{-m}\binom mj$, $0\leq j\leq m$.
\end{lemma}
The proof of \eqref{eq:ballproj}, including the Killing-field compatibility on a fixed Euclidean annulus, is in \cref{app:ball}. Interior solenoidality does not imply that the zero extension is solenoidal. For \eqref{eq:frame}, use the Hermitian orthonormal basis $e_\pm=(e_1\pm i e_2)/\sqrt2$. The normalized symmetric products have distinct frequencies $m-2j$ and squared amplitudes $2^{-m}\binom mj$; Parseval gives the assertion.

Let $\overline C_L=\max_{0\leq m\leq L}C_m$, $b_L=2^{-L/2}/\overline C_L$, and
\begin{equation}\label{eq:jointconstant}
 u_0=\max_n\avg{T_0^2|U_0|^2}_n^{1/2},\qquad
 \lambda_L=\frac{\kappa_L}{\sqrt{1+((\kappa_L+u_0)/(t_-b_L))^2}}.
\end{equation}
The two ranks in \eqref{eq:tensorrep} have opposite frequency parity, so the unweighted circle cross term vanishes. Equations \eqref{eq:slice}, \eqref{eq:ballproj}, and \eqref{eq:frame} give
$\norm{I_{T_0}q^s}_Y\geq t_-b_L\norm{F^s}_2$.
Using the density estimate and the triangle inequality then yields
\begin{equation}\label{eq:joint}
 \norm{A_0(h,q^s)}_Y\geq\lambda_L
 (\norm h_2^2+\norm{F^s}_2^2)^{1/2}.
\end{equation}
Unlike $\kappa_L$, this joint constant is not claimed to be optimal.

\subsection{An explicit weight perturbation bound}
For a compactly supported matrix weight $E(y,d)$ with finite spatial Fourier-Wiener norm, set
\begin{equation}\label{eq:algebra}
 \cA(E)=(2\pi)^{-3/2}\int_{\R^3}\sup_d
 \norm{\F_3E(\eta,d)}_{\mathrm{op}}\dd\eta.
\end{equation}
Then, for $I_Ez=\int E(y,d)z(y)\dd t$, the first estimate below holds whenever $\cA(E)<\infty$. For the second, assume additionally that $E$ is spatially $C^4$ with compact support and derivatives bounded uniformly in $d$:
\begin{equation}\label{eq:weightbound}
 \norm{I_Ez}_Y\leq\cA(E)\norm z_2,\qquad
 \cA(E)\leq\frac1{8\pi}\int\sup_d
 \norm{(1-\Delta_y)^2E(y,d)}_{\mathrm{op}}\dd y.
\end{equation}
To verify the first estimate, expand $E$ by Fourier inversion and apply Minkowski to constant angular multipliers acting on $e^{i\eta\cdot y}z$. The coarea identity cancels the $|\xi|$ in \eqref{eq:Y} before the translation $\xi\mapsto\xi-\eta$; no bounded trace of a general $L^2$ function on one plane is assumed. The second estimate follows from integration by parts and $\int_{\R^3}(1+|\eta|^2)^{-2}\dd\eta=\pi^2$.

Use \eqref{eq:exactdifference} to write $g=A_W(h,F)$, where
\begin{equation}\label{eq:W}
 W(y,d)(h,F)=T_1(ry,d)\big(F_L[d^L]+F_{L-1}[d^{L-1}]-U_2(ry,d)h\big).
\end{equation}
Let $W_0$ use $T_0,U_0$ instead. Choose $r$ so that $\overline{B_{2r}(x_0)}\subset\Omega$, and a smooth $\chi=1$ on $B_1$, supported in $B_2$. Thus the rescaled physical weight is defined on the cutoff support; define $C_E=\cA(\chi(W-W_0))$.

\begin{theorem}[Local nonlinear density stability]\label{thm:nonlinear}
Under \cref{ass:class}, suppose $\kappa_L(x_0)>0$. For a pair of candidates satisfying $\gamma_L<1$ when $L\geq1$ and $C_E<\lambda_L$,
\begin{equation}\label{eq:localstability}
 \norm{(\sigma_2-\sigma_1)(r\,\cdot)}_2
 \leq\frac{\norm{[\cR_2-\cR_1](r\,\cdot,\cdot)/r}_Y}{\kappa_L(1-C_E/\lambda_L)}.
\end{equation}
For every fixed finite $L$ with $\kappa_L(x_0)>0$, a sufficiently small radius and a nonempty relative $C^6$ neighborhood satisfy these conditions uniformly for all pairs. Equal complete data imply equal density in that class. No source uniqueness is asserted.
\end{theorem}
\begin{proof}
Apply \cref{lem:gauge} with the exact attenuation $T_1$ in \eqref{eq:W}. It leaves both $h$ and the data unchanged. For the same $z=(h,F^s)$, \eqref{eq:joint} gives $\norm z_2\leq\norm{A_0z}_Y/\lambda_L$. Subtract the perturbation bound before applying the sharp density estimate:
\[
 \norm g_Y\geq(1-C_E/\lambda_L)\norm{A_0z}_Y
 \geq\kappa_L(1-C_E/\lambda_L)\norm h_2.
\]
The original absorption condition remains $C_E<\lambda_L$.
At the reference, $\chi[W_*(ry,d)-W_0(d)]\to0$ in spatial $C^4$, uniformly in $d$, as $r\to0$. Equation \eqref{eq:weightbound} makes $C_E\to0$, and $\gamma_L\to0$ as well. Continuity of the forward integrals and their derivatives under $C^6$ coefficient perturbations preserves the strict inequalities in a small relative neighborhood. The physical margins persist under small $C^0$ changes. Scaling the offset by $r$ gives a factor $r^{3/2}$ in both the density $L^2$ norm and the data norm, so the same constant holds in unscaled coordinates.
\end{proof}

This theorem proves more than symbol ellipticity, but less than a large-domain uniqueness theorem. It controls the full operator by a frozen estimate and a quantified perturbation. The reference/exterior information and smoothness bounds are priors; they are not recovered from RGB by this argument.

\subsection{Physical nonempty examples for every finite order}\label{sec:positiveexamples}
For each fixed $L$, set $M=L+1$, $\sigma_*=1$, and $j_{*,k}=1/2+x_k^M/4$ on $B_1$. These fields have isotropic color and strict physical margins. At the center,
$T_0=e^{-1}$ and $U_{0,k}=(1-e^{-1})/2+B_Md_k^M$, where
$B_M=\tfrac14\int_0^1e^{-t}t^M\dd t>0$.
Only the two highest circle frequencies remain after degree-$L$ projection. Their squared energy is $B_M^2 2^{1-2M}(1-n_k^2)^M$. Since $\sum_k(1-n_k^2)=2$, convexity gives the minimum at $n_k^2=1/3$ (with equality throughout for $M=1$). Thus
\begin{equation}\label{eq:everyL}
 \kappa_L(0)=e^{-1}B_{L+1}6^{-L/2}>0.
\end{equation}
These are different reference scenes for different $L$. They refute a universal claim that every finite degree beyond a fixed threshold destroys local identifiability, but do not describe a fixed-scene improvement.

\section{Fragility of the actual volume problem}\label{sec:fragility}
Frozen optimality alone is not an algorithm-independent statement about a spatially varying nonlinear problem. We now localize the envelope and increase its oscillation count simultaneously.

\begin{theorem}[True-operator and nonlinear obstruction]\label{thm:fragility}
Assume the margins and smoothness in \cref{ass:class}. Let $\mathcal U$ contain a relative $C^6$ neighborhood of the reference among physical, degree-$L$, compact interior perturbations. Then any estimate
\begin{equation}\label{eq:inverseC}
 \norm{\sigma_2-\sigma_1}_2\leq C\norm{\cR_2-\cR_1}_Y
 \quad\text{on }\mathcal U
\end{equation}
requires $C\geq1/\kappa_L(x_0)$. If $\kappa_L(x_0)=0$, no finite such constant exists. This conclusion does not by itself imply exact nonlinear nonuniqueness.
\end{theorem}
\begin{proof}
Choose the minimizing normal and spherical polynomial $P$ from \cref{thm:frozen}. Put
$F(n)=\avg{|T_0(P-U_0)|^2}_n$. It is even, $C^2$, and has a minimum $\kappa_L^2$ at $\pm n_0$. Therefore
\begin{equation}\label{eq:Dpacket}
 D_0=\sup_{n\ne\pm n_0}
 \frac{F(n)-\kappa_L^2}{\min(|n-n_0|,|n+n_0|)^2}<\infty.
\end{equation}
Let $\phi\in C_c^\infty(B_1)$ be real, $\norm\phi_2=1$, and
$H_N(y)=\phi(y)\cos(Nn_0\cdot y)/c_N$, where $c_N$ is the normalizing norm. For all large $N$, $c_N\geq1/2$. Define physical perturbation directions
\begin{equation}\label{eq:packet}
 h_{r,N}(x)=r^{-3/2}H_N((x-x_0)/r),\qquad q_{r,N}(x,d)=P(d)h_{r,N}(x).
\end{equation}
They are real and $\norm{h_{r,N}}_2=1$. Fourier concentration and
$|\xi/|\xi|-n_0|\leq2|\xi-Nn_0|/N$ give
\begin{equation}\label{eq:packetfrozen}
 \kappa_L^2\leq\norm{I_{T_0(P-U_0)}H_N}_Y^2
 \leq\kappa_L^2+C_{\rm wp}/N^2,
 \quad C_{\rm wp}=16D_0\norm{\nabla\phi}_2^2.
\end{equation}
For $W(x,d)=T_*(x,d)[P(d)-U_*(x,d)]$, a cutoff of
$W(x_0+ry,d)-W(x_0,d)$ has spatial $C^4$ norm $O(r)$. Equation \eqref{eq:weightbound} implies an operator error at most $C_{\rm sp}r$. The physical data scaling is $r^{-1/2}G((p-p_0)/r,d)$, whose $Y$ norm equals $\norm G_Y$. Consequently,
\begin{equation}\label{eq:twoscales}
 (\kappa_L-C_{\rm sp}r)_+\leq\norm{D\cR_*[h_{r,N},q_{r,N}]}_Y
 \leq\sqrt{\kappa_L^2+C_{\rm wp}/N^2}+C_{\rm sp}r.
\end{equation}
Let $r_j\to0$, $N_j\to\infty$. The derivative gains tend to $\kappa_L$.

It remains to justify a nonlinear physical sequence. For each fixed smooth direction $v_j=(h_j,q_j)$, there are finite $K_j$ and $b_j>0$ such that
\begin{equation}\label{eq:remainder}
 \norm{\cR(\theta_*+av_j)-\cR(\theta_*)-aD\cR_*v_j}_Y\leq K_j a^2
 \quad (|a|\leq b_j),
\end{equation}
where $\theta=(\sigma,j)$. This follows from \eqref{eq:exactdifference}, the $C^4$ differentiability of $U_{\theta_*+av_j}$ in $a$, and \eqref{eq:weightbound}. Choose $b_j$ also to enforce the physical margins and the neighborhood bound for both signs. With
$a_j=\min\{b_j/j,1/[j(1+K_j)]\}$, the fields converge in $C^6$ to the reference and their nonlinear data-to-density ratios converge to $\kappa_L$. Applying \eqref{eq:inverseC} proves the claim. The constants $K_j$ are not assumed uniform over shrinking packets.
\end{proof}

\begin{corollary}[Any-estimator small-noise bound]\label{cor:minimax}
For deterministic $Y$-bounded noise, let
\begin{equation}\label{eq:minimaxdef}
 M_{\mathcal U}(\delta)=\inf_{\widehat\sigma}
 \sup_{\theta\in\mathcal U,\ \norm e_Y\leq\delta}
 \norm{\widehat\sigma(\cR(\theta)+e)-\sigma_\theta}_2.
\end{equation}
Then $\liminf_{\delta\downarrow0}M_{\mathcal U}(\delta)/\delta\geq1/\kappa_L(x_0)$, with infinity when $\kappa_L=0$.
\end{corollary}
\begin{proof}
Fix $\eta>0$ and one physical perturbation direction of unit density norm with derivative gain at most $\kappa_L+\eta$. For all sufficiently small $a>0$, \eqref{eq:remainder} bounds its nonlinear data change by $a(\kappa_L+2\eta)$. For every sufficiently small $\delta$, take $a=2\delta/(\kappa_L+2\eta)$. The midpoint of the two data sets is within $\delta$ of each. Every estimator makes density error at least $a/2$ on one of the two fields. First take $\liminf$ in $\delta$, then $\eta\downarrow0$. This order proves a bound for all small noise levels, not merely a selected subsequence.
\end{proof}

Combining \cref{cor:ladder,cor:minimax}, the necessary asymptotic inverse coefficient for a fixed scene cannot decrease with $L$. When $S_s>0$, it is at least $(L+1)^s/(t_+S_s)$. The limit is taken at fixed $L$ first; no useful uniform finite-noise cutoff or finite-pixel risk follows without additional frequency and sampling estimates. These are classical packet and two-point principles applied to the explicit source-refitted density gain.

\section{Same-reference capacity separation: complete arguments}\label{app:capacityproofs}
We prove the source-genericity and capacity-hiding statements used in \cref{m:coexistence}, retaining the sharp nonlinear density coefficient. The setting is the unit ball, complete oriented lines, black background and effective sources, either free across RGB or constrained to fixed channel proportions. Density and finite-degree source coefficients are smooth on the closed ball; a fixed exterior jump to vacuum is permitted. All differences use \eqref{eq:Y}. The RGB source class is \eqref{m:XD}--\eqref{m:Amu}, with compact sets $K,D\subset\Omega$ with $K\Subset\operatorname{int}D$ and nonempty interior for $K$.

\begin{figure}[htbp]
\centering\includegraphics[width=\textwidth]{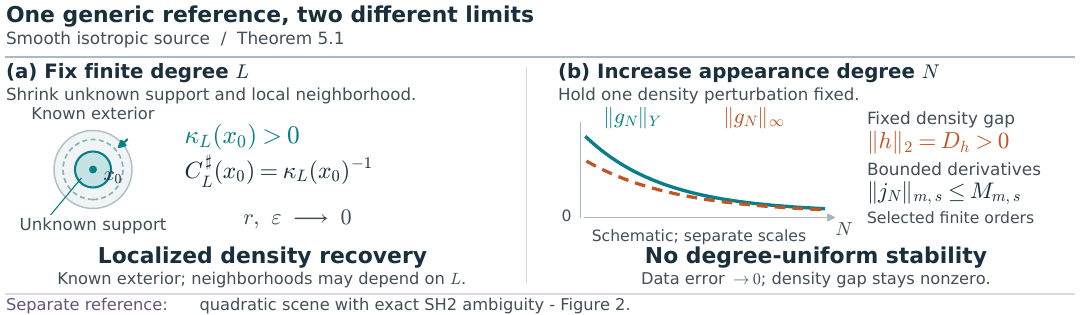}
\caption{\textbf{The two limits in the analytic same-reference theorem.} All drawings and curves are schematic, with separate scales and no measured rate. The fixed-degree neighborhood limit gives $C_L^\sharp(x_0)=1/\kappa_L(x_0)$. Increasing degree then hides a fixed density difference within selected finite derivative bounds. The lower note identifies the quadratic exact SH2 example as a separate reference. The numerical affine-source montage in \cref{m:teaser} illustrates R and is outside $\mathcal G$.}\label{cap:overviewfigure}
\end{figure}

\subsection{Physical hiding with the actual positive approximants (R)}\label{cap:R}
This construction does not require genericity. Fix any physical smooth reference with finite source degree $L_0$ and a nonzero $h\in C_c^\infty(\operatorname{int}K)$, $h\geq0$. Its remaining radiance $U$ is smooth on a neighborhood of $K\times\Sph^2$, since the exit distance is smooth there and the coefficients extend smoothly to the physical boundary. No division by density is used, including in vacuum. The transport comparison gives $0\leq U_k\leq1$.

For $N\geq L_0$ define $P_N=\cV_NU$ by \eqref{m:positiveoperator} and $j_N=j+hP_N$. This is the classical positive spherical de la Vall\'ee Poussin operator \cite[(4)--(7)]{yang2011}: it preserves constants and has degree at most $N$, but is not the degree-$N$ orthogonal projection. For our normalization its mass and second chord moment are
\begin{equation}\label{cap:moments}
 (N+1)\int_0^1u^N\dd u=1,\qquad
 4(N+1)\int_0^1(1-u)u^N\dd u=\frac4{N+2}.
\end{equation}
Positivity gives $0\leq P_N\leq1$ and hence
\begin{equation}\label{cap:box}
 j_N\geq0,\qquad \sigma+h-j_N=(\sigma-j)+h(1-P_N)\geq0.
\end{equation}
Any source/absorption slack of the reference on the changed support is preserved.

The zonal operator commutes with spatial derivatives and rotations, hence with every ordered rotation word. Applying its contraction and approximate-identity property to this finite list of derivatives yields
\begin{equation}\label{cap:contraction}
 \norm{P_N}_{m,s;K}\leq\norm U_{m,s;K},\qquad
 \norm{P_N-U}_{m,s;K}\longrightarrow0.
\end{equation}
Uniformity over $x\in K$ follows from compactness: the kernel mass outside $|d-e|<\delta$ is at most $4/((N+2)\delta^2)$. This also verifies the mixed-derivative statement, rather than only convergence at each spatial point.

Leibniz' rule now supplies the global bound
\begin{equation}\label{cap:jetbound}
 \norm{j_N}_{m,s}\leq\norm j_{m,s}+2^m\norm h_{C^m}\norm U_{m,s;K}.
\end{equation}
Outside $K$ the source does not change. The density $\sigma+h$ is fixed. Thus every prescribed finite collection of derivative orders has common finite bounds independent of $N$. This is the quantifier order $\forall(m,s)\ \exists M_{m,s}\ \forall N$, not one bound for every derivative order.

The classical transport identity in \cref{lem:transport} gives $\cR(\sigma+h,j+hU)=\cR(\sigma,j)$: the reference $U$ solves the modified exit-value equation. Linearity in the source at density $\sigma+h$ then gives
\begin{equation}\label{cap:residual}
 g_N:=\cR(\sigma+h,j_N)-\cR(\sigma,j)
 =X_{T_{\sigma+h}(P_N-U)}h.
\end{equation}
Choose a fixed smooth $\chi=1$ near $\supp h$, supported in $\operatorname{int}K$, and put $w_\chi=\chi T_{\sigma+h}$. Extended by zero, $E_N=w_\chi(P_N-U)$ is spatially compact and $C^4$. The weight estimate \eqref{eq:weightbound} gives, with $D_{4,N}=\norm{P_N-U}_{4,0;K}$,
\begin{align}
 \norm{g_N}_Y&\leq C_\chi D_{4,N}\norm h_2\longrightarrow0,\label{cap:Ybound}\\
 C_\chi&=\frac{169\sqrt3\,|\supp\chi|}{8\pi}
 \max_{|\alpha|\leq4}\norm{D_x^\alpha w_\chi}_\infty.\label{cap:Cchi}
\end{align}
Indeed the product coefficients in $(1-\Delta)^2$ are bounded in total by $1+2\cdot3\cdot4+3^2\cdot16=169$. The $\sqrt3$ converts channel maxima to the RGB norm. This finite constant may depend on the fixed $h$, but not on $N$. Separately,
\begin{equation}\label{cap:RGBbound}
 \norm{g_N}_\infty\leq\beta_h\norm{P_N-U}_{0,0;K}\to0,\qquad
 \beta_h=\sup_{\text{rays}}\int T_{\sigma+h}h\dd t\leq1.
\end{equation}
Small maximum RGB alone would not imply \eqref{cap:Ybound}; its proof uses spatial $C^4$ convergence.

\paragraph{Optional sufficient rate.}
Let $H_{4,2}$ bound the intrinsic spherical Hessians of all $D_x^\alpha U_k$, $|\alpha|\leq4$. At fixed geodesic distance $\theta$, the zonal kernel cancels the first-order Taylor term. The second-order remainder is at most $\theta^2H_{4,2}/2$. Since $\theta\leq\pi|d-e|/2$, \eqref{cap:moments} implies
\begin{equation}\label{cap:rate}
 D_{4,N}\leq\frac{\pi^2H_{4,2}}{2(N+2)}.
\end{equation}
This classical $O(N^{-1})$ sufficient estimate is not an optimal degree law, nor a numerical degree certificate without bounds on $H_{4,2}$ and $C_\chi$. For any specified finite derivative orders, enough finite spatial regularity also suffices; the closed-ball $C^\infty$ hypothesis supplies all such orders simultaneously.

\paragraph{Any prescribed finite-derivative neighborhood.}
Fix $0\ne h_0\geq0$ in $C_c^\infty(\operatorname{int}K)$. For each constraint
\[
 \norm{\sigma'-\sigma}_{C^{m_i}}+\norm{j'-j}_{m_i,s_i}<\delta_i,
 \qquad \delta_i>0,
\]
set $Q_i=\norm{h_0}_{C^{m_i}}(1+2^{m_i}\norm U_{m_i,s_i;K})$. Choose one $0<a<\min_i\delta_i/Q_i$ and set $h=a h_0$. By \eqref{cap:jetbound}, every $(\sigma+h,j+hP_N)$ meets every constraint. Both data norms tend to zero, but density distance is always $a\norm{h_0}_2>0$. Any density modulus tending to zero uniformly in degree is therefore impossible on this fixed neighborhood. The order is neighborhood, then fixed amplitude, then increasing degree.

For simultaneous bounded noise radii $\eta_Y,\eta_\infty>0$, choose $N$ making both data differences at most twice their respective radius. The common midpoint observation is consistent with either field; every estimator incurs density error at least $\norm h_2/2$ for one of them. Restriction to camera subsets or nonnegative normalized pixel footprints contracts the maximum-RGB discrepancy, but no such contraction of $Y$ is assumed.

\subsection{Realizable genericity in an isotropic source class (G)}\label{cap:G}
The fixed-support space $X_D$ in \eqref{m:XD} is closed in $C^\infty(\overline\Omega;\R^3)$: a convergent sequence of functions zero outside $D$ has the same property in its limit. It is Fr\'echet and Baire; $\mathcal A_\mu$ is an open subset of its affine translate. Compact strict slack makes it open. We do not use the varying-support inductive-limit topology. At fixed $\sigma$, the remaining-radiance map $j\mapsto U_j$ is linear, while $T_\sigma$ is fixed and has positive lower bound $t_-$ on $K\times\Sph^2$.

\begin{theorem}[Generic positive compensated gain]\label{cap:generic}
For each finite $L\geq0$, the set
\[
 \mathcal G_L=\{j\in\mathcal A_\mu:\inf_{x\in K}\kappa_{L,j}(x)>0\}
\]
is open and dense. Thus $\mathcal G=\bigcap_{L\geq0}\mathcal G_L$ is dense residual. Each of its members is a single isotropic reference with positive gain at every fixed finite degree.

The same assertion holds inside a fixed-chromaticity class. Fix $\mathbf a\in(0,1]^3$, $\max_k a_k=1$, and a smooth physical scalar $f_b$ on the entire volume. Replace $\mathcal A_\mu$ by the nonempty scalar class
\[
 \mathcal A_\mu^{\rm sc}=\{f=f_b+v:v\in C^\infty(\overline\Omega),\ \supp v\subset D,\ \mu<f<\sigma-\mu\text{ on }D\},
\]
rendered as $j=\mathbf a f$ throughout $\Omega$, including the known exterior of $D$. Genericity is relative to this scalar Fr\'echet class.
\end{theorem}
\begin{proof}
For $q=1,2,3$, put $m_q=L+q$. In a local oriented frame $(u(n),v(n))$ of $n^\perp$, use $d(\theta)=u\cos\theta+v\sin\theta$ and moments of the \emph{first channel alone}:
\begin{equation}\label{cap:six}
 \alpha_q(x,n)=\avg{U_{j,1}(x,d(\theta))\cos m_q\theta},\qquad
 \beta_q(x,n)=\avg{U_{j,1}(x,d(\theta))\sin m_q\theta}.
\end{equation}
The components $F_j=(\alpha_1,\beta_1,\alpha_2,\beta_2,\alpha_3,\beta_3)$ define a rank-six bundle section over $(\operatorname{int}D)\times\Sph^2$. A frame rotation by $\phi$ rotates the $q$th pair by $m_q\phi$. This is the direct sum of three distinct frequency bundles; its norm and zero set are intrinsic. The six modes are mutually orthogonal, orthogonal to $\cV_L$, and have squared norm $1/2$. Bessel in the first channel gives
\begin{equation}\label{cap:Bessel}
 E_{L,j}(x,n)\geq2t_-^2|F_j(x,n)|^2.
\end{equation}

At $(x_0,n_0)$ choose $0<r_-<r_+<\operatorname{dist}(x_0,D^c)$ and $\chi\in C_c^\infty(r_-,r_+)$ with integral one. For scalar $H\in C^\infty(\Sph^2)$, define a spatial shell source
\begin{equation}\label{cap:shell}
 v_H(x_0+td)=\chi(t)\exp\!\left(\int_0^t\sigma(x_0+sd)\dd s\right)H(d).
\end{equation}
Here $t=|y-x_0|$ and $d=(y-x_0)/|y-x_0|$ locate the spatial point $y$; $v_H$ is isotropic and has smooth zero extension. Attenuation cancels exactly, giving $U_{v_H}(x_0,d)=H(d)$. Choose the six spherical polynomials $2\operatorname{Re}(d\cdot u_0+i d\cdot v_0)^{m_q}$ and $2\operatorname{Im}(d\cdot u_0+i d\cdot v_0)^{m_q}$. Their circle restrictions are $2\cos m_q\theta$ and $2\sin m_q\theta$. In the RGB class use $\delta j=(v_H,0,0)$, giving the identity moment derivative. In the scalar class use $\delta f=v_H$, hence $\delta j=\mathbf a v_H$, giving $a_1\Id$. Both are surjective; signed directions preserve the physical class at sufficiently small amplitude.

Continuity preserves each such nonzero minor on a neighborhood. A finite cover of $K\times\Sph^2$ supplies one finite list of spatial directions $v_1,\ldots,v_q$ whose moment derivative is onto over an open neighborhood $B$ of that compact set. Write these directions in RGB coordinates, with $v_i=\mathbf a w_i$ for scalar perturbations $w_i$ in the constrained class. A small parameter ball $P\subset\R^q$ ensures $j_p=j+\sum_i p_iv_i$ stays in the chosen physical source class and any prescribed relative Fr\'echet neighborhood. This uses finitely many seminorm constraints and physical slack.

The section $(p,x,n)\mapsto F_{j_p}(x,n)$ is transverse to zero because its parameter derivative alone is onto. Its zero set is a smooth manifold of dimension $q+5-6=q-1$. Every point of its projection to $P$ is critical. Sard's theorem in countably many charts shows that its image has measure zero. Some arbitrarily small $p$ therefore gives no zeros on $B$. Compactness and \eqref{cap:Bessel} give a uniform positive gain on $K$, proving density. This argument occurs on the open base $B$, not on a potentially nonsmooth compact $K$.

The same finite-cover argument applies in each of the two source spaces; the scalar space is itself closed and Baire. For openness, path length at most two gives $\norm{U_j-U_{j'}}_\infty\leq2\norm{j-j'}_\infty$ in channel maxima. Each $E^{1/2}$ is distance to the fixed closed nuisance space $(T_\sigma\cV_L)^3$. Its 1-Lipschitz property gives
\[
 |E_{L,j}^{1/2}-E_{L,j'}^{1/2}|\leq2\sqrt3\norm{j-j'}_\infty
\]
uniformly on the compact base. For $j=\mathbf a f$, $\norm{j-j'}_\infty=\norm{f-f'}_\infty$. Taking minima preserves the bound. Baire proves the countable-intersection statement in each source class.
\end{proof}
Stefanov \cite[Section~5]{stefanov2014} establishes scalar joint source/attenuation recovery under his hypotheses; all-finite-order generic tensor results are also known \cite{cekic2024}. Here the source-realizability step uses one physical isotropic spatial source at fixed three-dimensional density. The fixed-chromaticity assertion follows from surjectivity \emph{inside} the scalar class, rather than intersecting an ambient RGB residual set with a thin subspace.

\paragraph{Fixed channel proportions throughout recovery and hiding.}
Let $\kappa^{\rm sc}_{L,f}$ be the scalar weighted-circle gain. Since $U_{\mathbf a f}=\mathbf a U_f$ on every remaining ray and weighted projection is linear,
\begin{equation}\label{cap:chromaticity}
 \kappa_{L,\mathbf a f}=\norm{\mathbf a}_2\kappa^{\rm sc}_{L,f}.
\end{equation}
The identity holds for independent RGB refits and for tied polynomials $p_k=a_kp$, whose optimal scalar $p$ attains all three channel minima. It is a channel-norm factor. On $D$, $j_k>a_k\mu$ and $\sigma-j_k>\mu$, so the full nonlinear theorem applies. Restricting its local class to $j'=\mathbf a f'$ retains its upper bound. The sharp packets also remain in that class: choose the optimal scalar lift $P$ and source increment $q=\mathbf a P h_{\rm packet}$. Thus the squeeze in \cref{cap:S} gives the same limiting coefficient $1/(\norm{\mathbf a}_2\kappa^{\rm sc}_{L,f})$ for the restricted class.

For hiding, \cref{cap:R} gives $j_N=\mathbf a(f+h\cV_NU_f)$, with $0\leq f+h\cV_NU_f\leq\sigma+h$. The same fixed $h$, support and finite derivative constraints are retained. Scalar $Y$ discrepancies acquire factor $\norm{\mathbf a}_2$ and maximum-channel discrepancies factor $\max_k a_k=1$. Hence both norms tend to zero with fixed chromaticity, including grayscale $\mathbf a=(1,1,1)$. The shared proportions apply on the entire optical path, including the known exterior; spatial and directional intensity remain adjustable.

\subsection{The sharp limiting nonlinear coefficient (S)}\label{cap:S}
At fixed $L,x_0$ and $\kappa_L(x_0)>0$, use precisely the local class of \cref{ass:class}. The coupled gauge and ball projection of \cref{sec:nonlinear,app:ball} give, for the \emph{same} $z=(h,F^s)$,
\begin{align}
 g&=A_0z+I_Ez,\qquad \norm{I_Ez}_Y\leq C_E\norm z_2,\label{cap:threebounds}\\
 \norm{A_0z}_Y&\geq\kappa_L\norm h_2,\qquad
 \norm{A_0z}_Y\geq\lambda_L\norm z_2.\notag
\end{align}
If $C_E<\lambda_L$ and the original gauge condition holds, then
\begin{equation}\label{cap:sharpbound}
 \norm g_Y\geq(1-C_E/\lambda_L)\norm{A_0z}_Y
 \geq\kappa_L(1-C_E/\lambda_L)\norm h_2.
\end{equation}
The source representative need not be physical; the original two candidate fields do. The absorption threshold remains $C_E<\lambda_L$, not $C_E<\kappa_L$.

Let $\mathcal U_L(r,\eps)$ be the \emph{full} physical relative coefficient $C^6$ neighborhood with perturbations supported in $\overline{B_r(x_0)}$ and fixed exterior. Define
\begin{equation}\label{cap:Copt}
 C_{\mathrm{opt},L}(r,\eps)=
 \sup_{\substack{\theta_1,\theta_2\in\mathcal U_L(r,\eps)\\\sigma_1\ne\sigma_2}}
 \frac{\norm{\sigma_2-\sigma_1}_2}{\norm{\cR(\theta_2)-\cR(\theta_1)}_Y},
\end{equation}
with infinity for a zero denominator. Fix once the norm of the finite real-SH coefficient vector and its derivatives through spatial order six. These classes contain small multiples of every smooth compactly supported degree-$L$ direction. Density and data norms both rescale by $r^{3/2}$, so \eqref{cap:Copt} is unchanged by the proof's coordinate rescaling.

Take $r_i\downarrow0$ with $\overline{B_{2r_i}(x_0)}\subset\operatorname{int}D$, then choose $\eps_i\downarrow0$ such that the gauge condition holds uniformly and
$c_i=\sup_{\theta_1,\theta_2\in\mathcal U_L(r_i,\eps_i)}C_E\to0$.
Such choices exist by the spatial $C^4$ weight estimate and coefficient $C^6$ continuity used in \cref{thm:nonlinear}. For all sufficiently large $i$,
\begin{equation}\label{cap:squeeze}
 \frac1{\kappa_L(x_0)}\leq C_{\mathrm{opt},L}(r_i,\eps_i)
 \leq\frac1{\kappa_L(x_0)(1-c_i/\lambda_L)}.
\end{equation}
The upper side is \eqref{cap:sharpbound}. For the lower side, apply \cref{thm:fragility} separately inside each positive-radius class: packet support may be smaller than $r_i$, and amplitude is chosen after frequency and nonlinear remainder constants. This preserves both physicality and the prescribed $C^6$ radius. Squeezing proves \eqref{m:sharpC}. Every limit here keeps $L,x_0$ fixed. No spatial bandlimit or finite primitive representation can be inserted without rechecking the packet inclusion.

\subsection{Proof of the synthesis and its consequence}\label{cap:synthesis}
Choose one $j\in\mathcal G$ from \cref{cap:generic}. It is isotropic, hence belongs to every finite SH reference class. Strict source/absorption margins on $D$ and finite density give both-sided physical perturbations and positive transmittance near each $x_0$. Apply the gauge, annular and weight estimates, and then \cref{cap:S}, separately at each fixed $L,x_0$. This proves part 1 of \cref{m:coexistence}, including its small-support and known-exterior qualifications. Part 2 follows from \cref{cap:R} at this \emph{same} reference, with one fixed $h$ or the single positive amplitude chosen after any prescribed finite list of neighborhood constraints. A modulus tending to zero would contradict the fixed density separation and vanishing data. These are compatible assertions because their neighborhoods and order of limits differ.

At fixed $L$ and \emph{matching spatial derivative order}, the field-based and finite-coefficient seminorms are equivalent, with constants depending on $L$ and the chosen angular order. This follows from norm equivalence in the finite-dimensional angular space, applied uniformly to each spatial derivative. It does not equate the $C^\infty$ Fr\'echet topology of genericity with a $C^6$ topology, or arbitrary mixed-derivative neighborhoods with coefficient $C^6$ neighborhoods. G selects a smooth reference; S subsequently uses its own $C^6$ neighborhood; R uses its explicitly specified finite seminorm constraints.

For completeness, at this reference and fixed $x_0$, define
\begin{equation}\label{cap:Ss}
 S_s^2=\sup_n\sum_k\avg{|\partial_\theta^s U_k(x_0,d(\theta))|^2}_n.
\end{equation}
Closed-ball smoothness and compactness of the circle-frame bundle give $S_s<\infty$ for every fixed integer $s\geq1$. The angular-tail estimate \eqref{eq:tailrate} and positivity yield
\begin{equation}\label{cap:superproof}
 0<\kappa_L(x_0)\leq\frac{t_+S_s}{(L+1)^s},\qquad
 C_L^\sharp(x_0)\geq\frac{(L+1)^s}{t_+S_s}.
\end{equation}
Here $S_s>0$, since $S_s=0$ would contradict positive gain. For any real $p>0$, choose a fixed integer $s>p$ and divide by $(L+1)^p$. This proves \cref{m:superpoly}. The constants depend on the reference, point and derivative order. The conclusion is stronger than a single polynomial lower bound, but gives no evaluated practical degree threshold, training-error law or common recovery-neighborhood size. The exact SH2 construction in \cref{thm:ambiguity} remains a separate finite-degree nonuniqueness statement.

\section{A physical separation between angular degrees one and two}\label{sec:separation}
Return to \eqref{m:reference}, and write $a=1/4$, $b=1/2$. At the center,
\begin{equation}\label{eq:AB}
 t_0=e^{-1},\qquad U_{0,k}=A+B d_k^2,\qquad
 A=a(1-e^{-1}),\quad B=b(2-5e^{-1}).
\end{equation}
Evenness makes the additional degree-one modes orthogonal to $U_0$ on each circle. Since
$\avg{d_k^2}_n=(1-n_k^2)/2$ and
$\operatorname{Var}_n(d_k^2)=(1-n_k^2)^2/8$, we have
\begin{equation}\label{eq:quadraticgain}
 \kappa_0(0)=\kappa_1(0)=\frac{e^{-1}(2-5e^{-1})}{2\sqrt6}
 \simeq0.0120601579,\qquad\kappa_L(0)=0\ (L\geq2).
\end{equation}
The necessary local inverse coefficient is at least $82.91765$ for $L=0,1$. This is distinct from the conservative joint coefficient used for a sufficient finite-radius estimate.

\begin{proposition}[Explicit degree-one prior class]\label{prop:radius}
For $0<r\leq0.01$, let each candidate have globally $C^6$ changes supported in $\overline{B_r}$ and source degree at most one. For $y=x/r$ define, with zero extension,
\begin{equation}\label{eq:prior}
 a_i(y)=\sigma_i(ry)-1,\quad
 \beta_i(y,d)=(j_i-j_*)(ry,d)-U_*(ry,d)a_i(y).
\end{equation}
Suppose $\cA(a_i)\leq0.1$ and $\cA(\beta_i)\leq0.1$ for both candidates, using the direction supremum inside the Fourier integral. Then the physical box holds and
\begin{equation}\label{eq:279}
 \norm{\sigma_2-\sigma_1}_{L^2(\Omega)}
 \leq263.057\norm{\cR_2-\cR_1}_Y.
\end{equation}
\end{proposition}
The complete scalar majorant and its outward interval evaluation are in \cref{app:radius}. These are field-norm priors, not an arbitrary ten-percent pointwise neighborhood. Certifying them for a learned field is an additional task. The previously chosen practical radius screen $r\geq0.05$ was not met; failure of that sufficient bound does not establish failure of injectivity.

\subsection{An exact nonlinear compensation construction}\label{sec:abel}
Fix $0<r\leq0.01$, put $s=r/2$, $R=r^2$, $S=s^2$, and let
\begin{equation}\label{eq:bump}
 \psi(z)=\begin{cases}\exp[-z/(1-z)],&0\leq z<1,\\0,&z\geq1,\end{cases}
 \qquad \rho_\eps(u)=1+\eps\psi(u/S).
\end{equation}
Set $\sigma_2(x)=\rho_\eps(|x|^2)$. The density perturbation is smooth and supported strictly inside the known-exterior ball. For an oriented line $p+td$, let $v=|p|^2$, $\ell(v)=\sqrt{1-v}$ and
\begin{align}
 H_\rho(v)&=\int_0^{\ell(v)}\rho(v+t^2)\dd t,&
 Z_\rho(v,t)&=\int_0^t\rho(v+w^2)\dd w,\nonumber\\
 T_\rho(v,t)&=e^{-H_\rho(v)-Z_\rho(v,t)},&
 I_m^\rho(v)&=\int_{-\ell(v)}^{\ell(v)}T_\rho(v,t)t^m\dd t.\label{eq:moments}
\end{align}
Use $m=0,1,2$ and $\Delta I_m=I_m^1-I_m^\rho$. Rendering the unchanged source at radial density $\rho$ gives
\begin{equation}\label{eq:raycoeff}
 \cR_k(\rho,j_*)=aI_0^\rho+bp_k^2I_0^\rho
 +2bp_kd_kI_1^\rho+bd_k^2I_2^\rho.
\end{equation}
The data retain channel-specific angular and spatial structure; they are not a scalar radial image.

Seek a single three-dimensional correction of the form
\begin{equation}\label{eq:fourprofile}
 \Delta j_k(x,d)=q_0(u)+q_B(u)x_k^2+q_D(u)x_kd_k+q_C(u)d_k^2,
 \qquad u=|x|^2.
\end{equation}
All four profiles will vanish for $u\geq S$. This source has degree at most two. Define
\begin{equation}\label{eq:Km}
 K_mq(v)=\int_{-\sqrt{R-v}}^{\sqrt{R-v}}T_\rho(v,t)t^mq(v+t^2)\dd t.
\end{equation}
The four coefficients contributed by \eqref{eq:fourprofile}, in the order $1,p_k^2,p_kd_k,d_k^2$, are
\begin{equation}\label{eq:triangular}
 K_0q_0,\quad K_0q_B,\quad 2K_1q_B+K_0q_D,\quad
 K_2q_B+K_1q_D+K_0q_C.
\end{equation}
Matching these with the difference in \eqref{eq:raycoeff} is a triangular system. Its weights use the changed density exactly.

\subsection{The diagonal inverse and explicit profiles}
For $0\leq v\leq u\leq R$ put $z(v,u)=\int_0^{\sqrt{u-v}}\rho(v+t^2)\dd t$. Pairing positive and negative $t$ yields
\begin{align}
 K_mq&=e^{-H_\rho}A_mq,&
 A_mq(v)&=\int_v^R\frac{a_m(v,u)q(u)}{\sqrt{u-v}}\dd u,\label{eq:Am}\\
 a_0&=\cosh z,&a_1&=-\sqrt{u-v}\sinh z,&a_2&=(u-v)\cosh z.\nonumber
\end{align}
The three kernels are smooth on the closed triangle; $a_0(v,v)=1$, while $a_1,a_2$ have zero diagonal. In particular the odd moment is not zero: entry attenuation is not even in $t$.

The right Abel inverse is
\begin{equation}\label{eq:AbelB}
 \mathcal Bg(v)=-\frac1\pi\frac\dd{\dd v}\int_v^R\frac{g(w)}{\sqrt{w-v}}\dd w.
\end{equation}
For data vanishing near $R$, it also equals $-\pi^{-1}\int_v^R g'(w)/\sqrt{w-v}\dd w$. It is injective on this terminal support class. Indeed, if $Cg(v)=\int_v^R g(w)/\sqrt{w-v}\dd w$ and $\mathcal Bg=0$, then $Cg$ is constant and its terminal value is zero. Applying $C$ once more and using the beta integral yields $\pi\int_v^R g(w)\dd w=0$, hence $g=0$. This justifies recovering the original Abel equation after applying $\mathcal B$. Define
\begin{equation}\label{eq:Vm}
 F_m(v,u)=\int_0^1\frac{a_m(v+(u-v)\theta,u)}{\sqrt{\theta(1-\theta)}}\dd\theta,
 \qquad V_m(v,u)=-\pi^{-1}\partial_vF_m(v,u).
\end{equation}
Use $V_m$ also for integration of its kernel from $v$ to $R$. Fubini and the beta integral give
$\mathcal BA_0=\Id+V_0$, $\mathcal BA_1=V_1$, $\mathcal BA_2=V_2$.
If $M_0=\max|V_0(v,u)|$, then
$\norm{V_0^n}_{\infty\to\infty}\leq M_0^nR^n/n!$. Thus
$(\Id+V_0)^{-1}=\sum_{n\geq0}(-V_0)^n$ converges in operator norm. This is the classical Volterra resolvent mechanism, not a small-norm contraction requirement \cite{quinto1983,dehoop2017}.

Define
\begin{align}
 P_m&=(\Id+V_0)^{-1}\mathcal B[e^{H_\rho}\Delta I_m],
 &J_m&=(\Id+V_0)^{-1}V_m\quad(m=1,2),\label{eq:PJ}\\
 q_0&=aP_0,&q_B&=bP_0,\nonumber\\
 q_D&=2bP_1-2J_1q_B,&q_C&=bP_2-J_2q_B-J_1q_D.\label{eq:profiles}
\end{align}
Here $P_m$ is defined for $m=0,1,2$. Equations \eqref{eq:PJ}--\eqref{eq:profiles} specify the alternative field by ordinary integrals and a convergent series, without any fitting to a finite set of rays.

\begin{theorem}[Exact physical degree-two ambiguity]\label{thm:ambiguity}
For every $0<r\leq0.01$ and $0<\eps\leq1/6000$, the reference \eqref{m:reference} and the field $(\sigma_2,j_2)=(\rho_\eps,j_*+\Delta j)$ defined above have different densities, the same exterior of $B_r$, and exactly equal RGB data on every external oriented line. The source perturbation is smooth, supported in $\overline{B_s}$, and has genuinely nonzero degree-two content. Uniformly over position and direction, $0.2\leq j_{2,k}\leq0.8<\sigma_2$. For fixed $r$, the alternatives approach the reference in every finite $C^k$ norm as $\eps\downarrow0$.
\end{theorem}
\begin{proof}
By construction $K_0P_m=\Delta I_m$ and $J_m=K_0^{-1}K_m$. Substituting \eqref{eq:profiles} into \eqref{eq:triangular} gives exactly
$a\Delta I_0,b\Delta I_0,2b\Delta I_1,b\Delta I_2$. Equation \eqref{eq:raycoeff} proves all-ray equality, including reversed orientation. Rays with $v\geq S$ miss every changed coefficient.

The targets $\Delta I_m$ are smooth and flat at $S$, and vanish above $S$. The Abel inverse and the smooth Volterra resolvent preserve this upper support and regularity. Smoothness in $u$ at zero implies smoothness of the composed spatial profiles at the origin. Parameter dependence is smooth and all targets vanish at $\eps=0$. The explicit physical bounds and proof of nonzero degree-two content are given in \cref{app:physical}; neither is inferred from the positive theorem or from finite sampling.
\end{proof}

\begin{corollary}[Separation at one reference]\label{cor:separation}
For any fixed $0<r\leq0.01$, the reference \eqref{m:reference} has a degree-one local class with density uniqueness, including \cref{prop:radius}. Enlarging only the angular class to degree two introduces arbitrarily close physical fields of unequal density and identical full data. Every degree-$L$ class with $L\geq2$ contains these alternatives.
\end{corollary}
\begin{proof}
Use \cref{prop:radius,thm:ambiguity}. For fixed $r$, compact-support $C^4$ convergence controls the Fourier norms in \eqref{eq:prior} by \eqref{eq:weightbound}. Hence sufficiently small exact alternatives satisfy the same numerical norm caps. Their degree-two component is the distinguishing condition. This argument does not assert that the particular $\eps=10^{-4}$ example has numerically certified Fourier budgets.
\end{proof}

This is not a universal transition at $L=2$. Isotropic radial ambiguity already rules out universal degree-one uniqueness, while \eqref{eq:everyL} provides other references with positive local gain at any fixed finite degree.

\subsection{Why moving away from the center does not restore the symbol}
The conjecture that $\kappa_2(x)$ becomes positive off center is false for this reference. At $x\ne0$, choose $n=x/|x|$. For every $d\in C_n$, $x\cdot d=0$ and the exit length is $\sqrt{1-|x|^2}$, constant on that circle. Thus
\begin{equation}\label{eq:tangentcircle}
 U_k(x,d)=(a+bx_k^2)J_0+2bx_kd_kJ_1+bd_k^2J_2,
 \quad J_m=\int_0^{\sqrt{1-|x|^2}}e^{-t}t^m\dd t.
\end{equation}
It is a degree-two polynomial on the chosen circle, so $\kappa_L(x)=0$ for all interior $x$ and $L\geq2$. The center follows from \eqref{eq:AB}. The same reasoning extends to radial densities and the four-profile family: along a tangent direction, $|x+td|^2=|x|^2+t^2$, so all radial integral coefficients are direction independent on that circle.

The whole-sphere function may have higher harmonics away from the center. That does not imply a positive minimum over circles. Conversely, $\kappa_L(x)=0$ alone characterizes neither the complete nonlinear kernel nor which prescribed supports admit exact ambiguities.

\section{Proof of the Euclidean tensor ball estimate}\label{app:ball}
We give the boundary argument because replacing an interior-solenoidal field by a globally solenoidal zero extension would invalidate the nonlinear estimate. The construction specializes classical symmetric-tensor elliptic theory \cite{sharafutdinov1994,pz2016}.

\begin{lemma}[Two classical inputs on the fixed annulus]\label{lem:annularinputs}
Let $A=B_2\setminus\overline{B_1}$ and let $D$ act on rank-$(m-1)$ symmetric tensors, $m\geq1$.
\begin{enumerate}
\item For a finite $K_m$, Korn coercivity holds modulo the full Killing kernel:
\begin{equation}\label{eq:korn}
 \inf_{k\in\ker D}\norm{v-k}_{H^1(A)}\leq K_m\norm{Dv}_{L^2(A)}.
\end{equation}
\item Every weak Killing tensor on $A$ extends as a Killing polynomial across $B_1$ to $B_2$.
\end{enumerate}
\end{lemma}
The tensor Neumann system is regular elliptic by \cite[Lemma 3.3 and the preceding argument, pp.~1911--1913]{pz2016}. Its standard second Korn estimate controls $H^1$ by $\norm{Dv}_2+\norm v_2$; compactness removes the last term on the $L^2$-orthogonal complement of $\ker D$. Otherwise a normalized violating sequence converges in $L^2$ to a Killing tensor orthogonal to the kernel, hence to zero, contradicting that estimate. Equation \eqref{eq:korn} is this consequence, not a verbatim assertion of the cited lemma.

For the second input, weak solutions are smooth by interior ellipticity. The flat rank-$(m-1)$ Killing equation gives $\nabla^m k=0$, hence polynomial components of degree at most $m-1$ \cite[Corollary 7.3, p.~81]{heil2017}. Connectedness identifies one polynomial on the annulus; its Killing identity extends across the inner ball. These two inputs use the unchanged Euclidean metric. The general extension in \cite[Proposition 3.2]{pz2016}, which permits an exterior metric change, is not substituted for this argument.

\paragraph{Extension functional and projection bound.}
Let $E:H^1(A)\to H^1(B_2)$ be a bounded extension agreeing with $v$ on $A$, and choose $e_m$ with $\norm{D(Ev)}_{L^2(B_1)}\leq e_m\norm v_{H^1(A)}$. For interior-solenoidal $F$, define
\begin{equation}\label{eq:ellF}
 \ell_F(v)=-\int_{B_1}\avg{F,D(Ev)}\dd x.
\end{equation}
It is independent of the interior extension, since the difference has zero trace on $\partial B_1$. If $k$ is Killing on $A$, use its polynomial extension to see $\ell_F(k)=0$. Equation \eqref{eq:korn} gives
$|\ell_F(v)|\leq e_mK_m\norm F_2\norm{Dv}_{L^2(A)}$.
Lax--Milgram on $H^1(A)/\ker D$ supplies $u$ with
\begin{equation}\label{eq:Neumann}
 \int_A\avg{Du,Dv}\dd x=\ell_F(v),\qquad
 \norm{Du}_{L^2(A)}\leq e_mK_m\norm F_2.
\end{equation}
Set $V=F$ on $B_1$, $V=Du$ on $A$, and $V=0$ outside $B_2$. For every global compactly supported test tensor $z$, solenoidality of $F$ and \eqref{eq:Neumann} give
$\int_{\R^3}\avg{V,Dz}=0$. Thus $V$ is globally solenoidal, including both interfaces, and $\norm V_2\leq C_m\norm F_2$ with $C_m=\sqrt{1+(e_mK_m)^2}$. Since $\Pi_mV=V$,
\begin{equation}\label{eq:dualball}
 \norm F_2^2=\avg{\widetilde F,V}
 =\avg{\Pi_m\widetilde F,V}\leq\norm{\Pi_m\widetilde F}_2 C_m\norm F_2.
\end{equation}
This proves \eqref{eq:ballproj}. Rank zero is immediate. Constants for general $m$ are fixed-domain operator norms, not explicit decimal certificates.

For $m=1$ one can take $C_1=\sqrt2$. If $(\Id-\Pi_1)\widetilde v=\nabla p$, then $p$ is harmonic inside and outside the ball with common trace. A spherical harmonic of degree $\ell$ has interior Dirichlet energy $\ell|p_{\ell m}|^2$ and exterior energy $(\ell+1)|p_{\ell m}|^2$. Hence $E_{\rm in}\leq E_{\rm total}/2$. Orthogonality gives
$E_{\rm total}=\avg{v,\nabla p}_{B_1}\leq\norm v_2\sqrt{E_{\rm in}}$, so $E_{\rm total}\leq\norm v_2^2/2$ and $\norm{\Pi_1\widetilde v}_2^2\geq\norm v_2^2/2$. Weak approximation extends the calculation to $L^2$ fields.

\section{Physical bounds and genuine degree two}\label{app:physical}
The estimates here justify the physical part of \cref{thm:ambiguity} without relying on a finite mesh. Throughout $0<r\leq.01$, $s=r/2$, $R=r^2$, and $0<\eps\leq1/6000$.

Since $|\psi|\leq1$ and $|\psi'|\leq4/e<2$, we have $1\leq\rho\leq2$, $|\rho'|\leq8\eps/R$. Write
$z=\sqrt{u-v}\,k(v,u)$ with $k=\int_0^1\rho(v+(u-v)\zeta^2)\dd\zeta$.
Then $k\leq2$, $|\partial_vk|\leq16\eps/(3R)$, $\sinh z\leq3\sqrt{u-v}$, and $\cosh z<2$. Differentiating the smooth kernels in \eqref{eq:Am} gives
\begin{align}
 |\partial_va_0|&\leq3+16\eps,&
 |\partial_va_1|&\leq\tfrac72+\tfrac{32}3\eps,&
 |\partial_va_2|&\leq2+3R+16\eps R.\label{eq:kernelbounds}
\end{align}
All are below four. In \eqref{eq:Vm}, differentiation contributes an additional $(1-\theta)$; the beta integral is $\pi/2$. Thus $|V_m(v,u)|\leq2$ and
\begin{equation}\label{eq:resolventbounds}
 \norm{(\Id+V_0)^{-1}}_{\infty\to\infty}\leq e^{2R}<2,
 \qquad\norm{J_m}_{\infty\to\infty}\leq4R.
\end{equation}

To control the differentiated target in the Abel inverse, let
$B_\psi(v)=\int_0^\infty\psi((v+t^2)/S)\dd t$. Then
$|B_\psi|\leq s$, $|B_\psi'|\leq2/s$, and $H_\rho=\ell+\eps B_\psi$.
The additional past optical depth is
\begin{equation}\label{eq:deltatau}
 \Delta\tau(v,t)=\eps\left(B_\psi(v)+\int_0^t\psi((v+w^2)/S)\dd w\right).
\end{equation}
On each physical chord, $0\leq\Delta\tau\leq2\eps s$ and $|\partial_v\Delta\tau|\leq4\eps/s$. Including the moving endpoints when differentiating the moments gives
\begin{equation}\label{eq:targetbounds}
 |\Delta I_m|\leq4\eps s,\qquad
 |\Delta I_m'|\leq11\eps/s,\qquad
 \norm{\partial_v(e^{H_\rho}\Delta I_m)}_\infty\leq96\eps/s.
\end{equation}
For the last inequality, $e^{H_\rho}<8$ and the bracket
$11+4s^2+8\eps s$ is below twelve. The entry endpoint contributes zero to the first difference; the exit contribution is bounded by $2\eps s$. It follows from \eqref{eq:AbelB} that
$\norm{\mathcal B[e^{H_\rho}\Delta I_m]}_\infty\leq192\eps$, and hence $\norm{P_m}_\infty\leq384\eps$.
Equation \eqref{eq:profiles} gives
\begin{equation}\label{eq:qbounds}
 \norm{q_0}_\infty\leq96\eps,\quad
 \norm{q_B}_\infty\leq192\eps,\quad
 \norm{q_D}_\infty<385\eps,\quad
 \norm{q_C}_\infty<193\eps.
\end{equation}
Thus for every position and direction in the changed region,
\begin{equation}\label{eq:300}
 |\Delta j_k|\leq(96+192r^2+385r+193)\eps<300\eps.
\end{equation}
Globally $j_{2,k}\in[1/4-300\eps,3/4+300\eps]\subset[.2,.8]$, while $\sigma_2\geq1$. For the concrete pair $r=.01$, $\eps=10^{-4}$, the source enclosure is $[.22,.78]$ and the central density difference is exactly $10^{-4}$.

For support, smoothness, and nonzero degree two, use a second expression for the target. Set
$G_m(v,t)=\int_t^{\ell(v)}e^{-w}w^m\dd w$,
$B_t(v)=\int_0^t\psi((v+w^2)/S)\dd w$, and $Q_m(v)=G_m(v,0)>0$.
Exact transport pseudolinearization gives
\begin{equation}\label{eq:targetAbel}
 e^{H_\rho(v)}\Delta I_m(v)=\eps\int_v^S
 \frac{b_m(v,u)\psi(u/S)}{\sqrt{u-v}}\dd u,
\end{equation}
where, with $t=\sqrt{u-v}$,
$b_m(v,u)=[e^{-\eps B_t(v)}G_m(v,t)+e^{\eps B_t(v)}G_m(v,-t)]/2$.
Pairing the two terms makes the kernel smooth with diagonal $Q_m(v)$. Abel composition and the Volterra resolvent imply
\begin{equation}\label{eq:Pmexpansion}
 P_m(v)=\eps Q_m(v)\psi(v/S)+\eps\int_v^S L_m(v,u)\psi(u/S)\dd u
\end{equation}
for bounded smooth kernels $L_m$. This proves flat support and smooth dependence on $\eps$. Since $J_1,J_2$ are strictly integral Volterra operators and $\psi$ is positive and nonincreasing below $S$, division by $\eps\psi(v/S)$ makes each integral remainder $O(S-v)$. Therefore
\begin{equation}\label{eq:nonzeroqC}
 \lim_{v\uparrow S}\frac{q_C(v)}{\eps\psi(v/S)}
 =b\int_0^{\sqrt{1-S}}e^{-t}t^2\dd t>0.
\end{equation}
No other term of \eqref{eq:fourprofile} has degree two in $d$, so this component cannot cancel. Dividing by the positive scalar density preserves angular degree.

\section{An explicit degree-one radius certificate}\label{app:radius}
This appendix supplies the numerical constants in \cref{prop:radius}. It concerns a prior class of fields; it does not certify all input norms of an arbitrary learned model.

For $q=q_0+q_1\cdot d$, circle orthogonality at the quadratic reference gives
\begin{equation}\label{eq:l1circle}
 \avg{|W_0(h,q)|^2}_n=t_0^2\left(
 |q_0-v(n)h|^2+\tfrac12\sum_k|P_nq_{1,k}|^2
 +\tfrac{B^2}8\sum_k(1-n_k^2)^2|h|^2\right),
\end{equation}
where $v_k(n)=A+(B/2)(1-n_k^2)$. Put
\begin{equation}\label{eq:l1joint}
 a_f=B^2/6,\quad V^2=3(A+B/2)^2,\quad
 \mu_f=\frac{a_f}{1+V^2+a_f},\quad
 k_J=t_0\sqrt{\min(\mu_f,1/4)}.
\end{equation}
The determinant/trace bound for the $(h,q_0)$ block and the rank-one ball estimate in \cref{app:ball} give a full joint lower bound $k_J$. Numerically $k_J\simeq.01140136991$, smaller than the sharp density constant \eqref{eq:quadraticgain}.

For this special degree-one case no smallness is needed to choose the source gauge. Solve $\Delta\varphi_k=\Div q_{1,k}$ with zero boundary trace and replace
$q_{1,k}$ by $q_{1,k}-\nabla\varphi_k$ and $q_{0,k}$ by $q_{0,k}+r\sigma_1(ry)\varphi_k$. This preserves the exact data and makes the vector component interior solenoidal.

We bound the actual ray operator directly using tensor moments instead of a spatial cutoff. On $B_1$, multiplication of a direction-only matrix weight by $(y\cdot d)^m|y|^{2n}$ has no larger $L^2\to Y$ norm than the original angular supremum bound. Indeed, place $y_{i_1}\cdots y_{i_m}|y|^{2n}z$ in an ordered tensor, contract with $d^{\otimes m}$, and use Fourier slice, coarea, and Plancherel. The same estimate remains valid after multiplication by a field with finite norm \eqref{eq:algebra}, at the cost of that norm. This follows by Fourier inversion and Minkowski, including for modulated coefficient vectors.

Set $\delta_r=1-\sqrt{1-2r^2}$ and define the following positive scalar majorants:
\begin{align}
 p_0&=\sqrt{3a^2+4ab+4b^2},\qquad q_0^\flat=\sqrt{3a^2+10ab+25b^2},\nonumber\\
 e_t&=e^{r+\delta_r}-1-r,\qquad e_h=e^{2\delta_r}-1,\nonumber\\
 q_{\rm rem}&=br^2+2br(r+\delta_r)+4b\delta_r+b(r+\delta_r)^2,\nonumber\\
 R_S&=t_0\{3br^2+br^3+e_t(p_0+2br+br^2)\}\nonumber\\
 &\quad+t_0^2\{q_{\rm rem}+e_h(q_0^\flat+4br+q_{\rm rem})\},\nonumber\\
 m_S&=r[2bt_0(1-2t_0)+\sqrt3at_0]+R_S,\nonumber\\
 m_Q&=\sqrt2t_0(e^{r+\delta_r}-1),\nonumber\\
 M_{\rm ref}(r)&=\sqrt{m_S^2+m_Q^2},\qquad N_T(r)=t_0e^{r+\delta_r}.\label{eq:majorantref}
\end{align}
For clarity, these are operator majorants, not values of the Fourier weight norm $C_E$. Their derivation retains the first-order cancellation in $T_*U_*$. With $x=ry$, $s_x=x\cdot d$, and $\delta=1-\sqrt{1-|x|^2+s_x^2}$, exact integration yields
\begin{align}
 T_*&=t_0e^{\delta-s_x},\qquad T_*U_*=t_0e^{\delta-s_x}P-t_0^2e^{2\delta}Q,\nonumber\\
 P&=a\one+2bd^2+2bx d+bx^2,\nonumber\\
 Q&=a\one+5bd^2+4b(xd-s_xd^2)+Q_{\rm rem},\nonumber\\
 Q_{\rm rem}&=bx^2-2bxd(s_x+\delta)-4b\delta d^2+b(s_x+\delta)^2d^2.\label{eq:exactS}
\end{align}
Products here are componentwise. The linear term after subtracting the center weight is
$2bt_0(1-2t_0)(xd-s_xd^2)-at_0s_x\one$.
The nonnegative series of $1-\sqrt{1-w}$, and the tensor-moment bound, majorize $\delta$ by $\delta_r$. Expanding $e^{\delta-s_x}=1-s_x+\text{remainder}$ and $e^{2\delta}=1+\text{remainder}$ gives \eqref{eq:majorantref}, including
$\norm{A_{W_*-W_0}z}_Y\leq M_{\rm ref}(r)\norm z_2$.

For the prior budgets $\rho_s=\rho_b=.1$ in \eqref{eq:prior}, define
\begin{align}
 E_s&=e^{2r\rho_s},\qquad H(r)=\frac{\rho_b[1-e^{-2r(1-\rho_s)}]}{1-\rho_s},\nonumber\\
 N_0&=t_0\sqrt{2+3A^2+2AB+B^2},\nonumber\\
 M_{\rm total}(r)&=M_{\rm ref}(r)+(E_s-1)(N_0+M_{\rm ref}(r))
 +N_T(r)E_sH(r).\label{eq:majoranttotal}
\end{align}
To verify the nonlinear terms, on $B_1$ write
\[
T_1(ry,d)=T_*(ry,d)L_1(y,d),\qquad
L_1=\exp\left[-r\int_0^2a_1(y-td)\dd t\right].
\]
In the unital Fourier algebra, $\cA(L_1)\leq E_s$ and $\cA(L_1-1)\leq E_s-1$.
The exact forward-radiance difference has the extension
\begin{equation}\label{eq:Uextension}
 G_2(y,d)=r\int_0^2e^{-rt}
 \exp\left[-r\int_0^ta_2(y+ud)\dd u\right]\beta_2(y+td,d)\dd t,
\end{equation}
which agrees with $U_2(ry,d)-U_*(ry,d)$ on $B_1$. The norm is at most $H(r)$ by translation invariance and submultiplicativity. Applying the tensor-moment multiplier estimate to
\[\begin{aligned}
W-W_0={}&(W_*-W_0)+(L_1-1)W_*\\
&{}-T_*L_1G_2[\text{density column}].
\end{aligned}\]
proves the operator error bound $M_{\rm total}$. It applies to all supported coefficient vectors, including the gauge-modified representative, without imposing new priors on that auxiliary representative.

All relevant majorants are nondecreasing on $0\leq r\leq.01$: the series defining them have nonnegative coefficients, and $H(r)$ is increasing. Exact outward interval evaluation gives
\begin{equation}\label{eq:intervals}
\begin{gathered}
 k_J>.011401369910273487,\qquad
 M_{\rm total}(.01)<.00780756,\\
 \frac{1}{\kappa_1(0)[1-M_{\rm total}(.01)/k_J]}<263.057.
\end{gathered}
\end{equation}
Retaining the sharp density estimate in the absorption step gives
\[
 \|g\|_Y\geq(1-M_{\rm total}/k_J)\|A_0z\|_Y
 \geq\kappa_1(0)(1-M_{\rm total}/k_J)\|h\|_2.
\]
The enclosed coefficient is $263.0562873469795\ldots$. Monotonicity of the majorants proves \eqref{eq:279} for every smaller radius. Fourier inversion of the prior caps also gives $.9\leq\sigma_i\leq1.1$ and $.075\leq j_{i,k}\leq.425+.5r^2$ inside the unknown ball, so the physical box is strict. Small multiples of any compact smooth density direction belong to the class; it is not a disguised fixed-density class.

The evaluator uses integer intervals $[l,u]/10^{40}$ with outward rounding. Square roots use integer square roots. For nonnegative arguments, exponentials use the degree-$64$ Taylor polynomial and tail bound $q^{65}/65!\,(1-q/66)^{-1}$ for upper endpoint $q<66$; negative exponentials are reciprocal intervals. Strict inequalities compare integer endpoints. This certifies the scalar arithmetic conditional on the analytic estimates and input norm caps. It is not proof-assistant verification of the theorem or interval verification of a general renderer.

\section{A certified finite-amplitude physical ambiguity}\label{app:finitecertificate}
This section proves \cref{m:twofold}. It supplies a physical exact pair at $s=1/2$, $S=1/4$, $\eps=1$; the elevenfold voxel comparison in \cref{app:integercertificate} concerns a separate finite representation. The method uses classical Abel inversion, residual bounds, and layerwise propagation \cite{dehoop2017}. The new result is the verified physical finite-amplitude instance, rather than a new Volterra inversion principle.

\subsection{Exact equations and an analytic proxy}
Write $\psi_S(u)=\psi(u/S)$. The construction in \cref{sec:abel} is algebraically exact for every finite positive $\eps$ and $S<1$: its smooth Volterra kernels have factorially convergent resolvents. Smallness in \cref{app:physical} was used only for its physical-box estimates. Put $P=q_0/a=q_B/b$, $D=q_D$, $C=q_C$, with $a=1/4$, $b=1/2$.

For $0\leq w\leq u\leq S$, set $\delta=u-w$, $t=\sqrt\delta$, and
\begin{align}
 e(w,u)&=\eps\int_0^1\psi_S\big(w+(u-w)y^2\big)\dd y,&
 e_w(w,u)&=\eps\int_0^1(1-y^2)\psi_S'\big(w+(u-w)y^2\big)\dd y,\label{eq:noptical}\\
 k&=1+e,& z&=tk,\qquad z_e=te.\nonumber
\end{align}
For $a_0=\cosh z$, $a_1=-t\sinh z$, $a_2=\delta\cosh z$, their first-coordinate derivatives are
\begin{align}
 A_0&=\operatorname{sinhc}(z)(-k^2/2+\delta k e_w),\nonumber\\
 A_1&=k\operatorname{sinhc}(z)/2+\cosh(z)(k/2-\delta e_w),\label{eq:nAk}\\
 A_2&=-\cosh z+\delta A_0,\qquad
 c_e=\cosh z_e,\qquad c_{e,w}=\operatorname{sinhc}(z_e)(-e^2/2+\delta e e_w).\nonumber
\end{align}
Here $\operatorname{sinhc}(0)=1$ gives the removable diagonal values. Define a positive averaging operator of mass $1/2$,
\begin{equation}
 \mathcal M[f](v,u)=\frac1\pi\int_0^1
 \frac{(1-\theta)f(v+(u-v)\theta,u)}{\sqrt{\theta(1-\theta)}}\dd\theta.
\end{equation}
The kernels $V_m$ in \eqref{eq:Vm} are $-\mathcal M[A_m]$.

Let $\ell(w)=\sqrt{1-w}$ and
\begin{equation}
 L_0=e^{-\ell},\quad L_1=(1+\ell)e^{-\ell},\quad
 L_2=(2+2\ell+\ell^2)e^{-\ell},\quad J_m=m!-L_m.
\end{equation}
Thus $L_0'=e^{-\ell}/(2\ell)$, $L_1'=e^{-\ell}/2$, $L_2'=\ell e^{-\ell}/2$.
The paired target kernels in \eqref{eq:targetAbel} are
$b_0=a_0-L_0c_e$, $b_1=a_0+a_1-L_1c_e$, and
$b_2=2a_0+2a_1+a_2-L_2c_e$. Setting $W_m=-\mathcal M[\partial_w b_m]$, Abel composition gives
\begin{align}
 (\Id+V_0)P&=g_0,\nonumber\\
 (\Id+V_0)D+V_1P&=g_1,\label{eq:nblock}\\
 (\Id+V_0)C+V_1D+\tfrac12V_2P&=\tfrac12g_2,\nonumber\\
 g_m(v)&=\eps\left[J_m(v)\psi_S(v)+\int_v^SW_m(v,u)\psi_S(u)\dd u\right].\nonumber
\end{align}
These equations specify the exact source. Flat support at $S$ and spatial smoothness follow as in \cref{app:physical}. The terminal-support Abel inverse is injective, so solving \eqref{eq:nblock} matches all four original line-polynomial coefficients. This proves all-ray equality, including reversed lines, before testing physicality.

Use the explicit proxy $\widehat P=\eps\psi_SJ_0$, $\widehat D=\eps\psi_SJ_1$, $\widehat C=\eps\psi_SJ_2/2$. Its source correction is
\begin{equation}\label{eq:nproxypositive}
 \widehat j_k-j_{1,k}=\eps\psi_S(|x|^2)
 \int_0^{\ell(|x|^2)}e^{-t}\left[a+b(x_k+td_k)^2\right]\dd t\geq0.
\end{equation}
The multiplier of $\eps\psi_S$ is bounded above by
$(a+bS)J_0(0)+sJ_1(0)+bJ_2(0)<1$. Consequently $\widehat j_k\geq1/4$ and, inside $B_s$,
$\sigma_2-\widehat j_k\geq3/4-S/2=5/8$. Outside $B_s$ both reference margins are at least $1/4$.

\subsection{Residual enclosures on the complete integration domain}
Let $r$ be the left side of \eqref{eq:nblock} at the proxy minus its right side. Algebraic cancellation gives
\begin{equation}\label{eq:nresidual}
 r_m(v)=\eps\int_v^S\psi_S(u)\mathcal M[F_m](v,u)\dd u,
\end{equation}
where the three components, indexed by $P,D,C$, are
\begin{align}
 F_P&=A_0L_0(u)-L_0'(w)c_e-L_0(w)c_{e,w},\nonumber\\
 F_D&=A_0L_1(u)+A_1L_0(u)-L_1'(w)c_e-L_1(w)c_{e,w},\label{eq:nreskernels}\\
 F_C&=\tfrac12\big[A_0L_2(u)+2A_1L_1(u)+A_2L_0(u)-L_2'(w)c_e-L_2(w)c_{e,w}\big].\nonumber
\end{align}
In particular, the cancellation is performed before taking absolute values.

Partition $[0,S]$ into $N$ closed cells $I_i=[ih,(i+1)h]$, $h=S/N$. For each $i\leq j$, enclose \eqref{eq:nAk} and \eqref{eq:nreskernels} on the entire rectangle $w\in I_i$, $u\in I_j$, restricted to $w\leq u$. To enclose \eqref{eq:noptical}, split $y\in[0,1]$ into $Q$ cells $[l,r]$. Its argument belongs to
\begin{equation}
 [(1-l^2)w_-+l^2u_-,\ (1-r^2)w_++r^2u_+].
\end{equation}
The bump decreases. With $\zeta=S/(S-z)$, its derivative magnitude is
$\zeta^2e^{1-\zeta}/S$, which has its sole maximum at $\zeta=2$. Endpoint enclosures and the possible interior maximum $4/(eS)$ therefore bound both integrands on every cell. For $e_w$, multiply by the exact integral of $1-y^2$. Positivity and mass $1/2$ of $\mathcal M$ then give rational bounds
\begin{equation}
 |V_m(v,u)|\leq M^m_{ij},\qquad
 |\mathcal M[F_m](v,u)|\leq R^m_{ij}
 \quad(v\in I_i,\ u\in I_j,\ v\leq u).
\end{equation}
The maxima include every $w$-cell between $i$ and $j$. Thus
\begin{equation}\label{eq:nrho}
 |r_m(v)|\leq\rho_{m,i}:=\eps\sum_{j\geq i}h\psi_S(jh)R^m_{ij},\qquad v\in I_i.
\end{equation}
Including the whole diagonal cell is conservative. These envelopes cover the points between partition endpoints as well as every optical and angular integration variable.

\subsection{Segmented propagation and physicality}
Let $E_{P,i},E_{D,i},E_{C,i}$ bound the profile errors on $I_i$. Solve from $i=N-1$ down to zero and put $a_i=\exp(hM^0_{ii})$. The scalar local Volterra resolvent has norm at most $a_i$ by its factorial series. Induction gives the following continuous bounds:
\begin{align}
 E_{P,i}&=a_i\left[\rho_{P,i}+\sum_{j>i}hM^0_{ij}E_{P,j}\right],\nonumber\\
 E_{D,i}&=a_i\left[\rho_{D,i}+hM^1_{ii}E_{P,i}
            +\sum_{j>i}h(M^1_{ij}E_{P,j}+M^0_{ij}E_{D,j})\right],\label{eq:npropagate}\\
 E_{C,i}&=a_i\left[\rho_{C,i}+h(\tfrac12M^2_{ii}E_{P,i}+M^1_{ii}E_{D,i})
            +\sum_{j>i}h(\tfrac12M^2_{ij}E_{P,j}+M^1_{ij}E_{D,j}+M^0_{ij}E_{C,j})\right].\nonumber
\end{align}
To verify this recurrence, bound the already solved tail on every later cell; on the current cell absorb the remaining $V_0$ integral using its local resolvent. The second and third components additionally use the already bounded earlier components. The off-diagonal sums account for accumulated segment error.

\begin{table}[t]
\caption{Outward bounds for the exact profiles relative to the analytic proxy at $s=1/2$, $\eps=1$, $N=Q=32$. Upper values are rounded up and lower values down.}\label{tab:continuumcertificate}
\centering
\begin{tabular}{@{}lrrr@{}}\toprule
Component & $P$ & $D$ & $C$\\\midrule
Residual upper bound & $.046538$ & $.038327$ & $.021660$\\
Exact profile-error upper bound & $.049959$ & $.044609$ & $.027105$\\\midrule
\multicolumn{3}{@{}l}{Exact-source lower bound $j_{2,k}$} & $.2062$\\
\multicolumn{3}{@{}l}{Exact absorption margin $\sigma_2-j_{2,k}$} & $.25$\\
\multicolumn{3}{@{}l}{Nonzero degree-two coefficient $q_C(0)$} & $.0531$\\\bottomrule
\end{tabular}
\end{table}

The rational evaluation yields \cref{tab:continuumcertificate}; in particular the maximum profile error is strictly below $.05$. On a radial cell, set $r_i=\sqrt{(i+1)h}$. The source error is at most
\begin{equation}
 (a+br_i^2)E_{P,i}+r_iE_{D,i}+E_{C,i}<.043733.
\end{equation}
Combining with \eqref{eq:nproxypositive} proves $j_{2,k}>.2062$ and the global absorption margin $1/4$. Also
$q_C(0)\geq J_2(0)/2-E_{C,0}>.0531$. At $x=0$ this yields a nonzero degree-two angular component; division by the positive density preserves its degree. This completes the physical part of \cref{m:twofold}.

The exact relative density change squared is $3\int_0^{1/2}r^2\psi(4r^2)^2\dd r$. Monotone endpoint bounds on $512$ radial cells, with exact integration of $r^2$, give the outward square-root interval
\begin{equation}
 [.14522190727469929229,\ .14591961734810516549].
\end{equation}
The changed support has volume fraction $1/8$, and central density doubles exactly.

\paragraph{Arithmetic and reproducibility.}
Every decision uses rational endpoints on the integer scale $10^{40}$. Exponentials use the degree-$64$ series with analytic tails as in \cref{app:radius}; far-tail bump values use monotonic bounds. The hyperbolic functions use their degree-$32$ series in $z^2$ with positive factorial tails. Integer square roots enclose geometric endpoints. No external floating-point elementary function determines acceptance. The supplementary checker uses only Python's standard library and reproduces the full interval tables. This is a computer-assisted enclosure conditional on the stated analytic identities, rather than proof-assistant verification.

The $N=Q=8$ enclosure already gives $j_{2,k}>.1434$ and a positive degree-two lower bound; the stated refined enclosure gives the stronger physical margins. This continuous physical certificate concerns the twofold pair. The elevenfold comparison in \cref{app:integercertificate} is a separate finite-model result.

\section{Observation equivalence and extensions of the physical pair}\label{app:extensions}
\subsection{Image objectives and image-only estimators}
Let $F=(\sigma,j)$ denote a field and $\cR(F)$ its full oriented-ray RGB. If $\cR(F_1)=\cR(F_2)$, then $\Phi(\cR(F_1);z)=\Phi(\cR(F_2);z)$ for any function $\Phi$ and fixed auxiliary input $z$. Expectations over a field-independent distribution of cameras, prompts, augmentations, and noise preserve this equality. This proves \cref{m:rgbfiber}, including RGB-only reconstruction supervision. Three-dimensional labels and geometric penalties use additional information and need not agree.

More explicitly, a camera outside the volume selects a subset of external oriented lines, whether orthographic or perspective. A pixel may integrate them against any fixed footprint measure. Equality on every ray implies equality after this linear integration, and applying the same deterministic tone mapping, quantizer, feature extractor, or learned image loss preserves equality. Identical observation-noise kernels give identical observation distributions as well. Thus additional external views and arbitrarily fine resolution cannot distinguish this pair under the same acquisition model. Changing background, illumination, or adding direct transmission measurements changes the observation model and can break the equivalence.

For a differentiable exact-data path $F(s)$ and any image cotangent $g$, the chain rule gives
\begin{equation}\label{eq:fibertangent}
 \langle D\cR(F)^*g,F'(s)\rangle
 =\langle g,D\cR(F)F'(s)\rangle=0.
\end{equation}
The RGB score-distillation update has this pullback form \cite{poole2023}. Its parameter-space values may differ across an observation fiber because the renderer Jacobians differ; \eqref{eq:fibertangent} concerns their action along the fiber, not equality of training trajectories. DreamFusion additionally changes shading and uses orientation/opacity regularizers, so the full algorithm is outside this fixed-observation corollary.

An estimator based only on the common RGB has the same output distribution $\widehat\sigma$ under either field. The triangle inequality yields
\begin{equation}\label{eq:exactpairminimax}
 \max_{i=1,2}\mathbb E\norm{\widehat\sigma-\sigma_i}_2
 \geq\tfrac12\norm{\sigma_1-\sigma_2}_2.
\end{equation}
For \cref{m:twofold}, the right side exceeds $.0726\norm{\sigma_1}_2$. This noiseless two-point limitation holds within the SH2 class. It is an elementary consequence of the exact pair, not a claim that any particular foundation model empirically attains this error.

\subsection{A family of differently textured physical scenes}
The same profiles $P,D,C$ certified in \cref{app:finitecertificate} cover a family of spatial textures. The matrices below may have unrelated eigenspaces across RGB channels.

\begin{proposition}[Quadratic-source family]\label{prop:quadraticfamily}
For each channel let
\begin{equation}
 f_k(x)=\tfrac14+\ell_k\cdot x+x^TQ_kx,\qquad
 |\ell_k|\leq.01,\quad0\preceq Q_k\preceq\tfrac12\Id.
\end{equation}
Use $\sigma_1=1$, $\sigma_2=1+\psi(4|x|^2)$, $j_{1,k}=f_k$, and
\begin{equation}\label{eq:quadraticfamily}
 j_{2,k}=f_k+P f_k+D(\ell_k\cdot d+2x^TQ_kd)+2C d^TQ_kd.
\end{equation}
Then both fields have exactly the same full-ray RGB and satisfy $j_{2,k}>.1955$ and $\sigma_2-j_{2,k}\geq.24$. If any $Q_k$ has nonzero trace-free part, the changed source has genuine degree two.
\end{proposition}
\begin{proof}
On $x=p+td$, expand $f_k(x)$ into its constant, linear, and quadratic powers of $t$. Their coefficients are $f_k(p)$, $\ell_k\cdot d+2p^TQ_kd$, and $d^TQ_kd$. The three paired moment equations for $P,D,C$ match each coefficient exactly; polarization allows every symmetric $Q_k$.

The proxy correction is $\psi(4|x|^2)\int_0^{\sqrt{1-|x|^2}}e^{-t}f_k(x+td)\dd t$. Its integrand is positive because $|x+td|\leq1.5$ on $B_{.5}$ and $f_k(x+td)\geq.25-.015$. Its multiplier is bounded by
\[
 (.25+.5S+.01s)J_0(0)+(s+.01)J_1(0)+.5J_2(0)<1.
\]
The original error bound increases by at most $.01(sE_P+E_D)<.000696$, giving total source error below $.044429$. Since $f_k\geq.24$ in $B_1$, $j_{2,k}>.1955$. Inside $B_{.5}$ the proxy absorption margin is at least $.75-.01s-.5S=.62$; outside it is at least $.24$. The stated global bound follows. At zero, the quadratic angular term is $2C(0)d^TQ_kd$, with $C(0)>.0531$.
\end{proof}
This enlarges the source-texture family, while keeping the density radial. It makes no common positive-$\kappa_1$ assertion for all these textures.

\subsection{Finite-amplitude pairs at higher optical depths}
Let the constant reference density be $\nu>0$, with $j_{1,k}=\nu(1/4+x_k^2/2)$, and put $\sigma_2=\nu+\eps\psi_S$. Dividing the original first-kind equations by $\nu$ gives the same profiles, now with $k=\nu+e$ in \eqref{eq:nAk}. The source correction becomes $\nu[P/4+Px_k^2/2+Dx_kd_k+Cd_k^2]$. Define
\begin{align}
 J_m(w)&=\int_0^{\sqrt{1-w}}e^{-\nu t}t^m\dd t
       =\frac{m!}{\nu^{m+1}}-L_m(w),\nonumber\\
 L_0&=e^{-\nu\ell}/\nu,\quad
 L_1=(\ell/\nu+1/\nu^2)e^{-\nu\ell},\quad
 L_2=(\ell^2/\nu+2\ell/\nu^2+2/\nu^3)e^{-\nu\ell}.\label{eq:nuprimitive}
\end{align}
Their derivatives are $L_0'=e^{-\nu\ell}/(2\ell)$, $L_1'=e^{-\nu\ell}/2$, and $L_2'=\ell e^{-\nu\ell}/2$. The paired target kernels become
\begin{equation}\label{eq:nutarget}
 b_0=a_0/\nu-L_0c_e,\quad
 b_1=a_0/\nu^2+a_1/\nu-L_1c_e,\quad
 b_2=2a_0/\nu^3+2a_1/\nu^2+a_2/\nu-L_2c_e.
\end{equation}
These formulas follow by the three tail integrations in \eqref{eq:nuprimitive}. Abel composition gives \eqref{eq:nblock} with these $J_m$ and $b_m$. Substituting the proxy cancels the constant primitive terms, leaving exactly the residual expressions \eqref{eq:nreskernels} with the new $L_m$. Thus the complete interval construction and segmented recurrence apply unchanged.

The proxy source correction is $\nu\eps\psi_S\int_0^\ell e^{-\nu t}[1/4+(x_k+td_k)^2/2]\dd t$. The physicality transfer requires
\begin{equation}
 \nu[(1/4+S/2)J_0(0)+sJ_1(0)+J_2(0)/2]<1.
\end{equation}
It then has source margin $\nu/4$ and interior absorption margin $\nu(3/4-S/2)$. The profile-to-source error is the previous bound multiplied by $\nu$.

\begin{proposition}[Three optically thick exact pairs]\label{prop:densepairs}
For each row of \cref{tab:densepairs}, the integral-defined SH2 source is physical, is equal to the reference outside $B_s$, and exactly preserves every oriented-ray RGB. Central density doubles. The source and absorption bounds in the table hold for every point and direction.
\end{proposition}
\begin{proof}
Use the preceding generalized identities and evaluate \eqref{eq:npropagate} with the outward integer arithmetic of \cref{app:finitecertificate}, $N=Q=16$. The proxy multiplier, source margin, absorption margin, and $C(0)>0$ inequalities all pass. The standalone checker reproduces the rational endpoints; the table rounds lower bounds downward. Volterra invertibility and the original moment identities supply exact equality independently of these inequalities.
\end{proof}

\begin{table}[t]
\caption{Continuous certificates at higher reference optical depths. $\tau_*=2\nu$ is the reference diameter optical thickness. Every row doubles central density; increasing thickness here also changes support radius.}\label{tab:densepairs}
\centering
\begin{tabular}{@{}rrrrrrr@{}}\toprule
$\nu$ & $\eps$ & $s$ & $\tau_*$ & $j_{2,k}>$ & $\sigma_2-j_{2,k}\geq$ & $C(0)>$\\\midrule
2&2&.3&4&.4489&.5&.0643\\
4&4&.2&8&.9749&1&.0435\\
8&8&.1&16&1.9991&2&.0153\\\bottomrule
\end{tabular}
\end{table}

As a check, setting $\nu=1$ reproduces every original $N=16$ rational endpoint. Independent first-kind solves at $96$ radial coefficients give held-out moment residuals below $1.1\times10^{-16}$ and sampled profile-error/enclosure ratios below $.677$ for all three cases. Those floating-point cross-checks are supplementary diagnostics, not the continuous enclosure proof. These examples remain emitting, absorbing, nonscattering volumes with radial density; they neither guarantee recovery behind an opaque barrier nor extend the first-order positive radius.

\subsection{A controlled nonradial approximate ambiguity}
\begin{proposition}[Common added absorption]\label{prop:nonradialrobust}
Let $(\sigma_i,j_i)$ be any physical exact-RGB pair. Add an arbitrary common $a(x)\geq0$ to both densities, keeping the sources fixed. If the largest integral of $a$ along any scene-crossing line is at most $\delta$, then the modified fields remain physical, have the same source degrees and density difference, and satisfy
\begin{equation}\label{eq:commonabsorption}
 \norm{\cR(\sigma_1+a,j_1)-\cR(\sigma_2+a,j_2)}_\infty
 \leq1-e^{-\delta}\leq\delta.
\end{equation}
\end{proposition}
\begin{proof}
On one ray and channel the original positive measures $\mu_i=T_i j_i\dd t$ have equal total mass $C\leq1$. With $A(t)$ the accumulated added absorption and $\tau$ its total, $0\leq A(t)\leq\tau\leq\delta$. Each new radiance $\int e^{-A(t)}\dd\mu_i$ lies in $[e^{-\tau}C,C]$, whose length is at most $1-e^{-\delta}$. Physicality follows from $j_i\leq\sigma_i\leq\sigma_i+a$.
\end{proof}
For example, $\delta\leq\log(256/255)$ gives a channelwise difference at most $1/256$, below one linear-RGB level of $1/255$, while preserving the certified pair's finite density difference. The added absorption may be nonradial. This is a continuity-based \emph{approximate} ambiguity result; an exact nonradial SH2-density construction remains unresolved here.

\section{A numerical angular ladder in one nonradial scene}\label{app:nonradialnumerics}
We fix one physical reference throughout the degree sweep:
\begin{align}
 \sigma(x)&=1+\sum_{i=1}^3w_i\exp(-|x-z_i|^2/a_i^2),\nonumber\\
 c_k(x)&=.5+.22\sin((Ax)_k+b_k),\qquad j_k=\sigma c_k.\label{eq:nonradialscene}
\end{align}
The density centers are $z_1=(.35,.10,-.15)$, $z_2=(-.20,.30,.25)$, $z_3=(.10,-.30,.35)$, with $w=(.70,.45,.30)$ and $a=(.40,.30,.35)$. The color is isotropic and belongs to $[.28,.72]$. Its coefficients are
\[
 A=\begin{pmatrix}4&1&.5\\-.5&3.5&1\\1&-1&4.5\end{pmatrix},
 \qquad b=(.2,.8,-.5)^T.
\]
The three evaluation points are $(0,0,0)$, $(.2,-.1,.1)$, and $(-.25,.15,.05)$. All choices precede computation. We start from $96$ deterministic Fibonacci normals plus coordinate and center-offset axes, then perform six local normal searches per degree. All evaluated normals are shared across $L$ before taking minima. Thus the sampled degree classes remain nested. Between $2528$ and $2600$ distinct normals are evaluated per point.

Optical integrals of the Gaussian density are analytic; cell attenuation weights use these integrals and colors use midpoint sampling. The primary resolution is $96$ circle angles and $256$ ray cells. We reevaluate the union of all six selected minimizing normals at $192$ angles and $512$ ray cells, again using this common normal set for all degrees. \Cref{tab:nonradialladder} reports this audit. The largest relative resolution change is $1.011\times10^{-5}$. Numerical values are well above that integration discrepancy, but finite-normal search and local minimization do not prove a positive infimum over all normals. Optimizer termination is not an injectivity certificate.

\begin{table}[t]
\caption{Estimated $\widehat\kappa_L$ in the \emph{same} nonradial scene. The three columns are the three points specified above; these are numerical diagnostics.}\label{tab:nonradialladder}
\centering
\begin{tabular}{@{}rrrr@{}}\toprule
$L$ & Point 1 & Point 2 & Point 3\\\midrule
0&$.029774$&$.025616$&$.026539$\\
1&$.011857$&$.011341$&$.015706$\\
2&$.0052801$&$.0056389$&$.0063861$\\
3&$.0011850$&$.0011672$&$.0026505$\\
4&$.00038659$&$.00039738$&$.0010602$\\
5&$.00010679$&$.000080992$&$.00015611$\\\bottomrule
\end{tabular}
\end{table}
\begin{figure}[t]
\centering\includegraphics[width=.95\textwidth]{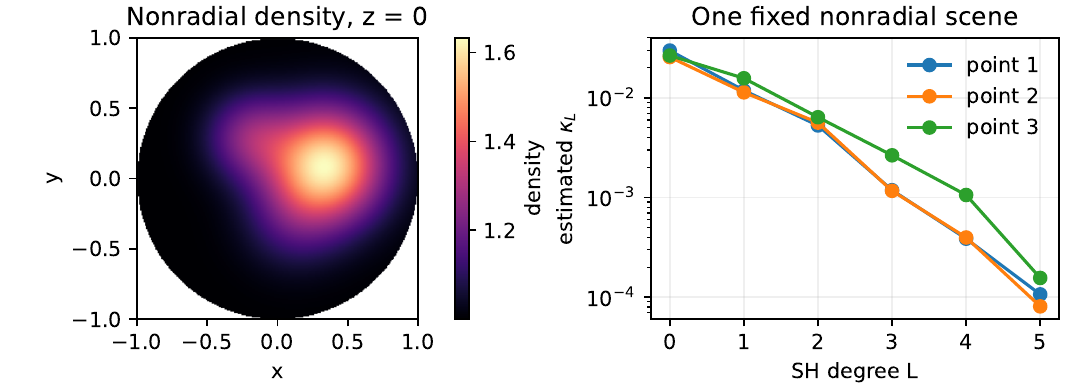}
\caption{Nonradial density slice and the fixed-scene angular ladder. The three curves use the same density and isotropic spatial color at all degrees. The positive plotted values are finite-search estimates, not certified lower bounds.}\label{fig:nonradialfull}
\end{figure}

All curves decline, without the reference-specific exact degree-two cutoff. This supports the usefulness of computing the criterion beyond the radial example; it does not validate the criterion as a depth-confidence predictor or establish a new practical reconstruction method.

\section{Certified directional witnesses in an asymmetric volume}\label{app:circlecertificate}
The density and isotropic RGB are exactly those of \cref{eq:nonradialscene}, with all listed decimal scene parameters interpreted as rational numbers. We certify the central point $x=0$. Previously selected normal directions are retained; no new minimizing-normal search is needed.

\begin{proposition}[One-circle certificates]\label{prop:onecircle}
Fix a normal $n$ and any vector trigonometric polynomial $p\in V_L^3$. If verified integration gives
\begin{equation}
 E(n,p):=\avg{T^2|U-p|^2}_n\leq\overline E,
\end{equation}
then $\kappa_L(0)\leq\sqrt{\overline E}$ and every local inverse coefficient is at least $1/\sqrt{\overline E}$ in the norms of \cref{m:Y}. Neither optimality of $p$ nor global optimality of $n$ is required.
\end{proposition}
\begin{proof}
The definition of $\kappa_L$ minimizes first over $p$ and then over $n$, hence its squared value is at most this feasible objective. Apply the shrinking-packet necessary bound. In particular, a positive upper bound on $\kappa_L$ says nothing about whether its global infimum is positive.
\end{proof}

\begin{table}[h]
\centering\small
\caption{Certified central-point bounds in the fixed nonradial scene. Values are rounded outwards. Only the stated inequality directions are asserted.}\label{tab:circlecertificate}
\begin{tabular}{@{}crr@{}}\toprule
$L$ & Upper bound on $\kappa_L(0)$ & Necessary inverse lower bound\\\midrule
0 & $.029773685$ & $33.5867$\\
1 & $.011857206$ & $84.3369$\\
2 & $.005280051$ & $189.3921$\\
3 & $.001184962$ & $843.9092$\\\bottomrule
\end{tabular}
\end{table}

\subsection{Integration remainders and arithmetic}
Use $512$ equally spaced circle angles and $16$ ray cells, each integrated by $12$-point Gauss--Legendre quadrature. In the complex radius-$R=1/16$ neighborhood of the real ray segment, the entire Gaussian density obeys $|\sigma|<3$, and the sinusoidal colors obey $|c_k|<1$. The holomorphic integrand $e^{-\int_0^t\sigma}\sigma c_k$ is bounded by $3e^{51/16}<80$. Cauchy's estimate and the classical Gauss derivative remainder give, for a cell of length $h\leq R$ and holomorphic bound $B$,
\begin{equation}\label{eq:gausscirclebound}
 |\mathrm{error}_{\rm cell}|\leq
 \frac{hB(h/R)^{24}}{25\binom{24}{12}^2}.
\end{equation}
After summing cells, take $B=3$ for optical depth and $B=80$ for the source integral. Prefix and partial-cell optical errors are included before exponentiation; hence the quadrature encloses the continuous transport weights, rather than assuming sampled weights exact.

For complex angle $|\operatorname{Im}\theta|\leq1/8$, an orthonormal complexified circle has $|\operatorname{Im}d|\leq\sinh(1/8)$. The scene's widths are at least $.3$, so $|\sigma|<3$, $|c_k|<1$, $|T|<21$, and $|U_k|<61$. Stored polynomial witnesses satisfy $\sum_j|p_{j,k}|e^{L/8}<10$. The holomorphic function $f=\sum_k[T(U_k-p_k)]^2$ therefore has $|f|<10^7$ in this strip. The normalized periodic trapezoid remainder \cite[Theorem 3.2]{trefethen2014} is at most
\begin{equation}\label{eq:circletrap}
 \frac{2\cdot10^7}{e^{512/8}-1}<4\cdot10^{-21}.
\end{equation}
For real angles $f$ is exactly the nonnegative residual energy. This supplies a continuous-circle upper bound, beyond a finite angular sample.

The verifier uses outward-rounded binary64 intervals for basic arithmetic, with explicit outward reductions. Exponentials and trigonometric functions use Taylor remainder enclosures and range reduction, not unbounded library evaluations. Each Gauss root is isolated by rational sign-changing brackets; interval formulas enclose its weight. Rational square-root bounds normalize the direction frame. Verification also checks $24$ exact Gauss moments, rational enclosures of all analytic constants, and $50$-digit independent transport evaluations on four prescribed directions. The maximum backlight interval width is below $2.1\times10^{-11}$. These checks support an IEEE-754 interval certificate, not a proof-assistant verification.

\subsection{A comparison using the correct sides of the bounds}
Let $M$ bound optical depth on every central half-ray, and set
\begin{equation}\label{eq:fixednonradialgain}
 \zeta=e^{-M}\left[c_g(0)-\frac{18}{25}(1-e^{-M})\right],
 \qquad c_g(0)=\frac12+\frac{11}{50}\sin\frac45.
\end{equation}
Since $c_g\leq18/25$, we have $U_g\leq(18/25)(1-e^{-M})$ and $T\geq e^{-M}$. The fixed-color linearization has weight $T(c-U)$; its green component is at least $\zeta>0$ pointwise. Hence its gain on \emph{every} great circle is at least $\zeta$.

We certify $M$ using a $64\times128$ latitude--longitude covering. Every direction is within angular distance $\pi/64$ of a grid direction. For $\sigma=1+\sum_i a_i e^{-|x-c_i|^2/w_i^2}$, the half-ray depth has angular Lipschitz constant bounded by
\begin{equation}
 G=\frac{429}{1000}\sum_i\frac{a_i}{w_i}
 >\frac12\sqrt{2/e}\sum_i\frac{a_i}{w_i}.
\end{equation}
Thus interval quadrature at grid directions plus $G\pi/64$ yields
\begin{equation}
 M<1.552383,\quad \zeta>.01911472,\quad
 C_{\rm fixed}^{\rm frozen}<52.316.
\end{equation}
Together with the exact stored SH3 bound this proves
\begin{equation}
 \frac{C_{\rm SH3}^{\rm frozen}}{C_{\rm fixed}^{\rm frozen}}
 >16.1310.
\end{equation}
The proof combines an upper bound for the denominator with a lower bound for the numerator. Dividing the SH3 and SH0 \emph{lower bounds} in \cref{tab:circlecertificate} would not prove a ratio of their actual inverse coefficients. The SH3 absolute lower bound transfers to the nonlinear local problem; the fixed-color upper bound here is frozen, so no nonlinear ratio is asserted.

The inequalities have strict margins. Transport along the compact ray set and the finite-polynomial witness energies vary continuously under uniform perturbations of bounded physical density and color. The positive fixed-color bound is continuous as well. Therefore a sufficiently small neighborhood of smooth nonradial scenes retains a ratio above $16$. This is an open-family severe-instability statement, not a claim that all nonradial scenes have the same gap or that exact ambiguity is generic.

\section{Directional effective sources and physical interpretation}\label{app:effectivephysical}
Thermal emission by a locally isotropic medium supplies a direction-independent source. If that is the correct known prior, SH0 is appropriate. For extinction $\sigma=\sigma_a+\sigma_s$, a single-scattering effective source under a prescribed incident field is
\begin{equation}
 j_{\rm eff}(x,d)=\sigma_a(x)L_e(x)+\sigma_s(x)
 \int_{\Sph^2}p_x(d,d')L_{\rm inc}(x,d')\,\mathrm dS(d').
\end{equation}
A nonnegative, normalized phase function and channelwise $0\leq L_e,L_{\rm inc}\leq1$ imply $0\leq j_{\rm eff}\leq\sigma$ \cite{pbr2023}. The angular dependence may be approximated by a physically constrained SH representation. This explains a directional \emph{effective} source without attributing directional thermal emission to smoke or flames.

The free-source model is more permissive than a specified phase function and illumination coupled to the unknown density. In particular, the exact SH2 source is not thereby proved realizable under a fixed microscopic scattering law. Likewise PINF's smoke field \cite{chu2022} maps space and time to color and density; it is not cited as evidence that all smoke reconstructions use directional color. The question here is what density information a learned representation loses when its appearance is allowed to refit in a given angular class.

\section{Matched fixed-density fits on the certified twofold pair}\label{app:twofoldfits}
The continuous density is exactly $1+\psi(4|x|^2)$ from \cref{m:twofold}. The finite representation uses its float64 nodal samples on the same $32^3$ cube $[-.8,.8]^3$ as the earlier experiment. Colors are trilinearly interpolated SH coefficients inside $B_{.8}$; the exterior is known. Thus the discrete density is an approximation of the certified density, while the analytic data are the reference field's exact RGB. We retain the same cameras, including their randomization, train/test split, ray masks, and $100$-view subset of $400$ views.

There are $44{,}800/179{,}200$ retained training rays and $7{,}936$ independent test rays. Training uses $128$ midpoint compositor samples, and evaluation uses $256$. Each of four full-batch CPU solves handles the quadratic reference and cosine control as independent RGB right-hand sides. The maximum is $2{,}500$ iterations or $120$ seconds per solve, with the lowest feasible training objective retained. No density, geometry, or held-out image is used for model selection.

For SH2, nodal color is initialized from the independently computed $128$-coefficient Chebyshev representation of the twofold source and projected to the physical coefficient set. SH1 uses the truncated SH1 part. The cosine control starts from its reference DC color. These initializations supply candidate compensation functions; the experiment does not show that random joint training discovers them. The physical SH1 constraint is exact, and the SH2 weighted-group constraint is a sufficient all-direction color box.

\begin{table}[h]\centering\small
\caption{Twofold candidate: certified training bounds and independent evaluation. Lower bounds include exact rational dual rescaling; upper bounds retain the original feasible witnesses.}\label{tab:twofoldfull}
\begin{tabular}{@{}llrrrr@{}}\toprule
Scene & Views & Degree & Lower RMSE & Upper RMSE & Test PSNR\\\midrule
Quadratic &100&1&$0$&$2.721\cdot10^{-5}$&$48.35$\\
Quadratic &100&2&$0$&$1.806\cdot10^{-7}$&$97.47$\\
Quadratic &400&1&$.000351417764$&$9.262\cdot10^{-4}$&$57.75$\\
Quadratic &400&2&$0$&$5.276\cdot10^{-6}$&$99.28$\\
Cosine &100&1&$0$&$3.256\cdot10^{-5}$&$44.52$\\
Cosine &100&2&$0$&$2.803\cdot10^{-6}$&$68.94$\\
Cosine &400&1&$.000884562911$&$.001434306$&$53.89$\\
Cosine &400&2&$0$&$.000123035$&$73.71$\\\bottomrule
\end{tabular}
\end{table}

The criterion is a strictly positive SH1 lower bound at least ten times the feasible SH2 upper bound. The quadratic $400$-view certificate has ratio $66.61$, while the $100$-view lower certificate is zero. The control's $400$-view ratio is $7.19$, below ten, and its $100$-view lower bound is also zero. Thus the tenfold separation is certified in one of the two primary camera cases. The bounds in \cref{tab:twofoldfull} use scalar-optimized dual witnesses with the same data, model and feasible coefficients.

\paragraph{Exact dual rescaling.}
For one channel let $F^*=\inf_{a\in K_1}\norm{Aa-b}_2^2/2$. Completing the square gives, for any data vector $y$,
\begin{equation}\label{tf:dual}
 F^*\geq D(y):=-\tfrac12\norm y_2^2-b^Ty+\inf_{a\in K_1}(A^Ty)^Ta.
\end{equation}
Writing $g=A^Ty$, the last term is explicitly
\begin{equation}
 \tfrac12\sum_v\left[g_{v0}-\max\{\lvert g_{v0}\rvert,\norm{g_{v,1:3}}_2/\sqrt3\}\right].
\end{equation}
Let $N=\norm y_2^2/2$ be the exact rational value for the saved dyadic witness, and let $H_-$ be the verified lower bound on the two homogeneous linear terms in \eqref{tf:dual}. For every $\alpha\geq0$,
\begin{equation}\label{tf:rescale}
 F^*\geq-\alpha^2N+\alpha H_-.
 \qquad
 F^*\geq\frac{(H_-)_+^2}{4N}\quad(N>0).
\end{equation}
The maximizing multiplier is $(H_-)_+/(2N)$; for $N=0$, use zero. The interval errors in $H_-$ scale with $\alpha$ and are retained. Exact rational arithmetic avoids rounding the rescaled witness. Summing the three channel bounds and rounding the RMSE down gives the reported lower bound. Residual-based dual rescaling is a classical certification device \cite{ndiaye2017}; the present calculation optimizes the guaranteed quadratic lower envelope for the physical SH set. It requires neither convergence nor a new fit. Replaying the integer checker from the saved optical weights and colors reproduces every original component and witness hash; independently summing each squared witness entry using arbitrary-precision integers verifies all six norms.

\paragraph{What the new bounds distinguish.}
For the quadratic reference, the certified $400$-view SH1 lower bound is more than $12.9$ times the $100$-view feasible upper bound. This compares opposite-sided bounds on the true optima, not two unfinished iterates. Finite observations genuinely change the compensation capacity on these nested camera sets, although spatial bandwidth and continuum quadrature remain unresolved. The SH1/SH2 ordering at $400$ views survives any target perturbation with RGB RMS norm less than $(.000351417764-.000005275487)/2=.0001730711385$, or $.0441331$ on a $0$--$255$ linear RGB scale. These are sub-level error budgets. All retained feasible SH1 fits have training residual below $1/255$; no one-level exclusion is established.

The retained primary fits stop after $516/250$ iterations for SH1/SH2 at $100$ views, and $144/62$ at $400$ views. Their optimality gaps remain unresolved. Equal wall budgets provide fewer iterations for the larger ray set, so achieved losses alone would not order the true optima; the opposite-sided certificates above do. The positive nonlinear theorem has an unknown support radius at most $.01$, whereas this experiment allows colors to vary throughout $B_{.8}$.

\paragraph{Discretization diagnostics.}
On $52{,}440$ independent uniform-in-volume samples, density interpolation has relative $L^2$ error $.0037146$ and sampled maximum error $.02197$. The interpolated central density is $1.99198$ versus the exact value $2$. Doubling compositor resolution on the \emph{same} training witnesses changes predictions by $3.87\cdot10^{-5}/3.98\cdot10^{-6}$ RMS for SH1/SH2 at $100$ views, and $1.63\cdot10^{-5}/3.37\cdot10^{-6}$ at $400$ views. In particular, the very small saved-grid SH2 residual is below its integration diagnostic and is not a continuous-rendering certificate. The source field's analytic all-ray identity and physicality are proved independently.

\paragraph{Geometry of the analytic pair.}
For impact parameter $b$, integrate the exact radial density to obtain optical thickness $\tau(b)$ and opacity $\alpha(b)=1-e^{-\tau(b)}$. Conditional absorption distance from entry is
\begin{equation}
 z(b)=\frac{\int_0^{D(b)}T(t)\,\mathrm dt-D(b)T(D(b))}{1-T(D(b))},
 \qquad D(b)=2\sqrt{1-b^2}.
\end{equation}
This follows by integration by parts from $\int t\sigma(t)T(t)\,\mathrm dt$. At $b=0$, the alternative has $\tau=2.60345016$, $\alpha=.92598224$, and $z=.65023838$, compared with $2,.86466472,.68696471$ at the reference. Over a unit image disk, integrating with measure $2b\,\mathrm db$ gives RMS changes $.01548836$ and $.00821076$ in opacity and depth. The maximum changes occur on the central ray. Doubling the depth integration resolution changes depths by at most $1.17\cdot10^{-8}$. These scalar geometry evaluations are ordinary numerical diagnostics; the continuum physicality and equality certificate is \cref{m:twofold}.

\section{From discrete residuals to continuous ray integration}\label{app:continuoustransfer}
The residual test in \cref{app:twofoldfits} can be strengthened from a saved compositor to continuous integration of the same nodal field. This section specifies the remaining model assumptions and gives a uniform error bound. Increasing the sample count at one returned color array would not suffice: a lower bound must cover every allowed color refit.

\subsection{Model and transfer inequality}
Interpret the saved density nodes as exact dyadic numbers in $[1,2]$ on the rational grid $[-4/5,4/5]^3$, with spacing $8/155$. Their trilinear interpolant is $\sigma_h$. Physical nodal SH coefficients define $c_h(x,d)\in[0,1]$ by the same interpolation, with $j_h=\sigma_h c_h$.

Each prescribed ray segment is $x(t)=p+td$, $-h\leq t\leq h$, with stored dyadic inputs. We use arc length $\dd s=|d|\dd t$ and normalized direction $d/|d|$ in the SH basis. With target $y$, near contribution $b_-$, and far radiance $b_+$ fixed at their stored values, the continuous forward map is
\begin{equation}\label{ct:map}
 R_h(a)=b_-+e^{-(\ell-h)|d|}
 \left[\int_0^D e^{-\int_0^s\sigma_h}\sigma_h c_h\dd s
       +e^{-\int_0^D\sigma_h}b_+\right].
\end{equation}
These are finite-ray boundary inputs to one common interior field. Density interpolation and error between the prescribed boundary data and the analytic reference's full-ray data are separate model errors.

Let $Q_h(a)=B_Qa+b_Q$ be a saved dyadic optical operator. Suppose, uniformly over the relevant physical color class,
\begin{equation}\label{ct:uniform}
 \norm{R_h(a)-Q_h(a)}_{\rm RMS}\leq\mu.
\end{equation}
If exact arithmetic certifies a lower bound $l_1$ on the SH1 optimum for $Q_h$, and a feasible SH2 residual $u_2$, then
\begin{equation}\label{ct:transfer}
 \rho_1(R_h;y)\geq\max(0,l_1-\mu),\qquad
 \rho_2(R_h;y)\leq u_2+\mu.
\end{equation}
Both inequalities follow directly from the triangle inequality. The SH1 bound uses the full physical nodal degree-one cone, while the SH2 witness is repaired into the physical sufficient set of \cref{app:twofoldfits}.

\subsection{Uniform integration enclosure}
\paragraph{Cell locations.}
Partition each ray with sorted dyadic approximations to grid crossings. Rather than assuming exact crossing decisions, the verifier encloses both endpoint coordinates in the chosen cell and checks that they lie in $[-\delta,1+\delta]^3$, $\delta=10^{-10}$. Affinity gives the same enclosure throughout the segment. For nodal values with range at most one, the cell polynomial extension and true interpolant differ by at most
\[
 3\delta+3\delta(1+2\delta)^2<10\delta=:\epsilon.
\]
This applies to density and every physical directional color. With path length $D\leq1.61$, $|b_+|\leq1$, and positive extended density, the resulting rendering error, including the terminal transmittance, is bounded by
\begin{equation}\label{ct:cellerror}
 (4D+3D^2+D)\epsilon<2\cdot10^{-8}.
\end{equation}
Constant extension beyond the outer grid handles rounded endpoints without changing this estimate.

\paragraph{Gauss remainder.}
Within one cell, $q(s)=\int_0^s\sigma_h$ is quartic and $w(s)=\sigma_hc_h$ has degree at most six. If the cell optical depth is $a\leq.2$, the degree-six Taylor polynomial of $e^{-q}$ about $q=a/2$ gives $wP_6(q)$ of degree at most $30$. A $16$-point Gauss rule is exact on this polynomial. Since $|c_h|<2$ for the tiny cell extensions, its error is at most $4a(a/2)^7/7!$. Total optical depth below $3.3$ therefore gives
\begin{equation}\label{ct:gauss}
 E_G\leq4(3.3)(.1)^7/7!<3\cdot10^{-10}.
\end{equation}
Density optical depths are integrated algebraically. Cell splitting and controlled quadrature are classical volume-rendering tools \cite{weiss2021}; the role here is to make the residual exclusion uniform in all admissible colors.

\paragraph{Arithmetic.}
Rational bisection isolates all sixteen Legendre roots, and outward intervals enclose the positive weights. A degree-sixteen Taylor evaluation of $e^{x/16}$, followed by four outward squarings, encloses $e^x$ for $-8\leq x\leq1$. Before squaring, the coefficient, Horner, and Taylor errors sum to less than $2^{-45}$; this bound is checked with rational arithmetic. Correct rounding of a library exponential is not assumed.

Outward binary64 interval operations enclose optical weights, prefixes, and exterior attenuation. For each row, the sum of distances from local weight intervals to their saved midpoints bounds their action on \emph{every} nodal color in $[0,1]$. Nonnegative sparse accumulation contributes less than $2^{-40}$. Adding these radii, the offset radius, \eqref{ct:cellerror}, and \eqref{ct:gauss} gives
\begin{equation}\label{ct:mu}
 \mu\leq2.501628\cdot10^{-8}.
\end{equation}
The calculation assumes IEEE-754 round-to-nearest basic arithmetic with outward inflation; it is an executable interval certificate, not a proof-assistant derivation. The released checker includes the exact constants, root enclosures, coefficient repairs, and integer primal/dual witnesses.

\subsection{Results and checks}
All $179\,200$ retained rays at $400$ views are enclosed, including the nested $44\,800$ rays at $100$ views. The density, color region, targets, and saved color proposals are unchanged; the proposals generate witnesses for the accurate optical matrix without another fit. Inserting the exact rational certificates into \eqref{ct:transfer} yields \cref{ct:results}. The primary $400$-view residual ratio exceeds $62.97$, while the control's ratio is about $6.9257$. The $100$-view lower bounds remain empty. A zero lower bound means the witness does not establish rejection, not that an exact fit exists.

\begin{table}[t]
\centering\small
\caption{Continuous integration of the fixed nodal field. SH1 entries are lower bounds, SH2 entries are feasible upper bounds. Displayed bounds are rounded outward.}\label{ct:results}
\begin{tabular}{@{}lcrr@{}}\toprule
Scene & Views & SH1 lower & SH2 upper\\\midrule
Quadratic & 100 & $0$ & $1.556610\cdot10^{-6}$\\
Quadratic & 400 & $3.454608\cdot10^{-4}$ & $5.485578\cdot10^{-6}$\\
Cosine & 100 & $0$ & $7.063575\cdot10^{-6}$\\
Cosine & 400 & $8.530096\cdot10^{-4}$ & $1.231659\cdot10^{-4}$\\\bottomrule
\end{tabular}
\end{table}

Let $\ell_1,u_2$ be the displayed primary $400$-view bounds. On the linear $0$--$255$ scale their certified gap is
\begin{equation}\label{ct:gapradius}
 255(\ell_1-u_2)=.08669368161.
\end{equation}
Distances to either feasible image set are $1$-Lipschitz in the target, so a perturbation of RMS $\eta$ preserves the ordering whenever $2\eta<\ell_1-u_2$. The sufficient radius is therefore half the gap, $.043346840805$ linear levels. In normalized units, the conservative bound $\eta<.0001699876$ used here is sufficient. The gap and this perturbation radius are distinct quantities. External calibration or density-interpolation errors must also fit the applicable budget before the result supports a claim about measured images.

Several checks supplement the arithmetic proof. The root/weight intervals enclose all $32$ exact Gauss moments. Independent $75$-digit evaluations check $2\,003$ exponential inputs. Comparing the independently assembled numerical and interval-midpoint matrices over all rows gives a uniform color-box difference below $4.52\cdot10^{-15}$. A separate adaptive transport ODE, using pointwise trilinear interpolation instead of cell polynomials, agrees on six prescribed rays within $1.68\cdot10^{-12}$. These numerical checks detect implementation errors; they do not replace the analytic remainders.

\paragraph{An unsuccessful direct transfer.}
Bounding the old $128$-step compositor's difference from continuous integration by the independent color box gives about $8.91\cdot10^{-4}$ RMS at $400$ views, exceeding its SH1 lower bound $3.51\cdot10^{-4}$. That route cannot certify a positive residual. Reevaluating the dual at the accurately enclosed operator avoids paying for this coarse-renderer discrepancy. The improvement is classical verification, not a new fitting algorithm or a continuum uniqueness claim for unrestricted spatial sources.

\clearpage
\section{A larger candidate at a one-level discrepancy budget}\label{app:elevencontinuous}
This experiment fixes the positive density candidate with central value eleven, separately from the certified exact twofold pair. It fixes the nodal density, cameras, targets, physical color class, and saved coefficient witnesses. Its purpose is a finite-ray density--appearance hypothesis test, so it does not require the analytic elevenfold Volterra source to satisfy the physical box. The SH2 fit is a different, independently verified physical nodal color.

\subsection{Continuous model and interval extension}
The exact stored density nodes lie in $[1,11]$ on the rational $32^3$ grid with spacing $8/155$ in $[-4/5,4/5]^3$. They define a continuous trilinear field inside the prescribed ray segments. SH directions use normalized stored directions; ray lengths use the corresponding arc-length speed. Targets and the near/far boundary radiances are prescribed exact inputs, as in \cref{app:continuoustransfer}. Calibration, interpolation relative to the analytic radial density, and changes in boundary inputs require separate error allowances.

The cell-partition check allows endpoint coordinates in $[-\delta,1+\delta]$, $\delta=10^{-10}$. Put $\varepsilon_0=10\delta=10^{-9}$. Relative to the true interpolants, the selected cell extensions change color by at most $\varepsilon_0$ and density by at most $10\varepsilon_0$. The extensions have $|c|<2$ and $|\sigma c|<12$. Source discrepancy is at most $31\varepsilon_0$, optical-depth discrepancy at most $10D\varepsilon_0$, and path length $D\leq1.61$. Including the far radiance gives the uniform rendered discrepancy
\[
 (31D+120D^2+10D)\varepsilon_0<4\cdot10^{-7}.
\]
These bounds hold for every physical SH coefficient array, not only the retained fits.

The verifier checks optical thickness at most one per cell and less than eighteen along the interior. A $24$-point Gauss rule integrates the source polynomial times the degree-ten Taylor approximation to attenuation exactly: their product has degree at most $6+4\cdot10=46$. Centering the exponential approximation at half the cell optical depth, positivity gives the summed remainder
\[
 E_G\leq\frac{4\cdot18\,(1/2)^{11}}{11!}<10^{-9}.
\]
Optical depths are enclosed by exact polynomial integration. Rational root isolation brackets every Gauss node, and interval weights enclose all forty-eight exact moments.

To enclose $e^x$ for $x\in[-32,1]$, evaluate its degree-sixteen Taylor polynomial at $x/64$, inflate endpoint values by $2^{-45}$, and apply six outward interval squarings. The Horner/coefficient/remainder proof in \cref{app:continuoustransfer} applies because $|x/64|\leq1/2$. A $10^{-12}$ prefix-sum inflation exceeds $\gamma_{92}\,18$. At most $736$ node contributions, their interval radii, offset error, and the same integer accumulation error are included. The resulting uniform per-channel bound is
\begin{equation}\label{eq:elevenmu}
 \|R_{\rm cont}(a)-(B_Qa+o_Q)\|_{\rm RMS}\leq\mu,
 \qquad \mu< 0.000011994577,
\end{equation}
for every physical nodal color $a$. This is an IEEE-754 interval certificate with exact rational/integer final inequalities, not a proof-assistant verification.

\subsection{Rejection and compatibility at a declared budget}
Apply the same exact SH1 dual and physical SH2 primal checker to the saved interval-matrix midpoint. Exact nonnegative rescaling tightens each dual proposal. If their matrix bounds are $\ell_1,u_2$, then
\[
 \rho_{1,\rm cont}\geq\max(0,\ell_1-\mu),\qquad
 \rho_{2,\rm cont}\leq u_2+\mu.
\]
\Cref{tab:elevencontinuous} expresses the resulting certificates in linear $0$--$255$ RGB RMS units. Every SH1 lower bound exceeds one; every physical SH2 upper bound lies below one. The primary quadratic hypothesis is therefore rejected under SH1 and compatible under SH2 at both view counts. The cosine control passes the same budget test. Its $400$-view ratio is below ten, giving a smaller quantitative separation.

\begin{table}[h]
\centering\small
\caption{Continuous elevenfold-candidate certificates. The final column is the additional target RMS perturbation that the lower bound can tolerate while keeping a rejection budget of one linear level. Values are rounded outwards.}\label{tab:elevencontinuous}
\begin{tabular}{@{}lrrrr@{}}\toprule
Scene & Views & SH1 lower & SH2 feasible upper & SH1 lower minus one\\\midrule
Quadratic & 100 & $1.05081$ & $0.00431$ & $0.05081$ \\
Cosine & 100 & $1.79127$ & $0.01118$ & $0.79127$ \\
Quadratic & 400 & $1.56965$ & $0.02054$ & $0.56965$ \\
Cosine & 400 & $2.53608$ & $0.29079$ & $1.53608$ \\
\bottomrule\end{tabular}
\end{table}

The budget is a total norm-ball discrepancy allowance for the supplied target, not a statistical noise guarantee or uniform pixelwise visibility threshold. Since distance to a feasible image set is $1$-Lipschitz in the target, if $y$ changes by RMS $\eta$, rejection at budget $b$ requires $\ell_1-\mu>b+\eta$. In particular, the primary $100$-view case has only about $.05$ additional linear-level margin. The experiment must not be described as tolerating a further arbitrary one-level perturbation on top of its one-level rejection budget.

\subsection{The twofold comparison in the same metric}
\begin{table}[h]
\centering\small
\caption{Continuous integration on the same candidate densities and rays. Values are RGB RMS in linear $0$--$255$ units, rounded outwards. The declared rejection budget is $1$. A zero lower bound is inconclusive.}\label{tab:budgetallviews}
\begin{tabular}{@{}lrrr@{}}\toprule
Candidate & Views & SH1 lower & SH2 feasible upper\\\midrule
Twofold & 100 & $0.000$ & $0.0004$ \\
Twofold & 400 & $0.088$ & $0.0014$ \\
Elevenfold & 100 & $1.050$ & $0.0044$ \\
Elevenfold & 400 & $1.569$ & $0.0206$ \\
\bottomrule\end{tabular}
\end{table}
On the existing twofold continuous operator, an uncompensated baseline uses the nodal interpolant of the true isotropic color. The same physical integer-primal checker, with all higher coefficients zero, encloses its RGB RMS and that of the retained SH1 fit. No lower residual bound is used as an achieved value. On $100$ views the baseline RMS is in $[.0041999755,.0042002778]$ and SH1 achieves at most $.000039168$; on $400$ views the corresponding values are $[.0042002128,.0042005188]$ and $.000926464$. Thus certified RMS suppression is at least $99.06\%$ and $77.94\%$, respectively; squared-error suppression is at least $99.991\%$ and $95.13\%$. The achieved residuals are below $.0100$ and $.2363$ linear levels. These bounds establish quantitative hiding at the stated finite precision while retaining the separate qualitative continuum results.

\subsection{Independent checks and limits}
An independent $32$-point cellwise calculation on $256$ fixed rays differs from the interval matrix by at most $7.67\cdot10^{-13}$ in its physical-color box envelope. Pointwise trilinear interpolation and an independent adaptive transport ODE on six fixed rays differ by at most $2.66\cdot10^{-11}$ for the physical SH2 witness. Such numerical crosschecks detect implementation errors; the guarantee comes from the uniform remainder and outward arithmetic above. The discrete certificate and its input scope are given in \cref{app:integercertificate}.

\clearpage
\section{Appearance compensation in trained voxel fields}\label{app:trainedcompensation}
The mathematical results establish properties of the rendering map; the experiment here asks whether RGB optimization actually reaches fields whose appearance masks incorrect density. We analyze all $48$ random-start runs of \cref{app:training}: two analytic scenes, $100/400$ training views, $L=0,1,2,3$, and three seeds. The field-function interventions use no warm-start alternative, training update, or color refit; \cref{app:learnedrefit} subsequently refits complete appearance classes at one fixed checkpoint. Each final field was learned from RGB alone. This is a post-hoc mechanism analysis of retained checkpoints: the intervention definitions and thresholds were fixed before their evaluation, but the earlier endpoint RGB and density results were already known.

\subsection{Field-function interventions}
Write $\sigma_*,c_*$ for the analytic truth and $\widehat\sigma,\widehat c$ for a trained field. On identical rays define
\begin{align}
 R_{00}&=R(\sigma_*,c_*),&R_{10}&=R(\widehat\sigma,c_*),\\
 R_{01}&=R(\sigma_*,\widehat c),&R_{11}&=R(\widehat\sigma,\widehat c).
\end{align}
The four renders reuse exactly the same samples and known exterior. Density and color are evaluated as separate frozen field functions; no parameter sharing can silently change the other field. Both the analytic and learned colors satisfy the physical box. Truncating the learned SH expansion retains its DC term and removes complete higher-degree groups, which preserves the sufficient physical box used in training.

Two orders of decomposition give
\begin{align}
 g&=R_{10}-R_{00},&a&=R_{11}-R_{10},\label{eq:comporder1}\\
 a'&=R_{01}-R_{00},&g'&=R_{11}-R_{01}.
\end{align}
Both $g+a$ and $g'+a'$ equal $R_{11}-R_{00}$. We measure the RGB-vector cosines between the two terms in each order. Opposite, large terms with a small sum quantify compensation in the trained endpoint. These finite differences are not a unique differential attribution or an observed optimization trajectory. Their infinitesimal relation is
\[
 \delta R=\int T\big[(c-U)h+\sigma\,\delta c\big]\,\dd s
          =\int T(q-Uh)\,\dd s,
 \qquad q=ch+\sigma\delta c.
\]
The novel measurements are the sizes of the individual terms and geometric errors in trained models, not the algebraic telescoping identity.

The primary set contains all $7{,}936$ retained unknown-intersecting rays from $32$ unseen directions; no exact-background or known-only pixels dilute the metrics. Secondary training measurements use eight uniformly spaced ray indices per training view. Density error uses the original $8{,}192$ independent points in $B_{.75}$. We also measure conditional absorption depth, transmission, and optical thickness. True whole-ray transmission ranges from approximately $.136$ to $.294$ on the test rays, so fully hidden opaque interiors do not explain these measurements.

Before rendering the new interventions, we fixed the following descriptive conjunction: test PSNR at least $40$ dB, relative density error at least $.1$, true-color replacement increasing RGB RMSE at least threefold, and cosine at most $-.9$ in both orders, together with passing the integration audit below. All $36$ directional-color runs pass. Across all $48$ runs, $42$ pass; the six quadratic $L=0$ runs miss the threefold criterion, with ratios $2.49$--$2.95$. These are scene/seed outcomes, not independent population-level statistical tests.

\subsection{Results and angular ablation}
\Cref{tab:trainedcompfull} reports every scene/view/degree cell, averaging all three seeds. Across SH1--SH3, density error is $17.4$--$51.0\%$ although unseen-view PSNR is $46.85$--$57.04$ dB. Replacing learned appearance with true appearance at the same learned density increases image RMSE $4.63$--$22.77$ times. Both-order cosines are below $-.9767$. Keeping learned DC color but removing all directional terms increases heldout RMSE $2.40$--$16.06$ times in every directional-color run. These coefficients support the learned rendering. Matched true-density controls in \cref{app:densityoracle} distinguish angular reliance from density-specific compensation.

\begin{table*}[t]
\centering\small
\caption{All retained random-start groups, three seeds per row. PSNR uses only unknown-intersecting heldout rays. $Q$ is RMSE after replacing learned color with true color divided by learned RMSE; $Q_{\rm DC}$ instead removes all nonconstant SH terms. Depth error is in scene-radius units. Rows average per-run metrics.}\label{tab:trainedcompfull}
\resizebox{\textwidth}{!}{\begin{tabular}{@{}lrrrrrrrr@{}}\toprule
Scene & Views & $L$ & Density err. (\%) & Learned PSNR & True-color PSNR & $Q$ & $Q_{\rm DC}$ & Depth RMSE\\\midrule
Quadratic & 100 & 0 & 11.9 & 58.87 & 50.86 & 2.52 & 1.00 & 0.0071 \\
Quadratic & 100 & 1 & 19.2 & 53.28 & 39.92 & 4.65 & 2.43 & 0.0175 \\
Quadratic & 100 & 2 & 30.0 & 51.98 & 34.22 & 7.73 & 5.31 & 0.0309 \\
Quadratic & 100 & 3 & 38.0 & 49.89 & 31.68 & 8.13 & 3.64 & 0.0403 \\
Quadratic & 400 & 0 & 11.3 & 60.78 & 51.45 & 2.93 & 1.00 & 0.0066 \\
Quadratic & 400 & 1 & 17.5 & 56.95 & 40.44 & 6.69 & 3.63 & 0.0163 \\
Quadratic & 400 & 2 & 29.2 & 56.66 & 34.24 & 13.20 & 8.89 & 0.0306 \\
Quadratic & 400 & 3 & 36.7 & 53.57 & 31.86 & 12.18 & 5.15 & 0.0393 \\
Cosine & 100 & 0 & 37.6 & 45.50 & 29.03 & 6.66 & 1.00 & 0.0329 \\
Cosine & 100 & 1 & 43.5 & 46.89 & 27.23 & 9.61 & 5.57 & 0.0422 \\
Cosine & 100 & 2 & 50.7 & 49.45 & 26.15 & 14.63 & 10.47 & 0.0492 \\
Cosine & 100 & 3 & 48.7 & 47.49 & 26.49 & 11.21 & 7.71 & 0.0468 \\
Cosine & 400 & 0 & 37.3 & 45.72 & 29.08 & 6.79 & 1.00 & 0.0327 \\
Cosine & 400 & 1 & 42.8 & 48.02 & 27.31 & 10.85 & 6.32 & 0.0418 \\
Cosine & 400 & 2 & 50.6 & 53.22 & 26.13 & 22.63 & 15.99 & 0.0493 \\
Cosine & 400 & 3 & 49.0 & 51.03 & 26.42 & 17.01 & 11.20 & 0.0473 \\
\bottomrule\end{tabular}}
\end{table*}

For the quadratic $400$-view scene, learned SH2 attains $56.66$ dB with $29.2\%$ density error; true-color replacement reduces PSNR to $34.24$ dB. Its absorption-depth RMSE is $.0306$ of the scene radius and its transmission RMSE is $.0628$. These are consequences for quantities used outside RGB view synthesis. The cosine control also exhibits strong compensation, including degree-zero compensation. Its density error decreases slightly from SH2 to SH3. Therefore neither training error nor this intervention is a universal monotone function of $L$, and neither establishes a reference-specific SH2 transition.

\begin{figure*}[t]
\centering\includegraphics[width=.93\textwidth]{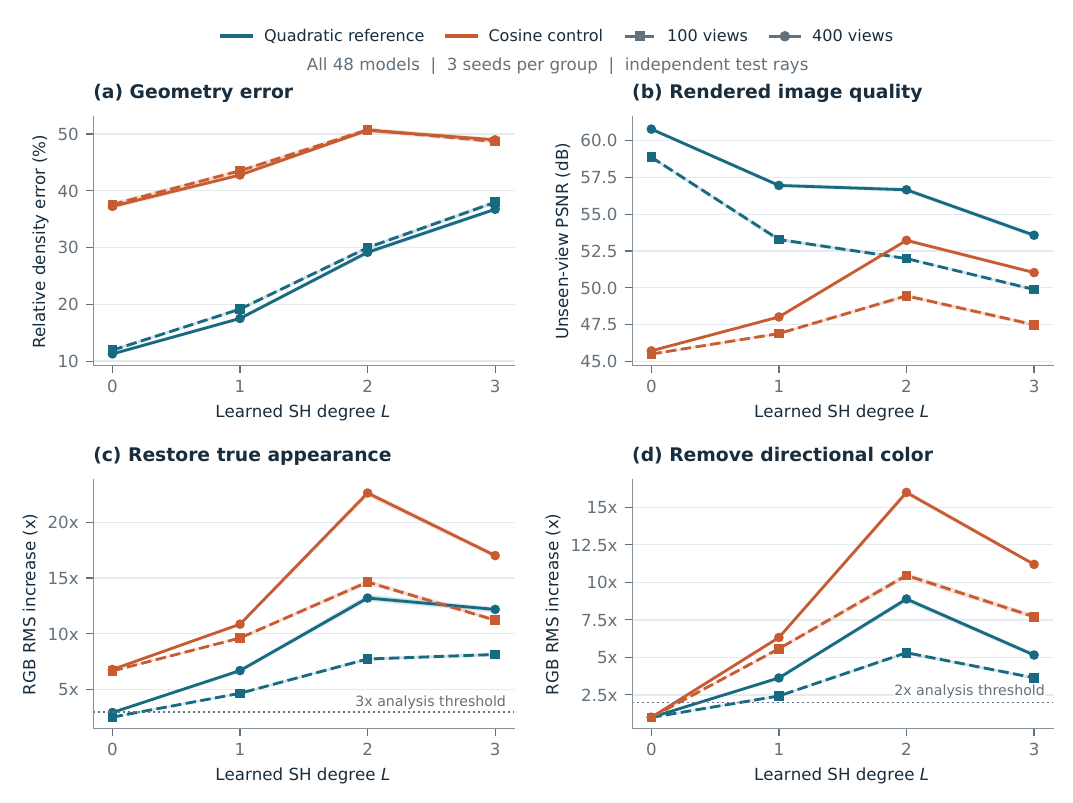}
\caption{Every trained group under field-function interventions. Lines average three seeds; narrow bands give the full seed range. Good unseen-view RGB coexists with incorrect density, and reverting appearance exposes a larger image error. Removing directional terms at fixed learned density also exposes error. The control shares the phenomenon, so this is evidence of appearance compensation in training, not of a quadratic-specific SH2 transition.}\label{fig:compensationsummary}
\end{figure*}

\subsection{Matched appearance training at true density}\label{app:densityoracle}
To separate density compensation from appearance-fitting effects, we repeat color training with density fixed at its exact truth, $\sigma_*=1$. The six controls use the quadratic and cosine scenes, $400$ views, SH1--SH3, and the preselected seed $1101$. Each matches its original joint run's initial color, Adam color-optimizer settings, learning-rate schedule, $384$-ray batch stream and $1{,}500$-step endpoint. The density parameters remain exactly one. No test data enter fitting or endpoint selection. The same nodal constraints enforce physical color in all directions.

\paragraph{A matched intervention contrast.}
Let $e_{\rm full},e_{\rm true},e_{\rm DC}$ denote test RGB RMSE in linear gray levels ($255$ times unit-range RMSE), using learned color, analytic true color, and learned DC color, respectively. Each intervention holds its field's density fixed. The contrast
\[
 \Delta_{\rm true}
 =\big(e_{\rm true}-e_{\rm full}\big)_{\rm joint}
 -\big(e_{\rm true}-e_{\rm full}\big)_{\rm oracle}
\]
subtracts the replacement effect at true density from the effect at learned density. It measures a matched endpoint difference; it does not decompose optimization error into unique causes.

\begin{table}[h]
\centering\small
\caption{All six matched controls. Errors use the same $7{,}936$ unknown-intersecting rays from $32$ unseen directions. Oracle density is exact; joint density is the original learned checkpoint. True-color oracle RMSE is below $.00011$ in every row. Positive $\Delta_{\rm true}$ separates the true-color intervention effects.}\label{tab:densityoracle}
\begin{tabular}{@{}lr rrr rrr r@{}}\toprule
 & & \multicolumn{3}{c}{True density: color-only fit} & \multicolumn{3}{c}{Learned density: joint fit} & \\
Scene & $L$ & Full & True & DC & Full & True & DC & $\Delta_{\rm true}$\\\midrule
Quadratic & 1 & 0.198 & $<.00011$ & 1.918 & 0.364 & 2.455 & 1.308 & 2.289 \\
Quadratic & 2 & 0.319 & $<.00011$ & 3.878 & 0.372 & 4.991 & 3.326 & 4.938 \\
Quadratic & 3 & 0.532 & $<.00011$ & 2.803 & 0.533 & 6.561 & 2.756 & 6.560 \\
Cosine & 1 & 0.431 & $<.00011$ & 4.851 & 1.014 & 10.970 & 6.404 & 10.387 \\
Cosine & 2 & 0.480 & $<.00011$ & 7.556 & 0.561 & 12.560 & 8.900 & 12.479 \\
Cosine & 3 & 0.503 & $<.00011$ & 6.307 & 0.720 & 12.143 & 8.022 & 11.926 \\
\bottomrule\end{tabular}
\end{table}

\paragraph{True-color replacement separates the two cases.}
At true density, learned-color RMSE is $.198$--$.532$ gray levels and replacing it with analytic color reduces error below $.00011$. At the matched learned densities, replacement raises error to $2.46$--$12.56$ gray levels. Thus $\Delta_{\rm true}=2.29$--$12.48$ across all six pairs. The learned appearance conceals the rendering effect of erroneous density in the joint endpoints, whereas the true-density controls retain only color-fitting and numerical errors under this intervention.

\paragraph{Directional removal measures angular reliance.}
Even with density exactly correct, removing directional terms raises error $5.27$--$15.74$ times. The corresponding matched DC-increment contrasts range from $-.78$ to $1.50$ gray levels, with three negative and three positive values. Angular coefficients therefore also accommodate finite-budget color fitting and spatial discretization. Directional-removal ratios quantify reliance on those coefficients; they do not independently diagnose density error. The strict capacity gaps after complete refitting at one learned density are established separately in \cref{app:learnedrefit}.

\paragraph{Numerical resolution.}
All metrics above use $256$ midpoint segments. Repeating every test ray with $512$ segments changes any intervention image by at most $.00127$ gray levels in RMS, far below the matched true-color contrasts. An independent archived renderer agrees within $6.67\cdot10^{-16}$ per RGB channel on six prescribed rays per arm at both resolutions. Sampling the analytic true color onto the $32^3$ grid before rendering gives $.0284$--$.0397$ gray levels of error, distinguishing interpolation error from the analytic-color integration floor. All saved predictions and per-view metrics accompany the six fixed-budget controls.

\subsection{Certifying the learned density under full appearance refitting}\label{app:learnedrefit}
The intervention in \cref{m:trainedcompensation}b fixes the first quadratic/$400$-view/SH2 checkpoint (seed $1101$). Here we keep that same learned density, bound the complete physical SH0 and SH1 appearance classes, and construct a physical SH2 refit on its original $32^3$ color grid. Its independently evaluated density error remains $29.4\%$. All $179{,}200$ unknown-intersecting training rays enter the objective. The $7{,}936$ rays from $32$ unseen cameras are used only to evaluate the retained appearances.

\paragraph{The original shared color model.}
For ray $r$, let $\phi_{riv}$ be the trilinear weight of node $v$ at midpoint $i$, $\alpha_{ri}=1-e^{-\widehat\sigma_{ri}\Delta_{ri}}$, and $T_{ri}$ the preceding interior transmission. The $256$-step compositor used for the interventions is
\begin{align}
 F_r(c)&=k_r+\sum_v B_{rv}c_v(d_r),\label{lr:colorlinear}\\
 B_{rv}&=T_{r,\rm near}\sum_iT_{ri}\alpha_{ri}\phi_{riv},\\
 k_r&=E_{r,\rm near}+T_{r,\rm near}T_{r,\rm inside}E_{r,\rm far}.
\end{align}
All density values, samples and exterior terms remain fixed. Since $B_{rv}\geq0$ and $\sum_vB_{rv}=T_{r,\rm near}(1-T_{r,\rm inside})\leq1$, this is the positive shared-parameter compositor of \eqref{q:render}. Unlike the constructed source-grid example in \cref{q:certificate}, it interpolates color before multiplication by sample density. Equation~\eqref{lr:colorlinear} retains that distinction.

\paragraph{Complete physical classes.}
Per channel, write SH1 as $c_v(d)=\tfrac12+u_v+b_v\cdot d$. Its exact all-direction physical constraint is
\begin{equation}
 |u_v|+\norm{b_v}_2\leq\tfrac12.\label{lr:physical}
\end{equation}
SH0 sets $b_v=0$. These node parameters are shared across every ray; trilinear interpolation preserves physicality. With $y=t-k-\tfrac12B\one$, let $K_L$ denote the corresponding centered appearance map and
\[
 \rho_L=\min_{z\in\mathcal C_L}\norm{y-K_Lz}_\infty.
\]
For any feasible $z$ and nonzero dual vector $\lambda$, support-function weak duality \cite{boyd2004} gives
\begin{equation}\label{lr:dual}
 \frac{\lambda^\top y-\tfrac12\sum_v h_L((K_L^\top\lambda)_v)}{\norm\lambda_1}
 \leq\rho_L\leq\norm{y-K_Lz}_\infty,
\end{equation}
where $h_0(g)=|g_0|$ and $h_1(g)=\max(|g_0|,\norm{g_{1:3}}_2)$. The support follows by allocating the radius in \eqref{lr:physical} to either its scalar or vector component. RGB maximum error is the maximum of the three independent channel optima. Thus the lower bound covers every allowed low-degree appearance, irrespective of convergence of the numerical proposals.

\paragraph{Complete physical appearance refits.}
\Cref{m:querytable} reports
\[
 \rho_0\in[1.1228,1.8820]/255,\qquad
 \rho_1\in[0.4418,0.9888]/255.
\]
In particular $\rho_0>1/255>\rho_1$: directional capacity strictly enlarges the attainable compensation at this learned density, even after optimizing all spatial color coefficients. A physical SH2 refit yields $\rho_2\leq 0.4093/255$. Since $0.4093<0.4418$, these bounds also establish $\rho_2<\rho_1$. The retained MSE-trained SH2 point has maximum error $2.8791/255$; its squared-error objective differs from the maximum-error refit.

\paragraph{Physical SH2 construction.}
Write $c_v(d)=\tfrac12+a_v+b_v\cdot d+v_v\cdot\phi_2(d)$, where $\phi_2$ is the real degree-two SH basis divided by $\sqrt5$ under our normalization. The addition theorem gives $\|\phi_2(d)\|_2=1$. Thus
\[
 |a_v|+\|b_v\|_2+\|v_v\|_2\leq\tfrac12
\]
ensures $0\leq c_v(d)\leq1$ for every direction, and trilinear interpolation preserves these bounds. This convex sufficient class contains every physical SH1 field by setting $v_v=0$. Its feasible points therefore upper-bound the complete physical SH2 optimum. We initialize from the retained SH1 fit and its convex mixture with the trained SH2 field, then minimize the same maximum-RGB objective with primal--dual updates \cite{chambolle2011}. Projection shrinks the scalar and two vector group norms by a common simplex threshold. All spatial coefficients remain shared across rays.

\begin{table}[t]
\centering\small
\caption{Independent unseen-view evaluation at the same fixed learned density. Each retained refit was selected solely on training rays. PSNR and maximum error are different objectives.}\label{lr:testtable}
\begin{tabular}{@{}lrr@{}}\toprule
Appearance & PSNR (dB) & Max. RGB (levels)\\\midrule
True color & $34.17$ & $10.9740$\\
Learned DC only & $37.69$ & $9.3612$\\
Physical SH0 refit & $47.10$ & $3.2156$\\
Physical SH1 refit & $52.90$ & $2.0526$\\
Physical SH2 refit & $60.55$ & $1.4589$\\
Retained trained SH2 & $56.73$ & $1.9069$\\\bottomrule
\end{tabular}
\end{table}

\paragraph{Certificate and implementation checks.}
Primal--dual optimization \cite{chambolle2011} produces feasible appearances and dual witnesses; least-squares residuals provide additional dual proposals. There is no density update, regularizer, angular clipping or heldout fitting. An additional $75$ primal--dual iterations at the same density yield the retained SH1 feasible point in $23.82$ CPU seconds. The reported enclosures use outward arithmetic for the finite renderer, the analytic quadratic target and known exterior. A uniform rowwise matrix-error bound transfers \eqref{lr:dual} to the exact saved-scalar compositor; guarded sparse products and exact dyadic dual normalization bound evaluation error. All $32{,}768$ nodal SH1 cones are checked by rational squared inequalities. The SH2 refit uses $500$ bounded primal--dual updates ($594.90$ CPU seconds including setup). Its final full-sphere group bounds are verified by outward square roots at every node and channel after a $1-2^{-30}$ centered contraction. Interval SH evaluation and the original uniform matrix-error transfer enclose all fitting residuals.

Replaying the original checkpoint renderer differs by less than $7.4\cdot10^{-14}$; three prescribed rays also pass an independent $80$-digit scalar evaluation against the residual enclosures. On $1{,}024$ prescribed unseen rays, doubling the midpoint count from $256$ to $512$ changes maximum RGB by at most $0.0018$ levels over all three retained appearances. The certificate concerns the stated finite ray set and compositor; this refinement check separately assesses quadrature sensitivity.

The SH2 candidate additionally passes independent $80$-digit saved-matrix replay on three prescribed rays and the ray with largest enclosed residual. Its separate $256/512$-step check changes maximum RGB by at most $0.0008$ gray levels on the prescribed $1{,}024$ unseen rays. The SH2 upper in \cref{m:querytest} applies to the original fitting-ray compositor; \cref{lr:testtable} reports its independently evaluated unseen views.

\subsection{Numerical checks and interpretation}
The primary reevaluation uses binary64 and $256$ midpoint segments. On $1{,}024$ fixed rays per field, the largest RMS change over all interventions when increasing to $512$ segments is $5.52\cdot10^{-6}$, below $5\%$ of the learned RGB error in every run. Replaying the saved $128$-segment renders differs by at most $1.75\cdot10^{-7}$ per channel; independently written double-precision rendering also agrees. An additional implementation uses SciPy trilinear interpolation and cellwise Gauss integration on $48$ prescribed rays across eight fields. Its $8$- versus $16$-point quadrature differs by at most $1.12\cdot10^{-15}$, while its largest RMS discrepancy from the primary renderer is $5.64\cdot10^{-6}$. These are numerical implementation checks, not interval certificates. Sufficient nodal color bounds lie inside $[.0000999,.999901]$; interpolation preserves physicality.

This experiment connects the compensation mechanism to ordinary finite-budget, RGB-only optimization of the SH-voxel representation used by Plenoxels. It does not claim a reproduction of its sparse implementation or benchmark protocol. No optimizer convergence, all-ray equality, or lower-degree impossibility after refitting follows from a color-removal intervention. The original dispersion-jump and curriculum thresholds remain unmet. The unknown region is $B_{.8}$, beyond the small-support positive theorem, and only final checkpoints are available. \Cref{app:nonradialtraining} adds nonradial density truth and recorded training trajectories; \cref{app:capturedsmoke} examines captured smoke. In particular, this endpoint evidence cannot determine whether longer optimization would remove the compensation.

\begin{figure*}[t]
\centering\includegraphics[width=.92\textwidth]{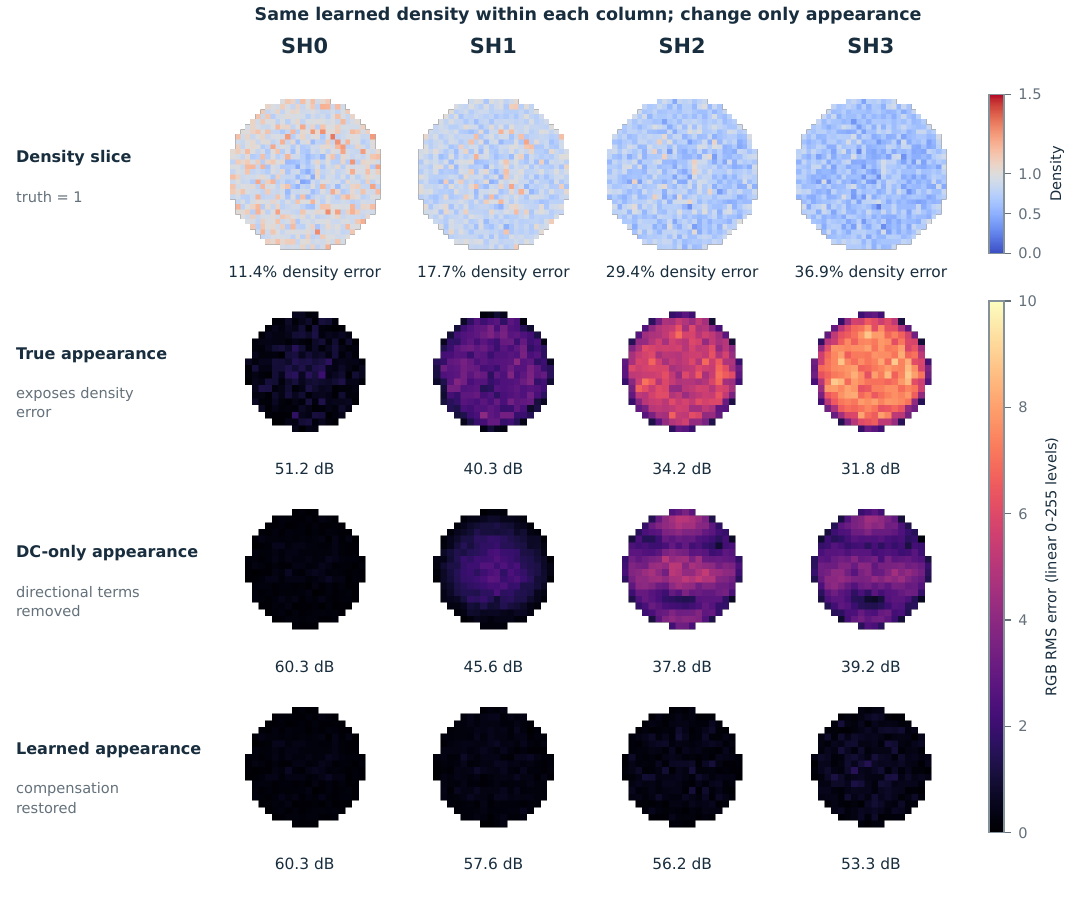}
\caption{Predetermined illustration: quadratic scene, $400$ training views, seed $1101$, first heldout camera; every $L$ is shown. Top: $z=0$ density slice inside the unknown ball, using a common $[0,1.5]$ scale; labels report three-dimensional density error. Remaining rows show per-pixel RGB RMS error after true-color replacement, directional-color removal, and full learned rendering. Labels give PSNR for the displayed view. RGB error uses a common $[0,10]$ range in linear $0$--$255$ levels; neither color scale clips any displayed value. Density is unchanged between the three lower rows in each column.}\label{fig:compensationintervention}
\end{figure*}

\clearpage
\section{Nonradial training trajectories and appearance interventions}\label{app:nonradialtraining}
\subsection{Question and generating field}
The preceding endpoint audit uses radial generating density. Here we ask whether the same compensation mechanism appears during longer RGB training of a nonradial volume. The sampling functions are
\begin{align}
 \sigma_0(x)&=1+\sum_{i=1}^3 a_i\exp(-|x-b_i|^2/w_i^2),\\
 c_0(x)&=\tfrac12+ .22\sin(Ax+\varphi),
\end{align}
with componentwise sine, $a=(.70,.45,.30)$, $w=(.40,.30,.35)$,
\[
 b_1=(.35,.10,-.15),\quad b_2=(-.20,.30,.25),\quad b_3=(.10,-.30,.35),
\]
\[
 A=\begin{pmatrix}4&1&.5\\-.5&3.5&1\\1&-1&4.5\end{pmatrix},
 \qquad\varphi=(.2,.8,-.5).
\]
We sample these functions on a $24^3$ grid in $[-1,1]^3$. Their \emph{trilinear interpolants}, restricted to the unit ball, are the generating truth. Thus spatial representation error is separated from optimization error. Vacuum and black background outside the ball are known; its entire interior is unknown. This shares the sampling functions of \cref{app:circlecertificate}, but interpolation changes the scene, so the earlier numerical $\kappa_L$ constants are not transferred. The piecewise polynomial training model and large unknown support are distinct from the positive $C^6$ small-ball theorem.

The true source is $j_0=\sigma_0c_0$ and is isotropic. Its nodal color is in $[.28,.72]$, and interpolation preserves that range. At the independent geometry points, true density spans approximately $1.000$--$1.686$. Whole-ray test transmittance spans $.0692$--$.7860$, so this experiment retains interior visibility instead of hiding all structure behind an opaque barrier.

\subsection{Frozen training and evaluation protocol}
There are $100$ orthographic training directions and $32$ independently rotated unseen directions. Images have $32^2$ and $24^2$ pixels, respectively, with image half-width $1.04$. We retain all $69{,}200$ training and $12{,}288$ test rays intersecting the unknown ball, excluding black background. Test rays and $8{,}192$ independent uniform geometry points in $B_{.95}$ are never used for optimization or checkpoint selection. Targets use $1{,}024$ midpoint segments with exact constant-segment attenuation. The $2{,}048$-segment check changes RGB RMS by at most $3.85\cdot10^{-8}$.

SH1 and SH3 each allocate and evaluate all $16$ angular coefficients, applying a degree mask for SH1. They use identical initial density/DC coefficients and precomputed ray-batch sequences for each seed. Density is initialized uniformly in $[.5,1.5]$; DC is $.4$ plus Gaussian noise of standard deviation $.03$, with all higher modes initially zero. Each update uses $96$ rays and $96$ segments. Adam uses density/color learning rates $.02/.005$, decays by a factor ten over the first $1{,}500$ steps, and retains those rates through step $6{,}000$. There is no TV. The same sufficient nodal physical projection as the earlier study enforces bounded color. The true-color oracle freezes the true nodal color and optimizes density from the same initialization; it uses extra information and is not a practical RGB-only baseline.

States are saved at steps $0,500,1{,}500,3{,}000,6{,}000$. All endpoint interventions in \cref{app:trainedcompensation} are repeated in double precision, using $256$ segments and a $512$-segment check on $1{,}024$ fixed unseen rays. The predeclared numerical tolerance is $\max(5\cdot10^{-6},.05\,{\rm RMSE})$. Metrics include density error relative to its true $L^2$ norm, conditional absorption-depth error, transmittance error, true-color replacement, and removal of nonconstant SH modes. These metrics remain ordinary numerical measurements, not interval certificates.

The fixed CPU budget completed $6$ of the planned nine runs. The paired complete cohort has $2$ seeds per directional arm and $2$ oracle runs. The time cap was applied without extending it after observing outcomes. Partial checkpoints and unstarted cells remain in the released coverage record; they are not counted as final-step replications. The protocol, camera rays, batches, checkpoints, and all results are retained.

\begin{figure}[t]
\centering\includegraphics[width=.99\textwidth]{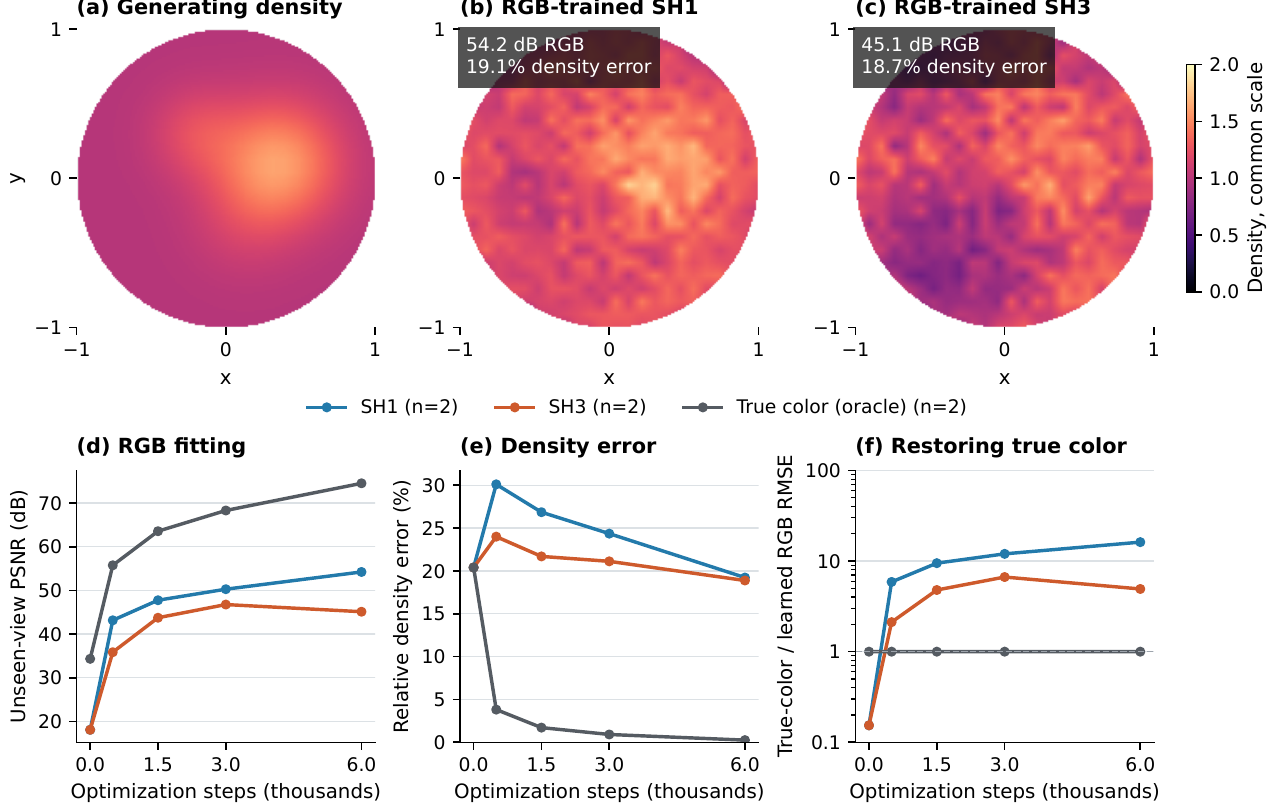}
\caption{\textbf{Appearance compensation beyond radial density.} Top: generating density and final SH1/SH3 density slices at $z=0$, using the first declared seed and one color scale. Bottom: complete-run means and full seed ranges for unseen RGB, density error, and the image error revealed by restoring true color. The oracle supplies true appearance during training. Finite optimization is nonmonotone; these trajectories do not establish convergence or universal deterioration with SH degree.}\label{fig:nonradialtraining}
\end{figure}

\subsection{Complete outcomes and finite-difference trajectories}
\begin{table}[t]
\centering\small
\caption{Every completed nonradial run at step $6{,}000$. Density error is a percentage; depth is in scene-radius units; $Q$ replaces learned color with true color and $Q_{\rm DC}$ removes only directional coefficients. The final column is fixed training-monitor MSE at $6{,}000$ divided by its value at $3{,}000$, not a convergence certificate.}\label{tab:nonradialtraining}
\begin{tabular}{@{}llrrrrrrr@{}}\toprule
Arm & Seed & PSNR & Density err. & Depth & Transmission & $Q$ & $Q_{\rm DC}$ & MSE ratio\\\midrule
SH1 & 5101 & 54.18 & 19.14 & 0.0320 & 0.0570 & 16.00 & 11.31 & 0.406 \\
SH3 & 5101 & 45.09 & 18.70 & 0.0365 & 0.0458 & 4.89 & 6.43 & 2.050 \\
Oracle color & 5101 & 74.40 & 0.26 & 0.0004 & 0.0004 & 1.00 & 1.00 & 0.286 \\
SH1 & 6202 & 54.26 & 19.30 & 0.0323 & 0.0578 & 16.32 & 11.47 & 0.451 \\
SH3 & 6202 & 45.14 & 19.05 & 0.0366 & 0.0462 & 4.95 & 6.51 & 1.980 \\
Oracle color & 6202 & 74.65 & 0.24 & 0.0003 & 0.0003 & 1.00 & 1.00 & 0.277 \\
\bottomrule\end{tabular}
\end{table}
Of the $4$ completed directional runs, $4$ satisfy the full descriptive conjunction at \emph{both} $3{,}000$ and $6{,}000$: at least $40$ dB test PSNR, at least $15\%$ density error, at least threefold true-color replacement error, at least $1.5$-fold DC-only error, cosines at most $-.9$ in both endpoint exchange orders, and the integration check. These four outcomes do not cover the two uncompleted directional runs. Final endpoint cosines are at most $-0.9740$.

The final directional models have transmittance RMS error $0.0458$--$0.0578$ and absorption-depth RMS error $0.0320$--$0.0366$. Good RGB therefore coexists with errors in density-dependent quantities. The oracle mean density error is $0.25\%$. This comparison calibrates the role of appearance information; it does not prove a particular regularizer or low-degree schedule will attain the oracle.

For two consecutive recorded states $A=(\sigma_A,c_A)$ and $B=(\sigma_B,c_B)$ we also evaluate
\[
 g=R(\sigma_B,c_A)-R(A),\qquad a=R(B)-R(\sigma_B,c_A),
\]
and the reverse-order decomposition, on the same $800$ training-monitor and $1{,}024$ unseen rays. Both sums equal $R(B)-R(A)$. We record the two cosines, individual term norms, net change, density drift, and change in density error. In the final SH1 interval, both-order cosines are about $-.95$, while density error improves by approximately $5.1$ percentage points. The corresponding SH3 cosines are only $-.47$ to $-.53$. Thus strong endpoint compensation need not imply that every update cancels; cancelling increments can even accompany correction \emph{toward} true density. All interval values are in the release.

The directional fixed-monitor MSE ratios from step $3{,}000$ to $6{,}000$ span $0.406$--$2.050$. Only $0$ completed runs meet the conservative diagnostic of less than $1\%$ monitor-MSE change and at most $.01$ relative density drift. Even meeting that diagnostic would not establish convergence. Neither the exact radial ambiguity nor the local stability theorem predicts a universal ordering of these optimization endpoints.

Independent SciPy trilinear interpolation and cellwise Gauss quadrature check six preselected rays from every final field. Two Gauss points integrate the cubic within-cell density exactly in real arithmetic; emission uses $8/16$-point quadrature. Their maximum discrepancy is $1.3e-15$, and the largest RGB RMS difference from the primary integrator is $4.3e-06$. This independent numerical check passes the declared tolerance. All floating-point and integration limits are kept separate from the exact-arithmetic certificates elsewhere in the paper.

\subsection{Longer matched GPU trajectories}\label{app:gpulong}
The longer cohort repeats all three arms from their common initialization for three seeds, using the same generating volume, camera splits, and first $6{,}000$ ray-batch sequence. Adam retains the original post-decay learning rates through step $30{,}000$. Every run completes its prescribed steps. The cohort is evaluated in float64 on all $12{,}288$ held-out rays and the same $8{,}192$ geometry points; it contains nine runs and seven retained checkpoints per run.

\begin{figure}[t]
\centering\includegraphics[width=\textwidth]{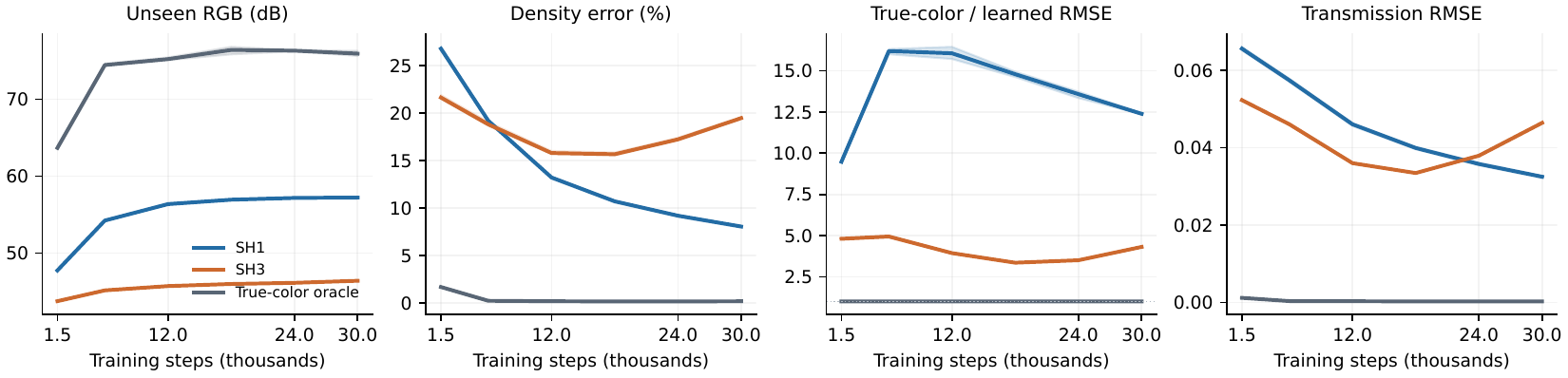}
\caption{\textbf{Longer optimization separates density improvement from appearance compensation.} Lines and bands show means and full ranges over all three seeds. SH1 density error decreases while its learned appearance continues to mask the remaining error. SH3 retains a larger density discrepancy. The true-color oracle supplies generating appearance during training. Curves use the prescribed checkpoints, with no endpoint selection.}\label{fig:gpulong}
\end{figure}

\begin{table}[t]\centering\small
\caption{Every run at $30{,}000$ steps. Density is relative $L^2$ error in percent; transmission and depth are RMS errors. $Q$ replaces learned color by true color, and $Q_{\rm DC}$ removes directional terms. The final column is monitor MSE at $30{,}000$ divided by its value at $24{,}000$.}\label{tab:gpulong}
\begin{tabular}{@{}llrrrrrrr@{}}\toprule
Arm & Seed & PSNR & Density & Trans. & Depth & $Q$ & $Q_{\rm DC}$ & MSE ratio\\\midrule
SH1 & 5101 & 57.35 & 7.95 & 0.0322 & 0.0138 & 12.44 & 15.45 & 0.915 \\
SH3 & 5101 & 46.46 & 19.34 & 0.0464 & 0.0362 & 4.34 & 7.17 & 0.957 \\
Oracle color & 5101 & 76.31 & 0.20 & 0.0003 & 0.0003 & 1.00 & 1.00 & 0.989 \\
SH1 & 6202 & 57.14 & 8.13 & 0.0327 & 0.0141 & 12.37 & 15.19 & 0.969 \\
SH3 & 6202 & 46.36 & 19.56 & 0.0463 & 0.0361 & 4.29 & 7.09 & 0.780 \\
Oracle color & 6202 & 75.64 & 0.21 & 0.0003 & 0.0003 & 1.00 & 1.00 & 1.131 \\
SH1 & 7303 & 57.20 & 8.12 & 0.0325 & 0.0140 & 12.37 & 15.26 & 1.042 \\
SH3 & 7303 & 46.40 & 19.49 & 0.0465 & 0.0363 & 4.31 & 7.16 & 0.992 \\
Oracle color & 7303 & 75.88 & 0.20 & 0.0003 & 0.0003 & 1.00 & 1.00 & 1.101 \\
\bottomrule\end{tabular}\end{table}

At the common $30{,}000$-step endpoint, SH1 density error is $7.95$--$8.13\%$, while SH3 is $19.34$--$19.56\%$. Across all six directional runs, unseen PSNR is $46.36$--$57.35$ dB, and true-color replacement multiplies RGB RMSE by $4.29$--$12.44$. The oracle mean density error is $0.204\%$. More optimization therefore improves density substantially in SH1 while endpoint compensation persists in both angular classes.

\paragraph{Training-monitor comparison.}
The fixed monitor contains $800$ training rays out of $69{,}200$; it is evaluated at each saved checkpoint independently of the stochastic training batch. The entries below are recomputed from saved monitor RGB at step $30{,}000$. MSE is a monitor estimate, not a whole-training-set evaluation.
\begin{center}\small\begin{tabular}{@{}lrrrr@{}}\toprule
Arm & Seed & Monitor MSE & Monitor PSNR & Test PSNR\\\midrule
SH1 & 5101 & $1.501\times10^{-6}$ & 58.24 & 57.35 \\
SH3 & 5101 & $1.582\times10^{-5}$ & 48.01 & 46.46 \\
True-color oracle & 5101 & $3.452\times10^{-8}$ & 74.62 & 76.31 \\
SH1 & 6202 & $1.686\times10^{-6}$ & 57.73 & 57.14 \\
SH3 & 6202 & $1.616\times10^{-5}$ & 47.91 & 46.36 \\
True-color oracle & 6202 & $3.558\times10^{-8}$ & 74.49 & 75.64 \\
SH1 & 7303 & $1.619\times10^{-6}$ & 57.91 & 57.20 \\
SH3 & 7303 & $1.708\times10^{-5}$ & 47.68 & 46.40 \\
True-color oracle & 7303 & $3.298\times10^{-8}$ & 74.82 & 75.88 \\
\bottomrule\end{tabular}\end{center}
SH3 has larger residual on both monitor and test rays. The degree comparison therefore combines distinct optimization outcomes with distinct appearance classes; it does not isolate a causal capacity effect on final density. The interventions within each fixed field establish persistent endpoint compensation. The capacity test that holds density and data objective fixed is instead \cref{app:learnedrefit}. The true-color oracle trains density with appearance known; it is a different control from training appearance at known density.

The preregistered persistence conjunction is tested at both $24{,}000$ and $30{,}000$: $40$ dB RGB, $15\%$ density error, threefold true-color error, $1.5$-fold DC-only error, both exchange cosines at most $-.9$, and physical/integration checks. It holds for 3 of six directional runs. The SH1 endpoints fall below the $15\%$ density threshold; their retained compensation ratios describe the remaining error. 0 of nine runs meet the separate $1\%$ monitor-MSE and $.01$ density-drift plateau diagnostic. These fixed-budget trajectories distinguish endpoint compensation from optimizer convergence.

All $63$ checkpoints pass the recorded physical and $256/512$-segment checks. Recomputing metrics from saved predictions and independently interpolating the density reproduces every reported value. Independent cellwise density integration and $8/16$-point Gauss emission quadrature check six predetermined rays from each final field. All nine final fields pass the original numerical tolerance. This numerical audit is separate from the analytic all-ray guarantees.

The first declared seed supplies three prescribed perspective cameras and a complete $36$-frame orbit. RGB panels share the same sRGB display transform; opacity and depth share scales within every comparison. Unmodified linear arrays and all cameras are retained. The perspective images use $512$ segments and the orbit uses $256$; the tables use the separately specified held-out evaluation. The RTX 4070 Ti cohort is a separate CUDA replay, whose backward operations need not match the earlier CPU trajectories bitwise.

\subsection{Appearance refitting on captured smoke}\label{app:capturedsmoke}
The ScalarReal example distributed with PINF~\cite{chu2022} contains five synchronized, preprocessed videos of real ScalarFlow smoke~\cite{eckert2019}. We use its calibrated camera poses and volume transform. Cameras $0,2,4$ provide fitting images, camera $1$ is reserved for validation, and camera $3$ provides the reported test images. All settings and frames are fixed before reconstruction; validation does not select parameters. The experiment uses frames $40,60,80$ of a $120$-frame, $30$ Hz capture. Images are area-downsampled to $96\times54$, then averaged across decoded channels. Pixel values remain on their supplied $[0,1]$ encoded-intensity scale.

\paragraph{A common density perturbation.}
We fit a nonnegative $16\times24\times16$ trilinear reference density from the three fitting cameras with unit color and black background. The optical-depth objective uses $-\log(1-y)$, mean squared neighboring-density differences weighted by $10^{-6}$, and density bounds $[0,100]$. L-BFGS-B has a $300$-iteration cap. This reference is an image-fitted field; no reconstructed density or held-out image supervises it. We then compare the reference and its prescribed $1.5\times$ scaling. At each fixed density, $600$ projected-gradient iterations per degree refit one spatially shared appearance field, with a nested feasible warm start. The color-held-fixed baseline uses unit color. The same refitting procedure at the unscaled reference provides a control for improvement unrelated to the prescribed density change.

At every node, the appearance is a constant, a linear directional term and five degree-two terms. We impose
\begin{equation}\label{smoke:physical}
 |a_0-\tfrac12|+\|a_{1:3}\|_2+\sum_{k=4}^{8}|a_k|\le\tfrac12.
\end{equation}
Here the five quadratic basis functions are $2d_xd_y,2d_xd_z,2d_yd_z,d_x^2-d_y^2,(3d_z^2-1)/2$, each bounded by one on the sphere. The linear contribution is at most $\|a_{1:3}\|_2$. Thus \cref{smoke:physical} gives $0\le c\le1$ in every direction; trilinear interpolation preserves this bound. It describes the complete nodal physical SH1 class and a sufficient SH2 subclass containing it. Backtracking checks every accepted quadratic descent step, and the inherited candidate is retained if better. The reported fits are feasible solutions of this finite optimization problem.

\begin{table}[t]\centering\small
\caption{Test camera PSNR (dB) for every prescribed time and density scaling. ROI covers rays through the calibrated volume. The final column restricts SH0 to observed foreground intensity $>.02$, exposing the contribution of dark background to ROI PSNR. $\dagger$ denotes an unseen time, interpolated between fitted fields.}\label{tab:capturedsmoke}
\begin{tabular}{@{}llrrrrr@{}}\toprule
Frame & Density scale & Fixed color & SH0 & SH1 & SH2 & SH0 foreground\\\midrule
40 & 1.0 & 29.45 & 29.92 & 30.11 & 30.08 & 19.70 \\
40 & 1.5 & 26.15 & 31.14 & 30.82 & 30.64 & 20.93 \\
60 & 1.0 & 29.14 & 29.60 & 29.74 & 29.72 & 21.14 \\
60 & 1.5 & 25.20 & 30.29 & 30.03 & 29.92 & 21.72 \\
80 & 1.0 & 28.38 & 28.71 & 28.87 & 28.85 & 21.24 \\
80 & 1.5 & 24.49 & 29.45 & 29.16 & 29.15 & 21.88 \\
50$^\dagger$ & 1.0 & 27.01 & 26.94 & 27.03 & 27.00 & 18.04 \\
50$^\dagger$ & 1.5 & 24.54 & 26.97 & 26.89 & 26.81 & 18.50 \\
70$^\dagger$ & 1.0 & 27.60 & 27.61 & 27.83 & 27.84 & 19.83 \\
70$^\dagger$ & 1.5 & 24.30 & 27.77 & 27.68 & 27.61 & 20.18 \\
\bottomrule\end{tabular}\end{table}

\paragraph{What the real images establish.}
At $1.5\widehat\sigma$, SH0 refitting reduces test-camera ROI RMSE by $43.5$--$44.3\%$ relative to color held fixed, across all three fitted times. Test PSNR rises from $24.49$--$26.15$ to $29.45$--$31.14$ dB. The maximum difference in predicted transmission is $.14815$ at the same rays and densities. On fitting rays, SH0, SH1 and SH2 RMS residuals span $.01732$--$.02048$, $.01329$--$.01516$ and $.01245$--$.01430$, respectively. The test camera instead favors SH0 over the higher degrees at the perturbed density. Thus this capture exhibits substantial appearance compensation, primarily through isotropic refitting; the extra fitting capacity has a separate generalization behavior. \Cref{tab:capturedsmoke} retains the full ordering and the unscaled-density controls.

\paragraph{Time interpolation.}
Frames $50$ and $70$ enter no fitting objective. Their density and appearance coefficients are arithmetic averages of the two adjacent fitted fields. At scaled density, SH0 yields $26.97$ and $27.77$ dB on camera $3$, compared with $24.54$ and $24.30$ dB under fixed color. This evaluates intermediate observations of a real fluid sequence through field interpolation.

\paragraph{Measurement scope and numerical checks.}
The experiment uses a compressed, externally illuminated capture with a discretized emission--absorption image model. The $50\%$ change is relative to the fitted reference; independent physical density, color, transmission and velocity truth are unavailable. The time test measures interpolation, rather than a velocity or fluid-dynamics model. Coarse spatial resolution and image formation contribute to the remaining foreground error. All cameras, full-frame and foreground metrics, coefficient fields and predicted images are retained. All $24$ fields satisfy the nodal physical bound. An independent SciPy interpolator and scalar transport loop agree on $153$ fixed rays to $5.6\times10^{-16}$. Refining $96$ to $192$ segments on all $19{,}383$ volume-intersecting rays changes RGB by at most $0.00193$ (maximum field RMS $0.000145$). The complete reconstruction and refitting pilot takes $67.3$ CPU seconds with two threads and $1.00$ GiB peak resident memory.

\paragraph{Finer evaluation and main-paper visualization.}
All fitted density and appearance coefficients above are held unchanged and evaluated on $384\times216$ pixels in every camera and every prescribed frame. Original decoded video pixels are area-downsampled for quantitative comparison. The main-paper image uses the same middle frame and test camera as the coarse evaluation, with a common expanded crop retaining the plume head, with the native-resolution captured image displayed alongside numerical predictions. This changes the evaluated ray density, while the fitted spatial grid and optimization budget remain fixed. No image enhancement, exposure fitting or per-panel normalization is applied.

\begin{table}[t]\centering\small
\caption{Finer evaluation of the same frozen fields: test camera PSNR (dB), at $384\times216$. The final column uses observed foreground intensity $>.02$. $\dagger$: time withheld from every fitting objective.}\label{tab:smokehighres}
\begin{tabular}{@{}llrrrrr@{}}\toprule
Frame & Density scale & Fixed color & SH0 & SH1 & SH2 & SH0 foreground\\\midrule
40 & 1.0 & 28.52 & 29.09 & 29.19 & 29.17 & 19.07 \\
40 & 1.5 & 25.67 & 30.53 & 30.29 & 30.05 & 20.65 \\
60 & 1.0 & 27.92 & 28.51 & 28.56 & 28.53 & 20.26 \\
60 & 1.5 & 24.65 & 29.48 & 29.23 & 29.09 & 21.27 \\
80 & 1.0 & 27.44 & 27.92 & 28.00 & 27.99 & 20.55 \\
80 & 1.5 & 24.06 & 28.87 & 28.66 & 28.60 & 21.50 \\
50$^\dagger$ & 1.0 & 26.20 & 26.26 & 26.31 & 26.27 & 17.51 \\
50$^\dagger$ & 1.5 & 24.02 & 26.36 & 26.29 & 26.22 & 18.14 \\
70$^\dagger$ & 1.0 & 26.38 & 26.57 & 26.70 & 26.70 & 19.00 \\
70$^\dagger$ & 1.5 & 23.63 & 26.85 & 26.78 & 26.68 & 19.63 \\
\bottomrule\end{tabular}\end{table}

\begin{figure}[t]\centering\includegraphics[width=\textwidth]{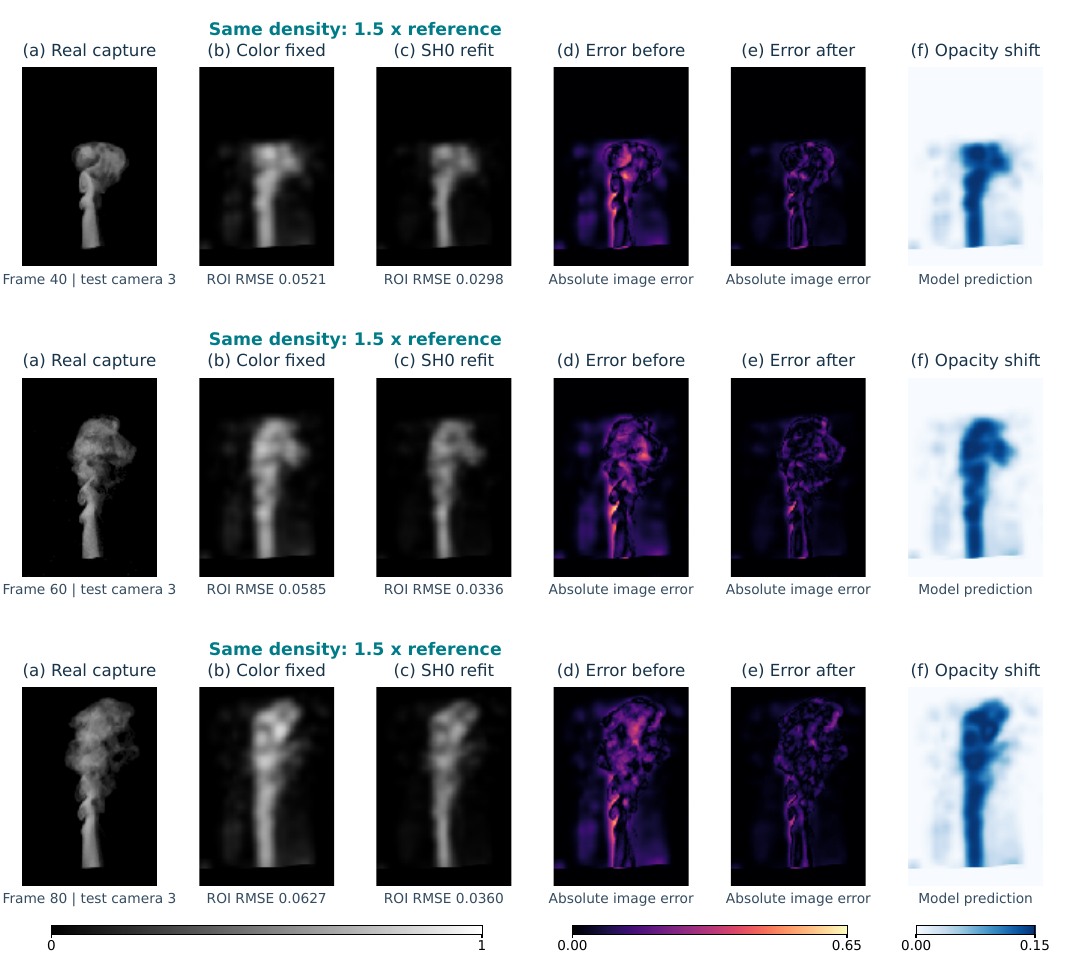}
\caption{The three prescribed fitting times, all evaluated from held-out camera $3$, using unchanged fields and finer rays. Columns show capture, unit-color and SH0 predictions, their absolute errors, and opacity change; all rows share the displayed scales. The middle capture retains native video pixels and the outer captures show the $384\times216$ observations used for evaluation.}\label{fig:smokehighresall}
\end{figure}

At the perturbed density, SH0 reduces test ROI RMSE by $42.5$--$42.9\%$. \Cref{tab:smokehighres} reports every degree and the unscaled-density controls. Higher degrees continue to lower fitting residuals without surpassing SH0 on the held-out camera at the perturbed density. The $97$ predetermined domain rays per frame agree with an independent SciPy interpolator and scalar transport loop to $8.9e-16$. Their $96$- to $192$-segment refinement changes intensity by at most $0.00097$. The original low-resolution predictions are independently reproduced to $5.6e-16$. Finer evaluation uses $91.6$ CPU seconds, two threads and $0.49$ GiB peak resident memory, with no fitting updates.

\paragraph{Density-scale profiles and reciprocal-color controls.}\phantomsection\label{smoke:scaleprofile}
We freeze the same three reference densities and evaluate all scales $s\in\{.5,.75,1,1.5,2,3\}$. At each scale, physical SH0--SH2 refits use the original cameras, grid and $600$-iteration budget; coefficients at $s=1,1.5$ are reused exactly. The other scales start from unit color, with nested degree warm starts and no cross-scale warm start. All $36$ new fits finish before held-out evaluation. Color-fixed controls use both unit color and the saved SH0 field at $s=1$. Reciprocal controls divide each by $s$ and clip to $[0,1]$; the unmodified controls are physical for $s\geq1$. All five cameras are evaluated at $384\times216$ pixels.

Reciprocal scaling is a thin-medium approximation. For cumulative optical depth $\tau(t)$ and $j=\sigma c$, scaling $(\sigma,c)$ to $(s\sigma,c/s)$ keeps $j$ fixed but changes the rendering by
\begin{equation}\label{smoke:reciprocal}
 R_s-R_1=\int\bigl(e^{-s\tau(t)}-e^{-\tau(t)}\bigr)j(t)\,dt.
\end{equation}
For $s\geq1$, its magnitude is at most $(s-1)\int\tau(t)j(t)\,dt$, by the mean-value bound on the exponential. For unit reference color, the scaled-color rendering is $(1-e^{-s\tau_{\rm tot}})/s$; this follows by integrating the derivative of $e^{-s\tau(t)}$. This control preserves the leading optically thin term, whereas spatial SH0 refitting can also compensate residual structure. Exact compensation in \cref{m:exactcomp} generally depends on direction through $U$ and is a distinct identity. PINF discusses color--density ambiguity in this real capture~\cite[Section~3.4]{chu2022}.

\begin{figure}[t]\centering\includegraphics[width=\textwidth]{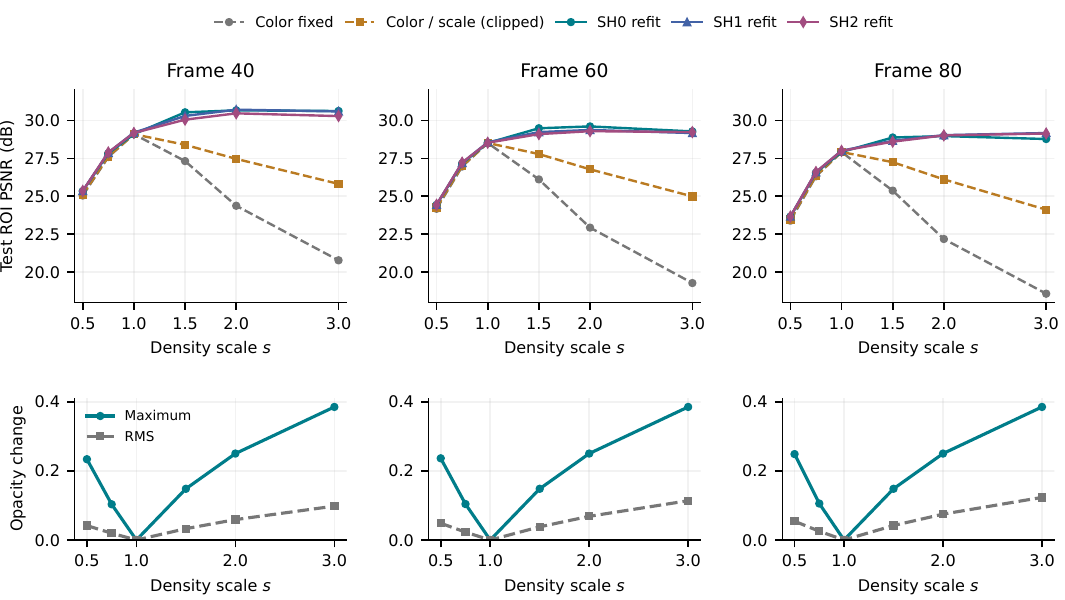}
\caption{\textbf{Density-scale sensitivity in all three captured-smoke frames.} Top: held-out-camera PSNR after the five displayed color treatments. Fixed color is the SH0 field fitted at $s=1$; the reciprocal control is clipped only where it would exceed the physical box. Bottom: predicted opacity changes relative to $s=1$, independent of color. Markers are the six evaluated scales; connecting segments guide the eye.}\label{fig:smokescaleall}
\end{figure}

\begin{table}[t]\centering\small
\caption{Every density scale on the held-out camera at finer resolution. Fixed and reciprocal colors use the saved $s=1$ SH0 field. FG denotes SH0 PSNR on observed foreground intensity $>.02$. All values are dB; density scales are relative to a fitted reference.}\label{tab:smokescales}
\begin{tabular}{@{}rrrrrrrr@{}}\toprule
Frame & $s$ & Fixed & Reciprocal & SH0 & SH1 & SH2 & FG\\\midrule
40 & 0.5 & 25.05 & 25.11 & 25.34 & 25.36 & 25.36 & 14.04 \\
40 & 0.75 & 27.59 & 27.62 & 27.83 & 27.88 & 27.89 & 17.23 \\
40 & 1 & 29.09 & 29.09 & 29.09 & 29.19 & 29.17 & 19.07 \\
40 & 1.5 & 27.32 & 28.39 & 30.53 & 30.29 & 30.05 & 20.65 \\
40 & 2 & 24.38 & 27.46 & 30.66 & 30.70 & 30.46 & 20.79 \\
40 & 3 & 20.78 & 25.82 & 30.62 & 30.59 & 30.28 & 20.90 \\
60 & 0.5 & 24.18 & 24.26 & 24.44 & 24.45 & 24.45 & 14.90 \\
60 & 0.75 & 26.96 & 26.99 & 27.16 & 27.20 & 27.22 & 18.34 \\
60 & 1 & 28.51 & 28.51 & 28.51 & 28.56 & 28.53 & 20.26 \\
60 & 1.5 & 26.11 & 27.78 & 29.48 & 29.23 & 29.09 & 21.27 \\
60 & 2 & 22.94 & 26.78 & 29.60 & 29.37 & 29.29 & 21.33 \\
60 & 3 & 19.29 & 24.99 & 29.27 & 29.16 & 29.23 & 20.96 \\
80 & 0.5 & 23.40 & 23.48 & 23.66 & 23.67 & 23.67 & 15.23 \\
80 & 0.75 & 26.31 & 26.35 & 26.56 & 26.59 & 26.60 & 18.72 \\
80 & 1 & 27.92 & 27.92 & 27.92 & 28.00 & 27.99 & 20.55 \\
80 & 1.5 & 25.37 & 27.25 & 28.87 & 28.66 & 28.60 & 21.50 \\
80 & 2 & 22.19 & 26.11 & 28.96 & 29.03 & 29.01 & 21.59 \\
80 & 3 & 18.59 & 24.12 & 28.77 & 29.17 & 29.13 & 21.41 \\
\bottomrule\end{tabular}\end{table}

Across $s=1.5,2,3$, SH0 test PSNR spans $.13$, $.33$ and $.19$ dB for frames $40,60,80$; the saved-color curves decrease by $6.53$, $6.83$ and $6.78$ dB. SH0 refits exceed reciprocal-color controls by $1.62$--$4.80$ dB over these nine comparisons. Predicted maximum opacity changes reach $.14815$, $.25$ and $.38490$ at scales $1.5$, $2$ and $3$. SH1--SH2 lower the fitting objective, while their test ordering varies with time and scale. These profiles separate spatial isotropic compensation from angular fitting capacity.

The reference is fitted with unit color and is not independent density truth. The sampled SH0 plateau describes test-image sensitivity, not a noise-calibrated density interval, a continuum ambiguity proof, or identification of the true scale. The all-ray theoretical constructions and complete-class learned-density certificates provide those separate information-limit statements. The reference optical-depth 99th percentiles are $.612$, $.640$, $.683$ and maxima $.990$, $1.026$, $1.238$ across volume-intersecting evaluation rays; the darkest rays do not imply that every smoke ray is optically thin. The raw records retain full-frame, ROI and foreground metrics, all controls, physical-box checks and color-clipping fractions.

\paragraph{Why the scale profile is asymmetric.}\phantomsection\label{smoke:ceiling}
Unit color places the fitted reference at the brightness ceiling. For zero background, any physical color field $0\leq c\leq1$ satisfies
\begin{equation}\label{smoke:brightnessbound}
 R=\int T\sigma c\,dt\leq\int T\sigma\,dt=1-e^{-\tau},\qquad
 \tau=\int\sigma\,dt.
\end{equation}
The last equality follows by integrating $-\partial_t T=T\sigma$. Thus a noise-free ray intensity $y<1$ requires $\tau\geq-\log(1-y)$. With a stated absolute reading bound $b$, replace $y$ by $\max(y-b,0)$. This constrains integrated optical thickness, not pointwise density.

At a fixed scaled density, the ceiling gives a residual bound independently of angular degree or spatial color sharing:
\begin{equation}\label{smoke:ceilingresidual}
 \operatorname{MSE}\geq\frac1{|\mathcal I|}\sum_{r\in\mathcal I}
       [y_r-(1-T_{s,r})]_+^2.
\end{equation}
Indeed, each admissible prediction is at most $1-T_{s,r}$, so any positive brightness shortfall is unavoidable; squaring and averaging proves the inequality. On the saved test ROI, this lower bound accounts for $93.8$--$95.3\%$ of the achieved SH0 MSE at $s=.5$, and $73.7$--$78.1\%$ at $s=.75$. The corresponding PSNR ceilings and achieved values are:
\begin{center}\small
\begin{tabular}{@{}rrrrr@{}}\toprule
Frame & $s$ & PSNR ceiling & SH0 PSNR & Forced MSE fraction (\%)\\\midrule
40 & .50 & 25.62 & 25.34 & 93.8\\
60 & .50 & 24.70 & 24.44 & 94.1\\
80 & .50 & 23.87 & 23.66 & 95.3\\
40 & .75 & 29.12 & 27.83 & 74.3\\
60 & .75 & 28.49 & 27.16 & 73.7\\
80 & .75 & 27.63 & 26.56 & 78.1\\\bottomrule
\end{tabular}
\end{center}
The table reuses all saved predictions and involves no additional fit. It explains the low-scale decline through a physical color constraint. Higher-scale refitting improves held-out fit over $s=1$, which locates a better image fit within this model; it does not identify the physical density scale. The one-sided optical-thickness inequality itself supplies no upper bound. An unbounded compatible set under a shared finite-degree appearance model would require an additional construction, and does not follow from these six sampled scales.

\paragraph{Scale-scan checks and cost.}
All $1{,}890$ metric rows are independently recomputed from stored predictions. Constant-color telescoping and $T_s=T_1^s$ are checked on every evaluation ray. Every nodal color meets \cref{smoke:physical} to floating-point tolerance $10^{-12}$. The $126$ field checks agree with an independent scalar renderer to $1.0\times10^{-15}$ on $97$ fixed rays per frame; doubling segments changes those intensities by at most $.001163$. A floating-point Frank--Wolfe gap bounds each quadratic fitting objective diagnostically: the largest gap is $7.35\times10^{-7}$, at most $.461\%$ of the corresponding objective. This is separate from interval-certified refit bounds. Existing predictions at $s=1,1.5$ replay identically and all frozen input hashes match. The scan costs $236.6$ seconds on two CPU threads, peaks at $.680$ GiB, and uses no GPU.

\paragraph{An interior-color reference control.}\phantomsection\label{smoke:grayreference}
To test the role of reference brightness, we repeat the density fit at constant color $a=.75$ and then repeat the full scale sweep. The color is the smallest value in $\{.5,.75,1\}$ above the largest training intensity, $.67075$; the choice uses cameras $0,2,4$ only. A half-gray reference would cap every prediction at $.5$, below $93$, $122$ and $148$ training pixels at the three times. Observations retain their original intensity convention and are neither clipped nor rescaled.

The reference objective uses optical-depth targets $-\log(1-y/.75)$, with the same grid, regularization, initialization, bounds and $300$-iteration limit as the unit-color fit. All $54$ SH0--SH2 refits use $600$ iterations, the same physical class and nested degree warm starts, initialized from unit color separately at each scale. The references and all color fits are completed before evaluating the five cameras at $384\times216$. The new density reference is fitted afresh: matching relative scales in the two scans are not identical absolute density fields.

\begin{center}\begin{minipage}{\textwidth}\centering
\includegraphics[width=\textwidth]{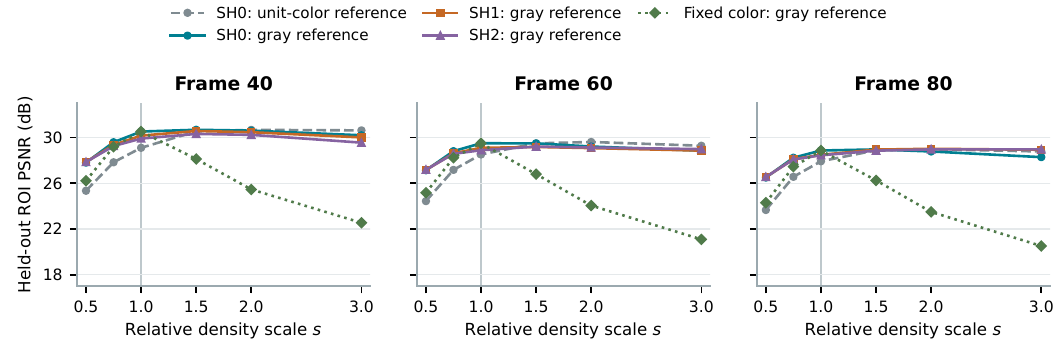}
\captionof{figure}{\textbf{Moving the reference away from unit color changes the scale profile.} Gray-reference curves use $a=.75$ for the density fit; the dashed gray curve repeats the original unit-color-reference SH0 result. Fixed color is the new reference's fitted SH0 field. Every marker is an evaluated scale, with the same training/test split and intensity convention. Connecting lines guide the eye.}\label{fig:smokegray}
\end{minipage}\end{center}

\begin{center}\begin{minipage}{\textwidth}\centering\small
\captionof{table}{Interior-color reference: held-out camera $3$ at every prescribed time and density scale. PSNR values use the calibrated ROI. The last column is the fraction of SH0 MSE forced by the physical brightness ceiling in \cref{smoke:ceilingresidual}.}\label{tab:smokegray}
\begin{tabular}{@{}rrrrrr@{}}\toprule
Frame & $s$ & SH0 (dB) & SH1 (dB) & SH2 (dB) & Forced MSE (\%)\\\midrule
40 & 0.5 & 27.79 & 27.85 & 27.85 & 71.4 \\
40 & 0.75 & 29.58 & 29.35 & 29.25 & 27.2 \\
40 & 1 & 30.51 & 30.16 & 29.90 & 9.9 \\
40 & 1.5 & 30.67 & 30.56 & 30.31 & 2.6 \\
40 & 2 & 30.59 & 30.46 & 30.23 & 1.6 \\
40 & 3 & 30.19 & 30.00 & 29.54 & 1.1 \\
60 & 0.5 & 27.13 & 27.16 & 27.17 & 71.4 \\
60 & 0.75 & 28.80 & 28.61 & 28.57 & 26.7 \\
60 & 1 & 29.49 & 29.09 & 28.88 & 11.9 \\
60 & 1.5 & 29.48 & 29.17 & 29.20 & 4.6 \\
60 & 2 & 29.21 & 29.07 & 29.12 & 2.5 \\
60 & 3 & 28.84 & 28.85 & 28.98 & 1.4 \\
80 & 0.5 & 26.50 & 26.53 & 26.54 & 75.9 \\
80 & 0.75 & 28.23 & 28.10 & 28.07 & 32.0 \\
80 & 1 & 28.86 & 28.48 & 28.44 & 17.6 \\
80 & 1.5 & 28.95 & 28.94 & 28.84 & 10.7 \\
80 & 2 & 28.77 & 29.01 & 28.95 & 7.9 \\
80 & 3 & 28.28 & 28.92 & 28.95 & 4.8 \\
\bottomrule\end{tabular}\end{minipage}\end{center}

Over $s=.75,1,1.5$, SH0 test PSNR spans at most $1.09$ dB, compared with $2.32$--$2.70$ dB for the unit-color references. At $s=.75$, the brightness ceiling accounts for $26.7$--$32.0\%$ of the new SH0 MSE. This comparison measures how the reference convention changes the low-scale response. Angular degree and spatial isotropic refitting remain separate comparisons in \cref{tab:smokegray}.

The mechanism follows directly from \cref{smoke:brightnessbound}. A constant-color reference predicts $y=a(1-e^{-\tau})$ on a ray. This same intensity respects the scaled-density ceiling exactly when $s\geq-\log[1-a(1-e^{-\tau})]/\tau$ for $\tau>0$. Rearranging $y\leq1-e^{-s\tau}$ proves this condition. Concavity of $1-e^{-s\tau}$ implies that the threshold is at most $a$, whereas at $a=1$ it equals $1$. An interior reference therefore removes a per-ray brightness obstruction below $s=1$; fitting a single shared color field across the cameras is still required. Neither the reference nor the sampled profile identifies the physical density scale.

All $1{,}890$ metric rows and the frozen inputs pass an independent saved-array check. Physical nodal bounds hold throughout; the $97$ fixed rays per time agree with independent interpolation and scalar compositing to $7.8\times10^{-16}$, with $96$-to-$192$-segment changes at most $0.001786$. The new fits and evaluations use $264.2$ seconds on two CPU threads, peak at $0.679$ GiB and use no GPU.

\paragraph{Spatial density perturbations on captured smoke.}\phantomsection\label{smoke:spatialcontrol}
The scale controls above isolate a uniform density change. We additionally test spatially varying changes at the same three interior-color references, using nodal multipliers $1+u_x/2$, $1+u_y/2$ and $1+u_z/2$, where each $u_a\in[0,1]$ is its normalized grid coordinate. A uniform multiplier $1.25$ provides a matched control. All axes and frames are retained. The grid, images and camera split remain fixed; no density is retrained. The changes preserve zero-density nodes and increase occupied-node densities by at most $50\%$.

Each edited density admits the same three nested physical appearance classes. SH0 starts from the reference SH0 coefficients divided by the nodal multiplier; SH1 and SH2 embed the preceding fit. Each receives $600$ projected-gradient iterations against the captured training images. The reciprocal nodal color is a feasible product-scaling control; it is not an exact transport compensation between interpolation nodes. Color-unchanged and reciprocal controls are evaluated alongside all fits. All fits finish before evaluating the five cameras at $384\times216$.

\begin{table}[t]\centering\small
\caption{\textbf{Spatial density edits separate fitting from transfer to a new camera.} ROI PSNR on test camera $3$; every prescribed edit and time is shown. Fixed uses the reference SH0 color and Reciprocal divides it nodally by the multiplier. SH0--SH2 refit appearance at the same edited density. The last two columns give the improvement over SH0 on validation camera $1$. These are different held-out viewpoints of the same fitted field.}\label{tab:smokespatial}
\begin{tabular}{@{}rlrrrrrrr@{}}\toprule
& & \multicolumn{5}{c}{Test-camera PSNR (dB)} & \multicolumn{2}{c}{Validation gain (dB)}\\
Frame & Edit & Fixed & Reciprocal & SH0 & SH1 & SH2 & SH1 & SH2\\\midrule
40 & $1.25$ & 29.69 & 30.22 & 30.66 & 30.50 & 30.28 & 0.32 & 0.46\\
40 & $x$ ramp & 29.57 & 30.11 & 30.70 & 30.53 & 30.30 & 0.30 & 0.43\\
40 & $y$ ramp & 30.30 & 30.44 & 30.60 & 30.32 & 30.08 & 0.29 & 0.39\\
40 & $z$ ramp & 29.72 & 30.24 & 30.68 & 30.54 & 30.33 & 0.30 & 0.42\\
60 & $1.25$ & 28.48 & 29.15 & 29.59 & 29.27 & 29.14 & 0.09 & 0.13\\
60 & $x$ ramp & 28.31 & 29.05 & 29.60 & 29.30 & 29.21 & 0.09 & 0.12\\
60 & $y$ ramp & 28.96 & 29.32 & 29.56 & 29.17 & 29.02 & 0.06 & 0.11\\
60 & $z$ ramp & 28.43 & 29.20 & 29.67 & 29.29 & 29.16 & 0.07 & 0.08\\
80 & $1.25$ & 27.92 & 28.53 & 28.99 & 28.77 & 28.74 & 0.34 & 0.31\\
80 & $x$ ramp & 27.73 & 28.42 & 29.00 & 28.82 & 28.76 & 0.32 & 0.28\\
80 & $y$ ramp & 28.17 & 28.67 & 28.97 & 28.66 & 28.59 & 0.34 & 0.30\\
80 & $z$ ramp & 27.86 & 28.57 & 29.01 & 28.82 & 28.78 & 0.35 & 0.29\\
\bottomrule\end{tabular}\end{table}

The unedited SH0 references have test PSNR $30.51$, $29.49$ and $28.86$ dB. The nine nonuniform changes have relative nodal density $L^2$ distances $12.7$--$27.5\%$ from those references. Higher degree reduces the training objective in every comparison and improves validation-camera RGB in all nine nonuniform cases. On test camera $3$, SH1 and SH2 instead increase RMSE by $1.61$--$4.63\%$ and $2.75$--$6.41\%$ relative to SH0. Thus the fitted angular corrections transfer differently to the two held-out viewpoints. This experiment evaluates learned feasible fields: nested-class optimum monotonicity concerns one fixed data objective, whereas the two camera evaluations measure transfer of the selected fields. The references provide controlled density changes, not physical density ground truth.

All $945$ saved metric rows pass independent recomputation (maximum difference $1.1\times10^{-14}$). Every nodal field satisfies the all-direction physical bound. Independent interpolation and sequential compositing on $97$ prescribed rays per field agree to $1.2\times10^{-15}$. A $96$-to-$192$-segment check covers every test-camera domain ray for each fitted degree; the stored discrepancies measure quadrature sensitivity. The $36$ fits and evaluations take $234.9$ CPU wall seconds, with $32.1$ seconds for saved-array auditing and no GPU use.

\subsection{Isolating angular compensation across held-out cameras}\label{smoke:angularreplication}

The scale profiles establish isotropic compensation. To test the additional role of angular capacity, we repeat the comparison across all five camera folds and a reserved time subset. The experiment separates two questions: whether angular terms improve unseen-camera fit, and whether editing density increases that improvement over its value at the reference density.

\paragraph{Frozen protocol and camera isolation.}
We use the same processed real-capture sequence, previously inspected frames $40,60,80$, and frame $100$, declared before decoding it. For each frame and held-out camera, both reference-density fitting and appearance fitting use only the other four cameras. This gives $20$ separate reference densities; none is reused from a fit that included its test camera. Frame $100$ extends the time subset of this capture: its four training-camera images are fitted, so this is neither temporal forecasting nor an independent scene.

For each reference, the density multiplier is $1$, $1.25$, or $1+.5u_j$, $j=x,y,z$, where $u_j\in[0,1]$ is the normalized spatial coordinate. The resulting density stays fixed across SH0/SH1/SH2 comparisons. Every fit uses the same $16\times24\times16$ trilinear grid and the nested physical inner classes in \cref{smoke:physical}; the SH2 class is not the complete physical SH2 class. The reference gray is $.75$, with a predeclared training-only brightness fallback to $1$ that was not needed in any fold. Each appearance fit receives $600$ iterations. All $300$ fits finish before test scoring; all cameras, edits and degrees are retained. Measurements are encoded grayscale at $96\times54$, with $96$ integration segments and full held-out ROI comparisons at $192$ segments. Fixed-gray and reciprocal-gray controls are also saved.

\paragraph{Ordinary angular gain and the edit-associated increment.}
Let $e_L(h)$ be the held-out ROI RMSE at density edit $h$ after the degree-$L$ refit, with $h=0$ denoting the unedited reference. We report the direct paired differences
\begin{equation}\label{smoke:pairedangular}
 G_L(h)=e_0(h)-e_L(h),\qquad
 D_L(h)=G_L(h)-G_L(0).
\end{equation}
A positive $G_L$ measures angular fitting benefit; $D_L$ subtracts the benefit already available without the density edit. Thus ordinary appearance modeling cannot automatically be credited to density compensation. These differences are descriptive comparisons, not a causal identification theorem.

Bars below mean equal-weight averages over every prescribed frame, fold and edited setting within each stratum. Before fitting, success was defined to require all of: $(\overline e_0-\overline e_L)/\overline e_0\geq5\%$, $\overline D_L\geq1/255$, positive mean $G_L$ on at least four of five cameras, and no camera's mean RMSE worsening by more than $5\%$. In addition, $\overline G_L$ and $\overline D_L$ must exceed three times their corresponding sums of integration sensitivities. Specifically, if $s_L(h)$ is the full held-out ROI RMS difference between the $96$- and $192$-segment predictions, these thresholds are
\begin{align*}
 Q_G&=3\overline{s_0(h)+s_L(h)},\\
 Q_D&=3\overline{s_0(h)+s_L(h)+s_0(0)+s_L(0)}.
\end{align*}
The same degree must pass separately in both strata. The unit $1/255$ is a declared encoded-gray threshold, not measured sensor noise or a linear-radiance error bound; the integration differences are sensitivity checks, not certified continuum-error bounds.

\begin{table}[t]
\centering\small
\caption{All registered cross-camera endpoints. Relative gain is positive when the angular refit improves mean held-out ROI RMSE. The last two columns give $255\overline D_L$ and its integration-sensitivity threshold $255Q_D$. Neither degree passes in either stratum.}
\label{tab:smokeangularreplication}
\begin{tabular}{@{}lrrrrr@{}}\toprule
Frames & SH & Gain (\%) & $G>0$ cameras & $255\overline D$ & $255Q_D$\\\midrule
$40/60/80$ & 1 & $+.35$ & $3/5$ & $.119$ & $.391$\\
$40/60/80$ & 2 & $-.77$ & $3/5$ & $.010$ & $.388$\\
$100$ & 1 & $+4.59$ & $5/5$ & $.256$ & $.444$\\
$100$ & 2 & $+4.04$ & $5/5$ & $.172$ & $.440$\\\bottomrule
\end{tabular}
\end{table}

\paragraph{Results and interpretation.}
On frames $40/60/80$, mean held-out RMSE is $.036279$ for SH0, $.036152$ for SH1 and $.036558$ for SH2. Camera~0 worsens by $9.70\%$ and $8.18\%$ respectively, despite improvements on three camera means. On frame $100$, the means are $.041392$, $.039491$ and $.039719$: both angular models improve all five camera means, and their ordinary gains exceed $Q_G$. These benefits remain below the declared $5\%$ threshold. More decisively, each stratum/degree mean edit-associated increment is below both $1/255$ and its $Q_D$ threshold (\cref{tab:smokeangularreplication}). No registered endpoint passes. The frozen test therefore establishes modest ordinary angular benefits in the reserved frame but does not establish reproducible additional compensation attributable to density edits under this protocol. It does not establish that angular compensation is absent in real scenes.

Training objectives improve in all $160$ edited SH0/SH1 or SH0/SH2 comparisons beyond the saved floating-point solver-gap diagnostic. Thus improved training fit alone does not settle camera transfer. These gaps are not interval certificates. No independent physical density or transmission measurements are available; opacity changes remain model predictions. Correlations within one capture preclude an independent-scene or population claim.

\paragraph{Checks, reproducibility and cost.}
An independent implementation recomputes all $7{,}500$ metric rows (maximum discrepancy $7.2\times10^{-15}$), checks all $500$ model instances against the whole-direction physical constraint to floating-point tolerance $10^{-12}$, and checks every fold's density and appearance isolation. It independently interpolates and composites $9{,}500$ held-out ray/model pairs, agreeing within $1.4\times10^{-15}$, and recomputes all $300$ full-ROI integration sensitivities. Frozen inputs, density equality across degrees and test-scoring order are verified. The released record includes every field, prediction, control, camera/edit table, protocol and failure gate. Fitting and evaluation take $1{,}135$ seconds on two CPU threads, peak at $1{,}164$ MiB, and use no GPU. There are no adaptive retries or additional fits after observing these results.

\section{Calibrating the inverse coefficient against fixed color}\label{app:calibration}
For completeness, the positive-gain examples $Q_L$ use \emph{different} isotropic references $\sigma=1$, $j_k=1/2+x_k^{L+1}/4$. Their degree-$L$ refitted gains satisfy $\kappa_L(0)=e^{-1}B_{L+1}6^{-L/2}$, where $B_m=\tfrac14\int_0^1e^{-t}t^m\dd t$ (\cref{sec:nonlinear}). They show existence at every fixed finite degree, not a same-scene ladder.
\begin{center}\small
\begin{tabular}{@{}lrr@{}}\toprule
Reference and fitted degree & Gain & Inverse\\\midrule
$Q_0$, $L=0$ & $.02430$ & $41.15$\\
$Q_1$, $L=1$ & $.006030$ & $165.84$\\
$Q_2$, $L=2$ & $.001746$ & $572.63$\\
$Q_3$, $L=3$ & $.0005497$ & $1819.31$\\\bottomrule
\end{tabular}\end{center}
All comparisons below use exactly the $Y$ normalization in \eqref{eq:Y} and the same density $L^2$ norm. Let $t=e^{-1}$, $a=1/4$, $b=1/2$, $A=a(1-t)$, $B=b(2-5t)$, and $D=at$. At the center of \eqref{m:reference}, $U_k=A+Bd_k^2$ and $c_{*,k}=a$.

\begin{proposition}[Fixed-color calibration]
The fixed-color frozen derivative has sharp gain
\begin{equation}\label{eq:fixedcolorcalibration}
 \gamma_c=t\sqrt{3D^2-2DB+B^2/2}.
\end{equation}
Its inverse is $23.11535489515\ldots$, compared with $82.91765398408\ldots$ after degree-zero or degree-one color refitting: a ratio $3.58712441838\ldots$.
\end{proposition}
\begin{proof}
With color fixed, $j=\sigma c_*$, hence $q_k=c_{*,k}h$. The frozen coefficient is $t(a-U_k)=t(D-Bd_k^2)$. For a circle normal $n$,
\[
 \avg{d_k^2}_n=(1-n_k^2)/2,\qquad
 \avg{d_k^4}_n=3(1-n_k^2)^2/8.
\]
Therefore $\sum_k\avg{d_k^2}=1$ and $\sum_k\avg{d_k^4}=3(1+\sum_k n_k^4)/8$. The last expression is minimized at $n_k^2=1/3$, where it equals $1/2$. Expanding $t^2\sum_k\avg{(D-Bd_k^2)^2}$ proves \eqref{eq:fixedcolorcalibration}. Fourier slice/coarea gives the lower bound, and localized real packets with this normal prove sharpness. For variable spatial color the perturbation $q_k(x)=c_{*,k}(x)h(x)$ preserves the fixed-color constraint exactly; shrinking packets gives the same necessary inverse coefficient for the actual local problem.

When color is refitted through degree one, evenness of $d_k^2$ makes its circle mean the optimal fit. Since $\operatorname{Var}_n(d_k^2)=(1-n_k^2)^2/8$, minimizing the sum gives $\kappa_0=\kappa_1=tB/\sqrt6$ as in \eqref{m:numbers}. Dividing the two gains gives the stated ratio.
\end{proof}

Holding the \emph{source} fixed would instead impose $q=0$. Its gain is $t\sqrt{3A^2+2AB+B^2/2}$, with inverse $8.44835688873\ldots$; it is not the fixed-color comparator. For the uniform reference $\sigma=1$, $c_k=j_k=a$, the fixed-color gain is $\sqrt3\,at^2$, with inverse $17.06429411109\ldots$. Allowing even degree-zero refitting gives $\kappa_L=0$ for every $L$, because $U_k=a(1-t)$ is direction independent. This last assertion is a frozen degeneracy; exact nonlinear ambiguity requires a separate construction.

The displayed decimals were enclosed by the same exact integer interval evaluator described in \cref{app:radius}, with scale $10^{40}$ and analytic exponential tails. These are sharp frozen comparisons and necessary local inverse coefficients. No sufficient fixed-color estimate for a prescribed finite radius or conversion to pixel-RGB noise is inferred from the table.

\section{A layered bridge to first-hit appearance fitting}\label{app:opaque}
\begingroup
\setlength{\abovedisplayskip}{5pt}
\setlength{\belowdisplayskip}{5pt}
\setlength{\abovedisplayshortskip}{3pt}
\setlength{\belowdisplayshortskip}{3pt}
This section relates two projection mechanisms without identifying their inverse problems. Wang--Miller's conditional surface loss \cite{wang2026} measures the best appearance fit on a specified candidate surface. Our $\kappa_L$ measures a volume-density derivative after source refitting and minimizes over every spatial frequency normal. The surface/volume connection itself is classical; texture--disparity frequency transport also predates both analyses \cite{chai2000}.

\subsection{Relation to earlier ambiguity analyses}
NeRF++ \cite{zhang2020} constructs a surface-based explanation of training images on an incorrect proxy and discusses the complexity of the required directional radiance. Ref-NeRF \cite{verbin2022} shows how emitters behind a semitransparent shell can imitate reflections, motivating structured appearance and normal penalties. These results motivate the density--appearance tradeoff studied here. Our exact pair is a complementary statement about a fixed continuum forward map: both members are physical finite-optical-depth volumes, color has degree at most two, and equality holds for every oriented external line. This statement concerns the observations themselves, independently of how a particular network is trained.

Wang--Miller's first-hit model \cite{wang2026} isolates a surface-location and angular-approximation problem. In the volume model, changing density also changes the weighting of emission elsewhere along the same ray. The distinction affects both the unknown and the data derivative; \cref{o:comparison} summarizes it. The layered limit below relates the conditional appearance-fitting losses under explicit assumptions. It does not identify the full volume-density inverse with surface recovery.

\begin{table*}[t]
\caption{Different inverse problems linked by constrained appearance fitting. Both works construct exact ambiguities within their stated observation and scene classes.}\label{o:comparison}
\centering\small
\begin{tabularx}{.98\textwidth}{@{}lYY@{}}\toprule
Property & Wang--Miller \cite{wang2026} & This work\\\midrule
Image formation & First-hit surface appearance & Distributed emission and exponential absorption\\
Geometric variable & Surface/proxy displacement & Shared three-dimensional volume density\\
Positive conclusion & Angular-capacity conditions separating true and misaligned surfaces & Sharp frozen density signal and local nonlinear stability\\
Exact ambiguity & Finite-SH billboard in a constructed viewing-circle model & Physical degree-two source compensation on every oriented external line\\
Quantitative object & Texture/parallax frequency and approximation error & Weighted great-circle residual; necessary and sufficient inverse coefficients\\
Empirical test & SH-capacity stress tests in 3DGS & Certified color fits and matched voxel optimization\\\bottomrule
\end{tabularx}
\end{table*}

Joint attenuation/source identification \cite{stefanov2014,holman2023} and fixed-attenuation tensor-source recovery \cite{psu2012,balmonard2019} supply other parts of the theoretical context. Our contribution is the finite-angular, shared-RGB volume problem and its quantitative same-reference results, using established transport, projection, and inversion tools. The nonlinear source gauge is handled separately from the sharp frozen density criterion. Likewise, color elimination \cite{fang2023} is an existing computational strategy; our discrete residual bounds certify compatibility tests for a declared density--appearance class rather than introduce a new color-fitting algorithm.

\begin{figure*}[t]
\centering\includegraphics[width=\textwidth]{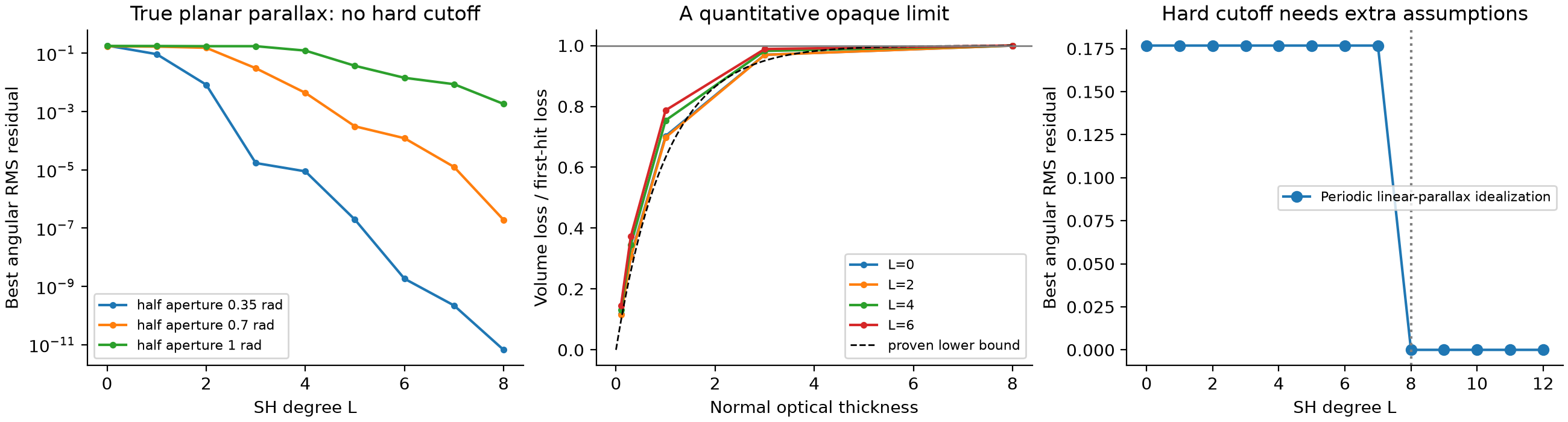}
\caption{Layered-model diagnostic. Left: exact planar parallax with $k\Delta=8$ has strongly aperture-dependent tails. Middle: the opacity-weighted best fit tends to the first-hit best fit within the proven sandwich. Right: a hard cutoff occurs in the separate periodic linear-parallax model. These are unconstrained angular projections, hence lower bounds for a globally physical SH fit; they are not measured $\kappa_L$ values or training outcomes.}\label{o:figure}
\end{figure*}

\subsection{A physical nonlinear slab identity}
Let the candidate point be $y=(u,0)$ and the true diffuse plane be $z=\Delta>0$. A viewing direction is $d=(v,\mu)$, where $\mu\geq\mu_0>0$. For a bounded texture $\rho$, the target on a ray through $y$ is
\begin{equation}
 F(y,d)=\rho(u+\Delta v/\mu).
\end{equation}
Put an isotropic slab of density $\tau/\epsilon$ in $0<z<\epsilon<\Delta$, with color $g(d)\in\cH_{\leq L}$. Its source is $j=\sigma g$; global bounds $0\leq g\leq1$ imply the physical source box. Since its ray optical thickness is $\tau/\mu$, direct integration gives
\begin{align}
 R_{\tau,g}&=\alpha_\tau g+(1-\alpha_\tau)F,
 &\alpha_\tau&=1-e^{-\tau/\mu},\label{o:identity}\\
 R_{\tau,g}-F&=\alpha_\tau(g-F).\nonumber
\end{align}
This is an exact nonlinear volume identity, for every positive finite thickness.

Fix a probability viewing measure $\nu$ and a nested angular class $G_L$. This can be the linear SH space, or its globally box-constrained subset. Define root mean square errors
\begin{align}
 D_L(F)&=\inf_{g\in G_L}\norm{g-F}_{L^2(\nu)},\\
 H_{\tau,L}(F)&=\inf_{g\in G_L}\norm{\alpha_\tau(g-F)}_{L^2(\nu)}.
\end{align}
\begin{proposition}[Quantitative first-hit limit]\label{o:limit}
For every $\tau>0$,
\begin{equation}\label{o:sandwich}
 (1-e^{-\tau})D_L(F)\leq H_{\tau,L}(F)\leq D_L(F).
\end{equation}
Thus $H_{\tau,L}$ converges to $D_L$ with relative root-error gap at most $e^{-\tau}$. It is nondecreasing in $\tau$ and nonincreasing in $L$. At the opposite limit, if
\(W_L(F)=\inf_{g\in G_L}\norm{(g-F)/\mu}_{L^2(\nu)},\)
then
\begin{equation}
 e^{-\tau/\mu_0}W_L(F)\leq H_{\tau,L}(F)/\tau\leq W_L(F).
\end{equation}
\end{proposition}
\begin{proof}
The uniform bounds $1-e^{-\tau}\leq\alpha_\tau\leq1$ give \eqref{o:sandwich} before taking infima. Pointwise monotonicity gives monotonicity in $\tau$, and nesting gives that in $L$. Finally,
\((1-e^{-x})/x=\int_0^1e^{-sx}\dd s\in[e^{-x},1]\)
with $x=\tau/\mu$ proves the small-opacity bounds. No exchange of a pointwise limit and an infimum is required.
\end{proof}
For the linear class, $H_{\tau,L}$ is a weighted orthogonal projection with weight $\alpha_\tau^2\nu$, and its opaque limit is the conditional first-hit projection. For the physical class it is a constrained projection. An unconstrained optimum is a lower bound, with equality only if a physical SH extension attains it. The $1/\mu$ in the thin-opacity limit is the isotropic slab's path-length factor.

\paragraph{Replacing the true surface by a volume.}
Put a second slab at $\Delta<z<\Delta+\eta$, with density $\beta/\eta$, color $\rho(x_\parallel)$, and a bounded background. If $\rho$ is $K$-Lipschitz, its volume target obeys
\begin{equation}
 \norm{F_{\eta,\beta}-F}_\infty\leq K\eta/\mu_0+e^{-\beta}.
\end{equation}
Indeed lateral travel within the slab is at most $\eta/\mu_0$, and the unabsorbed mass is at most $e^{-\beta}$. Distance to a fixed set is $1$-Lipschitz and multiplication by $\alpha_\tau$ is a contraction. Consequently
\begin{align}
 &|H_{\tau,L}(F_{\eta,\beta})-D_L(F)|\nonumber\\
 &\hspace{8mm}\leq K\eta/\mu_0+e^{-\beta}+e^{-\tau}D_L(F).
\end{align}
This is a physical layered emission--absorption limit. Extending it uniformly over a surface patch with spatially varying fitted $g(u,d)$ requires a common tangential Lipschitz bound $K_g$ and adds $K_g\epsilon/\mu_0$. Without that control, unconstrained sublayer spatial structure is not covered.

\subsection{Real planar parallax and an idealized hard cutoff}
For $d=(\sin\theta,0,\cos\theta)$ on $|\theta|\leq a<\pi/2$ and texture $\rho(u)=1/2+A\sin(ku+\phi)$,
\begin{equation}\label{o:planar}
 F(\theta)=\tfrac12+A\sin(\phi+ku+\omega\tan\theta),
 \qquad\omega=k\Delta.
\end{equation}
The instantaneous phase frequency is $\omega\sec^2\theta$. If $A\omega\ne0$ and the viewing measure has positive density on an open arc, the error to every finite trigonometric space is strictly positive. Otherwise closedness of that finite-dimensional space and real analyticity would extend exact equality to $(-\pi/2,\pi/2)$. A trigonometric polynomial has a finite limit at $\pi/2$, whereas \eqref{o:planar} oscillates indefinitely there, a contradiction. For a fixed nongrazing arc the error tends to zero as $L$ increases, but there is no exact finite-degree cutoff. With strict color margin, a continuous bounded extension to the sphere and uniform spherical-polynomial approximation also establish convergence for the physical class.

The linear-parallax approximation has the explicit error
\begin{equation}
 \norm{F-F_{\rm lin}}_\infty\leq A|\omega|(\tan a-a).
\end{equation}
Distance to the fitting space changes by at most this amount. This does not transfer full-circle Fourier orthogonality to a short viewing arc.

Under the \emph{additional} periodic idealization $F=1/2+A\sin(\phi+m\theta)$, integer $m>0$, and uniform full-circle measure, orthogonality gives
\begin{equation}
 D_L(F)=\begin{cases}A/\sqrt2,&L<m,\\0,&L\geq m.\end{cases}
\end{equation}
Both optima are physical: use the constant $1/2$ below the threshold, and above it the globally bounded extension
\begin{equation}
 g(d)=\tfrac12+A\operatorname{Im}\{e^{i\phi}(d_1+id_2)^m\}.
\end{equation}
This recovers the algebraic window in Wang--Miller's Appendix E \cite{wang2026}. A radial front layer viewed from its center has direction-independent optical thickness and multiplies this error by $1-e^{-\tau}$. It realizes the angular identity physically, but is not a plane with globally linear parallax. If an additional residual has norm at most $\epsilon_R$ and the correct surface's fitting error is at most $a_L$, the strict comparison
\begin{equation}
 a_L<(1-e^{-\tau})(A/\sqrt2-\epsilon_R)
\end{equation}
prefers the correct surface over the specified wrong layer. It is not a global uniqueness theorem among all volume fields.

\subsection{Why the full density constant has no such automatic limit}
Convergence of $T,U$ on a whole sphere, with a uniform positive lower bound on $T$, implies convergence of each weighted circle projection. Uniformity over normals then permits taking their minimum. A cap, a single informative circle, or a surface-patch loss does not supply these hypotheses.

Moreover, behind a spherical screen of optical depth at least $M$ in every direction, $T\leq e^{-M}$ and $0\leq U_k\leq1$. The trial polynomial zero gives
\begin{equation}\label{o:hidden}
 \kappa_L(x)\leq\sqrt3\,e^{-M}\longrightarrow0.
\end{equation}
An opaque surface can constrain surface position while hiding interior density. In a thin opaque layer $\sigma=\tau/\epsilon$ and its derivatives also leave any common bounded smooth class; moving the layer produces a singular shape derivative. Our finite-optical-depth volume inverse theorem therefore cannot be passed to this limit with unchanged norms and constants.

All $1{,}080$ aperture, phase, frequency, order and opacity combinations satisfied the sandwich. Refining $2{,}048$ to $4{,}096$ Gauss nodes changed root errors by at most $2.45\times10^{-13}$; independent slab integration matched \eqref{o:identity} within $7.98\times10^{-14}$. At $\tau=8$, the opaque-limit gap was at most $5.23\times10^{-5}$. Equivalent stable polynomial bases gave condition numbers at most $41.3$. These numerical checks complement the proof.
\endgroup

\section{Angular capacity, geometric curvature, and regularization}\label{app:capacity}
We apply classical variable projection \cite[\S2]{golubpereyra1972} to the geometry/appearance split. Its constant-rank nonlinear form is reviewed in \cite{golubpereyra2003}. The formulas below identify the role of angular capacity; they introduce neither a projection method nor a training algorithm.

\subsection{Appearance elimination at one fixed reference}
Write the weighted image derivative as $Ah+B_Lv$, where $h$ is geometry and $v$ appearance. Compare degrees at the same reference, geometry coordinates, observations and data weight. Interior physical margins permit unconstrained infinitesimal $v$; active color bounds instead require feasible tangent directions. Orthogonal elimination gives
\begin{equation}\label{cap:reduced}
 \mathcal Q_L(h)=\min_v\norm{Ah+B_Lv}_2^2
 =h^*G_Lh,\qquad G_L=A^*(I-P_L)A,
\end{equation}
where $P_L=B_LB_L^\dagger$. Nested appearance ranges immediately imply
\begin{align}
 0\preceq G_{L+1}\preceq G_L\preceq A^*A,\label{cap:psd}\\
 h^*(G_L-G_{L+1})h=\norm{(P_{L+1}-P_L)Ah}_2^2.\label{cap:loss}
\end{align}
Thus loss is strict only when added appearance image directions overlap the geometric signal. The smallest generalized eigenvalue in a chosen geometry norm is a finite-model sensitivity, not automatically $\kappa_L^2$.

$G_L$ is the geometric Schur complement of the joint Gauss--Newton matrix. With locally constant appearance rank, it is also the quadratic term of the profiled nonlinear objective at zero residual. Nonzero residual adds higher-derivative terms. None of these identities compares distinct trained endpoints.

\subsection{Quadratic penalties and training implications}
For a linear appearance penalty $D$ and $\lambda\geq0$, the same least-squares elimination applied to the augmented matrix $[B_L;\sqrt\lambda D]$ yields
\begin{align}
 \mathcal Q_{L,\lambda}(h)
 &=\min_v\{\norm{Ah+B_Lv}_2^2+\lambda\norm{Dv}_2^2\},\label{cap:penalty}\\
 G_{L,\lambda}
 &=A^*A-A^*B_LH_\lambda^\dagger B_L^*A,\label{cap:schur}\\
 H_\lambda&=B_L^*B_L+\lambda D^*D.\nonumber
\end{align}
This remains valid at deficient rank because $\ker H_\lambda\subseteq\ker B_L$. Comparing the minimands and using $v=0$ gives
\begin{equation}\label{cap:regmonotone}
 G_{L,0}\preceq G_{L,\lambda_1}\preceq G_{L,\lambda_2}\preceq A^*A
 \quad(0\leq\lambda_1\leq\lambda_2).
\end{equation}
A penalty increases curvature of the penalized objective, not image information. For absolute coefficient penalties the update must also keep the linear term at the reference. The rotation-invariant spherical-gradient penalty has weights $\ell(\ell+1)$, but it suppresses genuine directional appearance as well as compensation; it need not select exactly the harmful directions in \eqref{cap:loss}.

An early low-degree stage or small appearance learning rate delays access to compensation. Releasing those directions later leaves the final ambiguity class unchanged unless an additional constraint remains. A local probe must therefore compare the profiled $\mathcal Q_L(h)$, rather than raw $\norm{Ah}$, for the same geometry changes. The packet tests in \cref{app:packetvalidation} evaluate a prescribed source subspace, not an unrestricted online diagnostic. Independent depth/transmission evidence or geometric priors can distinguish or select densities; an RGB-only loss cannot distinguish an exact pair.

Progressive SH activation and separate DC/higher-order learning rates are established in the official 3DGS implementation \cite{gsimplementation}. These derivative formulas apply to its differentiable fixed-order regime, but do not transfer the continuum physical-volume theorems to sorted splatting. Directional MLPs require their actual parameter tangent space. The interventions in \cref{app:trainedcompensation} demonstrate learned compensation, while the retained curriculum test failed its accuracy threshold (\cref{app:curriculum}); no reconstruction improvement is established here.

\section{Controlled prediction and appearance refitting}\label{app:packetvalidation}
This experiment tests a quantitative prediction of the sharpness construction, separately from trained-field endpoint interventions. It uses the quadratic reference throughout. The source refit is deliberately a prescribed spatial subspace; it does not minimize over arbitrary voxel source coefficients. A small residual in this subspace supplies a feasible compensation direction for a larger class, whereas a positive residual is not a lower bound for that larger class.

\subsection{Specified packets, source coordinates, and directional prediction}
For $|y|<1$ let $\phi(y)=\exp[1-(1-|y|^2)^{-1}]$, extended by zero. Set
\begin{equation}\label{pack:fields}
 h(x)=\phi(x/r)\cos(kn\cdot x/r),\qquad
 \sigma_\eps=1+\eps h,\qquad j_{\eps,k}=j_{*,k}+\eps h\,p_{k,L}(d).
\end{equation}
The two normals are $(1,1,1)/\sqrt3$ and $(1,0,0)$; $k\in\{12,24\}$, $r\in\{.02,.1,.3\}$, and $L=0,1,2,3$. Every combination is retained. The source polynomials use a real SH basis normalized in $L^2(\Sph^2,\mathrm dS/(4\pi))$. Since density stays positive, the resulting color $j_\eps/\sigma_\eps$ also has degree at most $L$. The no-refit comparator fixes the \emph{source}; it is not the fixed-color comparator in \cref{app:calibration}.

Let $B=\tfrac12(2-5/e)$. At the center, $T_0=e^{-1}$ and $U_{0,k}=\tfrac14(1-e^{-1})+Bd_k^2$. Odd circle modes cannot fit this even quadratic. Using $\operatorname{Var}_n(d_k^2)=(1-n_k^2)^2/8$ gives
\begin{equation}\label{pack:directional}
 \sqrt{E_L(0,n)}=\frac{B}{e}\sqrt{\frac{1+\sum_kn_k^4}{8}},\quad L=0,1;
 \qquad E_L(0,n)=0,\quad L\geq2.
\end{equation}
This is a directional gain; only the weak normal attains $\kappa_1$. With $q=h p_L(d)$, finite packets approximate the angular lifting used in the sharpness proof. Finite spatial width, finite frequency, and quadrature can each produce a discrepancy from \eqref{pack:directional}.

\subsection{A differentiable finite renderer and two separate norms}
For each ray, the segment inside $B_r$ is split uniformly. Density and the packet correction are midpoint constants on each segment; the quadratic reference source is integrated against its segment exponential using analytic moments through degree two. Unchanged near and far parts are analytic. The unperturbed reference is consequently exact to floating-point error independently of the segment count, rather than a different midpoint reference for each resolution. The density Jacobian is obtained by differentiating the same compositor analytically. An independent dummy finite-difference check agrees within $4.45\cdot10^{-13}$; exact-reference checks differ by less than $10^{-13}$.

The primary discretization uses $24$ Gauss nodes in $\mu=n\cdot d$ and $12$ uniform azimuths, a $32^2$ cell-centered offset plane on $[-r,r]^2$, $48$ internal segments, and $56^3$ Gauss points for $\norm h_2$. Rays missing $B_r$ have zero changes. A prescribed audit repeats both $k=24,r=.02$ cases at $36\times16$ directions, $48^2$ offsets, $72$ segments, and $80^3$ density quadrature. Neither radii nor samples were selected after seeing the outcomes.

On zero-padded ray changes, the discrete $Y$ energy is
\begin{equation}\label{pack:Y}
 \sum_d w_d\,\frac{\Delta p^2}{4\pi^2}
 \sum_{\xi}|\xi|\,|\operatorname{FFT}_{\rm ortho}g_d(\xi)|^2,
 \qquad \xi=2\pi\operatorname{fftfreq}(N,\Delta p).
\end{equation}
The three channel energies are summed. Ordinary unweighted X-ray transforms of the same packets give a maximum relative norm-calibration error $3.63\cdot10^{-5}$. Independently, pixel RMS averages over the prescribed foreground offsets and RGB channels, with solid-angle weights. Each response is divided by the same density $L^2$ norm. Pixel RMS depends on the stated camera/foreground normalization; it is never substituted for $Y$ in a continuum stability assertion.

\subsection{Profiled linear and nonlinear predictions}
For fixed density, source refitting is linear. If $A$ is the density derivative and $B_{\eps,L}$ maps packet source coefficients to observations at $\sigma_\eps$, we solve
\begin{align}
 g_{N,L}(h)&=\min_a\norm{W(Ah+B_{0,L}a)}_2/\norm h_2,\label{pack:linear}\\
 g_{N,L}^{\rm nl}(\eps)&=\min_a\norm{W\{[R(\sigma_\eps,j_*)-R_*]/\eps+B_{\eps,L}a\}}_2/\norm h_2.
 \label{pack:nonlinear}
\end{align}
Here $W$ is either of the two data weights just specified; all least squares are unregularized. There is no ridge-generated positive curvature. The source response separates into a scalar image per direction times an SH polynomial, permitting a reduced dense SVD. The final residual is evaluated explicitly in the full weighted image space, avoiding subtraction of nearly equal Gram energies.

\paragraph{Why a local profile can converge.}
This is a finite-dimensional variable-projection calculation \cite{golubpereyra2003}, not a new projection theorem. At a positive density, the finite compositor and its source matrix are analytic in $\eps$. If $WB_{0,L}$ has full column rank, continuity keeps its Gram matrix invertible for small $\eps$. The least-squares coefficients and residual in \eqref{pack:nonlinear} then converge to those in \eqref{pack:linear}. If the branch also has a strict physical margin, it solves the physically constrained problem \emph{within this coefficient family}. Rank loss, active physical bounds, or an enlarged spatial source family require a separate argument. This establishes the conditions for the comparison; numerical singular values below are diagnostics, not interval proofs of rank.

In all recorded fits the singular-value ratio is at least $.01050$, with full column rank. Relative rank thresholds $10^{-10},10^{-12},10^{-14}$ produce the same ranks and gains; maximum normalized normal-equation residual is $5.30\cdot10^{-15}$. A sufficient uniform source bound is
\begin{equation}
 Q_k=|a_{k,00}|+\sum_{\ell=1}^L\sqrt{2\ell+1}\norm{a_{k,\ell,\cdot}}_2,
 \qquad \frac14-|\eps|(1+\max_kQ_k)>0.
\end{equation}
It follows from $|h|\leq1$, $j_*\in[1/4,3/4]$, and the SH addition theorem. The smallest computed sufficient margin is $.1461$ across all amplitudes and both metrics. Thus the unconstrained fits are physically feasible without clipping. These finite tests do not claim the rescaled smoothness caps of the radius-$.01$ nonlinear sufficient theorem.

\subsection{Results, sampling audit, and departures from the limit}
\begin{table}[h]
\centering\small
\caption{Actual-renderer $Y$ gains after SH1 source refitting. The circle gains are $.01206016$ (weak normal) and $.01477062$ (axis). All primary cases are shown; these are numerical values, not interval certificates.}\label{pack:table}
\begin{tabular}{@{}llrrr@{}}\toprule
Normal & Frequency & $r=.02$ & $r=.1$ & $r=.3$\\\midrule
Weak & $12$ & $.01287599$ & $.01301846$ & $.01419802$\\
Weak & $24$ & $.01230651$ & $.01246757$ & $.01378320$\\
Axis & $12$ & $.01449925$ & $.01457301$ & $.01523875$\\
Axis & $24$ & $.01468110$ & $.01474239$ & $.01531347$\\\bottomrule
\end{tabular}
\end{table}

At $r=.02,k=24$, frozen finite-packet gains differ from their directional limits by at most $1.99\%$; adding true spatial transport changes them by at most $.055\%$. Higher-frequency packets improve directional agreement at small support. At $r=.3$, the weak-normal discrepancy reaches $14.29\%$ for $k=24$ and $17.73\%$ for $k=12$, so locality remains material. These larger-radius results were retained rather than interpreted as failures of a small-support theorem.

Across both norms and every case, the maximum relative nonlinear/linear difference is $.02735\%$, $.006052\%$, and $.001464\%$ at amplitudes $.08,.02,.005$, respectively. The sampling audit changes $Y$ gains by at most $.001916\%$ and pixel-RMS gains by at most $.6364\%$; the different foreground grids slightly change the latter normalization. The predeclared calibration, sampling, directional agreement, small-support transport, small-amplitude prediction, and rank-sensitivity gates all pass. Full arrays, coefficients, singular values, and tolerances are retained.

SH0 and SH1 have equal gains in this symmetric packet family, as do SH2 and SH3. The frozen SH2/SH3 residual is at floating-point roundoff; the actual-renderer SH2 gains are nonzero. For example, at $k=24$ the weak-normal gain grows from $.0004751$ at $r=.02$ to $.007198$ at $r=.3$. It would be incorrect to turn a zero frozen gain into a claim of exact cancellation by this prescribed source envelope.

\paragraph{What has and has not been checked.}
This is a consistency experiment connecting directional prediction, finite rendering, and a local physical refit branch. The finite Jacobian comparison uses the same compositor; the circle prediction comes independently from \eqref{pack:directional}, and sampling is separately refined. The reference is quadratic, and the appearance space is restricted. \Cref{app:weightedpacket} next tests nonradial packets with direction-dependent transport. Unrestricted spatial source compensation, broad structural perturbations, truly directional ground truth, and optimized training controls remain future tests. The existing nonradial instability certificates and training interventions address different questions and are not substitutes for them.

\begin{figure}[h]
\centering\includegraphics[width=.53\textwidth]{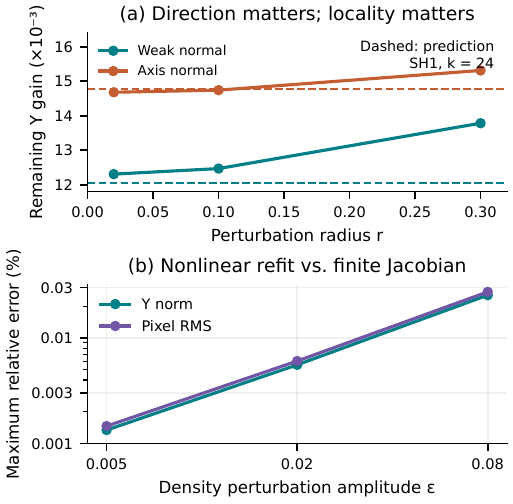}
\caption{The retained constant-transmittance packet study. Directional and locality effects are separated from finite-Jacobian/nonlinear agreement. The sampling-refinement audit covers only the two $k=24,r=.02$ cases, not the entire grid.}\label{pack:constantfigure}
\end{figure}

\subsection{Nonradial transport-weight ablation}\label{app:weightedpacket}
We next use exactly the nonradial density and isotropic sinusoidal color of \cref{app:nonradialnumerics}, without changing the scene to favor a predictor. The point $x_0=(.2,-.1,.1)$ was already specified there. The normals, fixed before the new responses, are
\[
 n_p=(.36003286697,.41661247720,.83475168676),\qquad n_x=(1,0,0),\qquad n_z=(0,0,1),
\]
with $n_p$ normalized in the calculation. It reuses the previous SH1 witness at the origin; it is not reoptimized at $x_0$ and is not asserted to minimize the new circle energy. We keep $k=24$, $r=.02,.08$, $L=0,1,2,3$ and amplitudes $.08,.02,.005$. Thus there are six primary cases and two sampling audits. Source changes remain $q=h\,p_L(d)$; the scene color is isotropic before these changes.

\paragraph{Competing predictions.}
The weighted prediction $g_w=\sqrt{E_L(x_0,n)}$ is compared with
\begin{align}
 g_{TU}&=\min_{p\in\cV_L}\avg{|T_0U_0-p|^2}_n^{1/2},\label{nwp:wrong}\\
 g_{\rm rms}&=\avg{T_0^2}_n^{1/2}\min_{p\in\cV_L}\avg{|U_0-p|^2}_n^{1/2}.
\end{align}
RGB energies are summed. The first baseline substitutes $\cV_L$ for $T_0\cV_L$ and therefore changes the admissible compensation. The second replaces direction-dependent attenuation by an RMS weight; it may be a good approximation without being the exact criterion. Constant $T_0$ makes all three predictions equal. Circle quadrature uses $256$ angles and $64$ source Gauss nodes; a $512$-angle, $128$-node calculation audits every normal and degree.

\begin{table}[h]
\centering\small
\caption{Transport variation at the fixed point. All three normals were prescribed before evaluating new responses.}\label{nwp:transport}
\begin{tabular}{@{}lrrrr@{}}\toprule
Normal & $T_{\min}$ & $T_{\max}$ & $T_{\max}/T_{\min}$ & $\operatorname{std}(T)/\operatorname{mean}(T)$\\\midrule
$n_p$ & 0.21377 & 0.39626 & 1.854 & 0.210\\
$n_x$ & 0.21938 & 0.35897 & 1.636 & 0.162\\
$n_z$ & 0.21313 & 0.39464 & 1.852 & 0.189\\
\bottomrule\end{tabular}
\end{table}

\paragraph{Renderer and fit.}
Gaussian density gives analytic optical-depth integrals in terms of $\operatorname{erf}$. The reference source is $j_*=\sigma_*c_*$. Within the changed ball the packet density/source is midpoint constant on each segment. The unperturbed source is integrated with Gaussian quadrature using its actual background attenuation. If $H(t)$ is the cumulative midpoint packet integral, changed attenuation is $T_*(t)e^{-\eps H(t)}$; the differentiated image uses $-T_*j_*H$. The far-side contribution retains its factor $e^{-\eps H_{\rm full}}$, while unchanged near-side radiance cancels. The nonlinear difference uses $\operatorname{expm1}(-\eps H)/\eps$ to avoid cancellation error.

Primary sampling uses $24\times16$ directions, $32^2$ offsets, $48$ segments, two inner Gauss nodes per segment, $48$ exterior source nodes, and $56^3$ density quadrature. Only $(n_p,r=.02)$ and $(n_x,r=.02)$ are repeated at $36\times24$ directions, $48^2$ offsets, $72$ segments, three inner nodes, $80$ exterior nodes, and $80^3$ density quadrature. These audit values are not substituted for primary measurements. Relative prediction error means $|g_{\rm pred}-g_{\rm measured}|/g_{\rm measured}$ throughout this comparison.

\paragraph{Physical feasibility.}
The reference obeys $\sigma_*\geq1$ and $.28\leq c_{*,k}\leq.72$. Since $|h|\leq1$, the addition-theorem bound $Q_k$ used above gives the sufficient all-space, all-direction margin $.28-|\eps|(1+\max_kQ_k)$. Every reported unconstrained fit has positive margin, so no clipping or ridge creates the result. The all-angle inequalities have an analytic form; their numerical evaluations here are not interval certificates.

\begin{table}[h]
\centering\small
\caption{Every primary small-support case, $r=.02$, in the $Y$ norm. ``Measured'' is the actual finite-renderer derivative after refitting the prescribed source envelope. Errors are for $g_w$. No normal or degree is selected from these outcomes.}\label{nwp:all}
\begin{tabular}{@{}lrrrrrr@{}}\toprule
Normal & $L$ & $g_w$ & $g_{TU}$ & $g_{\rm rms}$ & Measured & Error (\%)\\\midrule
$n_p$ & 0 & 0.038129 & 0.067975 & 0.039571 & 0.038116 & 0.03\\
$n_p$ & 1 & 0.013190 & 0.015472 & 0.012909 & 0.013504 & 2.33\\
$n_p$ & 2 & 0.005798 & 0.008147 & 0.005651 & 0.006445 & 10.04\\
$n_p$ & 3 & 0.001815 & 0.003949 & 0.001810 & 0.002630 & 31.00\\
$n_x$ & 0 & 0.035760 & 0.057131 & 0.035763 & 0.035874 & 0.32\\
$n_x$ & 1 & 0.013359 & 0.012648 & 0.013053 & 0.013576 & 1.60\\
$n_x$ & 2 & 0.009741 & 0.009945 & 0.009601 & 0.009848 & 1.08\\
$n_x$ & 3 & 0.002541 & 0.002950 & 0.002520 & 0.002963 & 14.25\\
$n_z$ & 0 & 0.034740 & 0.061432 & 0.035247 & 0.035098 & 1.02\\
$n_z$ & 1 & 0.014259 & 0.016128 & 0.014039 & 0.014450 & 1.32\\
$n_z$ & 2 & 0.007461 & 0.009058 & 0.007203 & 0.007715 & 3.30\\
$n_z$ & 3 & 0.003054 & 0.004543 & 0.003032 & 0.003546 & 13.86\\
\bottomrule\end{tabular}
\end{table}

\paragraph{Outcome and validity checks.}
For the primary SH1 comparison at $r=.02$, mean weighted-prediction error is $1.750\%$, and the median ratio of $g_{TU}$ error to $g_w$ error is $6.26$. All predeclared gates pass. The circle-refinement change is $5.77\times10^{-15}$ in relative units; the two-case sampling audit changes gains by at most $0.1206\%$ in $Y$ and $0.7580\%$ in pixel RMS. The X-ray norm-calibration error is $3.14\times10^{-6}$ in relative units. At $\eps=.005$, the finite Jacobian/nonlinear discrepancy is at most $1.4\times10^{-5}\%$ over both metrics, all degrees and all cases. The least physical margin is $0.1488$, the smallest singular-value ratio is $0.00893$, and the three rank thresholds $10^{-10},10^{-12},10^{-14}$ agree. Independent pivoted QR rechecks 128 saved linear/small-amplitude fits, with maximal relative SVD discrepancy $3.55\times10^{-15}$. These are numerical validity checks, not rigorous floating-point enclosures.

\paragraph{Limits and unfavorable comparisons.}
The $g_{\rm rms}$ approximation can be closer to a finite packet than the exact directional asymptote. Even at $r=.02$, the $n_x$ SH3 case has $14.25\%$ weighted-prediction error but only $.43\%$ error for $g_{TU}$. At $r=.08$, the SH3 weighted errors range from $59.68\%$ to $72.41\%$. Higher-order residuals are small enough that spatial transport variation and finite angular localization become large \emph{relative} errors. The saved frozen-operator comparison separates these sources: at $r=.02$, the three SH3 actual/frozen discrepancies are $19.39\%$, $12.89\%$, and $12.71\%$, relative to the actual gain. In the prior-normal case, circle/frozen disagreement is also $14.39\%$, relative to the frozen gain. These discrepancies exceed the sampling audit, so numerical refinement alone does not explain them. All six primary cases, every degree, and both data norms are retained; a favorable SH1 result does not establish uniform accuracy at higher degree. Nor does a positive residual in this restricted source space lower-bound the optimum in a richer spatial appearance class. These results test nonconstant transport weighting at one scene and point, not a learned-geometry predictor, a successful training schedule, or positivity over all normals.

\begin{figure}[t]
\centering\includegraphics[width=.72\linewidth]{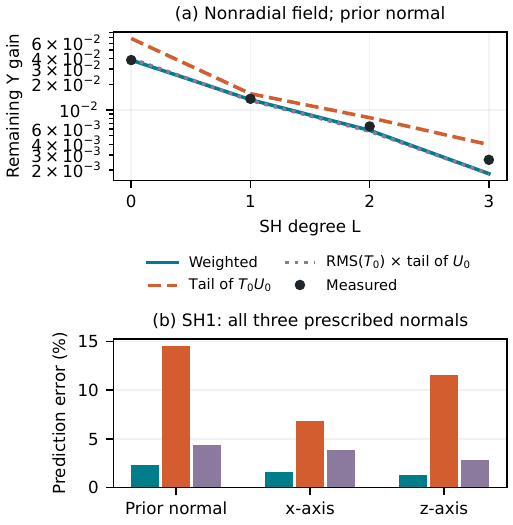}
\caption{\textbf{Transport weighting changes the prediction.} Nonradial packets at $r=.02,k=24$. (a) The first prescribed normal, with all degrees retained. (b) SH1 relative errors against the actual renderer at all three normals. Colors match (a). The comparison uses $Y$; the refinement audit covers only the prior-normal and $x$-axis cases.}\label{m:packetvalidation}
\end{figure}

\section{A transmission-sensitive control for the certified pair}\label{app:backgroundcontrol}
Two-background triangulation matting eliminates foreground appearance by differencing registered composites \cite{smithblinn1996}. The same identity isolates transmission here. Keep each candidate's effective source unchanged when the prescribed background switches from $\beta=0$ to $\beta=1$. Then
\[
 R_\beta=R_0+\beta T_{\rm full},\qquad R_1-R_0=T_{\rm full}.
\]
The invariance assumption is explicit: changing a physical background in a scattering experiment may change the source and invalidate this subtraction. The control applies to the specified emission--absorption model with fixed effective sources.

The twofold pair has identical $R_0$ on every ray but different transmission. On a central ray,
\[
 T_*=e^{-2},\quad T_2=e^{-2-\delta},\quad
 \delta=\int_{-1/2}^{1/2}\psi(4t^2)\dd t.
\]
For $|t|\leq1/4$, $\psi(4t^2)\geq e^{-1/3}\geq2/3$, so $\delta\geq1/3$. Also $e^2<8$ and $e^{1/3}\geq4/3$, giving the strictly positive, completely analytic bound
\begin{equation}\label{bg:gap}
 T_*-T_2=e^{-2}(1-e^{-\delta})>\frac18\left(1-\frac34\right)=\frac1{32}.
\end{equation}
For example, $e<11/4$ follows by bounding its factorial tail geometrically, which proves the inequality used above. No approximate Volterra source is required for \eqref{bg:gap}.

Suppose each of the two background measurements has absolute channel error at most $b$. Their difference has error at most $2b$. The transmission hypotheses have disjoint error intervals whenever $4b<1/32$, that is, $b<1/128$. Thus one central transmission-sensitive pair of ray measurements distinguishes these two candidates at a one-level-per-measurement budget $b=1/255$, even though any number of ordinary black-background views cannot separate them. This is pair distinguishability, not recovery of an arbitrary density from one ray. The following bound extends the control to a specified class of finite footprints.

\begin{proposition}[Finite-footprint separation]\label{bg:footprint}
For the certified twofold pair, let $\mu$ be any nonnegative probability measure on external oriented rays supported entirely on impact parameters $\rho\leq1/4$. Use the same $\mu$ in the black- and white-background measurements, and keep each candidate's effective source unchanged. The footprint-averaged transmissions satisfy
\begin{equation}\label{bg:footprintgap}
 \overline T_*-\overline T_2>\frac1{56}.
\end{equation}
If each averaged measurement has absolute channel error at most $b$, the two transmission hypotheses have disjoint error intervals whenever $b<1/224$. In particular, $b=1/255$ is covered.
\end{proposition}
\begin{proof}
Write a ray as $p+td$, with $p\perp d$ and $\rho=|p|$. For $\rho\leq1/4$ and $|t|\leq1/4$, the entire segment lies in the bump support and
\[
 4(\rho^2+t^2)\leq\frac12,\qquad
 \psi\bigl(4(\rho^2+t^2)\bigr)\geq e^{-1}>\frac13.
\]
Thus the additional optical depth and reference transmission obey
\begin{align}
 \delta(\rho)&\geq\tfrac12e^{-1}>\tfrac16,\label{bg:footprintdepth}\\
 T_*(\rho)&=e^{-2\sqrt{1-\rho^2}}\geq e^{-2}>\tfrac18.
\end{align}
Using $e^{1/6}>7/6$ gives the uniform lower bound
\begin{align}
 T_*(\rho)-T_2(\rho)
 &\geq e^{-2}\bigl(1-e^{-1/(2e)}\bigr)\notag\\
 &>\tfrac18(1-\tfrac67)=\tfrac1{56}.
\end{align}
Because $\mu\geq0$ and $\int\mathrm d\mu=1$, averaging preserves this uniform bound. The identical footprint in both measurements also preserves the subtraction identity:
\[
 \overline R_1-\overline R_0=\int T_{\rm full}\,\mathrm d\mu
 =\overline T.
\]
Two errors bounded by $b$ give a difference error at most $2b$. The two intervals, each of radius $2b$, are disjoint if $4b<1/56$, equivalently $b<1/224$. Since $224<255$, the one-level budget is strictly covered.
\end{proof}

The footprint may include rays from an exterior perspective camera; only its impact-parameter support and common weights enter the proof. Negative reconstruction weights, unmatched background footprints, or a background-dependent effective source require a different bound. This remains a statement about distinguishing the specified pair, not a general density-reconstruction guarantee.

The saved numerical geometry diagnostic gives a central transmission gap $.0613175$ (about $15.64$ linear levels) and full-disk RMS gap $.0154884$ (about $3.95$ levels). These numerical values illustrate the exact positive bound; they do not certify a real acquisition setup or a noise distribution. This control is post hoc, chosen to illustrate the known opacity difference, and is reported separately from the preregistered packet experiment.

\section{Fixed-reference positive-approximant demonstration}\label{app:positiveR}
\begin{figure*}[t]
\centering\includegraphics[width=\textwidth]{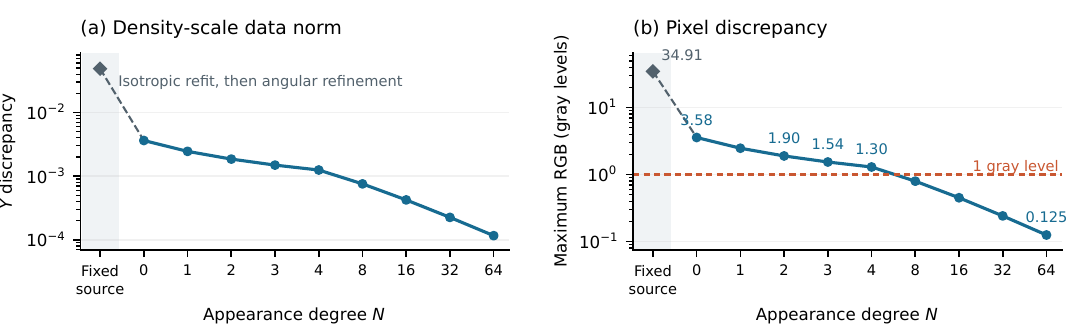}
\includegraphics[width=.77\textwidth]{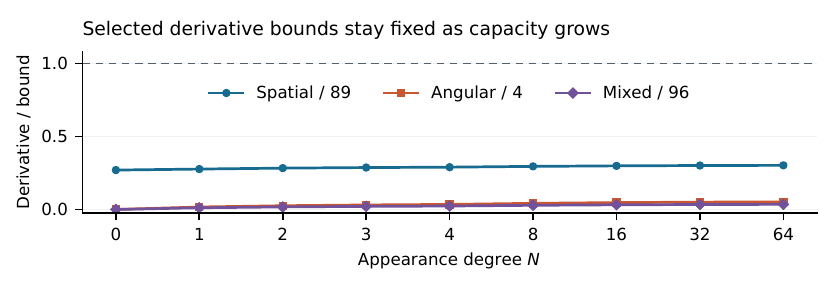}
\caption{\textbf{Two stages of compensation at one fixed density error.} Top: unchanged source (gray), isotropic compensation ($N=0$), and angular refinement, including $N=3$. Left shows $Y$ discrepancy; right shows numerical maximum RGB in linear gray levels. Bottom: sampled spatial, angular and mixed derivatives divided by their analytic all-degree bounds. All panels use the same affine reference and density addition (\cref{pos:results}).}\label{m:positiveR}
\end{figure*}
This experiment evaluates the positive approximants in \cref{cap:R} at one reference and one fixed density perturbation, including degree $3$ and an unchanged-source baseline. Each arm uses a fixed evaluation protocol. The affine reference illustrates the approximation mechanism; it is separate from the generic reference theorem.

\subsection{Reference, fixed perturbation and physical range}
In $\Omega=B_1$ use $\sigma=1$ and $j_k(x)=a_k+.15x_k$, with $a=(.35,.5,.65)$. Fix center $c=(.05,-.03,.02)$, semiaxes $b=(.3,.2,.15)$, and
\begin{equation}\label{pos:bump}
 q(x)=\sum_i((x_i-c_i)/b_i)^2,\qquad
 h(x)=\begin{cases}4e^{-q/(1-q)},&q<1,\\0,&q\geq1.\end{cases}
\end{equation}
Its support lies strictly inside $B_{.4}$. At its center density rises from 1 to 5. Independent radial quadratures of orders 128 and 256 give $\norm h_2=.31977429124949$, differing relatively by $3.30\times10^{-15}$. Neither density nor reference changes with $N$.

The reference backlight is analytic. Writing $\ell=-x\cdot d+\sqrt{(x\cdot d)^2+1-|x|^2}$,
\begin{equation}\label{pos:U}
 U_k=(a_k+.15x_k)(1-e^{-\ell})+.15d_k[1-(1+\ell)e^{-\ell}].
\end{equation}
We compute $P_N=\cV_NU$ and $j_N=j+hP_N$ at $N=0,1,2,3,4,8,16,32,64$, with no optimization, clipping or regularization. Globally $.2\leq j_k\leq.8$; positivity gives $j_{N,k}\geq.2$ and $(1+h)-j_{N,k}\geq.2$ for every $N$. These physical margins are analytic, not inferred from sampling.

This reference is outside G on $K=B_{.4}$: at $x=0$, \eqref{pos:U} becomes $a_k(1-e^{-1})+.15(1-2e^{-1})d_k$, of angular degree one, so $\kappa_1(0)=0$. The experiment illustrates finite-amplitude hiding from \cref{cap:R}, which does not require genericity; the same-reference coexistence statement is established analytically by \cref{m:coexistence}.

\subsection{Evaluating the positive operator}
The Funk--Hecke multiplier of \eqref{m:positiveoperator} at harmonic degree $\ell$ is
\begin{equation}\label{pos:multiplier}
 m_{N,0}=1,\qquad m_{N,\ell}=\prod_{r=0}^{\ell-1}\frac{N-r}{N+r+2}\quad(\ell\leq N),
 \qquad m_{N,\ell}=0\quad(\ell>N).
\end{equation}
All modes through $N$ are retained with these exact multipliers. Smooth radial Legendre coefficients use Gaussian quadrature and cubic interpolation. This evaluates the same positive operator; it is not ordinary SH truncation or an angular Taylor polynomial. Direct positive-kernel integration at 64 independently specified point-direction pairs differs from harmonic evaluation by at most $4.70\times10^{-14}$.

The nonlinear residual is evaluated as $\int T_{1+h}h(P_N-U)\dd t$, with changed attenuation integrated along the ellipsoid chord. Independent full-emission ODE integration on eight rays agrees within $5.35\times10^{-11}$; the exact limit source $j+hU$ reproduces the reference in that audit. Thus a theorem upper bound is not being substituted for measured rendered data.

\begin{table}[t]
\centering\small
\caption{Same reference and altered density, including the unchanged-source baseline. Maxima are numerical, including multistart ray refinement for RGB. Derivative entries are component maxima; complete mixed seminorms include lower orders.}\label{pos:results}
\begin{tabular}{@{}lrrrrrr@{}}\toprule
Source & $Y$ & Max. RGB & Gray levels & $\max|\partial_xj|$ & $\max|\Omega j|$ & $\max|\partial_x\Omega j|$\\\midrule
Unchanged $j$ & $0.048601719$ & $0.136901595$ & $34.910$ & $.15$ & $0$ & $0$\\
$N=0$ & $0.003630594$ & $0.014027883$ & $3.577$ & $24.02625$ & $0.000000$ & $0.00000$\\
$N=1$ & $0.002452099$ & $0.009697377$ & $2.473$ & $24.64066$ & $0.067386$ & $1.10920$\\
$N=2$ & $0.001854743$ & $0.007433374$ & $1.896$ & $25.21150$ & $0.102371$ & $1.67977$\\
$N=3$ & $0.001492635$ & $0.006033923$ & $1.539$ & $25.56058$ & $0.123867$ & $2.02840$\\
$N=4$ & $0.001249327$ & $0.005080907$ & $1.296$ & $25.79645$ & $0.138435$ & $2.26378$\\
$N=8$ & $0.000757111$ & $0.003118732$ & $0.795$ & $26.27688$ & $0.168205$ & $2.74253$\\
$N=16$ & $0.000423876$ & $0.001761993$ & $0.449$ & $26.60456$ & $0.188575$ & $3.06833$\\
$N=32$ & $0.000225533$ & $0.000942735$ & $0.240$ & $26.80048$ & $0.200775$ & $3.26274$\\
$N=64$ & $0.000116521$ & $0.000488573$ & $0.125$ & $26.90844$ & $0.207502$ & $3.36972$\\
\bottomrule\end{tabular}
\end{table}

The unchanged-source baseline is $\mathcal R(1+h,j)-\mathcal R(1,j)$; it differs from holding color fixed, which would use source $(1+h)c$. Its maximum RGB is $34.910$ gray levels, versus $3.577$ after isotropic compensation and $1.539$ at degree $3$. The baseline-to-$N=0$ reduction is $9.76\times$; subsequent angular refinement to $N=64$ supplies another $28.7\times$. These ratios compare maxima, not additive portions of a common residual vector. All share the same changed attenuation.

\subsection{Derivative budgets are independent of degree}
On $K=B_{.4}$, the exit length satisfies $.6\leq\ell\leq1.4$. With $z=x\cdot d$, its square-root factor exceeds $.9$, yielding component bounds
\[
 |\partial_x\ell|<1.9,\quad |\Omega\ell|<.58,\quad
 |\partial_x\Omega\ell|<2.1.
\]
Differentiate \eqref{pos:U}. The coefficient $A_k=e^{-\ell}(j_k+.15d_k\ell)$ obeys $|A_k|\leq1.01$ and $|\Omega A_k|\leq1.16\cdot.58+.21=.8828$. Therefore
\[
 |U_k|\leq1,\quad |\partial_xU_k|<2.1,\quad |\Omega U_k|<1,
 \quad |\partial_x\Omega U_k|<4.
\]
For the bump, putting $t=(1-q)^{-1}\geq1$ gives $|\partial_i h|\leq(8/.15)e^{1-t}t^2\leq32/(.15e)<80$, while $h\leq4$. Positivity and derivative commutation transfer the bounds to $P_N$. Product differentiation gives the analytic all-$N$ bounds
\begin{equation}\label{pos:derbounds}
 \max|\partial_xj_N|<89,\qquad \max|\Omega j_N|<4,
 \qquad\max|\partial_x\Omega j_N|<96.
\end{equation}
The full source seminorms $(1,0),(0,1),(1,1)$ are consequently bounded by $89,5,96$. Including the fixed density norms gives common budgets $169,10,176$. These are real-arithmetic bounds specified before the numerical run, not formal validation of floating-point source evaluation. Higher derivative orders are not numerically measured here.

The primary and refined derivative samples contain 734 and 2270 point-direction pairs, respectively, including the same fixed axis probes and nested Sobol samples. Direct-kernel quadrature uses $48\times128$ and $64\times160$ angular samples. Source-jet quadrature changes by at most $9.19\times10^{-12}$, and sample refinement changes derivative block maxima by at most $1.30\%$. The sampled full seminorms $(1,0),(0,1),(1,1)$ range over $[24.0262,26.9085]$, $[2.3007,2.4329]$, and $[24.0262,26.9085]$; boundedness for untested degrees follows from the analytic argument, not this finite sample.

\subsection{Data norm, refinement and limits}
For each sampled direction, images on a fixed ray-offset square are transformed with a unitary discrete Fourier transform. If its physical step is $\Delta x$, the approximation to \eqref{eq:Y} is
\[
 \frac{\Delta x^2}{4\pi^2}\sum_d w_d\sum_\xi|\xi|\,|\operatorname{FFT}_{\rm ortho}g(\xi,d)|_{\rm RGB}^2.
\]
Zero padding refines Fourier frequency spacing while preserving $\Delta x$. The primary run uses 128 directions, $64^2$ offsets, and 48 ray nodes; the refined run uses 288 directions, $96^2$ offsets, and 96 ray nodes. Padding factors four and eight are compared. The scalar calibration gives $\norm{Xh}_Y/\norm h_2=.999991682$. Primary-to-refined changes are at most $.00338\%$ in $Y$ and $1.668\%$ in grid-maximum RGB. Isolating ray-order refinement gives a relative change below $2.10\times10^{-10}$; padding changes $Y$ by at most $.00572\%$.

Four separated grid maxima initialize each degree's continuous ray search, with 250 evaluations per start. Eight of the original 32 searches reach that cap; one of the four added degree-$3$ searches reaches the cap. All trials and initial values are retained, and each degree has at least one converged trial. Local convergence does not certify the continuum supremum. The reported maxima and $Y$ values are refined numerical estimates. The final RGB estimate is about $.125$ of a linear 8-bit level. The common derivative bounds and physical box are analytic; the computed discrepancies are not all-ray numerical error certificates.

The initial integration and checks took 290 and 17 seconds on two CPU threads. The degree-$3$ and baseline follow-up took $81.92$ seconds; its degree-zero replay agrees exactly in $Y$, direct positive-kernel evaluation agrees within $2.9\times10^{-14}$, and independent ODE integration agrees within $3.4\times10^{-12}$. The degree-$3$ $Y$ value changes by $.00317\%$ under the prescribed grid refinement. Complete data, search limits, protocols and costs are retained. Constants in \eqref{cap:Ybound} and \eqref{cap:rate} were not numerically evaluated.

\subsection{Actual-field cutaways and matched image sequence}\label{pos:teaserrender}
Figure~\ref{m:teaser} visualizes these same fields; no density, source reference, or approximant definition changes. The cutaway removes the portion $z>c_z=.02$ of the unit ball. Colors on the exposed plane sample the actual density with common limits $[1,5]$; only the outer spherical surface uses common normal shading to show depth. These density views are not RGB observations.

The rendered RGB images use exterior orthographic rays with viewing direction
$d=-(2.6,-3.8,2.8)/\norm{(2.6,-3.8,2.8)}_2$, a square of ray offsets with half-width $1.1$, and $256^2$ pixel-center rays. Every image uses unit exposure, black background, and the same sRGB display encoding. All eight prescribed degrees are rendered; the teaser shows the preselected $N=0,8,64$. Each alternative has the same density $1+h$ and source $j_N=j+h\cV_NU$; even $N=0$ uses a constructed source, not unchanged appearance. The residual is the maximum absolute linear-RGB channel difference per pixel, with common color limits $[0,.015]$. Each magnified crop is the same camera-space square of half-width $.36$ centered on the projected bump center; its pixel footprint is marked on the reference thumbnail. Cropping and display encoding are not used when measuring the residuals.

Full emission--absorption integration uses analytic exterior segments and 96 Gaussian nodes on each bump chord. It agrees with the exact transport-identity evaluation within $3.54\times10^{-16}$ on this image grid; increasing to 144 nodes changes RGB by at most $3.34\times10^{-16}$. An independent adaptive ODE audit on three preselected rays agrees within $2.21\times10^{-12}$. These are numerical consistency checks, not all-ray formal certificates. The rendered values lie in $[0,.607]$, and the analytic physical margins of this construction remain those stated above.

For this one view, the maximum linear-RGB differences at $N=0,8,64$ are respectively $.0114344$, $.00250381$, and $.000390281$. They are view-specific maxima and differ from the multi-direction estimates in \cref{pos:results}. The affine reference remains outside $\mathcal G$; this montage illustrates finite-amplitude hiding, while the separately labeled theorem strip summarizes the analytic generic-reference conclusion. It does not turn this experiment into a generic recovery demonstration.

\clearpage
\section{Shared-query image classes and sharp angular resources}\label{app:querystructure}
This section characterizes compensation in finite shared representations, supporting the physical certificates in \cref{m:querytable}. \Cref{q:Q1} gives finite-query closure, \cref{q:antipodal} the exact SH1 criterion, and \cref{q:circledegree} sharp regular-circle resources. Geometric weights and queries are fixed. The statements concern complete attainable image sets and their fitting distances. Norm-distance/support duality \cite[\S8.1.3 and Exercise 3.35]{boyd2004}, Bernstein--Szeg\H{o} \cite[Introduction]{erdelyi2014}, and Fej\'er interpolation \cite[Chapter IV, \S6]{jackson1930} are classical. We retain the model-specific image-set argument, full-sphere physical extension, and sharp resource consequences.

\subsection{Model and a finite-query closure theorem}
For one channel let
\begin{equation}\label{q:render}
 F_r(c)=b_r+\sum_vw_{rv}c_v(q_{rv}),\qquad w_{rv}\geq0,\qquad \sum_vw_{rv}\leq1.
\end{equation}
There are finitely many observations and nodes. Every node has an independent shared function $c_v:\Sph^2\to[0,1]$. An observation uses \emph{one} query for that node; repeated evaluations at that same direction have already been combined in $w_{rv}$. Remove zero-weight incidences, and merge exactly equal queries into $Q_v=\{q_{v1},\ldots,q_{vJ_v}\}$. Nearly equal directions are not merged. Let
\[
 (H_v)_{rj}=\begin{cases}w_{rv},&q_{rv}=q_{vj},\\0,&\text{otherwise},\end{cases}
 \qquad H=[H_1\ \cdots\ H_V].
\]
For a fixed node, the nonzero columns of $H_v$ have disjoint supports. The full matrix need not have full rank. Let $E_L(Q_v)$ be the restrictions of globally physical $c_v\in\cH_{\leq L}$ to $Q_v$, and define
\begin{align}
 Z_L&=b+\sum_vH_v E_L(Q_v),&
 Z_Q&=b+H[0,1]^{\sum_vJ_v},\label{q:sets}\\
 \rho_L(y)&=\operatorname{dist}_\infty(y,Z_L),&
 \rho_Q(y)&=\operatorname{dist}_\infty(y,Z_Q).\label{q:distances}
\end{align}
These are compact convex sets: bounded sphere values bound every coefficient in the finite SH basis. RGB uses the product of three independent channel classes at the same density.

For a source-interpolated grid with nodal density $s_v$ and nonnegative hats $\phi_v$, samples have $\sigma_i=\sum_v\phi_{iv}s_v$ and $j_i=\sum_v\phi_{iv}s_vc_v(d_r)$. With step length $\Delta_i$,
\begin{equation}\label{q:gridweights}
 w_{rv}=s_v\sum_i T_{ri}\Phi(\sigma_{ri},\Delta_{ri})\phi_{riv},\qquad
 \Phi(s,\Delta)=\begin{cases}(1-e^{-s\Delta})/s,&s>0,\\\Delta,&s=0.\end{cases}
\end{equation}
Thus the total weight is $1-T_{\rm exit}\leq1$. Fixed inactive nodes contribute to $b_r$. Positivity gives $0\leq j\leq\sigma$ throughout each cell. Fixed-order alpha compositing has the same form with $w_{rv}=T_{rv}\alpha_{rv}$, but its query convention and allowed color range must be checked separately.

\begin{lemma}[Positive finite-query approximation]\label{q:positive}
Let $q_1,\ldots,q_J$ be distinct unit directions with minimum chord separation at least $\delta>0$. For $J\geq2$ set
\[
 \beta=1-\delta^2/16,\quad A_L=1+(J-1)\beta^L,\quad
 e_L=\frac{(J-1)\beta^L}{1+(J-1)\beta^L}.
\]
For any $a_j\in[0,1]$, the polynomial
\begin{equation}\label{q:peaks}
 p_L(d)=\frac1{A_L}\sum_{j=1}^J a_j\left(\frac{1+d\cdot q_j}{2}\right)^L
\end{equation}
has degree at most $L$, lies in $[0,1]$ on the entire sphere, and obeys $\max_j|p_L(q_j)-a_j|\leq e_L$. For $J=1$ use the constant $a_1$ and $e_L=0$.
\end{lemma}
\begin{proof}
At most one query lies within chord distance $\delta/2$ of a given $d$; every other peak is at most $\beta^L$. Thus the sum of all peaks is at most $A_L$ and \eqref{q:peaks} is a subconvex combination of the $a_j$ and zero. At $q_j$ the own peak is one, whereas the other peaks sum to at most $(J-1)(1-\delta^2/4)^L\leq A_L-1$. The numerator of $p_L(q_j)-a_j$ therefore lies in $[-(A_L-1),A_L-1]$. Sphere restrictions of degree-$L$ polynomials belong to $\cH_{\leq L}$.
\end{proof}

\begin{theorem}[Q1: finite-query closure with uniform resources]\label{q:Q1}
Use the preceding $e_{v,L}$ at each node. Then, uniformly over targets and nonnegative weights with the same queries and row sums at most one,
\begin{equation}\label{q:Q1bound}
 0\leq\rho_L(y)-\rho_Q(y)\leq E_L^{\rm query}:=\max_r\sum_vw_{rv}e_{v,L},
 \qquad \overline{\bigcup_{L\geq0}Z_L}=Z_Q.
\end{equation}
\end{theorem}
\begin{proof}
The values of every physical SH function lie in the query box, proving the first containment and lower bound. A closest query-box fit exists. Apply \cref{q:positive} to its values independently at every node; positivity bounds the image change by $E_L^{\rm query}$. Finitely many distinct query sets have positive separation, so this error tends to zero. The box image is closed, giving the closure identity. Only positivity and the row bound entered the estimate.
\end{proof}
For a per-node tolerance $0<\epsilon<1/2$, one sufficient degree is
\begin{equation}\label{q:generaldegree}
 L\geq\max_{v:J_v\geq2}\left\lceil
 \frac{\log((J_v-1)(1-\epsilon)/\epsilon)}{-\log(1-\delta_v^2/16)}
 \right\rceil.
\end{equation}
The constant half-color improves $e_{v,L}$ to $\min(e_{v,L},1/2)$; antipodal-only nodes have zero interpolation error already at $L=1$. The general bound is conservative and does not control angular derivatives. A larger or denser query set changes its resources.

\subsection{An exact criterion for the whole image class}
Define the additional angular obstruction
\begin{equation}\label{q:gap}
 \mathfrak G_L=\max_{z\in Z_Q}\operatorname{dist}_\infty(z,Z_L)
 =\max_{\norm\lambda_1\leq1}\{h_{Z_Q}(\lambda)-h_{Z_L}(\lambda)\}
 =\sup_y\{\rho_L(y)-\rho_Q(y)\}.
\end{equation}
Here $h_Z$ is the support function. The middle identity follows from classical distance duality and support calculus \cite[\S8.1.3 and Exercise 3.35]{boyd2004}; the last uses a closest point in $Z_Q$ and the triangle inequality. This nested-set Hausdorff distance is distinct from the geometric curvature $G_L$ in \eqref{m:profilecurvature}.

\begin{lemma}[From node tables to image sets]\label{q:nodelemma}
Under \eqref{q:render}, $Z_L=Z_Q$ if and only if $E_L(Q_v)=[0,1]^{J_v}$ for every active node.
\end{lemma}
\begin{proof}
Sufficiency follows by substitution. If one evaluation set is proper, separate it from a point of its cube with a moment vector $m$. For each query $j$ choose one incident row $r_j$ with $w_{r_jv}>0$. These rows are distinct. Assign $\lambda_{r_j}=m_j/w_{r_jv}$ and zero to other rows, so $H_v^T\lambda=m$. The support loss is positive at this node and nonnegative at every other node because their SH evaluation sets lie in their cubes. Hence the complete image sets differ. Normalize $\lambda$ if desired.
\end{proof}
Other nodes may compensate a particular target, but cannot eliminate the difference of the entire attainable sets.

\begin{theorem}[Antipodal-query criterion]\label{q:antipodal}
The following are equivalent: (i) $\rho_1(y)=\rho_Q(y)$ for every target; (ii) $Z_1=Z_Q$; (iii) every active $Q_v$ is contained in one antipodal pair. In this case $Z_L=Z_Q$ for all $L\geq1$. Otherwise $\mathfrak G_1>0$.
\end{theorem}
\begin{proof}
At queries $a,-a$, any values $a_+,a_-\in[0,1]$ are realized by
\[
 c(d)=\tfrac12(a_++a_-)+\tfrac12(a_+-a_-)(a\cdot d),
\]
a convex interpolation on the entire sphere. Single-query sets use constants. Conversely a physical SH1 function has the form $c=1/2+u+v\cdot d$ with $|u|+\norm v_2\leq1/2$. At separation $0<\theta<\pi$ its contrast is at most $\sin(\theta/2)<1$, so it cannot assign zero and one at that pair. Every set not contained in an antipodal pair has such a pair. Apply \cref{q:nodelemma}. Distance functions identify closed sets.
\end{proof}

For an explicit support form write $m_{vj}=(H_v^T\lambda)_j$. Then
\begin{equation}\label{q:SH1gap}
 \mathfrak G_1=\max_{\norm\lambda_1\leq1}\frac12\sum_v
 \left[\sum_j|m_{vj}|-\max\left(\left|\sum_jm_{vj}\right|,
             \left|\sum_jm_{vj}q_{vj}\right|\right)\right].
\end{equation}
The node term vanishes for a fixed witness exactly when all its nonzero moments have one sign, or all signed unit vectors $\operatorname{sign}(m_{vj})q_{vj}$ coincide. These are the scalar/vector triangle-equality cases. This fixed-witness condition is distinct from the all-witness criterion above.

If a node is visible on two observations with weights $a,b>0$ and distinct directions separated by $\theta$, take $\lambda_r=b/(a+b)$ and $\lambda_s=-a/(a+b)$. Its opposite moments have common magnitude $ab/(a+b)$, giving
\begin{equation}\label{q:pairgap}
 \mathfrak G_1\geq\frac{ab}{a+b}\bigl(1-\sin(\theta/2)\bigr).
\end{equation}
The bound is sharp for a single node and two equally weighted observations. Thus a family in which every member has a simultaneously visible pair with weights at least $a_{\min}>0$ and $0<\theta\leq\pi-\eta$ has the uniform gap $\frac12a_{\min}(1-\cos(\eta/2))$. Without a visibility lower bound the gap can vanish. At coincident queries the independent-value box itself changes; the distinct-query formula must not be applied without merging them.

\subsection{A physical two-ray density profile}\label{q:twolayer}
Shared queries can preserve a nonlinear density signal even when appearance is freely refitted within its full physical range. The following exactly solved volume makes this effect quantitative. It also shows how the antipodal construction in \cref{q:antipodal} removes that signal.

Consider two ordered regions with shared color functions $c_1,c_2:\Sph^2\to[0,1]$ and black background. The unknown front density is $s$, with $q=e^{-s}\in[1/3,2/3]$. Front path lengths are $1,2$; known rear optical thicknesses are $\log(4/3),\log4$. Exact segment integration gives
\begin{align}
 F_1(q,c)&=(1-q)c_1(d_1)+\tfrac q4c_2(d_1),\label{q:twoforward}\\
 F_2(q,c)&=(1-q^2)c_1(d_2)+\tfrac{3q^2}4c_2(d_2).\nonumber
\end{align}
The reference $q_*=1/2$, $c_{1,*}=c_{2,*}=1/2$ has $s_*=\log2$ and $y_*=(5/16,15/32)$. Define $\rho_L(q)=\min_c\|F(q,c)-y_*\|_\infty$, where each shared color ranges over the complete physical degree-$L$ class. Thus $\rho_L$ measures the image evidence that remains after every legal appearance adjustment at a candidate density.

\begin{proposition}[Exact finite density signal]\label{q:twoprofile}
For every $q\in[1/3,2/3]$, the complete SH0 profile is
\begin{equation}\label{q:tworho}
 \rho_0(q)=\frac{5|2q-1|}{32(2+q)}.
\end{equation}
Its sharp uniform coefficient in the density coordinate is
\begin{equation}\label{q:twogain}
 \rho_0(e^{-s})\geq\gamma_*|s-\log2|,
 \qquad \gamma_* =\frac{5}{224\log(3/2)}.
\end{equation}
For identical queries $d_1=d_2$, $\rho_L=\rho_0$ at every degree. For antipodal queries $d_2=-d_1$, $\rho_L=0$ throughout the same density interval at every $L\geq1$.
\end{proposition}

\begin{proof}
For shared constants, the contrast $\lambda=(-(1+q),1)$ cancels the front color:
\begin{equation}\label{q:twocontrast}
 \lambda\cdot(F-y_*)=(2q-1)\left(\frac q4c_2+\frac5{32}\right).
\end{equation}
Since $c_2\geq0$ and $\|\lambda\|_1=2+q$, \eqref{q:tworho} is a lower bound for every physical color. Equality is attained by $c_2=0$ and $c_1=25/[32(1-q)(2+q)]$: the two residuals are opposite and have the stated magnitude. This $c_1$ increases from $225/448$ to $225/256$, entirely inside $[0,1]$.

For sharpness in \eqref{q:twogain}, put $t=s-\log2$, $z=e^{-t}$ and $g(t)=5(1-z)/[16(4+z)]$. Then $\rho_0=|g|$, $g(0)=0$ and
\[
 g''(t)=\frac{25z(z-4)}{16(4+z)^3}<0
 \quad\text{for }z\in[2/3,4/3].
\]
Concavity makes $g(t)/t$ nonincreasing. Its minimum occurs at $q=1/3$, where the ratio is $5/[224\log(3/2)]$ and the physical minimizer above attains equality.

Identical queries expose one shared value per region, and constants realize every allowed pair of values. Hence their profile stays \eqref{q:tworho} at all degrees. For antipodal queries set
\begin{align}
 u_1(q)&=\frac{5}{16(1-3q/4)},&
 u_2(q)&=\frac{15}{32(1-q^2/4)},\nonumber\\
 c_{1,q}(d)=c_{2,q}(d)&=\frac{u_1+u_2}{2}
             +\frac{u_1-u_2}{2}(d\cdot d_1).\label{q:twoSH1}
\end{align}
The queried values are $u_1,u_2$, which reproduce both components of $y_*$ exactly. Their ranges are $[5/12,5/8]$ and $[27/56,135/256]$, respectively. Every sphere value in \eqref{q:twoSH1} is a convex combination of them, so all colors lie in $[5/12,5/8]$. At $q=1/2$ they equal the reference color $1/2$. This gives a physical SH1 ambiguity family through the same reference.
\end{proof}

\paragraph{Finite precision.}
The profile gives an exact fixed-reference density interval. For total fitting tolerance $0\leq\epsilon\leq5/224$, solving \eqref{q:tworho} yields
\begin{equation}\label{q:twonoise}
 \rho_0(q)\leq\epsilon\quad\Longleftrightarrow\quad
 q\in\left[
 \max\!\left\{\tfrac13,\frac{5-64\epsilon}{10+32\epsilon}\right\},
 \min\!\left\{\tfrac23,\frac{5+64\epsilon}{10-32\epsilon}\right\}
 \right].
\end{equation}
For larger tolerances the full prior is compatible. If the interval endpoints are $a,b$, the exact density radius is $\max\{\log(1/(2a)),\log(2b)\}$, by monotonicity of $s=-\log q$. One linear gray level of total fitting tolerance gives radius $0.06397065$. With error at most $\beta$ on each raw reading and on a candidate's fit, prediction compatibility uses $\epsilon=2\beta$; one gray level for each gives radius $0.13058460$. These are fixed-reference radii in $|s-s_*|$. If the known front region has volume $V$, the physical density $L^2$ error is $\sqrt V|s-s_*|$.

\paragraph{Physical sharing and realization.}
Each ray alone accepts every density in the prior, using shared constants $c_1=c_2=u_r(q)$. The joint guarantee comes from using the same two colors in both observations. The ordered weights are realized by two disjoint tubes: their front pieces share $s,c_1$, and their unit-density rear pieces share $c_2$ and have the specified lengths. Orienting the tubes in parallel or oppositely realizes the two query designs while preserving front-to-back order along each designated ray. This specifies a finite, piecewise-constant spatial model with one unknown density parameter.

\paragraph{Physical color constraints.}
For shared constants write \eqref{q:twoforward} as $F=B(q)c$. This appearance matrix has determinant $q(1-q)(2q-1)/4$. For $q\ne1/2$, its exact fit requires $c_1=15/[32(1-q)]$ and $c_2=-5/(8q)$. Unconstrained profiling therefore gives zero residual, while \eqref{q:tworho} retains a positive signal through nonnegativity. At the reference the appearance matrix has rank one. Away from it, an interior color pair can be continued as $B(q)^{-1}B(q_0)c_0$ for nearby $q$, preserving both data. The positive bound is consequently specific to the declared reference and prior; it quantifies the physical appearance fiber without requiring unique recovery of the colors.

\subsection{Relative sampling with bounded appearance}\label{q:relativesampling}
Finite sampling must preserve the density signal after \emph{every} admissible appearance refit. An absolute image approximation error is insufficient near the reference because the density separation tends to zero. The following bound makes the sampling error relative to that separation. It uses finite-dimensional compactness, as in finite-measurement stability theory \cite[Theorem~2]{alberti2021finite}, while retaining appearance directions that become visible when density changes.

Let $\mathcal Q$ be a compact metric space of declared ray or pixel queries and write $F(s,c)=B(s)c\in C(\mathcal Q;\mathbb R^k)$. Density and appearance have fixed finite coordinates. Fix $(s_0,c_0)$ and a compact density prior $S$, star-shaped about $s_0$. Every allowed appearance set $C(s)$ lies in a common bounded coefficient set $C$. Assume $B$ is $C^1$ in $s$ on a neighborhood of $S$, with values in operators from appearance coordinates to continuous query data. Use a declared density norm, such as its spatial mass norm, and the maximum query/channel norm for data.

Put $V_0=\operatorname{ran}B(s_0)$ and $M=\sup_{c\in C}\|c\|$. Define
\begin{align}
 K&=M\sup_{s\in S}\|D_sB(s)\|,\label{rs:constants}\\
 \omega_0(\delta)&=\sup_{\substack{f\in V_0,\ \|f\|_\infty\leq1\\d(r,r')\leq\delta}}
       |f(r)-f(r')|_\infty,\nonumber\\
 \omega_1(\delta)&=M\sup_{\substack{s\in S,\ \|u\|=1\\d(r,r')\leq\delta}}
       \|D_sB(s,r)[u]-D_sB(s,r')[u]\|.\nonumber
\end{align}
Here $\omega_0=0$ when $V_0=\{0\}$. The first modulus measures how quickly a normalized reference appearance image varies between nearby queries; the second measures the variation of its density response. Both tend to zero: the unit ball of $V_0$ is compact, and the derivative family is continuous on compact parameter and unit-sphere sets. $K$ is finite.

\begin{proposition}[Relative sampling after full appearance refitting]\label{rs:theorem}
Let $D$ be a finite $\delta$-net of $\mathcal Q$ and $g=B(s)c-B(s_0)c_0$. Uniformly over $s\in S$ and $c\in C(s)$,
\begin{equation}\label{rs:relative}
 \|g|_D\|_\infty\geq(1-\omega_0)\|g\|_\infty
              -(K\omega_0+\omega_1)\|s-s_0\|.
\end{equation}
If the full-query physical profile satisfies
$\inf_{c\in C(s)}\|g\|_\infty\geq\gamma\|s-s_0\|$ on $S$ with $\gamma>0$, then
\begin{equation}\label{rs:samplegain}
 \gamma_D=(1-\omega_0)\gamma-K\omega_0-\omega_1>0
\end{equation}
guarantees $\inf_{c\in C(s)}\|g|_D\|_\infty\geq\gamma_D\|s-s_0\|$.
No full-column-rank or constant-rank condition on $B(s)$ is required.
\end{proposition}

\begin{proof}
Keep the exact splitting
\[
 g=f+v,\qquad f=B(s_0)(c-c_0),\quad v=[B(s)-B(s_0)]c.
\]
Integrating $D_sB$ along $s_0+t(s-s_0)$ gives $\|v\|_\infty\leq K\|s-s_0\|$ and $|v(r)-v(r')|_\infty\leq\omega_1\|s-s_0\|$ for nearby queries. Since $f\in V_0$ and $\|f\|_\infty\leq\|g\|_\infty+K\|s-s_0\|$, the between-query variation of $g$ is at most $\omega_0\|g\|_\infty+(K\omega_0+\omega_1)\|s-s_0\|$. Choose a sample within $\delta$ of a maximizing query and rearrange to obtain \eqref{rs:relative}. Positivity in \eqref{rs:samplegain} implies $\omega_0<1$, so taking the infimum over all physical appearances proves the profile bound. The term $v$ includes every reference appearance null direction that opens under the density change.
\end{proof}

\paragraph{Noise and spatial resolution.}
If the observation differs from the reference by at most $b$ per reading and a candidate fits it within $a$, the triangle inequality and \eqref{rs:samplegain} give
\begin{equation}\label{rs:noise}
 \|s-s_0\|\leq(a+b)/\gamma_D.
\end{equation}
Thus two $b$-compatible predictions use $2b$, or $2(b+\tau)$ when both computed predictions have a uniform renderer error $\tau$. All quantities concern the fixed representation and reference. The full-query maximum-RGB coefficient $\gamma$ is an input to this transfer; the continuous $Y$-norm constant $\kappa_L$ is not substituted for it. With Lipschitz bounds $\omega_i\leq H_i\delta$, the sufficient sampled coefficient is $\gamma-\delta[H_0(\gamma+K)+H_1]$.

\paragraph{A physical case with changing rank.}
In \cref{q:twolayer}, take the continuous family of two-lobe pixel readings $F_t=(1-t)F_1+tF_2$, $t\in[0,1]$, with complete SH0 colors $[0,1]^2$. Each residual is affine in $t$, so its full-query maximum is its two-ray maximum and $\gamma=5/[224\log(3/2)]$. At the reference, $V_0=\operatorname{span}\{1+t/2\}$, giving $H_0=1/3$. For $q=e^{-s}\in[1/3,2/3]$, the two density-derivative row norms are bounded by $5/6$ and $14/9$, so $K=14/9$. The mixed query/density derivative has $\ell^1$ norm
\[
 \begin{cases}3q/4-q^2/2,&q\leq1/2,\\7q^2/2-5q/4,&q\geq1/2,\end{cases}
\]
whose maximum is $H_1=13/18$. The $64$ midpoint queries $t_i=(2i+1)/128$ have covering radius $1/128$. Substitution in \eqref{rs:samplegain} gives
\[
 \gamma_D\geq\gamma-(\gamma/3+67/54)/128>.0452.
\]
Consequently $a=b=.001$ implies $|s-s_0|<.04425$. The appearance matrix has rank one at $q=1/2$ and rank two elsewhere, with determinant $q(1-q)(2q-1)/4$. Its reference null vector $(1,-4)$ becomes $(1-2q,1-4q^2)$ in the two readings and is retained in the proof. Rational bounds and vertex enumeration check the formulas. This mixture design demonstrates the transfer at a rank change; the two original endpoint rays already give the stronger exact guarantee in \cref{q:twogain}.

\paragraph{Spatial resolution is an identifiability assumption.}\phantomsection\label{rs:spatialfreedom}
Fixing angular degree alone does not give finite-ray density recovery. Let $\Omega$ be bounded and convex, with known exterior and black background, and let $(\sigma_*,j_*)$ be a smooth physical isotropic reference. Suppose $\mu\leq j_{*,k}\leq\sigma_*-\mu$ on an open $D\Subset\Omega$, for some $\mu>0$. For any fixed finite set $\mathcal R$ of complete oriented chords, every sufficiently small $h\in C_c^\infty(D)$ admits one smooth isotropic source $j_h$, equal to $j_*$ outside $D$, such that
\begin{equation}\label{rs:spatialexact}
 0\leq j_h\leq\sigma_*+h,\qquad
 F_r(\sigma_*+h,j_h)=F_r(\sigma_*,j_*),\quad r\in\mathcal R.
\end{equation}
Thus exact compensation can use spatial freedom while remaining in SH0. Along $h=tg$ for any fixed smooth $g$, both fields approach the reference in $C^\infty$ as $t\to0$.

\begin{proof}
Merge identical oriented chords and discard lines missing $D$. On each remaining geometric line, choose an interval inside $D$ avoiding its finitely many intersections with other lines. Place one nonnegative smooth bump there for one observed orientation, or two ordered bumps for both orientations. Their closed support balls can be chosen mutually disjoint and disjoint from every other measured line, with $0\leq\psi_a\leq1$. The source response $B_{ra}(\sigma)=\int_rT_\sigma\psi_a$ is block diagonal. Single-orientation blocks are positive. For opposite orientations, write $A(t)=\int_0^t\sigma_*$ and $A_{\rm tot}=A(\ell_r)$. The weight ratio is $w_+(t)/w_-(t)=e^{A_{\rm tot}-2A(t)}$, strictly decreasing between the ordered supports since $\sigma_*\geq2\mu$. The $2\times2$ determinant is
\[
 \iint \psi_1(t)\psi_2(u)w_-(t)w_-(u)
 \left[e^{A_{\rm tot}-2A(t)}-e^{A_{\rm tot}-2A(u)}\right]\,dt\,du>0
\]
where $t,u$ run over the ordered trace supports, so $t<u$. Hence $B_*=B(\sigma_*)$ is invertible.

Let $M=\|B_*^{-1}\|_\infty$ and $\ell=\operatorname{diam}\Omega$. Each row has at most two nonzero entries, so $\|B(\sigma_*+h)-B_*\|_\infty\leq2\ell^2\|h\|_\infty$ when the density is nonnegative. Neumann's bound gives inverse norm at most $2M$ for $\|h\|_\infty\leq1/(4M\ell^2)$. For reference readings $y_{*,k}$, set
\[
 \alpha_k=B(\sigma_*+h)^{-1}
       [y_{*,k}-F(\sigma_*+h,j_{*,k})],\qquad
 j_{h,k}=j_{*,k}+\sum_a\psi_a\alpha_{a,k}.
\]
Source linearity proves exact equality. The transport difference and $0\leq U_{*,k}\leq1$ give $\|\alpha_k\|_\infty\leq2M\ell\|h\|_\infty$. Disjoint supports give the same bound on $\|j_h-j_*\|_\infty$. Choosing
\[
 \|h\|_\infty\leq
 \min\{(4M\ell^2)^{-1},\ \mu/[2(1+2M\ell)]\}
\]
preserves both physical margins at $\mu/2$. Fixed smooth bumps also bound every source derivative by a constant times $\|h\|_\infty$. If no measured line meets $D$, take $j_h=j_*$ directly.
\end{proof}
The bump space depends on the ray design. Its required spatial scales and derivative constants need not remain uniform as observations become dense. Finite-dimensional priors in sampling stability theory \cite{alberti2021finite} fix precisely this additional freedom. Our fixed-grid capacity comparisons retain their declared spatial class; \eqref{rs:spatialexact} describes what changes when that class is enlarged. Continuous pixel footprints require the separate source-response argument below.

\paragraph{Continuous pixels retain spatial-source ambiguity.}\phantomsection\label{rs:pixelspatial}
The same issue arises for finite pixels that integrate a continuum of rays. Consider distinct pinhole centers $o_c$ outside $\overline\Omega$ and normalized solid-angle averages over nonoverlapping angular pixels $P_{cp}$, of area $\omega_{cp}>0$. Fix a ball $D\Subset\Omega$. For each camera, the pixels cover all directions hitting $D$, their cone cells in $D$ have planar interfaces and a connected face-adjacency graph, and every adjacency has an open face patch disjoint from all other cameras' interfaces. Retain only cells with nonempty interior in $D$; all other pixels are unaffected by changes supported there.

Under these conditions there are fixed smooth isotropic source functions $\phi_1,\ldots,\phi_m$ supported in $D$ whose pixel-response matrix $A_\sigma$ is invertible for all sufficiently small nonnegative $\|\sigma\|_\infty$. At any physical isotropic reference with this invertibility and margins $\mu\leq j_{*,k}\leq\sigma_*-\mu$ on $D$, every sufficiently small prescribed $h\in C_c^\infty(D)$ has a physical isotropic source $j_h=j_*+\sum_a\phi_a\alpha_a$ with exactly the same pixel averages. Both fields equal the reference outside $D$ and converge to it in $C^\infty$ along $h=tg$. The \hyperref[rs:pixelspatialproof]{proof below} establishes continuous-pixel independence before solving at the changed density.

\phantomsection\label{rs:pixelspatialproof}
\begin{proof}
At zero attenuation, spherical coordinates about each camera give the exact source kernel
\begin{equation}\label{rs:pixelkernel}
 F_{cp}(0,v)=\int_D k_{cp}(x)v(x)\,dx,\qquad
 k_{cp}(x)=\frac{\mathbf1_{C_{cp}}(x)}{\omega_{cp}|x-o_c|^2},
\end{equation}
where $C_{cp}$ is the pixel cone. Suppose $\sum_{cp}a_{cp}k_{cp}=0$ almost everywhere. Across a clean face between pixels $p,q$ of camera $c$, the other cameras' terms are smooth. The jump therefore gives $a_{cp}/\omega_{cp}=a_{cq}/\omega_{cq}$. Connected adjacency makes this ratio one constant $b_c$ per camera, so coverage gives $\sum_c b_c/|x-o_c|^2=0$ on $D$. Real-analytic continuation extends the identity to the connected set $\mathbb R^3\setminus\{o_c\}$. Multiplication by $|x-o_c|^2$ and passage to $o_c$ gives $b_c=0$. Thus all pixel source functionals are independent.

Choose $m$ points away from interfaces whose kernel vectors are independent. Nonnegative, unit-volume smooth bumps $\phi_a$ in small disjoint balls about them yield an invertible $A_0=(F_{cp}(0,\phi_a))$. Fix these bumps and put $B=\|\sum_a|\phi_a|\|_\infty$, $M_0=\|A_0^{-1}\|_\infty$ and $\ell=\operatorname{diam}\Omega$. Positive normalized averaging and $|T_\sigma-T_{\widetilde\sigma}|\leq\ell\|\sigma-\widetilde\sigma\|_\infty$ give
\[
 \|A_\sigma-A_{\widetilde\sigma}\|_\infty
 \leq\ell^2B\|\sigma-\widetilde\sigma\|_\infty
\]
for nonnegative densities. Hence $\|\sigma_*\|_\infty\leq(2M_0\ell^2B)^{-1}$ suffices for invertibility. Positive physical references exist, for example small constant $\sigma_*=\epsilon$ with $j_*=\epsilon(1/4,1/2,3/4)$.

For $M=\|A_{\sigma_*}^{-1}\|_\infty$, choose
\begin{equation}\label{rs:pixelradius}
 \|h\|_\infty\leq
 \min\left\{\frac1{2M\ell^2B},\frac{\mu}{2(1+2MB\ell)}\right\}.
\end{equation}
Then $\|A_{\sigma_*+h}^{-1}\|_\infty\leq2M$. Set
\[
 \alpha_k=A_{\sigma_*+h}^{-1}
 [F(\sigma_*,j_{*,k})-F(\sigma_*+h,j_{*,k})].
\]
The exact transport difference, $0\leq U_{*,k}\leq1$ and normalized averaging bound the bracket by $\ell\|h\|_\infty$. Thus $\|j_h-j_*\|_\infty\leq2MB\ell\|h\|_\infty$; \eqref{rs:pixelradius} leaves both physical margins at least $\mu/2$. Source linearity gives exact nonlinear equality of all pixels. Fixed smooth bumps give convergence in every finite derivative norm along $h=tg$.
\end{proof}
For example, two cameras at $(0,0,3)$ and $(3,0,0)$ viewing $\Omega=B_1$ can each split a rectangular field of view into two pixels, with cone interfaces $x=0$ and $y=0$ in $D=B_{1/4}$. Their clean face patches satisfy the hypotheses. This result uses solid-angle pixel averages and unrestricted spatial source resolution. Uniform sensor-plane weights or overlapping smooth point-spread functions require their own independence argument. The neighborhood size depends on the pixel geometry and source basis; exact finite-pixel agreement does not imply agreement of all constituent rays.

\smallskip
\phantomsection\label{rs:pixelscaling}
\noindent\textbf{Extension along optical-depth scales.}
The small-attenuation condition above can be relaxed along each fixed density shape. For any smooth bounded $s\geq0$, the same bumps give an invertible $A_{\lambda s}$ for every $\lambda>0$ outside a discrete set $E_s$, finite on every bounded interval. At each such scale, every isotropic reference with strict physical margins admits the exact local compensation above. Thus this finite-pixel ambiguity extends beyond optically thin media. The \hyperref[rs:pixelscalingproof]{argument below} concerns the exact exponential response.

\phantomsection\label{rs:pixelscalingproof}
\begin{proof}
Extend the scale to $z\in\mathbb C$. Each entry of $A_{zs}$ integrates $e^{-z\tau_s}\phi_a$, where $0\leq\tau_s\leq\ell\|s\|_\infty$. Its power series is uniformly dominated on every bounded set of $z$, so $A_{zs}$ is entire. Since $\det A_0\ne0$, the zeros of $\det A_{zs}$ are isolated; take $E_s$ to be its positive real zeros. At every other scale, \eqref{rs:pixelradius} applies with $M=\|A_{\lambda s}^{-1}\|_\infty$ and that reference's margin. For the family $(\lambda s,\lambda j)$, the margin is $\lambda\mu_0$ when $\mu_0$ is the margin of $(s,j)$. On compact scale sets avoiding $E_s$ and zero, continuity of the inverse gives a common positive radius.
\end{proof}
An exceptional scale concerns the selected source basis; it need not obstruct other source corrections. Certifying a specified scale requires its response matrix, and the admissible radius can shrink at large optical depths.

\paragraph{Checking the full-appearance coefficient.}\phantomsection\label{rs:fiber}
The positive coefficient in \eqref{rs:samplegain} must cover every physical appearance that fits the reference data. The following sufficient test supplies this input without requiring a unique appearance fit. Take a fixed compact physical coefficient set $C$, put $y_0=B(s_0)c_0$, and retain $V_0=\operatorname{ran}B(s_0)$. Define
\begin{align}
 C_\eta&=\{c\in C:\|B(s_0)c-y_0\|_\infty\leq\eta\},\nonumber\\
 \mu_\eta&=\min_{\substack{c\in C_\eta\\\|u\|=1}}
 \operatorname{dist}_\infty(D_sB(s_0)[u]c,V_0).\label{rs:fibergain}
\end{align}
$C_0$ contains all appearances producing the reference image, including reference null directions. $\mu_\eta$ measures the smallest density response left outside the reference appearance range for its thickened fiber. Distance uses the maximum-reading norm. Assume a uniform Taylor remainder
\[
 M\|B(s_0+h)-B(s_0)-D_sB(s_0)[h]\|
 \leq\epsilon_r\|h\|,\qquad \|h\|\leq r,
\]
within the star-shaped density prior. Differentiability gives $\epsilon_r\to0$; a bound $M\|D_s^2B\|\leq H$ gives $\epsilon_r=Hr/2$. Then
\begin{equation}\label{rs:fiberbound}
 \|B(s)c-y_0\|_\infty\geq
 \min\{\mu_\eta-\epsilon_r,\eta/r-K\}\|s-s_0\|
\end{equation}
for every $c\in C$ and $\|s-s_0\|\leq r$ in the prior.

\emph{Proof of \eqref{rs:fiberbound}.}
Write $t=\|s-s_0\|>0$. If $c\in C_\eta$, quotienting the Taylor expansion by $V_0$ removes $B(s_0)c-y_0$ and leaves at least $(\mu_\eta-\epsilon_r)t$. Otherwise, the exact splitting used above gives $\|B(s)c-y_0\|_\infty>\eta-Kt\geq(\eta/r-K)t$. The case $t=0$ is immediate. Compactness of $C$ and the density unit sphere gives $\mu_\eta\to\mu_0$ as $\eta\downarrow0$: convergent subsequences of minimizers land in $C_0$, and distance to $V_0$ is continuous. Thus $\mu_0>0$ yields a positive coefficient for sufficiently small $\eta$, then $r$. This coefficient can be inserted into \eqref{rs:samplegain}. The density neighborhood is prescribed; the entire compact appearance class remains admissible.

For the two-layer model above, $\operatorname{dist}_\infty(D_sB(s_0)c,V_0)=c_1/10+3c_2/40$. On $C_\eta$, eliminating $c_1$ gives $\mu_\eta=1/16-2\eta/15$ for $0\leq\eta\leq15/32$, attained at $(5/8-4\eta/3,0)$. Row sums give $K=14/9$ and $H=28/9$. With $r=.001$ and $\eta=.005$, \eqref{rs:fiberbound} therefore supplies $\gamma_r=217/3600>.06027$ over all physical colors. The preceding $64$ pixel mixtures retain $\gamma_D\geq23237/460800>.05042$. At $a=b=10^{-5}$, \eqref{rs:noise} gives $|s-s_0|<.000397$ within this density prior.

\paragraph{Why one fitted appearance is insufficient.}
Consider two known front segments with transmissions $1/2,2/5$, followed by a shared unknown rear density $s$ with lengths $1,2$. With shared front/rear colors $(c_f,c_b)\in[0,1]^2$, black background, and $q=e^{-s}$, exact segment compositing gives
\[
 F_1=\tfrac12c_f+\tfrac12(1-q)c_b,\qquad
 F_2=\tfrac35c_f+\tfrac25(1-q^2)c_b.
\]
At $q_0=1/2$, $c_0=(1/2,1/2)$, the image is $(3/8,9/20)$ and $V_0=\operatorname{span}\{(1,6/5)\}$. The density derivative is $(c_b/4,c_b/5)$; its maximum-norm distance from $V_0$ is $c_b/22$, hence $1/44>0$ at $c_0$. Yet the physical color $(3/4,0)$ gives the same image for every $s$. It lies in $C_0$ and makes $\mu_0=0$. The test at one appearance misses this alternative. The example concerns unrestricted physical refitting; joint local density--appearance recovery uses a neighborhood of both variables. Conversely, the quotient test is only sufficient: active physical bounds can preserve a signal even when the relaxed distance in \eqref{rs:fibergain} vanishes.

\paragraph{A positive coefficient with 27 density nodes.}\phantomsection\label{rs:multivoxel}
The complete shared SH1 class also admits a certified positive local coefficient in a spatial voxel model. We use the $384$ RGB rays of the $27$-node design, with all three channels and $64$ segments per ray. Each node carries one density and four source coefficients per channel, in the basis $(1,d_x,d_y,d_z)$. All $324$ source coordinates may change, subject to the full-direction physical constraints $q_{v0}\geq\|q_{v,1:3}\|_2$ and $q_{v0}+\|q_{v,1:3}\|_2\leq s_v$ for each channel. The exterior is known vacuum and the background is black. The density loss is the $L^2$ norm of its Q1 interpolant on the unit cube.

The stored binary sampling weights and segment lengths define the exact finite compositor; stored direction vectors are normalized exactly. The stored isotropic source and density define the reference $(s_*,q_*)$. For every physical candidate with $\|s-s_*\|_\infty\leq r$, the maximum of its $1{,}152$ raw RGB differences obeys
\begin{align}
 \|F(s,q)-F(s_*,q_*)\|_\infty
 &\geq\gamma_r\|\sigma_{s-s_*}\|_{L^2},\label{rs:voxelbound}\\
 \gamma_r&=\frac{1-10^{-8}}{2401}-11e^{2r}r.\nonumber
\end{align}
The proof below certifies $\gamma_r>4.054\cdot10^{-4}$ at $r=10^{-6}$ and $\gamma_r>3.064\cdot10^{-4}$ at $r=10^{-5}$. Thus a finite spatial representation can retain density evidence after unrestricted physical SH1 refitting. The neighborhood restricts density alone.

\emph{Proof of \eqref{rs:voxelbound}.}
Put $h=s-s_*$, $v=q-q_*$, $z=(h,v)$ and $J=D F(s_*,q_*)$. An outward interval calculation of the entire $1{,}152\times351$ Jacobian and a stored approximate left inverse $R$ gives
\[
 \|I-RJ\|_\infty<10^{-8},\qquad \|R\|_\infty<2401.
\]
These are induced maximum norms. The certified values before rounding up are $6.8129\cdot10^{-9}$ and $2400.8921$, respectively. Along a ray let $a(u)$ be the accumulated density change in the segment model, and let $\delta j$ be its source change. The exact nonlinear remainder is
\[
 e=\int T_*\big[(e^{-a}-1+a)j_*+(e^{-a}-1)\delta j\big]du.
\]
The nonnegative sampling weights have row sums $P\leq1.000001$, and ray length $\ell\leq7/4$. Unit directions give $|a|\leq P\ell\|h\|_\infty$ and $|\delta j|\leq3P\|v\|_\infty$. Physicality gives $\int T_*j_*\leq1$. Consequently,
\[
 \|e\|_\infty\leq 11e^{2r}r\|z\|_\infty.
\]
This retains the mixed density--source term for every $v$. Since $g=F(s,q)-F(s_*,q_*)=Jz+e$, the inverse bounds give $(1-10^{-8})\|z\|_\infty\leq2401(\|g\|_\infty+\|e\|_\infty)$. The Q1 bound $\|\sigma_h\|_{L^2}\leq\|h\|_\infty$ proves \eqref{rs:voxelbound}.

\emph{Verification and scale.}
Every Jacobian and inverse-product entry is enclosed using outward elementary arithmetic and an exponential Taylor series with an explicit tail. An independent $70$-digit backward transport calculation checks $7{,}371$ predetermined Jacobian entries, $49$ inverse-product entries and all $351$ inverse and defect row bounds. The complete matrix enclosure and the analytic remainder establish the uniform claim. At $r=10^{-4}$ this sufficient bound is nonpositive. The positive result is a local certificate for the declared finite compositor; its small radius and coefficient do not supply a one-gray-level recovery guarantee. The files include the exact inputs, inverse, enclosures and both implementations.

\emph{A larger certified density neighborhood.}
Retaining each ray's attenuation moments sharpens the same full-appearance guarantee. With the identical reference, rays and physical SH1 class, the coefficient in \eqref{rs:voxelbound} can be replaced by
\begin{equation}\label{rs:raywisegain}
 \widetilde\gamma_r=
 \frac{1-7\cdot10^{-8}-356r\exp(1.413r)}{2401}.
\end{equation}
It exceeds $2.680\cdot10^{-4}$ at $r=.001$ and $1.191\cdot10^{-4}$ at $r=.002$. The certified nodal density radius therefore expands from $10^{-5}$ to $.002$ while all $324$ source coefficients remain free. This improvement comes from a tighter estimate of the same nonlinear residual.

\emph{Proof of \eqref{rs:raywisegain}.}
Set $D=\operatorname{diag}(2401I_{27},359I_{324})$ and $N=\|D^{-1}z\|_\infty$. Full interval row sums verify $\|D^{-1}R\|_\infty<1$ and $\|D^{-1}(I-RJ)D\|_\infty<7\cdot10^{-8}$. For ray $e$, let $p_e(u)$ be the sum of its stored sampling weights, $\lambda_e(u)=\int_0^u p_e(w)dw$, and $\nu_e=1+\|d_e\|_1$. For each channel define the nonnegative reference moments
\[
 \alpha_{ec}=\tfrac12\int T_*j_{*,c}\lambda_e^2du,
 \qquad \beta_e=\nu_e\int T_*p_e\lambda_e du.
\]
The exact remainder above obeys
$|e_{ec}|\leq e^{1.413r}r(\alpha_{ec}\|h\|_\infty+\beta_e\|v\|_\infty)$, since $|a(u)|\leq\lambda_e(u)\|h\|_\infty$ and $\max_e\lambda_e(\ell_e)<1.413$. Write $\alpha,\beta$ as reading vectors, repeating $\beta_e$ across channels. The complete interval calculation gives
\[
 \max_i\frac{\sum_{ec}|R_{i,ec}|(2401\alpha_{ec}+359\beta_e)}{D_{ii}}
 <355.474<356.
\]
Hence $N\leq\|g\|_\infty+(7\cdot10^{-8}+356r e^{1.413r})N$. Using $N\geq\|\sigma_h\|_{L^2}/2401$ proves the bound whenever its coefficient is positive. Moments are integrated exactly over each constant segment with outward arithmetic; an independent $70$-digit incomplete-gamma calculation checks the moment implementation and all $351$ scaled row bounds. At $r=.005$ this sufficient coefficient is nonpositive. The larger certified neighborhood retains a small noise scale: at $r=.002$, $\widetilde\gamma_r r$ is about $2.38\cdot10^{-7}$ in normalized RGB.

\paragraph{Density recovery with nonunique shared appearance.}\phantomsection\label{rs:sharedfiber}
The same layered construction also gives an exact positive example with $p$ unknown rear densities. Let $s_i\geq0$ and $q_i=e^{-s_i}$. The $i$th ray pair has the front transmissions and rear lengths above, one rear RGB color $c_i\in[0,1]^3$, and a front RGB color $f\in[0,1]^3$ shared across \emph{all} pairs:
\[
 F_{i1}=f/2+(1-q_i)c_i/2,\qquad
 F_{i2}=3f/5+2(1-q_i^2)c_i/5.
\]
Black background and exact segment integration are used. At $s_i=\log2$, the strictly interior colors $f=7/8$, $c_i=3/4$ give the RGB reference $y_i=(5/8,3/4)$ in every channel. Its complete appearance fiber is $f\in[3/4,1]$, $c_i=5/2-2f$, independently in each channel. Appearance remains nonunique across all $p$ regions.

Define $\rho(s)$ as the smallest maximum RGB discrepancy from this reference over every physical shared color. For $0<q_i\leq3/4$, its exact value is
\begin{equation}\label{rs:sharedprofile}
 \rho(s)=\max_i\frac{|2e^{-s_i}-1|}{4(9+4e^{-s_i})}.
\end{equation}
Thus the density vector is uniquely determined at zero discrepancy although the joint parameters are not. Moreover,
$\rho(\log2+h)=\|h\|_\infty/44+O(\|h\|_\infty^2)$,
so the optimal limiting density coefficient is $44$ when loss is the maximum density-coordinate error. The derivation and attaining colors are given \hyperref[rs:sharedproof]{below}. For $p=27$, the $54$ RGB rays leave three appearance-only Jacobian null directions; the joint rank is $108$ among $111$ unknown coordinates.

\emph{Proof of \eqref{rs:sharedprofile} and its consequences.}\label{rs:sharedproof}
The functional $\lambda_i=(4(1+q_i),-5)/(9+4q_i)$ has $\ell^1$ norm one and annihilates the rear-color column. Applying it to the two residuals gives, channelwise,
\[
 \|F_i-y_i\|_\infty\geq
 \frac{|2q_i-1|(5/4-f)}{9+4q_i}
 \geq\frac{|2q_i-1|}{4(9+4q_i)}.
\]
One common front color $f=1$ attains every regional bound with
$c_i=11/[4(1-q_i)(9+4q_i)]$; these colors lie in $[0,1]$ for $q_i\leq3/4$. The paired residuals are equal and opposite. This proves the profile over the entire physical appearance class. Taylor expansion at $s_i=\log2$ gives the stated coefficient, and a change in one coordinate attains it. At the reference, the contrast $(-6/5,1)$ removes appearance and leaves $-c_i h_i/10$. Hence each Jacobian null direction has $h_i=0$ and $\delta c_i=-2\delta f$, proving the rank and kernel statements.

The lower bound remains valid for every nonnegative density, including $q_i>3/4$. If the reference and a candidate fit one observation within $b=1/255$ each, then $\rho(s)\leq2/255$ implies the sharp bounds
\begin{equation}\label{rs:sharednoise}
 \log(478/327)\leq s_i\leq\log(542/183)
 \quad\text{for every }i.
\end{equation}
Indeed, solving the lower-bound inequality gives $183/542\leq q_i\leq327/478$. The attaining colors above are physical throughout this box, and the midpoint of candidate and reference RGB supplies a common noisy observation. Thus the interval $[.37965,1.08578]$, with exact endpoints in \eqref{rs:sharednoise}, describes possible density values compatible with \emph{some} such noisy observation. No small density-neighborhood prior is needed.

The positive signal is enforced by the physical color bound: the first reference reading exceeds the front layer's maximum contribution $1/2$, forcing rear emission. With lower reference data $(a,6a/5)$, $0\leq a\leq1/2$, the shared colors $f=2a$, $c_i=0$ instead fit every density vector. This distinguishes density recovery from appearance uniqueness within the same observation model.

\paragraph{Local recovery and a finite-noise capacity gap.}\phantomsection\label{rs:coupledwitness}
The coupled $27$-node model of \cref{rs:multivoxel} admits both certified local SH1 recovery and a finite density error concealed at one-gray noise. Keep its reference, all $384$ RGB rays, $64$ segments, source parameterization, black background and spatial $L^2$ loss. Candidate densities are nonnegative, and sources obey the full nodal physical bounds $q_0\geq\|q_{1:3}\|_2$ and $q_0+\|q_{1:3}\|_2\leq s$. Let $\rho_L(s)$ denote the optimal maximum RGB discrepancy from this reference over all physical degree-$L$ sources at density $s$.

For a saved candidate $s^\dagger$, outward evaluation certifies
\begin{equation}\label{rs:coupledgap}
\begin{aligned}
 \|\sigma_{s^\dagger}-\sigma_{s_*}\|_{L^2}&>.10099999,\\
 2.18174921<255\rho_0(s^\dagger)&<2.18174922,\\
 255\rho_1(s^\dagger)&<1.05804189.
\end{aligned}
\end{equation}
The SH0 bound covers every physical isotropic source; the SH1 bound is attained or improved by a saved globally physical directional source. These two certificates establish a strict capacity gap at one fixed density. Keeping the reference color instead leaves $14.98537$--$14.98538$ gray levels. The candidate has minimum nodal density above $.13632$ and nodal maximum change $1.55915$. The local bound \eqref{rs:raywisegain} applies within $.002$, explaining how positive local recovery and this finite-amplitude alternative coexist. The certificate and common-observation construction are proved \hyperref[rs:coupledproof]{below}.

\emph{One observation, two capacity decisions.}
Write $y_*$ for the reference RGB and $y_1$ for the saved SH1 candidate. Their midpoint $z=(y_*+y_1)/2$ obeys
\begin{equation}\label{rs:couplednoise}
\begin{aligned}
255\max\{\|z-y_*\|_\infty,\|z-y_1\|_\infty\}&<.52902095,\\
255\inf_{q\in\mathcal C_0(s^\dagger)}
\|F(s^\dagger,q)-z\|_\infty&>1.65272826.
\end{aligned}
\end{equation}
Thus the same one-gray observation budget accepts the reference and the SH1 alternative, while rejecting that alternative under every SH0 appearance. On any prior containing these two SH1 scenes, every density estimator has worst-case $L^2$ error greater than $.05049$ at noise $1/255$ per reading. The midpoint supplies a worst-case bounded-noise witness.

\emph{Proof and certificate of \eqref{rs:coupledgap}--\eqref{rs:couplednoise}.}\label{rs:coupledproof}
At fixed density, let $G(s)$ map isotropic source nodes to one RGB channel. For every $\|\lambda\|_1\leq1$ and $0\leq q\leq s$, the dual norm and the box support function give
\[
 \|G(s)q-y_*\|_\infty\geq
 \lambda^Ty_*-\sum_v s_v\max\{(G(s)^T\lambda)_v,0\}.
\]
Three saved dual vectors, normalized with outward $\ell^1$ bounds, give the SH0 lower certificate. Physical fitted sources give both upper certificates by exact segment compositing. Every matrix coefficient and all $1{,}152$ readings are enclosed with the exponential-series arithmetic of \cref{rs:multivoxel}. The mass matrix is $M_1\otimes M_1\otimes M_1$, where $M_1=\frac1{12}\bigl(\begin{smallmatrix}2&1&0\\1&4&1\\0&1&2\end{smallmatrix}\bigr)$; outward evaluation of $(s^\dagger-s_*)^TM(s^\dagger-s_*)$ gives the density separation. Nodal source bounds imply $0\leq j(x,d)\leq\sigma(x)$ for every position and direction by nonnegative Q1 interpolation.

The first line of \eqref{rs:couplednoise} follows by halving the SH1 discrepancy. The reverse triangle inequality gives its second line from $2.18174921-.52902095=1.65272826$. Finally, at this common observation, the triangle inequality prevents both estimation errors from being smaller than half the density separation. These arguments depend on the saved physical and dual witnesses, independently of solver convergence.

\emph{Construction and complete sweep.}
A single density direction is selected by the smallest singular vector of the reference density Jacobian after eliminating all shared SH1 source columns, normalized in the Q1 mass norm. Its largest-magnitude node fixes the sign. Both signs at amplitudes $.01,.03,.101,.201$ are tested. SH0 uses the complete physical box; SH1 fits use the sufficient constraints $\|q_{1:3}\|_1\leq\min(q_0,s-q_0)$ and are then checked against the full physical bounds. Their feasible values upper-bound the complete SH1 optimum. All positive-density cases are listed below; the remaining three cases violate density nonnegativity.

\begin{center}\small
\begin{tabular}{@{}crrr@{}}\toprule
Signed amplitude & Fixed color & SH0 optimum & SH1 upper\\\midrule
$-.01$ & $1.643$ & $[.24026008,.24026009]$ & $.033858$\\
$+.01$ & $1.613$ & $[.23539893,.23539895]$ & $.033851$\\
$-.03$ & $5.021$ & $[.73600890,.73600892]$ & $.214167$\\
$+.03$ & $4.749$ & $[.69233295,.69233297]$ & $.101518$\\
$+.101$ & $14.986$ & $[2.18174921,2.18174922]$ & $1.058042$\\\bottomrule
\end{tabular}
\end{center}
All entries are maximum channel discrepancy in linear gray levels; fixed-color values are rounded upper bounds. Source-unchanged diagnostics are retained in the artifacts; they violate physical bounds for the primary candidate, so fixed color supplies the physical baseline. Independent reverse compositing checks all $11{,}520$ fitted RGB scalars, with $105$ additional 75-digit checks and all nodal physical margins. The construction uses the declared finite renderer and query set; its angular physicality holds on the whole sphere.

\paragraph{Continuous integration preserves the capacity gap.}\phantomsection\label{rs:coupledcontinuous}
The saved $.101$-$L^2$ density alternative retains its one-gray SH0--SH1 separation when the $384$ original segments are rendered by continuous Q1 integration. Both endpoints of each saved chord define its exact path, length $\ell$, and unit direction; coordinate clamping extends the Q1 fields at sub-ulp boundary departures. Density and source coefficients remain fixed. The continuous SH1 discrepancy is below $1.131407$ gray levels, whereas the complete physical SH0 optimum exceeds $2.124238$. At the midpoint of the two continuous SH1 images, the reference and alternative each have error below $.565704$ levels; every SH0 source at the alternative density has error above $1.558535$. Thus continuous integration preserves the opposite capacity decisions at one-gray noise. The following error bound and dual transfer prove this statement.

On the clamped Q1 grid of spacing $1/2$, assume $0\le j\le\sigma\le S$. Let $R_N$ use $N$ equal midpoint segments and let $R$ denote the continuous integral. Then
\begin{equation}\label{eq:coupledquad}
 |R-R_N|\le E_N(S,\ell):=
 \left[\frac{\ell(18S^2+48S)}{12}+27S\right]\left(\frac{\ell}{N}\right)^2.
\end{equation}
This bound controls integration error uniformly over the physical source class, including refitted sources. Its proof below accounts for the Q1 derivative jumps.

\noindent\textit{Proof of \cref{eq:coupledquad} and certificate transfer.}
Within a Q1 cell, coordinate derivatives are bounded by $2S$ and mixed second derivatives by $8S$; consequently ray derivatives satisfy $|f'|\le6S$, $|f''|\le48S$. Global continuity retains the first bound as a Lipschitz bound. Let $T$ be the continuous entry transmittance and $U_N$ the midpoint field's remaining radiance. The transport identity gives
\[
 R-R_N=\int T\big[(j-j_N)-(\sigma-\sigma_N)U_N\big]dt.
\]
The weights $T$ and $TU_N$ are bounded by one and have derivative bounds $S$ and $2S$, respectively. On a smooth interval of width $\Delta$, the midpoint linear term cancels against a constant weight; the weight variation and quadratic remainder bound both source and density contributions by $(18S^2+48S)\Delta^3/12$. At most nine intervals contain crossings of the planes $x_i=-1/2,0,1/2$. Each such interval contributes at most $3S\Delta^2$ by Lipschitz continuity. Summation with $\Delta=\ell/N$ proves the bound.

At $N=1024$, outward interval arithmetic encloses every midpoint RGB value and source-matrix coefficient. Add $E_N$ for the reference and candidate to bound their continuous SH1 discrepancy. For a saved SH0 dual vector $\lambda$, subtract $\sum_r|\lambda_r|(E_{N,\mathrm{ref},r}+E_{N,\mathrm{alt},r})$ from its midpoint lower bound. The dual norm is at most one, so this subtraction covers every physical SH0 source. The resulting bounds above imply the common-observation statement by the same midpoint and triangle-inequality argument as \cref{rs:couplednoise}. The local inverse estimate \cref{rs:raywisegain} retains its stated finite compositor; this absolute integration bound transfers the finite-amplitude certificate.

\paragraph{Changing the queries tests a separate property.}
Keep the same source and evaluate all $425$ previously declared held-out rays: $17$ full-sphere Fibonacci directions with phase $.81$ and a $5\times5$ offset grid on $[-.34,.34]^2$. Their continuous SH1 discrepancy lies in $[2.880432,2.984881]$ gray levels. Its lower bound exceeds two, so this particular pair cannot fit a common observation within one gray level each on that query set. This separates stability under integration refinement from transfer to new queries. The statement concerns the saved source, rather than the optimum over alternative SH1 refits. The retained numerical maxima are:
\begin{center}\small
\begin{tabular}{lrrr}
\toprule
Queries & $N=64$ & $N=256$ & $N=1024$\\
\midrule
Original $384$ & $1.05804$ & $1.05973$ & $1.05980$\\
Held-out $425$ & $2.93255$ & $2.93266$ & $2.93266$\\
\bottomrule
\end{tabular}
\end{center}
All entries are maximum SH1 RGB discrepancies in linear gray levels. Independent reverse compositing checks all $29{,}124$ scalar outputs, and separate polynomial-density quadrature checks all $809$ continuous segments at two quadrature orders. The bounds use \cref{eq:coupledquad}, not the observed convergence of these numerical checks.

\paragraph{Local recovery for the same continuous renderer.}\phantomsection\label{rs:continuouslocal}
The continuous $27$-node Q1 reference in \cref{rs:coupledcontinuous} also has a positive local density gain over its full physical SH1 source class. Write $F$ for its $384$ continuous RGB ray integrals and $h$ for the density-node increment. For every $\|h\|_\infty\le10^{-4}$ and every physical shared SH1 source $j_*+q$ at $\sigma_*+\sigma_h$,
\begin{equation}\label{eq:continuouslocalgain}
 \|F(\sigma_*+\sigma_h,j_*+q)-F(\sigma_*,j_*)\|_\infty
 \ge 3.3116\times10^{-4}\|\sigma_h\|_{L^2}.
\end{equation}
Thus the same reference and continuous rendering map support local recovery and the $.101$-$L^2$ finite-amplitude hiding construction. The source may change throughout its physical class; the small neighborhood restricts density. The gain measures raw linear RGB per unit density, and its \hyperref[rs:continuouslocalproof]{proof below} controls both integration and nonlinear errors uniformly over all shared source coefficients.

\noindent\textit{Proof.}\phantomsection\label{rs:continuouslocalproof}
For one ray of length $\ell$, put $\Delta=\ell/N$. Let $J$ be the continuous derivative and $J_N$ the exact midpoint compositor's derivative. For signed Q1 density and source increments with nodal bounds $H$ and $Q$, respectively,
\begin{equation}\label{eq:continuousderivative}
 \begin{aligned}
 |(J-J_N)[h,q]|&\le\Delta^2(C_h H+C_q Q),\\
 C_q&=S(2\ell^2+16\ell)+4\ell+27,\\
 C_h&=S^2(2\ell^3+31\ell^2/2)+2S\ell^2+14S\ell+4\ell+30.
 \end{aligned}
\end{equation}
where the reference obeys $0\le j_*\le\sigma_*\le S$. This relative error bound vanishes with the parameter increment.

To derive it, a signed Q1 nodal bound $B$ gives ray-derivative bounds $12B,96B$ within cells. Midpoint cancellation bounds complete smooth-cell primitive errors by $4B\Delta^3$; each of at most nine grid/clamp crossing cells contributes $3B\Delta^2$, as does the final partial cell. Hence $\sup_t|\int_0^t(f-f_N)|\le B(4\ell+30)\Delta^2$, halved for nodal values in $[0,S]$. Put $A=\int_0^t\sigma_*$, $B_h=\int_0^t h$, $T=e^{-A}$, and use subscript $N$ for the midpoint fields and their exact primitives. Then $J[h,q]=\int T(q-j_*B_h)$. Subtract its midpoint analogue into four terms:
\[
 \int(T-T_N)(q-j_*B_h)+\int T_N(q-q_N)
 -\int T_N B_h(j_*-j_{*,N})-\int T_Nj_{*,N}(B_h-B_{h,N}).
\]
The primitive bounds control the first and last terms, using $\int T_Nj_{*,N}\le1$. In the two middle terms, midpoint cancellation uses weights $T_N$ and $T_NB_h$, with Lipschitz bounds $S$ and $(S\ell+1)H$. The resulting bounds are $Q[\ell(S+4)+27]\Delta^2$ and $H[\ell\{S(S\ell+1)/2+2S\ell\}+27S\ell/2]\Delta^2$. Their sum proves \cref{eq:continuousderivative}, including all Q1 interfaces.

Let $v$ contain the $324$ source increments and $D=\operatorname{diag}(2401I_{27},359I_{324})$. With $t=\|D^{-1}(h,v)\|_\infty$, the SH1 basis $(1,d_1,d_2,d_3)$ gives $H\le2401t$, $Q\le1077t$. Applying the inverse-residual criterion~\cite{Rump2010} at $N=4096$, outward evaluation of all $1152\times351$ derivative entries and the unchanged saved left-inverse approximation $R$ certifies $\|D^{-1}R\|_\infty\le1$ and
\[
 \|D^{-1}(I-RJ)D\|_\infty\le .028870.
\]
Here the finite derivative defect is below $.014989$; applying $D^{-1}|R|$ to the raywise bound in \cref{eq:continuousderivative} contributes less than $.013882$.

The exact exponential expansion retains the density--source cross term. At $\|h\|_\infty\le r$, its raywise nonlinear remainder is at most
$e^{\ell r}\ell^2(H^2+3H\|v\|_\infty)/2\le1739r e^{\ell r}\ell^2t$.
Applying $D^{-1}|R|$ bounds this by $.176011t$ at $r=10^{-4}$. Therefore the RGB residual is at least $(1-.028870-.176011)t$. Nonnegative Q1 interpolation on the unit-volume cube gives $\|\sigma_h\|_{L^2}\le\|h\|_\infty\le2401t$, proving \cref{eq:continuouslocalgain}. The same certificate gives gain $3.77759\times10^{-5}$ at $r=5\times10^{-4}$; its sufficient expression is nonpositive at $r=10^{-3}$ and $.002$.

Independent reverse differentiation checks every midpoint Jacobian entry, and piecewise-polynomial integration with two Gaussian quadrature orders checks every continuous ray. The certificate itself uses the uniform bound \cref{eq:continuousderivative} and the nonlinear remainder. The finite-amplitude result and the held-out-query outcome in \cref{rs:coupledcontinuous} retain their stated scope.

\paragraph{Stability between data-compatible reconstructions.}\phantomsection\label{rs:continuouspairwise}
The same certificate also bounds the difference between two admissible reconstructions. This makes the recovery statement independent of which fitting algorithm produced them. Keep the density-node prior $r=10^{-4}$ and let $K_\epsilon$ contain all physical shared SH1 fields $z=(s,v)$ satisfying $\|s-s_*\|_\infty\le r$ and $\|F(z)-F(z_*)\|_\infty\le\epsilon$. With $\delta=.028870$, $\beta=.176011$, $a=1-\delta-\beta=.795119$, and $A=2401$, every $z_1,z_2\in K_\epsilon$ obeys
\begin{equation}\label{eq:continuouspairwise}
 \begin{aligned}
 \|F(z_1)-F(z_2)\|_\infty
 &\ge\gamma(\epsilon)\|\sigma_1-\sigma_2\|_{L^2},\\
 \gamma(\epsilon)&=\frac{1-\delta}{A}-\frac{2\beta\epsilon}{ar},
 \qquad \gamma(\epsilon)>0.
 \end{aligned}
\end{equation}
At $\epsilon=10^{-8}$, $\gamma(\epsilon)\ge3.6019\times10^{-4}$. The constant measures the raw RGB discrepancy that must remain between two different densities after both sources have been refitted. If an observation $y$ is within $\eta=5\times10^{-9}$ of $F(z_*)$, any two physical SH1 reconstructions fitting $y$ within the same $\eta$ and satisfying the density prior differ by at most $2.777\times10^{-5}$ in spatial density $L^2$. Indeed they belong to $K_{2\eta}$ and their RGB difference is at most $2\eta$. These noise levels are much smaller than $1/255$; the bound establishes a conditional inverse, not a one-gray recovery guarantee.

\noindent\textit{Proof.}\phantomsection\label{rs:continuouspairwiseproof}
The preceding joint estimate gives $\|D^{-1}(z-z_*)\|_\infty\le\epsilon/a$ for $z\in K_\epsilon$, so the data condition controls all source coefficients as well as density. To pass from reference-centered to pairwise recovery, retain the fixed left preconditioner $R$ and control its derivative throughout the parameter triangle spanned by $z_*,z_1,z_2$. Uniform derivative control is the standard mechanism behind inverse Lipschitz estimates~\cite{alberti2021finite}; here all constants come from the preceding certificate.

For increments $u_i=(h_i,q_i)$, differentiation under the finite ray integral gives
\[
 D^2F_k(z)[u_1,u_2]
 =\int T\bigl(B_{h_1}B_{h_2}j-B_{h_1}q_2-B_{h_2}q_1\bigr),
 \qquad B_h(t)=\int_0^t h.
\]
At a physical field, $T\le1$ and $\int Tj\le1$. Density bounds $H_i$ and source bounds $Q_i$ therefore give the raywise upper bound $\ell_k^2[H_1H_2+(H_1Q_2+H_2Q_1)/2]$. Writing $t_i=\|D^{-1}u_i\|_\infty$, we have $H_i\le A t_i$ and $Q_i\le1077t_i$. The bound becomes $2A\cdot1739\ell_k^2t_1t_2$. Since the certified remainder above uses $b_k=1739r e^{\ell_k r}\ell_k^2$ and $\|D^{-1}|R|b\|_\infty\le\beta$, it follows that
\[
 \|D^{-1}R D^2F(z)[u_1,u_2]\|_\infty
 \le (2A\beta/r)t_1t_2.
\]
Physical admissibility is convex in the independent density and source coordinates. Thus the entire parameter triangle is physical and satisfies $\|D^{-1}(z-z_*)\|_\infty\le\epsilon/a$, although $K_\epsilon$ itself need not be convex. Integrating the second-derivative bound from $z_*$ to each point on the $z_1$--$z_2$ segment bounds the scaled preconditioned derivative change by $\theta=2A\beta\epsilon/(ar)$. Integrating along that segment and applying the reference derivative defect gives
$(1-\delta-\theta)\|D^{-1}(z_1-z_2)\|_\infty\le\|F(z_1)-F(z_2)\|_\infty$.
Finally $\|\sigma_1-\sigma_2\|_{L^2}\le A\|D^{-1}(z_1-z_2)\|_\infty$, proving \cref{eq:continuouspairwise}. The fixed preconditioner prevents cancellation; integrating derivative norm lower bounds alone would not prove the result.

The data condition matters: two distinct positive densities with $j_1=j_2=0$ produce identical black images, even within the density prior. Hence a positive pairwise gain cannot hold over its entire physical source class at arbitrary image values. The sufficient expression in \cref{eq:continuouspairwise} becomes nonpositive above $\epsilon\approx9.136\times10^{-8}$; this loss of a bound does not establish ambiguity. The SH3 construction below establishes exact partners separately, at the same reference and finite query set.

\paragraph{Noise relative to the density prior.}\phantomsection\label{rs:priornoise}
The positive local coefficient above concerns a small density prior. At one-gray noise, that entire prior is already compatible with the reference image under SH0. This distinguishes information supplied by RGB from accuracy supplied by the prior itself.

For the same continuous Q1 reference on the unit-volume cube $C=[-1/2,1/2]^3$, let $\mathcal S_r=\{s:\|s-s_*\|_\infty\le r\}$. Its isotropic source $j_*$ has nodal absorption slack $\min_{v,k}(s_{*,v}-j_{*,v,k})>.2507$. For $0<r\le.002$, every $s\in\mathcal S_r$ therefore admits the unchanged physical source $j_*$ and satisfies, over all external rays,
\begin{equation}\label{eq:priornoise}
 \|F(s,j_*)-F(s_*,j_*)\|_\infty\le\sqrt{3}\,r.
\end{equation}
Thus, if the deterministic per-reading noise bound is $\eta\ge\sqrt{3}r$, the density projection of the physical $\eta$-compatible class at $y_*=F(s_*,j_*)$ is all of $\mathcal S_r$. Its spatial-$L^2$ diameter is exactly $2r$, and its optimal worst-case reconstruction error is exactly $r$. At $r=10^{-4}$, \cref{eq:priornoise} is below $.044168$ gray levels; at $r=.002$, it is below $.88335$. Both are smaller than one gray level. Sharper inverse constants alone cannot improve the worst-case accuracy of these priors at that noise level.

\noindent\textit{Proof.}
Nonnegative Q1 interpolation gives $|h|=|\sigma_s-\sigma_*|\le r$ and transfers $0\le j_*\le\sigma_s$ from nodes to the cube. The exact transport difference with unchanged source is $F(s,j_*)-F(s_*,j_*)=-\int T_s hU_*$. Since $T_s,U_*\in[0,1]$ and every cube chord has length at most $\sqrt3$, this proves \cref{eq:priornoise}. The two admissible endpoints $s_\pm=s_*\pm r\mathbf1$ are both compatible with $y_*$, using errors $e_\pm=y_*-F(s_\pm,j_*)$ of size at most $\eta$. Partition of unity makes their density difference the constant $2r$ on a cube of volume one. The triangle inequality forces any estimate at $y_*$ to have error at least $r$ for one endpoint. Returning $\sigma_*$ attains error at most $r$ for every density in the prior. The same upper and lower arguments give the global minimax error $r$ over this prior and bounded noise.

The effective source stays fixed; the color $j_*/\sigma_s$ changes but remains isotropic. This noise-scale calibration requires no directional capacity, and is separate from the strict SH0--SH1 separation above. It concerns finite tolerance under deterministic noise, not exact image equality or a statistical averaging model.

\paragraph{A finite footprint preserves the local gain.}\phantomsection\label{rs:finitefootprint}
The local certificate extends to sufficiently narrow averages of continuous segment measurements. Keep the same $27$-node Q1 fields, coordinate clamping, SH1 basis, reference, and density prior $\|h\|_\infty\le10^{-4}$. Replace each of the $384$ ordered central segments $(a_k,b_k)$ by a fixed probability measure $\mu_k$ on ordered endpoint pairs $(a,b)$ satisfying
$\max\{|a-a_k|_2,|b-b_k|_2\}\le w$.
The measures are nonnegative, normalized, and independent of the unknown field; apply each to all three color channels. Let $F_\mu$ average the continuous segment renderer under these same measures for both fields. For every physical shared SH1 source at the candidate density and every $w\le10^{-6}$,
\begin{equation}\label{eq:finitefootprintgain}
 \|F_\mu(\sigma_*+\sigma_h,j_*+q)-F_\mu(\sigma_*,j_*)\|_\infty
 \ge1.8\times10^{-4}\|\sigma_h\|_{L^2}.
\end{equation}
The radius is measured in cube-side units. This sufficient radius shows that the positive gain survives a nonzero measurement footprint, with all spatial source coefficients still free. A calibrated camera pixel falls within this statement when its entire clipped-segment support obeys the endpoint condition. The \hyperref[rs:finitefootprintproof]{proof below} absorbs segment length into the SH1 basis before varying direction, avoiding a short-segment singularity.

\phantomsection\noindent\textit{Proof.}\label{rs:finitefootprintproof}
Write $v=b-a$, $\ell=|v|$, and $x(t)=(1-t)a+tb$ for $0\le t\le1$. For a source increment $q=q_0+\boldsymbol q\cdot d$, combine length and direction as
\[
 G_q(x,v)=\ell q_0(x)+\boldsymbol q(x)\cdot v.
\]
This equals $\ell q(x,v/\ell)$ for $\ell>0$ and extends continuously by zero to $\ell=0$. Let $H$ bound density-node increments and $V$ bound all source coefficients. Spacing-$1/2$ Q1 interpolation gives spatial Lipschitz bounds $7H$ for density increments and $21V$ for $q$ at fixed direction. The isotropic reference fields lie in $[0,1]$ and have Lipschitz bounds $4$. These bounds hold across cells and under coordinate clamping.

Every saved central length is below $1.413$. For $w\le.01$, both the central and perturbed lengths are below $1.5$, their difference is at most $2w$, and $|x'-x|\le w$. The spatial change of $G_q$ is bounded by $1.5\cdot21Vw$; its segment change at fixed $x$ is bounded by $(2+2\sqrt3)Vw$. Thus $|G_q'-G_q|\le38Vw$ and $|G_q|\le4.5V$. With $G_j=\ell j_*(x)$, $S(t)=\ell\int_0^t\sigma_*(x(u))\,du$, $B(t)=\ell\int_0^t h(x(u))\,du$, and $T=e^{-S}$, the corresponding bounds are
\[
 \begin{gathered}
 |G_j'-G_j|\le8w,\quad |G_j|\le1.5,\quad |T'-T|\le8w,\\
 |B'-B|\le13Hw,\qquad |B|\le1.5H.
 \end{gathered}
\]
Since $DF_*[h,q]=\int_0^1T(G_q-G_jB)\,dt$, subtraction yields
\[
 |(DF_*'-DF_*)[h,q]|\le w(50H+74V).
\]
This controls every joint increment. For $\tau=\|D^{-1}(h,v_{\rm coeff})\|_\infty$, use $H\le2401\tau$, $V\le359\tau$, and the unchanged preconditioner $R$ from \cref{rs:continuouslocalproof}. Averaging preserves the bound, so the scaled inverse defect is at most $.028870+146616w$.

The exact nonlinear remainder above is at most $1739r e^{1.5r}(1.5)^2\tau$, uniformly for every segment and every physical candidate source. At $r=10^{-4}$, $e^x\le(1-x)^{-1}$ bounds this by $(15651/39994)\tau<.392\tau$. The norm $\|D^{-1}R\|_\infty\le1$ therefore gives
\[
 \|F_\mu(z)-F_\mu(z_*)\|_\infty
 \ge\bigl(1-.028870-146616w-15651/39994\bigr)\tau.
\]
At $w\le10^{-6}$ the bracket exceeds $.432514>2401(1.8\times10^{-4})$. Nonnegative Q1 interpolation on the unit-volume cube gives $\|\sigma_h\|_{L^2}\le2401\tau$, proving \cref{eq:finitefootprintgain}. The argument is a quantified derivative perturbation of the existing certificate~\cite{alberti2021finite}; no minimum segment length is required.

For data within $\eta$ of $F_\mu(z_*)$, any candidate in this prior fitting the data within $\eta$ has density error at most $2\eta/(1.8\times10^{-4})$. For example, $\eta=5\times10^{-9}$ gives $1/18000$. The one-gray prior obstruction in \cref{rs:priornoise} also survives these normalized nonnegative averages. Thus the footprint extension and the noise scale impose separate conditions on recovery.

\paragraph{A statistical cost within the local density prior.}\phantomsection\label{rs:gaussianpairs}
Positive noiseless gain can coexist with a substantial cost of distinguishing two nearby densities. Use the same continuous $27$-node reference and $384$ ordered RGB segments, with $r=10^{-4}$. Consider the two physical fields $z_\pm=(\sigma_*\pm r,j_*\pm rp)$, whose density $L^2$ distance on the unit-volume cube is $2r$. Three specified shared source directions $p$ give the following continuous image-difference bounds; the unchanged-source row uses $p=0$.
\begin{center}
\small
\begin{tabular}{lcc}
\toprule
Source direction & $10^5\|F(z_+)-F(z_-)\|_2$ & $K_0$\\
\midrule
Unchanged & $[62.429,62.436]$ & $355$\\
SH0 refit & $[8.0435,8.0497]$ & $21{,}360$\\
SH1 refit & $[5.5465,5.5541]$ & $44{,}869$\\
\bottomrule
\end{tabular}
\end{center}
The norm sums all $1152$ scalar readings; the \hyperref[rs:gaussianproof]{proof below} controls the continuous integration and nonlinear errors. For independent additive Gaussian noise with per-reading standard deviation $s_e=1/255$, let $K$ count independent repetitions of this entire vector. For each row and every $K\le K_0$, no estimator can attain density error at most $r/2=5\times10^{-5}$ with probability at least $.95$ at both fields. Thus the prescribed RGB measurements impose a statistical accuracy limit even inside this small density prior. Each row is one feasible physical pair; the comparison measures achieved compensation within a fixed design.

\phantomsection\noindent\textit{Proof and interpretation.}\label{rs:gaussianproof}
The SH0 and SH1 directions fit $-D_\sigma F_N[1]$ using all shared source coefficients of the indicated degree at the saved $N=4096$ midpoint derivative. The saved coefficients define the fields exactly. Both signs satisfy the nodal sufficient conditions $j_0-\sum_i|j_i|\ge0$ and $\sigma-j_0-\sum_i|j_i|\ge0$; their minimum margins exceed $.1049$ and $.2506$. Nonnegative Q1 interpolation makes both fields physical at every point and direction.

Along a segment of length $\ell$, write $T_*$ for reference attenuation and $Q=\max_v\sum_{k=0}^3|p_{vk}|$ for the channel's source bound. Symmetric subtraction gives
\[
 F(z_+)-F(z_-)=2r\,DF_*[1,p]+e,\qquad
 |e|\le\frac{r^3e^{r\ell}\ell^3(1+Q)}{3}.
\]
Indeed the exact integrand is $T_*[-2\sinh(rt)j_*+2r\cosh(rt)p]$. Subtract its linear term, use the cubic and quadratic Taylor bounds for $\sinh$ and $\cosh$, and apply $\int T_*j_*\le1$, $T_*\le1$. The cancellation removes the quadratic error. Combining this remainder with \cref{eq:continuousderivative} encloses every continuous reading. Exact rational accumulation of squared interval distances from zero and squared maximum endpoints gives the table's lower and upper norm bounds. A separate $100$-digit calculation checks all $3\times1152$ intervals.

Let $V_+$ denote the unrounded upper bound on the squared norm. Equal-prior Gaussian testing~\cite{polyanskiywu2024draft} has optimal error $\Phi(-\sqrt K\|F(z_+)-F(z_-)\|_2/(2s_e))$, where $\Phi$ is the standard normal distribution function. This follows by projecting the sample mean onto the difference of the two means. The table uses
\begin{equation}\label{eq:gaussianpaircount}
 K_0=\left\lfloor\frac{9s_e^2}{V_+}\right\rfloor.
\end{equation}
For $K\le K_0$ the testing error exceeds $.05$: the Gaussian tail bound gives
$\Phi(-3/2)\ge(6/13)e^{-9/8}/\sqrt{2\pi}>.05$.
An estimate within $r/2$ of the true density identifies the nearer endpoint, since the endpoints are $2r$ apart. A $.95$ accuracy guarantee at both endpoints would therefore contradict this testing bound. This proves \cref{eq:gaussianpaircount} and the stated accuracy obstruction. The counts apply to the specified independent-noise model and measurement design; $K_0+1$ is not a sufficient reconstruction guarantee.

\paragraph{The same finite observations: SH1 recovery and SH3 exact ambiguity.}\phantomsection\label{rs:capacitytransition}
Keep the same $27$-node reference and all $384$ continuous rays from \cref{rs:continuouslocal}. Increasing angular degree changes local density identifiability at this fixed scene and observation set. SH1 has the positive lower bound \cref{eq:continuouslocalgain}. In the shared Q1--SH3 class, every nodal density change $\|h\|_\infty\le10^{-8}$ instead has a physical source $j_h$ satisfying
\begin{equation}\label{eq:capacitytransition}
 F(\sigma_*+\sigma_h,j_h)=F(\sigma_*,j_*).
\end{equation}
The source can be chosen smoothly with $j_0=j_*$, and both source and absorption margins exceed $.061$ everywhere and in every direction. Thus distinct densities occur arbitrarily close to this reference with identical finite observations. The \hyperref[rs:capacitytransitionproof]{construction below} gives an exact nonlinear partner, extending the comparison beyond linear sensitivity. Equality concerns the declared finite rays; the all-ray SH2 construction is a separate result.

\noindent\textit{Proof and physical construction.}\phantomsection\label{rs:capacitytransitionproof}
For one channel, let $G(s)$ map nodal source coefficients to the $384$ continuous readings. Use the $20$ monomials in $d$ of total degree at most three. Their restrictions span exactly SH degree at most three, although the coordinates are redundant. Each node uses one polynomial shared by all rays, giving $540$ coefficients per channel. A stored matrix $W\in\mathbb R^{540\times384}$ satisfies the verified bounds
\[
 \|I-G(s_*)W\|_\infty\le\delta=.001600,
 \qquad \|W\|_\infty\le M=201000.
\]
Every reference nodal source and absorption margin exceeds $\mu=.104$, and each ray has squared length below two. Set $H=\|h\|_\infty$. Since both densities are nonnegative, $|T_h(t)-T_*(t)|\le tH$. Nonnegative Q1 basis functions sum to one, and every direction monomial has absolute value at most one. Consequently
\[
 \|G(s_*+h)-G(s_*)\|_\infty\le20H.
\]
Writing $A_h=G(s_*+h)W$, the inverse-residual bound~\cite{Rump2010} gives $\|A_h^{-1}\|_\infty\le(1-\delta-20MH)^{-1}$. If $c_*$ denotes the isotropic reference in these coordinates and $y_*=G(s_*)c_*$, define
\begin{equation}\label{eq:capacitysource}
 c_h=c_*+WA_h^{-1}\bigl[y_*-G(s_*+h)c_*\bigr].
\end{equation}
Multiplying by $G(s_*+h)$ proves \cref{eq:capacitytransition}. Since $j_*\le1$, the bracket has maximum norm at most $H$. The resulting source change in every direction is therefore at most
\[
 V_H=\frac{20MH}{1-\delta-20MH}.
\]
At $H\le10^{-8}$ the denominator is at least $4791/5000$, and $V_H\le67/1597<.042$. The source and absorption margins are at least $\mu-V_H$ and $\mu-H-V_H$, both above $.061$. Q1 interpolation preserves these nodal bounds throughout the cube. The matrix entries and the inverse are smooth in $h$, and the bracket vanishes at zero. Applying the construction independently to the three channels proves the claim. No statement about SH2 identifiability is needed.

\noindent\textit{Continuous matrix certificate.}
For each ray, split its saved endpoint segment at the grid and clamp planes. These cuts are exact rational numbers in the affine ray parameter. On each interval, a nodal Q1 basis function $\phi_v$ is a nonnegative Bernstein polynomial of degree three, while its accumulated optical depth $A(t)$ has degree four. For $A\ge0$,
\[
 \left|e^{-A}-\sum_{k=0}^{24}\frac{(-A)^k}{k!}\right|
 \le\frac{A^{25}}{25!}.
\]
Integrate $\phi_v A^k/k!$ using Bernstein products and coefficient averages, enclosing every operation outward. Since $\sigma_*\le1$, the remaining ray integral is bounded by $\ell^{26}/25!$. This encloses all $10{,}368$ continuous coefficients $\int T_*\phi_v$; multiplication by direction monomials gives all $207{,}360$ entries of $G(s_*)$. Complete interval multiplication yields $\|I-G(s_*)W\|_\infty<.001599679$ and $\|W\|_\infty<200773.479$. Independent piecewise-polynomial quadrature checks all entries, and high-precision arithmetic checks every norm row. The exact equality follows from \cref{eq:capacitysource}; it is not inferred from a small rendered residual or a numerical rank threshold.

\subsection{Sharp resources for two queries}
\begin{lemma}[Sharp bounded SH contrast]\label{q:sharpcontrast}
For two directions separated by $\theta\in(0,\pi]$,
\begin{equation}\label{q:contrast}
 \max_{c\in\cH_{\leq L},\ 0\leq c\leq1}|c(q_2)-c(q_1)|
 =M_L(\theta):=\sin\bigl(\min\{L\theta/2,\pi/2\}\bigr).
\end{equation}
\end{lemma}
\begin{proof}
The shortest connecting great circle restricts $p=2c-1$ to a real degree-$L$ trigonometric polynomial bounded by one. The classical real Bernstein--Szeg\H{o} inequality $|p'|^2+L^2p^2\leq L^2$ \cite[Introduction]{erdelyi2014} makes $\arcsin p$ $L$-Lipschitz. Use $(1-\epsilon)p$ and pass to the limit to handle endpoints. When $L\theta\leq\pi$, the largest sine difference across an arcsine interval of length $L\theta$ is $2\sin(L\theta/2)$; otherwise the contrast bound one suffices.

For attainment before saturation write $q_{1,2}=a\cos(\theta/2)\mp b\sin(\theta/2)$ in an orthonormal plane frame and let $z(d)=a\cdot d+i b\cdot d$. The polynomial $c(d)=(1+\operatorname{Im}z(d)^L)/2$ is globally physical and has the stated contrast. At $\theta=\pi$ choose any $a\perp b=q_2$; $L=0$ is constant. After saturation let
\[
 \alpha=\frac{1-\cos(\pi/L)}{1-\cos\theta}\in(0,1],\qquad
 c(d)=\frac{1+T_L(1-\alpha+\alpha q_2\cdot d)}2.
\]
The Chebyshev argument lies in $[-1,1]$ on the whole sphere, and the two values are zero and one. Thus both constructions are physical spherical polynomials of degree at most $L$.
\end{proof}

\begin{corollary}[Minimum degree for two-query families]\label{q:pairdegree}
If every active node has at most two distinct queries, the least common degree with $Z_L=Z_Q$ is $\max_{v:J_v=2}\lceil\pi/\theta_v\rceil$, or zero if there is no such node. For a single node with equal weight $\beta>0$, subtract the fixed contribution $b$ and consider target $(0,\beta)$, equivalently uncentered target $b+(0,\beta)$. Then
\begin{equation}\label{q:pairresidual}
 \rho_L=\frac\beta2(1-M_L(\theta)).
\end{equation}
Hence accuracy $0\leq\epsilon<\beta/2$ requires and suffices that $L\geq\lceil2\arcsin(1-2\epsilon/\beta)/\theta\rceil$.
\end{corollary}
\begin{proof}
The contrast bound proves necessity. Once a physical polynomial realizes zero and one, an affine remapping interpolates any pair of physical values. Apply \cref{q:nodelemma}. In the equal-weight case the pre-saturation construction has equal and opposite fitting errors, proving \eqref{q:pairresidual}.
\end{proof}
These are minimum degree statements, not minimum training time or storage. For larger query sets the same balanced witness gives $\mathfrak G_L\geq ab(1-M_L(\theta))/(a+b)$, but pairwise tests alone need not determine the complete interpolation cost.

\subsection{Regular-circle queries: exact physical interpolation}
\begin{theorem}[Regular-circle consequence of Fej\'er interpolation]\label{q:circledegree}
Let $q_j=a\cos\phi_j+b\sin\phi_j$, $\phi_j=2\pi j/J$, $0\leq j<J$, for an orthonormal frame $a,b$ and $J\geq2$. Every table $a_j\in[0,1]$ has a globally physical SH interpolant of degree $J-1$. That degree is necessary in the worst case over tables.
\end{theorem}
\begin{proof}
Fej\'er's equally spaced trigonometric interpolation gives nonnegative degree-$(J-1)$ cardinal functions summing to one \cite[Chapter IV, \S6, pp.~142--143]{jackson1930}. Their harmonic disk extensions are
\begin{equation}\label{q:Fejer}
 K_j(z)=\frac1J\left[1+2\operatorname{Re}\sum_{k=1}^{J-1}
                  (1-k/J)e^{-ik\phi_j}z^k\right],\qquad |z|\leq1.
\end{equation}
The harmonic maximum principle extends positivity and the partition of unity from the circle to the disk. Therefore
\begin{equation}\label{q:circleinterpolant}
 c(d)=\sum_j a_jK_j(a\cdot d+i b\cdot d)
\end{equation}
lies in $[0,1]$ on $\Sph^2$ and interpolates exactly. Each real or imaginary part of $z^k$ is a homogeneous harmonic polynomial in three dimensions, so the SH degree is at most $J-1$.

For necessity choose one value equal to one and the other $J-1$ values zero. The great-circle restriction of a physical interpolant is a nonzero nonnegative trigonometric polynomial. Every zero has even multiplicity. A degree-$L$ trigonometric polynomial has at most $2L$ zeros with multiplicity, so $2L\geq2(J-1)$.
\end{proof}
This gives a sharp image-class threshold through \cref{q:nodelemma}. Subsets of a common $J$-gon inherit the upper bound by filling unused values arbitrarily. One fully queried node supplies necessity. For four regular queries the largest two-query threshold is two, whereas the all-table threshold is three. This is an explicit obstruction to inferring the general resource cost from pairs alone.

The construction uses $1+2(J-1)$ real sectoral coefficients per node/channel in the circle-aligned frame. This is one explicit storage bound, not a minimum-bit result. In contrast, \eqref{q:generaldegree} has a conservative separation-squared dependence. The sharp two-query law and regular-circle law show why that general construction should not be described as a minimum resource cost.

\subsection{A certified fixed-density mechanism test}\label{q:certificate}
\paragraph{Exact finite model.}
The reference has unit density and isotropic source $j_{*,k}(x)=1/2+\cos(3x_k)/4$ in $B_1$, with black background. Its target is the analytic continuous rendering. Sample a $370^3$ uniform grid in $[-1,1]^3$, and fix alternative nodal densities by evaluating
\begin{equation}\label{q:ellipsoid}
 s(x)=1+4\psi\left(\frac{(x_1-.05)^2}{.3^2}+
                        \frac{(x_2+.03)^2}{.2^2}+
                        \frac{(x_3-.02)^2}{.15^2}\right),
\end{equation}
with the smooth bump $\psi$ defined after \cref{m:thmseparation}. Only nodes in this ellipsoid have adjustable colors; inactive sources stay fixed at their reference nodal values. The observations lie in the plane $x_3=7/369$, use $d_j=(\cos(2\pi j/64),\sin(2\pi j/64),0)$ and offsets $\ell/200$ for $\ell=-80,\ldots,80$, giving $10304$ rays. Each alternative chord is evaluated with $256$ equal midpoint segments and the exact-real finite compositor \eqref{q:gridweights}. There are $6410$ active planar nodes. The three-dimensional grid spacing and off-plane basis are retained, but off-plane hats vanish on these rays.

Thus this certificate compares a finite compositor with specified analytic target data. It is not a continuous-integration or all-view certificate. The continuous density in \eqref{q:ellipsoid} has peak five; its grid samples, rather than that analytic function between samples, define the alternative.

\paragraph{The full physical SH1 class.}
At each active node let $c_v(d)=1/2+u_v+v_v\cdot d$ with
\begin{equation}\label{q:SH1physical}
 |u_v|+\norm{v_v}_2\leq1/2.
\end{equation}
This condition is necessary and sufficient for the color box on the entire sphere. Since all measured $d_3=0$, deleting $(v_v)_3$ preserves every observation and feasibility; conversely each reduced solution extends to a physical three-dimensional function. Unobserved active nodes can be filled with constants. The planar computation therefore covers the \emph{full} physical nodal SH1 class on this measurement design, rather than an arbitrarily selected perturbation subspace.

Let $B_{rv}=w_{rv}$, $y_c=y-b-B\one/2$, $g=B^T\lambda$, and $Q_v=\sum_rB_{rv}\lambda_rd_r$. Norm-distance duality \cite[\S8.1.3]{boyd2004} gives
\begin{equation}\label{q:SH1dual}
 \rho_1(y)=\max_{\norm\lambda_1\leq1}
       \left[\lambda\cdot y_c-\frac12\sum_v\max(|g_v|,\norm{Q_v}_2)\right].
\end{equation}
Indeed the support function of \eqref{q:SH1physical} is $\max(|g_v|,\norm{Q_v}_2)/2$. Any saved dual witness gives a lower bound without solver convergence; any legal color gives an upper bound. RGB channels are independent, so the joint maximum-norm optimum is the maximum of the channel optima.

\begin{center}\small
\begin{tabular}{@{}lrrr@{}}\toprule
Channel & SH1 lower & SH1 feasible upper & Query-relaxed feasible upper\\\midrule
Red & $1.2156472159$ & $1.3147534292$ & $.0015929157$\\
Green & $1.2875761195$ & $1.3850259355$ & $.0018108342$\\
Blue & $0$ & $.0136122121$ & $.0003039854$\\
Joint RGB & $1.2875761195$ & $1.3850259355$ & $.0018108342$\\\bottomrule
\end{tabular}
\end{center}
Values are outward-rounded linear gray levels, $255$ times maximum normalized RGB error. In particular the SH1 primal--dual gap is below $.09745$ levels. The blue channel alone does not discriminate at one level. The budget is one \emph{total fitting} radius about a fixed target, not one level on each of two independently noisy images and not an RMS or PSNR threshold.

\paragraph{Reference-density calibration.}\phantomsection\label{q:referencecontrol}
At $\sigma=1$, interpolate the reference $j_k=.5+.25\cos(3x_k)$ on the same grid and use the same $256$-segment compositor. This is a physical SH0 construction, also allowed under SH1. With spacing $s=2/369$, maximum segment length $q=1/128$ and nodal storage allowance $e_{J0}<5.961\cdot10^{-8}$, its all-ray error obeys
\begin{equation}\label{q:referencebound}
 \|I_{\rm grid,ref}-I_{\rm ref}\|_\infty
 \leq \frac{27}{32}s^2+\frac{13}{16}q^2+\frac5{128}q^3+e_{J0}
 <\frac{.019}{255}.
\end{equation}
The spatial term is the tensor-product interpolation bound using $|\partial_i^2 j_k|\leq9/4$. For each ray $|j'_k|\leq3/4$, $|j''_k|\leq9/4$; the midpoint defect with exact unit-density attenuation gives the displayed conservative quadrature terms. Exact rational substitution with the saved storage allowance yields an upper bound of $.018986313181$ linear levels. This calibrates the same grid, targets and integration convention; it is not a new fit or a bound on the optimum's lower side. The reference and alternative have different density-dependent optical weights.

\paragraph{Remaining target margin.}
Distances to a fixed image class are $1$-Lipschitz in the target. An additional unaccounted target discrepancy $\eta$ in linear gray levels preserves the one-level SH1 rejection if $\eta<1.2875761195-1=.2875761195$. This is a total fitting-margin statement, not two independent sensor-noise radii. The calibration bound above is not automatically an estimate of all physical-model errors.

\paragraph{Arithmetic and physical range.}
Primal values are treated as exact dyadic numbers. SH1 feasibility was recomputed with $70$-digit intervals on \eqref{q:SH1physical}; query-relaxed values are in $[0,1]$. Every ray coefficient is rebuilt from scalar enclosures of the exact grid, directions, bump and exponentials, using outward basic operations, guarded sums and fixed-cell checks. The floating proposal matrix is not assumed exact. An inherited stored-density/exterior allowance of at most $.000303985$ gray is subtracted from lower bounds and added to upper bounds once. The resulting certificates are conditional on the stated inclusion arithmetic and software assumptions; floating cross-checks are not formal verification.

\paragraph{Removing only angular coupling.}
For each direction, one nodal table serves all $161$ offsets. All $192$ direction/channel bounded linear programs supplied feasible solutions; their interval-evaluated maximum error is below $.001811$ gray. Thus the retained spatial sharing alone does not sustain the one-gray rejection. The limit residual is bounded above, not proved zero. Under \cref{q:circledegree}, the same table has an exact globally physical degree-$63$ extension and inherits the identical image bound with no interpolation penalty.

The saved outward adjoints verify positive visibility at all $64$ exact directions for $1116$ nodes. Hence $63$ is also the minimum degree for equality of the \emph{whole} finite image classes, by \cref{q:nodelemma}. It is not a lower bound on the degree required by this particular target. The exact extension uses $127$ real angular coefficients per node/channel, $2442210$ in total. Its floating coefficient conversion reproduces all saved query values within $1.133\times10^{-14}$; global physicality follows from \eqref{q:circleinterpolant}, not from this sampled check. SH2 and SH3 on this same model remain undetermined.

The general peak construction would instead give a very conservative degree-$32768$ upper below $.002525$ gray. The exact regular-query construction explains this large resource difference. Neither is an optimization runtime or a minimum-storage theorem.

\subsection{What a fixed witness attributes, and what it does not}
For one saved $\lambda$, let
\[
 m_{vj}=\sum_{r:q_{rv}=q_{vj}}\lambda_rw_{rv},\quad
 P_v=\sum_r|\lambda_r|w_{rv},\quad M_v=\sum_j|m_{vj}|,
 \quad g_v=\sum_jm_{vj},\quad Q_v=\sum_jm_{vj}q_{vj}.
\]
The supports of per-incidence colors, shared query values and physical SH1 give the exact signed identity
\begin{align}
 D_1(\lambda)&=D_{\rm ray}(\lambda)+S(\lambda)+A_1(\lambda),\label{q:decomposition}\\
 S&=\tfrac12\sum_v(P_v-M_v)\geq0,\qquad
 A_1=\tfrac12\sum_v\bigl(M_v-\max(|g_v|,\norm{Q_v}_2)\bigr)\geq0.\nonumber
\end{align}
The inequalities are triangle inequalities; repeated identical queries must be grouped before evaluating $A_1$. After division by $\norm\lambda_1$ and conversion to gray, the green witness gives
\[
 -35.2306144+0.1415918+36.3769027=1.2878801
\]
before the storage allowance. Red gives the same pattern; blue uses a zero witness. These are \emph{signed values of one functional}, not the optima of three different problems. Negative values are valid weak lower bounds, not negative distances. Strong cancellation makes percentages of ``error due to angle'' or ``error due to space'' meaningless. The separate feasible query fit is what establishes the decision change. From the two certified distances and \eqref{q:gap}, the whole RGB image-class gap obeys $255\mathfrak G_1>1.2857652$.

\subsection{Assumptions that cannot be dropped}\label{q:footprint}
\paragraph{Multiple queries per pixel.}
If one observation averages different queries of the same node, the columns of $H_v$ may overlap and \cref{q:nodelemma}'s necessity can fail. For example $F(c)=(c(q_1)+c(q_2))/2$ already covers $[0,1]$ with constants, even when the directions are not antipodal. For arbitrary nonnegative $H_v$, \eqref{q:SH1gap} remains exact. Equality requires every accessible moment vector $H_v^T\lambda$ to satisfy the sign/alignment equality condition for every node. Antipodal-only queries remain sufficient, but are no longer necessary.

\paragraph{Additional parameter constraints and generalization.}
The constructive upper requires independently realizable nodal functions. A shared small neural network or another cross-node constraint can invalidate it. The native Gaussian implementation uses a primitive/camera query convention and permits colors above one; its complete class is not the globally bounded class here. Primitive motion changes queries, weights and possibly sorting. These results give neither a native 3DGS geometry theorem nor an optimizer guarantee. New views add queries, and an infinite query set need not have positive separation. No complete-light-field, $Y$-norm, or degree-independent derivative bound follows from finite interpolation.

\paragraph{Implementation validation.}
The SH1 certificate has $41$ implementation checks, the query-relaxed fit $45$, and the resource construction $319$. Checks include direct high-precision ray evaluations, exact-query geometry, support identities, sharp two-query constructions, regular-circle cardinality and all-table reconstruction. The regular-circle theorem is an analytic argument, not a conclusion from those tests. All channels and negative controls are retained in the accompanying artifacts.
A replay driver reconstructs the optical coefficients and validates both saved witness classes in fresh directories without rerunning optimization; the stated arithmetic assumptions remain necessary.

\clearpage
\section{Additional transmission: a matched acquisition-cost bound}\label{app:acquisition}
The compensation mechanism suggests a complementary observation: change the background while keeping the effective source fixed. For each ray, black and white readings have the form
\begin{equation}\label{ac:pair}
 F_{r,0}(s,q)=E_r(s,q),\qquad F_{r,1}(s,q)=E_r(s,q)+T_r(s).
\end{equation}
Their difference removes appearance exactly. We give a finite-model example in which both a sufficient design and a necessary cost are certified at the same reference and noise budget. The lower bound constructs a whole physical family whose density changes survive every design with fewer than three paired rays.

\subsection{Model, observations, and the cost objective}
Let $\phi_v$ be the 27 tensor-product piecewise-trilinear nodal basis functions on $[-1/2,1/2]^3$, with three nodes per axis and vacuum exterior. Density is $\sigma_s=\sum_v\phi_v s_v$, $s\in[0,6/5]^{27}$, and
\begin{equation}\label{ac:physical}
 j_k(x,d)=\sum_v\phi_v(x)\bigl(q_{v,k,0}+q_{v,k,1:3}\cdot d\bigr),\qquad
 \|q_{v,k,1:3}\|_2\leq q_{v,k,0}\leq s_v-\|q_{v,k,1:3}\|_2.
\end{equation}
These are the full physical nodal SH1 source constraints. Every ray shares the same nodal coefficients. The exact density norm is $\|s-s_*\|_M^2=(s-s_*)^TM(s-s_*)$, where $M_{vw}=\int\phi_v\phi_w\,\mathrm dx$ is the rational Q1 mass matrix.

At the nodes, writing $(X,Y,Z)=2(x,y,z)$, the reference is
\begin{align}
 s_*&=.70+.12X+.08YZ,\label{ac:reference}\\
 c_*&=(.40+.12X+.07YZ,\ .45-.09Y+.08XZ,\ .35+.10Z-.08XY),\nonumber
\end{align}
with isotropic source coefficients $q_{*,v,k,0}=s_{*,v}c_{*,v,k}$ and zero degree-one coefficients. The certificate uses the saved dyadic reference values and encloses rounding when forming their products.

Each ray uses 64 constant cells. Stored nonnegative interpolation weights $\phi_{riv}$ and lengths $\Delta_{ri}$ define exact dyadic inputs. Cell density and source are $\sigma_{ri}=\sum_v\phi_{riv}s_v$ and $j_{ri,k}=\sum_v\phi_{riv}(q_{v,k,0}+q_{v,k,1:3}\cdot d_r)$. Integrating each constant cell gives the finite renderer
\begin{equation}\label{ac:compositor}
 E_r=\sum_i\exp\!\left(-\sum_{j<i}\Delta_{rj}\sigma_{rj}\right)
 j_{ri}\int_0^{\Delta_{ri}}e^{-t\sigma_{ri}}\,\mathrm dt,
 \qquad T_r(s)=e^{-A_rs},\quad A_{rv}=\sum_i\Delta_{ri}\phi_{riv}.
\end{equation}
The integral is well defined at zero density. This declaration fixes the model being certified, including finite sampling.

There are 384 mandatory black-background RGB rays, denoted $D_0$, and 144 selectable rays. The baseline uses 24 full-sphere Fibonacci directions and four offsets per image axis, $\{- .35,-7/60,7/60,.35\}$. The pool uses 16 upper-hemisphere Fibonacci directions with azimuthal phase $.37$ and offsets $\{-.33,0,.33\}$ per axis. Directions, cube chords and midpoint samples are stored as the exact input arrays below. Each selectable action consists of two additional \emph{scalar red-channel} readings at backgrounds zero and one, at cost two. All raw readings have absolute error at most $b=1/255$. For a selected set $D$, let $F_{D_0\cup D}$ stack the baseline and additional readings, and define
\begin{align}
 \Theta_D&=\{(s,q):s\in[0,6/5]^{27},\ \eqref{ac:physical}\text{ holds},\\
 \|F_{D_0\cup D}(s,q)-F_{D_0\cup D}(s_*,q_*)\|_\infty\leq2b\},\nonumber\\
 R_{\rm ref}(D)&=\sup_{(s,q)\in\Theta_D}\|s-s_*\|_M,\label{ac:radius}\\
 C_{\rm ref}(\varepsilon)&=\min\{2|D|:R_{\rm ref}(D)<\varepsilon\}.\label{ac:cost}
\end{align}
The factor $2b$ expresses intersection of two raw-noise boxes: the midpoint of their predictions is one common legal measurement. Thus the target is reference exclusion. The noise boxes, prescribed reference, pool and cost convention are identical on both sides of the following bound.

\begin{proposition}[Certified cost of reference exclusion]\label{ac:costtheorem}
For the finite model and acquisition pool above,
\begin{equation}\label{ac:bracket}
 6\leq C_{\rm ref}(.10)\leq152.
\end{equation}
Every set of at most two paired rays retains a shared physical SH1 alternative with density separation greater than $.1019$. A specified set of 76 pairs bounds every compatible density's separation by $.097965$.
\end{proposition}

\subsection{Why a whole physical family gives a necessary cost}
The essential construction is a three-dimensional affine family
\begin{equation}\label{ac:family}
 s(u)=s_*+Hu,\qquad q(u)=\bar q+Qu,\qquad \|u\|_2\leq1.
\end{equation}
All its coefficients are shared across rays. Exact rational arithmetic and outward interval bounds establish
\begin{align}
 H^TMH&>(.1019)^2 I_3,\label{ac:mass}\\
 0<s_v(u)&<6/5,\quad 0\leq j_k(x,d;u)\leq\sigma_{s(u)}(x),\label{ac:box}\\
 \sup_{\|u\|\leq1}\max_{r\in D_0,k}|E_{r,k}(u)-E_{r,k,*}|&<.006527<2/255,\label{ac:base}\\
 \sup_{\|u\|\leq1}\max_{r\text{ in pool}}|E_{r,R}(u)-E_{r,R,*}|&<.006339<2/255.\label{ac:pool}
\end{align}
The last condition is a property of the constructed family over the candidate pool; only selected actions are charged or observed.

For any $D$ with $|D|\leq2$, the matrix $A_DH$ has a unit null vector $u$. Hence
\[
 T_D(s_*+Hu)=T_D(s_*),\qquad \|Hu\|_M>.1019.
\]
On each selected ray, the white-reading difference therefore equals the black-reading difference. Equations~\eqref{ac:base}--\eqref{ac:pool} bound every measured coordinate by $2b$. The alternative remains compatible with a common raw observation and violates the $.10$ target. This proves the lower bound for all $\binom{144}{2}=10296$ two-pair designs at once. The dimension argument is classical; its physical content here is the verified finite-amplitude family satisfying \eqref{ac:box}--\eqref{ac:pool}.

\paragraph{Physicality over the full parameter ball.}
For either nodal inequality in \eqref{ac:physical}, the construction gives
$a_0+a^Tu\geq D_{\rm dir}\|Bu\|_2$, where $D_{\rm dir}=1+10^{-12}$ covers the unit sphere and all saved dyadic directions. Exact squares verify $a_0>\|a\|_2$. Each of the 162 node/channel/inequality cases has a saved $\lambda\geq0$ for which
\begin{equation}\label{ac:cone}
 \begin{pmatrix}
 aa^T-D_{\rm dir}^2B^TB+\lambda I & a_0a\\
 a_0a^T & a_0^2-\lambda
 \end{pmatrix}\succ0.
\end{equation}
Its quadratic form at $(u,1)$, plus $\lambda(1-\|u\|^2)$, is $(a_0+a^Tu)^2-D_{\rm dir}^2\|Bu\|^2$. Thus the classical one-ball S-procedure \cite[Appendix B.2, Eqs.~(B.6)--(B.7)]{boyd2004} proves the desired cone inequality. Rational $LDL^T$ pivots certify \eqref{ac:cone} and \eqref{ac:mass}. Nonnegative spatial interpolation preserves the nodal physical bounds.

\paragraph{Enclosing the nonlinear readings uniformly.}
In cell $i$, let $h_i,w_i\in\R^3$ be density and source responses to $u$, and let $\bar j_i$ be the center source. Write
\[
 v_i(t)=\sum_{j<i}\Delta_jh_j+t h_i,\quad
 \rho_i=\max\{\|v_i(0)\|,\|v_i(\Delta_i)\|\},\quad
 W_i=\int_{\mathrm{cell}\ i}T_*(t)\,\mathrm dt.
\]
Convexity bounds $\|v_i(t)\|\leq\rho_i$ throughout the cell. The integrated linear response is $\ell=\sum_i\int T_*(w_i-\bar j_i v_i)\,\mathrm dt$. Expanding $e^{-v_i^Tu}(\bar j_i+w_i^Tu)$ gives the following two valid cell remainders:
\begin{align}
 R_{1i}&=W_i\{\bar j_i(e^{\rho_i}-1-\rho_i)+\|w_i\|(e^{\rho_i}-1)\},\label{ac:remainder1}\\
 R_{2i}&=W_i\left\{\rho_i\max_{t\in\{0,\Delta_i\}}\|w_i-\bar j_i v_i(t)/2\|
 +\bar j_i(e^{\rho_i}-1-\rho_i-\rho_i^2/2)
 +\|w_i\|(e^{\rho_i}-1-\rho_i)\right\}.\label{ac:remainder2}
\end{align}
For the second bound, retain the quadratic term
$-(v_i^Tu)((w_i-\bar j_i v_i/2)^Tu)$ before bounding the remaining series. The bound over $t$ again follows by convexity. With $\delta_{\rm ctr}=|E(s_*,\bar q)-E(s_*,q_*)|$, one obtains
\begin{equation}\label{ac:uniform}
 \sup_{\|u\|\leq1}|E(u)-E_*|
 \leq\delta_{\rm ctr}+\|\ell\|_2+\sum_i\min\{R_{1i},R_{2i}\}.
\end{equation}
The actual upper endpoints are
\[\begin{aligned}
\text{baseline in \eqref{ac:base}:}\quad&.006526626509024162,\\
\text{pooled design in \eqref{ac:pool}:}\quad&.006338615268606398.
\end{aligned}\]
The center correction is below $1.183\times10^{-14}$ and every $\rho_i<.215$. The implementation encloses exponential tails by positive rational series and encloses basic operations outward. These bounds cover the whole parameter ball, rather than a sample of its directions.

\subsection{A sufficient design for the same objective}
For two candidates with raw prediction gap at most $2b$ on each selected endpoint, \eqref{ac:pair} yields
$|T_r(s)-T_r(s_*)|\leq\eta$, $\eta=4b$. Since each selected $T_{r,*}>\eta$, logarithmic inversion gives exact strips
\begin{equation}\label{ac:strips}
 -\log(1+\eta/T_{r,*})\leq A_r(s-s_*)\leq-\log(1-\eta/T_{r,*}).
\end{equation}
Put $h=s-s_*$ and $u_r=-\log(1-\eta/T_{r,*})$, so $|A_rh|\leq u_r$. Append coordinate rows with bounds $|h_v|\leq\max\{s_{*,v},6/5-s_{*,v}\}$. For any nonnegative multipliers $\lambda_i$ satisfying
\begin{equation}\label{ac:outer}
 \sum_i\lambda_i a_i^Ta_i\succeq M,\qquad
 \sum_i\lambda_i u_i^2<(.097965)^2,
\end{equation}
every compatible candidate satisfies
$h^TMh\leq\sum_i\lambda_i(a_ih)^2<(.097965)^2$.
This is a sufficient quadratic outer enclosure of the log-strip polytope.

The saved certificate uses 76 transmission rows and two density-coordinate bounds. Directed interval evaluation verifies \eqref{ac:outer}; a rationally specified congruence and interval diagonal dominance certify its matrix inequality. The minimum selected reference transmission exceeds $.328153$. All densities in the nodal box and all compatible physical source refits are covered. The acquisition cost is $2\times76=152$ scalar readings, proving the upper side of \eqref{ac:bracket}.

\paragraph{Certificate specification and interpretation.}
The accompanying \texttt{acquisition\_evidence} package contains the exact dyadic arrays, reference, all rays, family coefficients, multipliers, interval outputs and optimizer-free replay instructions. The lower certificate checks 162 rational cone matrices and the density mass matrix; the upper certificate verifies the selected log strips and their quadratic enclosure. The inputs and replay outputs are identified by SHA-256 hashes in its manifest. Floating-point proposals are used only to produce candidates for these checks. The declared arithmetic is rational arithmetic and outward real-function enclosures, with the implementation assumptions stated in the replay documentation.

This result quantifies fixed-reference exclusion in the declared compositor. Its norm is the continuous Q1 density norm, while its observations are finite-compositor readings. The raw-noise-box target differs from conditioning on one fixed realized baseline, and from uniform minimax estimator error, where a two-point separation contributes half its distance. The interval $[6,152]$ states necessary and sufficient costs for one objective without asserting that either endpoint is optimal. General compatible-set recovery and kernel-section arguments have classical foundations and recent linear-query developments \cite{filmus2026linearqueries}; here the shared physical source family supplies the radiative obstruction.

\clearpage
\section{Representation controls and experimental scope}\label{app:records}
This section gives the physical constraints and reproducible training settings, the complete discrete certificate, and the neural-field controls needed to interpret the main experiments.
\subsection{Physical appearance constraints and arithmetic certificates}\label{app:fixedprotocol}
\subsubsection{Fixed density, convex appearance compensation}\label{app:fixedcolor}
The joint optimization experiment combines representation capacity with nonconvex density updates. To isolate appearance compensation, we fix the alternative density grid and optimize only its SH color coefficients. The same $32^3$ grid, $B_{.8}$ unknown region, analytic targets, known exterior, and nested $100/400$ camera sets are retained. Both the quadratic reference and the cosine control use the same alternative density. A unit-density, degree-zero fit provides a discretization baseline for each scene and view count. The test cameras are independent; optimization and checkpoint selection use training RGB only.

\paragraph{A linear rendering operator.}
For fixed trilinear density, midpoint optical weights are constants. Write
\[
 (Aa)_r=\sum_v B_{rv}\sum_{m=1}^{(L+1)^2}\phi_m(d_r)a_{vm},
 \qquad f(a)=\tfrac12\|Aa-b\|_F^2.
\]
Here $B_{rv}\geq0$ combines trilinear interpolation with attenuated sample weights, and $b$ subtracts the known near and far exterior contributions. Colors and RGB channels remain linear; no sigmoid or directionwise clipping is used. Training uses $128$ midpoint samples, independent evaluation uses $256$, and a $512$-sample check measures integration sensitivity. This is a convex, finite-dimensional problem, but convexity alone does not establish convergence of a finite numerical run.

\paragraph{Physical color sets and an exact degree-one constraint.}
With the normalized SH convention of the joint experiment, set $e=(1/2,0,\ldots)$ at each node and channel, and define
\begin{align*}
 C_{\rm in}&=\left\{a:|a_0-1/2|+\sum_{l=1}^L\sqrt{2l+1}\|a_l\|_2\leq1/2\right\},\\
 C_{\rm out}&=\{a:\|a-e\|_2\leq1/2\}.
\end{align*}
If $C_{\rm phys}$ is the exact all-direction node constraint $0\leq c(d)\leq1$, then
\[
 C_{\rm in}\subseteq C_{\rm phys}\subseteq C_{\rm out}.
\]
The first inclusion follows from the SH addition theorem and Cauchy--Schwarz. The second follows from Parseval applied to $|c-1/2|\leq1/2$. Trilinear interpolation preserves nodewise physicality throughout the grid. For $L=0,1$, the first inclusion is an equality: the range of a degree-one color is exactly $[a_0-\sqrt3\|a_1\|_2,a_0+\sqrt3\|a_1\|_2]$. Thus a lower bound obtained with $C_{\rm in}$ at degree one applies to the entire physical degree-one voxel class. At higher degrees, $C_{\rm in}$ is only sufficient. A fit in $C_{\rm out}$ may be nonphysical and is never presented as a valid scene.

\paragraph{Bounds before optimizer convergence.}
Let $r=Aa-b$ and $g=A^Tr$. Convexity gives the globally valid lower bound
\[
 f_*\geq\max\{0,\ f(a)+\langle g,e-a\rangle-h_{C-e}(-g)\}.
\]
The support function of the product set is summed over nodes and channels. Its local value is $\|g\|_2/2$ for $C_{\rm out}$ and
\[
 \tfrac12\max\left\{|g_0|,\max_{1\leq l\leq L}
                 \|g_l\|_2/\sqrt{2l+1}\right\}
\]
for $C_{\rm in}$. Any feasible iterate supplies an upper bound. This separates a residual floor from insufficient optimization without requiring an exact least-squares solve. The original tables evaluate these bounds in ordinary double precision. The additional integer enclosures in \cref{app:integercertificate} certify the degree-one/two comparison for the prescribed discrete inputs.

\paragraph{Solver and interpretation.}
We use full-batch accelerated projected gradient \cite{beck2009} with a scalar Jacobi scaling at each voxel, exact projection onto the relevant convex set, and a quadratic majorization check at every step. The inner projection is weighted group soft-thresholding with its threshold found from sorted group norms. There is no TV or other regularization. The stopping rule uses a $1\%$ relative primal--dual gap or $10^{-10}$ absolute gap in RGB MSE, subject to $2500$ iterations or $120$ seconds per joint two-scene solve; unresolved cells retain their bounds and are not called exact optima. Selection takes the smallest feasible training objective, never test or geometry scores. Both constraints use the same initial coefficients: the analytic alternative for the quadratic scene and the reference DC coefficients for the control; higher unassigned coefficients are zero. These are controlled capacity fits with an oracle initialization for the constructed example, not random discovery of a density ambiguity.

The finite-view, radius-$.8$ model is outside the small-support positive theorem. A lower-bound/upper-bound separation compares physical appearance capacities for the prescribed density pair and voxel class; it does not prove density uniqueness among all fields or show that every remaining residual is interpolation error. A positive center $\kappa_L$ in the control also does not imply a large observable residual throughout $B_{.8}$. All original random-start thresholds remain unchanged.

\subsubsection{An arithmetic certificate for the discrete separation}\label{app:integercertificate}
This certificate fixes the elevenfold density candidate, separately from the exact twofold construction. We additionally verify the two appearance capacities by exact integer and rational postprocessing of the retained iterates. This uses classical verified-optimization principles \cite{keiljansson2006}; it introduces neither a new solver nor a new identifiability theorem. All saved binary64 optical weights, target RGB, known-exterior RGB, and ray directions are interpreted as exact inputs. Directions are normalized over the reals and their SH values are enclosed by rational intervals. The result certifies this prescribed discrete rendering model, not its quadrature, calibration, or continuum-model errors.

\paragraph{A dual witness without a convergence requirement.}
For one channel write $b=t-o$ and
\[
 f_* = \min_{a\in C_1}\tfrac12\|Aa-b\|_2^2,
 \qquad C_1=\{|a_{v0}-1/2|+\sqrt3\|a_{v1}\|_2\leq1/2\ \forall v\}.
\]
For \emph{any} ray vector $y$, put $g=A^Ty$. Completing the square and minimizing the resulting affine expression over $C_1$ gives
\begin{align}
 f_*&\geq D(y):=-\tfrac12\|y\|_2^2-\langle y,b\rangle
       +\tfrac12\sum_v \bigl(g_{v0}-s(g_v)\bigr),\label{eq:integerdual}\\
 s(g_v)&=\max\{|g_{v0}|,\|g_{v1}\|_2/\sqrt3\}.\nonumber
\end{align}
Indeed $\|Aa-b\|_2^2/2\geq\langle y,Aa-b\rangle-\|y\|_2^2/2$, and the last sum is $\inf_{a\in C_1}\langle g,a\rangle$. Thus even a rounded, inaccurate residual is a valid dual witness. Negative channelwise bounds may be replaced by zero before summing RGB objectives.

\paragraph{Integer enclosure of the adjoint.}
Choose dyadic $y_i=Y_i/Q_y$ and integers $\widehat B_{iv}$, $H_{ik}$ with
\[
 |B_{iv}-\widehat B_{iv}/Q_B|\leq Q_B^{-1},\qquad
 |y_i\phi_k(d_i)-H_{ik}/Q_H|\leq Q_H^{-1}.
\]
The first enclosure follows by rounding stored weights. The second includes the rational SH enclosure, multiplication rounding, and quantization. Let $Q=Q_BQ_H$ and define, on the saved sparse support,
\begin{align*}
 G_{vk}&=\sum_i\widehat B_{iv}H_{ik},\\
 E_{vk}&=\sum_{i:B_{iv}\ne0}\bigl(|H_{ik}|+1\bigr)
                  +\sum_i\widehat B_{iv}.
\end{align*}
The weights are nonnegative. Expanding $B(y\phi)-\widehat B H/Q$ therefore proves
\[
 |g_{vk}-G_{vk}/Q|\leq E_{vk}/Q.
\]
In particular $g_{v0}\geq(G_{v0}-E_{v0})/Q$, while
\[
 s(g_v)\leq\frac1Q\max\left\{
 |G_{v0}|+E_{v0},\,
 \left\lceil\sqrt{\frac{\sum_{k=1}^3(|G_{vk}|+E_{vk})^2}{3}}\right\rceil\right\}.
\]
The square-root ceiling is computed by integer square roots. If target and offset are independently rounded with denominator $Q_t$, their joint dot-product error is at most $\sum_i|Y_i|/(Q_yQ_t)$. Substituting these bounds into \eqref{eq:integerdual} gives a rational lower bound. Integer sparse products are accepted only after an upper bound on their sum of absolute terms excludes overflow. Subsequent scalar reductions use arbitrary-precision integers.

\paragraph{A physically feasible upper witness.}
We round the retained degree-two coefficients and move their deviations from the half-color toward zero. The largest coefficient change is below $5.59\times10^{-9}$. For integer coefficients $Z$ with denominator $Q_a$, the check
\[
 |Z_{v0}-Q_a/2|+
 \sum_{l=1}^2\left\lceil\sqrt{(2l+1)\sum_m Z_{vlm}^2}\right\rceil
 \leq Q_a/2
\]
proves the all-angle physical color constraint. Both view counts pass with at least one integer unit of slack. For these coefficients $\sum_m|a_{vm}|\leq1$ and $0\leq c_v(d)\leq1$. Hence rounding the optical weights contributes at most $N_{\max}/(2Q_B)$ to each RGB component, where $N_{\max}$ is the largest row support size. Rounding the SH values contributes at most $M_B((2Q_\phi)^{-1}+\epsilon_\phi)$, where $M_B$ is the largest rounded optical row mass and $\epsilon_\phi$ encloses normalized-SH evaluation. Intermediate and final dyadic rounding contribute, respectively,
\[
 \frac{\max_i\sum_m|\widehat\phi_{im}|}{2Q_PQ_\phi}
 \quad\hbox{and}\quad \frac{3}{2Q_{r}}.
\]
Adding these componentwise bounds to the exactly enclosed integer residual RMS gives a rigorous feasible upper bound. No reoptimization or test-ray selection is used.

\begin{table}[t]
\centering\small
\caption{Integer-certified training RMSE bounds for the fixed alternative density. All displayed lower bounds are rounded down and upper bounds up. The control's 100-view separation shows that this effect is not exclusive to the quadratic reference.}\label{tab:integercertificate}
\begin{tabular}{lrcc}
\toprule
Scene & Views & $L=1$ lower & $L=2$ upper\\\midrule
Quadratic &100& .004430018493 & .000001185647\\
Quadratic &400& .006126135873 & .000068440310\\
Control &100& .007309708730 & .000022267172\\
Control &400& .009951989652 & .001128158093\\\bottomrule
\end{tabular}
\end{table}

\paragraph{Finite RGB perturbations.}
Let $d_L(b)=\inf_{a\in C_L}\|A_La-b\|_{\rm RMS}$ for the same fixed density. The triangle inequality gives $|d_L(b+e)-d_L(b)|\leq\|e\|_{\rm RMS}$. Consequently lower and upper bounds $d_1(b)\geq\ell$, $d_2(b)\leq u$ imply
\[
 d_1(b+e)-d_2(b+e)\geq\ell-u-2\|e\|_{\rm RMS}.
\]
For the quadratic reference, strict ordering is therefore certified for RGB perturbations below $.002214416423$ and $.0030288477815$ RMS at 100 and 400 views. Multiplication by $255$ gives $.564676187865/.7723561842825$ linear-RGB levels RMS. This scale is neither a pixelwise bound nor a gamma-encoded sRGB tolerance. These compare angular appearance classes at a prescribed wrong density; they are not density-recovery error bounds. If a different renderer is uniformly within $\rho_L$ RMS on $C_L$, subtract $\rho_1+\rho_2$ from the right side. Such a uniform continuum or calibration bound has not been established here.

The certification retains more than $99.98\%$ of each quadratic floating lower bound. All four scene/view certificates take $7.8$ seconds on a local CPU after the original training, with no new training runs. A separate checking implementation tests 256 exact small physical-SH dual inequalities and reconstructs the saved witnesses. These are validated arithmetic computations with explicit input scope, not proof-assistant or external expert verification.

\subsection{Training settings and all endpoints}\label{app:training}
\subsubsection{Joint optimization protocol}
We use a dense Plenoxels-style \cite{plenoxels2022} representation: a $32^3$ trilinear grid stores density and real SH color coefficients, with degree $L\in\{0,1,2,3\}$. The unknown region is $B_{.8}$; the reference exterior up to the unit sphere is known and integrated analytically. This is an explicit finite-dimensional test of optimization behavior, not an implementation of the small-radius smooth theorem. Training receives RGB only, without density, depth, opacity, or internal-source losses.

\paragraph{Scenes and splits.}
The quadratic scene is \eqref{m:reference}. The control has the same density and $c_k(x)=1/2+\cos(3x_k)/4$. There are $100$ or $400$ orthographic full-sphere training views at $32^2$ pixels; the smaller set is a fixed subset of the larger. Independent test and validation camera sets contain $32$ views at $24^2$ and $12$ at $16^2$, respectively. Analytic reference RGB removes teacher-rendering quadrature error. We report PSNR on rays intersecting the unknown region, as well as the full foreground; exact black background never enters the reported scores. Geometry is measured on $8192$ independent uniform points in $B_{.75}$. Absorption-depth errors use independent test rays and analytic unit-density truth.

\paragraph{Finite angular degree and the physical box.}
The real SH basis is scaled so its degree-zero element is one and $\sum_m\phi_{lm}(d)^2=2l+1$. No directional sigmoid or clipping is applied. At each voxel and channel, after every update, DC is clipped to $[.002,.998]$ and the nonconstant coefficients are jointly scaled until
\[
 \sum_{l=1}^L\sqrt{2l+1}\,\|a_l\|_2
 \leq \min(a_0,1-a_0)-10^{-4}.
\]
Cauchy--Schwarz and convex trilinear interpolation imply an all-direction color box. This sufficient constraint restricts the full physical SH class and can affect optimization; both analytic warm starts satisfy it at the grid nodes without scaling. Density is projected to $[.05,12]$.

\paragraph{Optimization and hypotheses.}
The grid uses three matched seeds per scene, view count and angular degree: $48$ runs, each with $1500$ Adam updates and $384$ rays per batch. Density and color learning rates start at $.02$ and $.005$, decay exponentially to one tenth, and use Adam $\epsilon=10^{-8}$. Density starts uniformly in $[.5,1.5]$; DC starts at $.4$ with $.03$ Gaussian noise and higher degrees at zero. There is no TV penalty.

Midpoint rendering uses $64$ interior samples in training and $128$ in evaluation, with a separate $256$-sample check. Known exterior segments are exact. Final-step checkpoints are used throughout. The ratio of the last two $100$-step mean losses diagnoses continuing optimization; it does not establish convergence. Discretized warm-start losses are measured rather than assumed zero.

\paragraph{Why the same control has positive center gain at every finite degree.}
At the center, $T=e^{-1}$ and
\[
 U_k(d)=\tfrac12(1-e^{-1})+\tfrac14\int_0^1e^{-t}\cos(3t d_k)\dd t.
\]
On $C_n$, write $d_k=a_k\cos(\theta-\theta_k)$, $a_k=\sqrt{1-n_k^2}$. Choose any even integer $\nu>L$, $\nu\geq2$. The Jacobi--Anger expansion gives a degree-$\nu$ cosine amplitude of magnitude
\[
 \tfrac12\int_0^1 e^{-t}J_\nu(3t a_k)\dd t.
\]
For $0<z\leq3$, the alternating Bessel series has decreasing terms and
\[
 J_\nu(z)\geq \frac{(z/2)^\nu}{\nu!}
 \left(1-\frac{9}{4(\nu+1)}\right)>0.
\]
Since $\sum_k a_k^2=2$, some $a_k\geq\sqrt{2/3}$. Thus one uncancellable Fourier amplitude has a strictly positive normal-independent lower bound, proving $\kappa_L(0)>0$ for every finite $L$ in this one fixed control. This does not provide a whole-region lower bound or guarantee recovery by the finite-view voxel optimizer.

\subsubsection{Complete fixed-budget training results}
Density dispersion is $D=\sqrt{N^{-1}\sum_i\operatorname{Var}_{s=1,2,3}\sigma_s(x_i)}$, with variance divisor three. Errors and PSNR in \cref{t:allvoxel} are means and sample standard deviations over the three seeds. There are $7{,}936$ test rays through the unknown region, within $12{,}288$ foreground rays. The training sets contain $44{,}800$ and $179{,}200$ unknown-intersecting rays. All are retained.

\begin{table*}[t]\centering\small
\caption{All random-start cells. Depth is conditional absorption-depth RMSE in scene units. No seed is omitted.}\label{t:allvoxel}
\begin{tabular}{llrrrrr}\toprule
Scene & Views & $L$ & PSNR (dB) & Density error & Depth error & $D$\\\midrule
Quadratic & 100 & 0 & $58.87\pm0.04$ & $0.119\pm0.001$ & $0.0071\pm0.0000$ & 0.082\\
Quadratic & 100 & 1 & $53.28\pm0.08$ & $0.192\pm0.001$ & $0.0175\pm0.0001$ & 0.085\\
Quadratic & 100 & 2 & $51.98\pm0.02$ & $0.300\pm0.002$ & $0.0309\pm0.0001$ & 0.099\\
Quadratic & 100 & 3 & $49.88\pm0.05$ & $0.380\pm0.002$ & $0.0403\pm0.0001$ & 0.097\\
Quadratic & 400 & 0 & $60.78\pm0.05$ & $0.113\pm0.001$ & $0.0066\pm0.0000$ & 0.077\\
Quadratic & 400 & 1 & $56.95\pm0.07$ & $0.175\pm0.001$ & $0.0163\pm0.0001$ & 0.074\\
Quadratic & 400 & 2 & $56.66\pm0.08$ & $0.292\pm0.002$ & $0.0306\pm0.0001$ & 0.084\\
Quadratic & 400 & 3 & $53.57\pm0.09$ & $0.367\pm0.002$ & $0.0393\pm0.0001$ & 0.078\\
Cosine control & 100 & 0 & $45.50\pm0.02$ & $0.376\pm0.003$ & $0.0329\pm0.0001$ & 0.091\\
Cosine control & 100 & 1 & $46.89\pm0.04$ & $0.435\pm0.002$ & $0.0422\pm0.0001$ & 0.100\\
Cosine control & 100 & 2 & $49.45\pm0.07$ & $0.507\pm0.003$ & $0.0492\pm0.0002$ & 0.121\\
Cosine control & 100 & 3 & $47.49\pm0.06$ & $0.487\pm0.002$ & $0.0468\pm0.0002$ & 0.122\\
Cosine control & 400 & 0 & $45.72\pm0.02$ & $0.373\pm0.003$ & $0.0327\pm0.0001$ & 0.088\\
Cosine control & 400 & 1 & $48.02\pm0.02$ & $0.428\pm0.002$ & $0.0418\pm0.0001$ & 0.093\\
Cosine control & 400 & 2 & $53.22\pm0.06$ & $0.506\pm0.002$ & $0.0493\pm0.0001$ & 0.114\\
Cosine control & 400 & 3 & $51.03\pm0.05$ & $0.490\pm0.002$ & $0.0473\pm0.0002$ & 0.113\\
\bottomrule\end{tabular}\end{table*}
\subsection{Scope of optimizer and confidence interpretations}\label{app:limitations}
The all-ray local theorems do not predict a capacity-specific jump in random-start density dispersion: the two scenes' observed SH2/SH1 dispersion ratios are $1.14$--$1.22$, while losses continue decreasing. The intervention tests and fixed-density class certificates address compensation directly. A gain evaluated on a learned field is not automatically a calibrated depth-confidence score: six RGB-only pilot fits do not improve the tested depth-error ranking. The exact tangent-circle counterexample in \eqref{eq:tangentcircle} also shows why a finite grid of normals cannot certify positive minimum gain.

\paragraph{Optimization path.}\phantomsection\label{app:curriculum}
A matched 12-run comparison uses the same final SH3 model, two analytic scenes, three seeds, 1,500 updates and paired initialization/batches. The curriculum activates degrees 0/1/2/3 at updates 0/250/500/750. Relative to immediate activation, median density error decreases $11.9\%$ in the quadratic scene but increases $9.9\%$ in the control, with PSNR losses $1.96/2.86$ dB. Neither scene meets the predefined $20\%$ density improvement and at most $.5$ dB loss criterion. These comparisons establish no accuracy-improving schedule; the theory concerns information remaining after appearance refitting.

\subsection{Native neural fields: prediction and finite compensation}\label{app:nativepilot}
The references are a 12,000-update NeRF Lego model, released 200,000-step NeRF Fern/Lego checkpoints, and photo-trained Nerfacto Fern selected at update 20,000 of 30,000. Targets are reference renderings. R/D/H release the RGB readout, additionally its directional input columns, or the complete color head; these are parameter classes, not SH degrees.

\paragraph{Local prediction.}
Readout increments fitted on 512 NeRF rays predict nonlinear RMS on 256 separate rays within $.073\%$ at density amplitude $.001$ and $.72\%$ at $.01$. Expanding the hidden prefix from 16 to 128 decreases training residual $2.48\times$ but increases held-out residual $1.57\times$: accurate prediction of a fitted update does not guarantee camera transfer.

\paragraph{Finite changes.}\phantomsection\label{app:officialrefit}\label{app:nerfactocompensation}
All arms fix density, spatial features and proposals, and use disjoint fitting/validation/test rays. The trained NeRF uses 4,096/1,024 fitting/validation rays and 3,000 updates; released NeRF and Nerfacto use 3,072/512 and 4,000. Validation selects the checkpoint, including initialization. The $1.25\times$ scale and separately calibrated positive-bias edits give independently measured relative density changes $22.73$--$25.55\%$.

\begin{table}[!ht]\centering\small
\caption{All finite native-head refits. Entries give frozen/refitted test RGB RMSE ratios in R/D/H order; values above one mean improvement. No arm meets the joint criterion of $5\times$ reduction, reference agreement $\geq35$ dB, density change $\geq20\%$ and sampling consistency.}\label{tab:officialrefit}
\begin{tabular}{@{}llrrr@{}}\toprule
Reference & Density edit & R & D & H\\\midrule
Trained NeRF Lego & Scale & 1.216 & 1.189 & 1.106\\
 & Bias & 1.319 & 1.411 & 1.643\\
Released NeRF Fern & Scale & 1.025 & 1.000 & 1.000\\
 & Bias & 1.433 & 1.469 & 1.472\\
Released NeRF Lego & Scale & .999 & .942 & 1.000\\
 & Bias & 1.126 & 1.185 & 1.275\\
Nerfacto Fern & Scale & 1.082 & 1.100 & 1.141\\\bottomrule
\end{tabular}\end{table}
All 50,000 trained-NeRF, 35,910/130,000 released-Fern/Lego and 142,884 Nerfacto test pixels are included. For Nerfacto, agreement rises from $40.75$ dB before refitting to $41.43/41.58/41.89$ dB. Locking the last-sample background to the reference gives almost the same ratios, $1.081/1.097/1.140$. The corresponding reference photograph PSNR across cameras is $27.69$--$29.71$ dB; full-head refitting gives $27.32$--$29.38$ dB. Agreement with a learned reference and accuracy against photographs measure different quantities.

\paragraph{Integration and observation population.}\phantomsection\label{app:nerfactosamplingdiagnostic}\label{app:nerfactofixedrefinement}
Sampling checks pass for all six trained-NeRF arms and three released-Lego scale arms, but fail for the other nine released-checkpoint arms and all three native-sampled Nerfacto arms. On 96 fixed Nerfacto probes, interval subdivision with frozen backgrounds stabilizes empirically at 32/64/128 subdivisions; final maximum discrepancy changes are $.0368/.0168$ gray levels. Frozen/refitted ratios are below one both before and after refinement (finest R/D/H: $.9620/.9521/.9890$). These probes already differ from full-image averages at coarse resolution; they establish neither a continuous error bound nor loss of the full-image gain. The prescribed Nerfacto cut has negligible transported weight and is uninformative for reuse. The evidence supports local prediction and modest renderer-specific compensation, without a general reconstruction improvement.

\subsection{Matched spatial representation and angular control}\label{app:sharedshdiagnostic}
At the frozen Nerfacto Fern reference, shared trilinear $6^3/8^3$ grids cover contracted coordinates $[-2,2]^3$. Density is unchanged, scaled by $1.25$, or multiplied by $1+.5u_x$, $u_x\in[0,1]$. SH0/SH1/SH2 fitting uses 1,024 rays and 256 validation rays; 128 uniformly sampled pixels from each test camera (0, 8, 16) are fixed before scoring. Fixed intervals/backgrounds use 64 subdivisions for fitting and 32/64/128 for evaluation. All arms reach numerical training-RMSE brackets of width $\leq.05/255$; validation selects the reported iterate.

\paragraph{Replacement versus texture-preserving modification.}
Replacing the complete native color with shared SH color gives only $15.52/16.61$ dB agreement at the unedited reference, even at SH2. Although angular terms improve edited-density test RMSE by $3.66$--$8.09\%$, the increment after subtracting their unchanged-density benefit is only $.012$--$.163$ gray levels. All eight comparisons also fail the final integration criterion. This replacement class is a poor model of the native texture.

To remove that baseline mismatch, a separately frozen diagnostic preserves the original color $c_*$ and fits
\begin{equation}\label{control:residualcolor}
 c_u(x,d)=c_*(x,d)+\min\{c_*(x,d),1-c_*(x,d)\}\,M_u(x,d),
 \quad M_u=\sum_v b_v(x)\sum_k u_{vk}\phi_k(d).
\end{equation}
The formula is componentwise. Trilinear weights satisfy $b_v\geq0$, $\sum_vb_v=1$; the polynomial SH basis is
$(1,d_x,d_y,d_z,d_xd_y,d_xd_z,d_yd_z,d_x^2-d_y^2,(3d_z^2-1)/2)$, truncated at $L$. Since $|\phi_k|\leq1$, imposing $\sum_k|u_{vk}|\leq1$ at every node/channel gives $|M_u|\leq1$, hence $0\leq c_u\leq1$. The zero update reproduces the reference exactly. Degree describes the modulation; the final color is not SH-limited. The same splits, edits, grids and numerical tolerances give another 18 fits. This protocol was chosen after observing replacement underfitting and is a diagnostic follow-up, not independent confirmation.

\paragraph{Training compensation and camera transfer.}
All eight edited-density angular comparisons reduce training RMSE $5.83$--$12.39\%$ relative to degree-zero modulation, but increase held-out RMSE $1.44$--$6.81\%$. All four scale comparisons pass the final per-camera refinement thresholds ($.05/255$ absolute discrepancy change and $10\%$ relative); the spatial-ramp comparisons do not. Thus even the numerically stable scale arm fails to transfer the fitting benefit. Neither diagnostic supplies the required $5\%$ test improvement, one-gray-level extra compensation and positive gain in every test camera. On the same 384 refined test rays, the previously fitted native R/D/H heads still reduce error by $7.1$--$8.9\%$ relative to frozen color: these are different update classes and comparisons. Independent arithmetic and native-field replay checks cover all 36 fits. The controls establish the distinction between representational fit, optimization and camera transfer; they do not demonstrate strong angular compensation of unknown physical density in photographs.

\section{From the local ray norm to positive pixels in a smooth finite model}\label{long:pixeltransfer}
This section proves the transfer summarized in \cref{long:pixelmain}. The model is fixed before the measurements are selected. The argument uses finite-range norm equivalence and Fourier interpolation, in the setting of finite-measurement stability \cite{alberti2021finite,bourgeois2013}, while retaining the density-only target and the complete permitted appearance class.

\subsection{Model and hypotheses}
The domain is $B_1\subset\R^3$, with fixed exterior and black background. Oriented lines are $r=(p,d)$, where $p\perp d$, with metric
\begin{equation}\label{lp:metric}
d_{\cR}((p,d),(p',d'))=|p-p'|+\angle(d,d').
\end{equation}
Fix $0<a<1$. All unknown density and source-coefficient changes are supported in an interior ball contained in $B_a$. Hence their rendering differences vanish for $|p|>a$.

Choose fixed finite smooth compactly supported density basis functions and smooth spatial source basis functions multiplied by SH modes through a fixed degree $L$. The density basis is linearly independent. Let $s$ and $c$ be their coordinates and $(s_0,c_0)$ the reference; the density mass norm is chosen so that
$H=\|s-s_0\|_D=\|\sigma_s-\sigma_{s_0}\|_{L^2}$.
The renderer is affine in the source coordinates:
\begin{equation}\label{lp:model}
F(s,c)=b(s)+B(s)c.
\end{equation}
The baseline-source contribution $b(s)$ may depend on density. Let $S$ be a compact density prior star-shaped about $s_0$, and allow every $c$ in a declared physical set $C(s)$ with $\|c\|\le M_c$ uniformly. All candidates obey the support and $C^6$ neighborhood assumptions of the local theorem, including $0\le j_k\le\sigma$.

For smooth RGB ray functions with the stated offset support, write
\begin{equation}\label{lp:norms}
R(w)=\max_{r,k}|w_k(r)|,\quad
P(w)=\max_{r,k}|\nabla_p w_k(r)|,\quad
J(w)=\max_k\operatorname{Lip}_{d_{\cR}}w_k.
\end{equation}
Let $Z$ carry the norm $R+P+J$. Assume $B(s)$ and $b(s)-b(s_0)$ are continuously differentiable into this space on a neighborhood of $S$. The substantive inverse hypothesis is
\begin{equation}\label{lp:Yhyp}
H\le C_Y\|F(s,c)-F(s_0,c_0)\|_Y
\quad\text{for every }s\in S,\ c\in C(s).
\end{equation}
It follows by restriction of \cref{thm:nonlinear} when that theorem's positive-gain, gauge, support and neighborhood conditions hold. Finite dimensionality alone is insufficient. The physical-coordinate coefficient is unchanged under the support rescaling: both density $L^2$ and the $Y$ norm of a rendering difference scale by $r^{3/2}$ when the rescaled ray image is divided by $r$.

\subsection{Derivative bounds over the full appearance class}
Put $V_0=\operatorname{ran}B(s_0)$. Since this is a finite-dimensional space of actual image functions, there exist $A_p,A_q\ge1$ such that
\begin{equation}\label{lp:finiteimage}
P(f)\le A_pR(f),\qquad J(f)\le A_qR(f),\qquad f\in V_0.
\end{equation}
Use $A_p=A_q=1$ if $V_0=\{0\}$. This does not require an inverse on source coefficients: $B(s_0)$ may have a kernel. Choose finite bounds
\begin{align}
K_0&\ge\sup_{s,u,c}R(Db(s)[u]+DB(s)[u]c),\nonumber\\
K_p&\ge\sup_{s,u,c}P(Db(s)[u]+DB(s)[u]c),\label{lp:K}\\
K_q&\ge\sup_{s,u,c}J(Db(s)[u]+DB(s)[u]c),\nonumber
\end{align}
where $s\in S$, $\|u\|_D=1$, and $\|c\|\le M_c$. Compactness and the stated differentiability give finiteness. The coefficient ball here only bounds the derivative; it does not enlarge the physical candidate class.

For $g=F(s,c)-F(s_0,c_0)$ use the exact split
\[
f=B(s_0)(c-c_0),\qquad
v=b(s)-b(s_0)+[B(s)-B(s_0)]c,\qquad g=f+v.
\]
The segment from $s_0$ to $s$ lies in $S$. Integrating the density derivative with $c$ fixed gives $R(v)\le K_0H$, $P(v)\le K_pH$ and $J(v)\le K_qH$. Consequently, with $R=R(g)$,
\begin{equation}\label{lp:derivative}
\begin{aligned}
P(g)&\le A_pR+D_pH,&D_p&=A_pK_0+K_p,\\
J(g)&\le A_qR+D_qH,&D_q&=A_qK_0+K_q.
\end{aligned}
\end{equation}
This applies to all permitted source increments, including changes invisible at the reference and visible at another density. No bound proportional to $H$ is imposed on $c-c_0$.

\subsection{A quantitative maximum-RGB and pixel gain}
\begin{theorem}[Smooth finite-model pixel transfer]\label{lp:theorem}
Under the preceding hypotheses set $T=3a^2C_Y^2$ and
\begin{equation}\label{lp:gamma}
\gamma=\frac{2}{TD_p+\sqrt{T^2D_p^2+4TA_p}}>0.
\end{equation}
Then $R(g)\ge\gamma H$. Suppose the selected ray centers are a $\delta$-net of $Q_a=\{(p,d):|p|\le a\}$ and each measurement is a nonnegative normalized average supported within distance $\rho$ of its center in \eqref{lp:metric}. If
\begin{equation}\label{lp:threshold}
t=\delta+\rho<\frac{\gamma}{A_q\gamma+D_q},\qquad
\gamma_{\rm pix}=\gamma-t(A_q\gamma+D_q),
\end{equation}
then the maximum channelwise pixel difference is at least $\gamma_{\rm pix}H$.
\end{theorem}
\begin{proof}
The measure of $|p|\le a$ over the direction sphere and its offset planes is $4\pi^2a^2$. With the unitary Fourier convention, Plancherel and Cauchy--Schwarz give
\begin{equation}\label{lp:interpolation}
\|g\|_Y^2\le\frac{\|g\|_{L^2}\|\nabla_pg\|_{L^2}}{4\pi^2}
\le3a^2R P(g)\le3a^2R(A_pR+D_pH).
\end{equation}
The factor three sums the RGB channels. Combining this inequality with \eqref{lp:Yhyp}, for $H>0$ the ratio $x=R/H$ satisfies $1\le T(A_px^2+D_px)$. The positive root is \eqref{lp:gamma}; the case $H=0$ is immediate.

Let $e=\max_{i,k}|\int g_k\,d\mu_i|$ for the pixel measures. Positivity, unit mass and the coverage assumption give $R\le e+tJ(g)$. By \eqref{lp:derivative},
\[
e\ge(1-tA_q)R-tD_qH.
\]
The strict threshold implies $1-tA_q>0$, allowing substitution of $R\ge\gamma H$. The result follows.
\end{proof}

If the reference and a candidate both fit the same pixel data with errors $\eta_0,\eta_1$ and their forward calculations have certified errors $\tau_0,\tau_1$, the triangle inequality gives
\begin{equation}\label{lp:noise}
\|\sigma_s-\sigma_{s_0}\|_{L^2}
\le\frac{\eta_0+\eta_1+\tau_0+\tau_1}{\gamma_{\rm pix}}.
\end{equation}
These are maximum linear-RGB budgets. The coefficient is relative to one fixed reference and model; it is not a uniform coefficient across arbitrary resolutions, degrees or rank changes. The absence of an additive sampling floor is a consequence of \eqref{lp:derivative}, not an assertion that raw pixels are bounded functionals on the entire space $Y$.

\subsection{Regularity, quantitative constants and physical cameras}
For the smooth finite restriction of the local rendering class, the $Z$ regularity follows by writing the ray integrals over a fixed interval and differentiating transport weights on an interior cutoff. Source basis images are weighted integrals of smooth compactly supported functions. Density derivatives additionally contain a fixed density basis and the corresponding remaining radiance. The entry/exit denominators are bounded below by $\sqrt{1-b^2}$ on a cutoff ball $B_b$, $b<1$. Bounded coefficients and dominated differentiation give the required uniform first spatial/directional derivatives and continuous density derivatives. The finite density mass norm controls the smooth density seminorms. This reasoning uses the interior support, so it does not require smooth absolute reference data at the outer grazing boundary.

The finite-image constants can in principle be enclosed. Choose an independent basis $w_1,\ldots,w_m$ of $V_0$. There are $m$ query/channel evaluations forming an invertible matrix $E$, since otherwise a nonzero linear combination would vanish at every query and channel. An enclosed inverse norm gives sufficient choices
\[
A_p\ge\|E^{-1}\|_\infty\sum_jP(w_j),\qquad
A_q\ge\|E^{-1}\|_\infty\sum_jJ(w_j),
\]
enlarged to one. This matrix only bounds derivatives in the proven reference image space; it does not replace \eqref{lp:Yhyp} by a finite Jacobian-rank check.

A finite net can be realized by exterior pinhole cameras. For a center line $(p,d)$ choose $R_c>1$ and camera center $o=p-\sqrt{R_c^2-|p|^2}\,d$. Its offset map is $p_o(d')=o-(o\cdot d')d'$, with
$|p_o(d')-p_o(d)|\le2R_c\angle(d,d')$.
A positive angular footprint of radius
\[
0<\theta<\min\{\rho/(2R_c+1),\tfrac12\arccos(a/R_c)\}
\]
therefore has line-space radius at most $\rho$ and stays in the inward hemisphere. Its physical forward-ray reading agrees with the complete-line integral. Compactness supplies a finite net, and one camera per center is a sufficient existence construction. Its efficiency is a separate quantitative question addressed by the next section's scope bound.

\section{Physical finite-grid color compensation on a fixed pixel design}\label{long:gridcomp}
This corollary complements \cref{lp:theorem}: the measurements are fixed first, and spatial color resolution is then selected. It uses exact continuous pixel integrals and a common isotropic color field.

\subsection{Reference, response and a fixed grid}
Let $\Omega$ be bounded and convex, with diameter $\ell$, fixed black background and vacuum exterior. Let $D\Subset\Omega$ be an open ball. For each of $m$ fixed positive normalized camera-pixel measures $\nu_i$, define
\begin{equation}\label{fg:F}
F_i(\sigma,c)=\int\int_{r}T_\sigma(t)\sigma(r(t))c(r(t))\,dt\,d\nu_i(r).
\end{equation}
Rays missing the domain contribute zero. Take a continuous reference color and a smooth nonnegative reference density, with $\Sigma=\|\sigma_*\|_\infty$, $\sigma_*\ge s_->0$ on $D$, $0\le c_{*,k}\le1$ everywhere and
$\nu\le c_{*,k}\le1-\nu$ on $D$, where $0<\nu\le1/2$ is the color margin.

Assume an exact response right inverse: there exist $\psi_a\in C_c^\infty(D)$, $1\le a\le m$, such that
\begin{equation}\label{fg:A}
A_*[i,a]=\int\int_r T_*\sigma_*\psi_a\,dt\,d\nu_i
\end{equation}
is nonsingular. The pixel-source construction in \cref{rs:pixelspatial} supplies this condition in its stated geometry: divide its source bumps by $\sigma_*$ on their compact supports. It is a substantive measurement condition, not a consequence of counting grid nodes.

Use induced maximum matrix norms and define
\[
M=\|A_*^{-1}\|_\infty,\quad B=\Big\|\sum_a|\psi_a|\Big\|_\infty,
\quad\Lambda=\sum_a\operatorname{Lip}(\psi_a),\quad
d_0=\operatorname{dist}\big(\cup_a\supp\psi_a,D^c\big)>0.
\]
Extend the bumps by zero, and let $I_\epsilon$ be nodal trilinear interpolation on a Cartesian grid whose side lengths are at most $\epsilon$. Its nonnegative weights sum to one. Thus $\psi_a^\epsilon=I_\epsilon\psi_a$ obey
\begin{equation}\label{fg:interp}
\sum_a|\psi_a^\epsilon(x)|\le B,\qquad
\sum_a\|\psi_a^\epsilon-\psi_a\|_\infty\le\sqrt3\epsilon\Lambda.
\end{equation}
Choose and fix any grid with
\begin{equation}\label{fg:grid}
\sqrt3\epsilon<d_0,\qquad M\ell\Sigma\sqrt3\epsilon\Lambda\le\tfrac12.
\end{equation}
The first inequality preserves interior support. The exact matrix $A_\epsilon$ obtained from \eqref{fg:A} satisfies $\|A_\epsilon-A_*\|_\infty\le\ell\Sigma\sqrt3\epsilon\Lambda$. The Neumann bound gives $\|A_\epsilon^{-1}\|_\infty\le2M$.

\subsection{Exact nonlinear compensation with physical colors}
\begin{proposition}[One fixed color grid suffices locally]\label{fg:theorem}
With the preceding grid fixed, put
\begin{equation}\label{fg:radius}
\delta=\min\left\{\frac{s_-}{2},\frac{1}{4M\ell B(1+\ell\Sigma)},
\frac{\nu}{8M\ell B}\right\}.
\end{equation}
For every continuous density change $h$ compactly supported in $D$ with $H=\|h\|_\infty\le\delta$, there is an isotropic color increment $q_\epsilon\in\operatorname{span}\{\psi_a^\epsilon\}^3$ with
\begin{equation}\label{fg:partner}
\begin{gathered}
F(\sigma_*+h,c_*+q_\epsilon)=F(\sigma_*,c_*),\\
\|q_{\epsilon,k}\|_\infty\le4M\ell BH,\qquad
\nu/2\le c_{*,k}+q_{\epsilon,k}\le1-\nu/2\quad\text{on }D.
\end{gathered}
\end{equation}
All changes vanish outside $D$, and the resulting effective source obeys $0\le j_k\le\sigma_*+h$ everywhere.
\end{proposition}
\begin{proof}
Write $\tau=\sigma_*+h\ge0$. Nonnegative optical depth gives
\[
|T_\tau-T_*|\le\ell H,\qquad
|T_\tau\tau-T_*\sigma_*|\le(1+\ell\Sigma)H.
\]
Using \eqref{fg:interp}, the exact grid color response at $\tau$ satisfies
$\|A_\epsilon(\tau)-A_\epsilon\|_\infty\le\ell B(1+\ell\Sigma)H$.
By \eqref{fg:radius} its relative perturbation with respect to $A_\epsilon$ is at most $1/2$, so $\|A_\epsilon(\tau)^{-1}\|_\infty\le4M$.

For fixed \emph{color}, the exact transport difference is
\[
F_r(\tau,c_{*,k})-F_r(\sigma_*,c_{*,k})
=\int_r T_\tau h(c_{*,k}-U_{*,k})\,dt.
\]
Both $c_{*,k}$ and $U_{*,k}$ lie in $[0,1]$, hence the maximum pixel discrepancy is at most $\ell H$. Set
\[
\alpha_k=A_\epsilon(\tau)^{-1}
 [F(\sigma_*,c_{*,k})-F(\tau,c_{*,k})],\qquad
q_{\epsilon,k}=\sum_a\psi_a^\epsilon\alpha_{ak}.
\]
Linearity in color at fixed changed density gives exact equality. Also $\|\alpha_k\|_\infty\le4M\ell H$, and \eqref{fg:interp} yields the asserted correction bound. The last inequality in \eqref{fg:radius} makes that bound at most $\nu/2$. Support and physicality follow.
\end{proof}

If the reference density and color lie on a fixed coarse trilinear grid and the refinements are nested, both remain on every finer grid. Restricting $h$ to that finer grid gives two ordinary grid density/color fields with identical continuous pixels. The effective source is the product $\sigma c$ and need not itself be trilinear. The class is nonempty: a sufficiently small positive constant density and an interior constant color satisfy the pixel-response construction in \cref{rs:pixelspatial}; fine grids contain interior nodal hats. A nonzero small multiple of one such hat gives distinct densities tending to the reference in nodal coordinates.

For a general smooth reference, the assertion instead concerns the affine color class $c_*+V_\epsilon$. No claim of exact grid representation of that reference is needed. The amplitudes and grid threshold depend on the fixed response and may be small or expensive. The proposition establishes a failure of local injectivity for the declared finite observations, not complete-ray equality or held-out view agreement. It does not contradict the model-first positive certificate in \cref{lp:theorem}.

\section{Camera cost of the full-ray covering certificate}\label{long:cameracost}
The following bound evaluates the geometry required by \cref{lp:threshold}. It is a cost of this sufficient proof strategy. A separate exact example below rules out interpreting it as a universal number of cameras necessary for density recovery.

\subsection{A covering obstruction}
Use the unit ball and metric \eqref{lp:metric}. There are $M$ exterior camera centers $o_i=R_i u_i$, $R_i>1$, $|u_i|=1$. Selected pixels point inward and meet the ball on their forward half-ray. If their center lines form a $\delta$-net of $Q_a$ with $\delta<1$, then they approximate every central line $(0,d)$.

\begin{lemma}[Central-line camera bound]\label{lc:bound}
For $M\ge2$, such coverage requires
\begin{equation}\label{lc:ineq}
M(1-\sqrt{1-\delta^2})\ge2,\qquad
\delta^2\ge\frac{4(M-1)}{M^2}.
\end{equation}
One or two cameras cannot supply this coverage with $\delta<1$.
\end{lemma}
\begin{proof}
For a line $(p,v)$ from center $i$ within $\delta$ of $(0,d)$, put $x=|p|$ and $q=\angle(v,d)$. Then $x+q\le\delta$. Inward geometry gives $\angle(-u_i,v)=\arcsin(x/R_i)$, so
\[
\angle(-u_i,d)\le\arcsin(x)+\delta-x\le\arcsin\delta.
\]
The last function is nondecreasing in $x\in[0,\delta]$. Thus $M$ caps of angular radius $\arcsin\delta$ cover the sphere. Their area is $2\pi(1-\sqrt{1-\delta^2})$ each, and area subadditivity gives the first inequality. For $M\ge2$, squaring its nonnegative rearrangement gives the second. A single cap of radius less than $\pi/2$ cannot cover the sphere, and $M=2$ would force $\delta\ge1$.
\end{proof}

Suppose $V_0\ne\{0\}$ and its functions vanish for $|p|>b$, with $0<b\le a$. Any admissible $A_q$ in \eqref{lp:finiteimage} obeys $A_q\ge1/b$: take a nonzero component attaining $R(w)$ and move its offset at fixed direction to $|p|=b$, where continuity gives value zero. This path has length at most $b$, so $R(w)\le bJ(w)\le bA_qR(w)$.

Consequently \eqref{lp:threshold} requires $\delta+\rho<b$. If $r=b-\rho\in(0,1)$, \cref{lc:bound} implies the strict necessary count
\begin{equation}\label{lc:cost}
r^2M^2>4(M-1),\quad M\ge3,
\qquad M>\frac{2}{1-\sqrt{1-r^2}}.
\end{equation}
For ideal $\rho=0$, $b=2/5,1/4,1/10,1/100,1/1000$ gives smallest integers not excluded by this test of $24,63,399,39999,3999999$, respectively. With $b=2/5$ and $\rho=1/20$ the count is $32$. These necessary counts ignore cap overlap, noncentral offsets and the other transfer constants. They do not provide sufficient designs. In particular, five fixed camera centers cannot meet this sufficient certificate when $b\le4/5$, even with arbitrarily fine point sampling at those centers.

\subsection{One camera can identify a restricted physical volume}
Choose a nonnegative smooth radial bump $\phi$ supported in $B_b$, positive near the center and with central line integral one. The model is $\sigma_s=s\phi$, $j_{s,c}=cs\phi$, with unknown $s>0$ and $0<c<1$ in compact positive intervals. Exact rendering gives
\[
I(p)=c(1-e^{-sA(p)}),\qquad A(p)=X\phi(p).
\]
Continuity and radiality give $A(0)=1$, $A(b)=0$ and an offset with $A=1/2$. One exterior camera supplies inward rays at both offsets. If $I_1$ is the half-integral reading, $I_2$ the central reading, and $q=e^{-s/2}$, then
\begin{equation}\label{lc:onecamera}
q=I_2/I_1-1,\qquad s=-2\log q,\qquad
c=\frac{I_1^2}{2I_1-I_2}.
\end{equation}
The physical image for unrestricted $s>0$, $0<c<1$ is exactly
$0<I_1<I_2<2I_1-I_1^2$; direct substitution proves both directions. On a compact positive parameter rectangle, $I_1,q,1-q$ stay away from zero, so the inverse has bounded derivatives in a neighborhood of its image. It gives local small-noise stability, with density error $|s-s'|\|\phi\|_2$.

This ideal two-point-pixel example has a strong radial/shared-color prior. Its explicit recovery shows why a limitation of a four-dimensional line-covering proof cannot become a universal camera-count lower bound. Similarly, analytic dependence on finite coefficients does not imply analytic dependence on all rays: this nonzero smooth image vanishes on the open offset annulus outside $B_b$, which would force an analytic ray image to vanish identically.

\clearpage
\phantomsection


\addcontentsline{toc}{part}{References}
\begin{thebibliography}{46}
\providecommand{\natexlab}[1]{#1}
\providecommand{\url}[1]{\texttt{#1}}
\expandafter\ifx\csname urlstyle\endcsname\relax
  \providecommand{\doi}[1]{doi: #1}\else
  \providecommand{\doi}{doi: \begingroup \urlstyle{rm}\Url}\fi

\bibitem[Alberti and Santacesaria(2021)]{alberti2021finite}
Giovanni~S. Alberti and Matteo Santacesaria.
\newblock Infinite-dimensional inverse problems with finite measurements.
\newblock \emph{Archive for Rational Mechanics and Analysis}, 2021.

\bibitem[Assylbekov(2012)]{assylbekov2012}
Yernat~M. Assylbekov.
\newblock Attenuated tensor tomography on simple surfaces, 2012.
\newblock arXiv:1204.4388.

\bibitem[Assylbekov and Stefanov(2020)]{assylbekovstefanov2020}
Yernat~M. Assylbekov and Plamen Stefanov.
\newblock Sharp stability estimate for the geodesic ray transform.
\newblock \emph{Inverse Problems}, 36\penalty0 (2):\penalty0 025013, 2020.

\bibitem[Bal and Monard(2019)]{balmonard2019}
Guillaume Bal and Fran{\c c}ois Monard.
\newblock Inverse source problems in transport via attenuated tensor
  tomography, 2019.
\newblock arXiv:1908.06508.

\bibitem[Beck and Teboulle(2009)]{beck2009}
Amir Beck and Marc Teboulle.
\newblock A fast iterative shrinkage-thresholding algorithm for linear inverse
  problems.
\newblock \emph{SIAM Journal on Imaging Sciences}, 2\penalty0 (1):\penalty0
  183--202, 2009.

\bibitem[Boyd and Vandenberghe(2004)]{boyd2004}
Stephen Boyd and Lieven Vandenberghe.
\newblock \emph{Convex Optimization}.
\newblock Cambridge University Press, 2004.

\bibitem[Ceki\'c and Lefeuvre(2024)]{cekic2024}
Mihajlo Ceki\'c and Thibault Lefeuvre.
\newblock Generic injectivity of the {X}-ray transform, 2024.
\newblock Revised preprint.

\bibitem[Chai et~al.(2000)Chai, Tong, Chan, and Shum]{chai2000}
Jin-Xiang Chai, Xin Tong, Shing-Chow Chan, and Heung-Yeung Shum.
\newblock Plenoptic sampling.
\newblock In \emph{Proceedings of SIGGRAPH}, pages 307--318, 2000.

\bibitem[Chambolle and Pock(2011)]{chambolle2011}
Antonin Chambolle and Thomas Pock.
\newblock A first-order primal-dual algorithm for convex problems with
  applications to imaging.
\newblock \emph{Journal of Mathematical Imaging and Vision}, 40:\penalty0
  120--145, 2011.

\bibitem[Chu et~al.(2022)Chu, Liu, Zheng, Franz, Seidel, Theobalt, and
  Zayer]{chu2022}
Mengyu Chu, Lingjie Liu, Quan Zheng, Erik Franz, Hans-Peter Seidel, Christian
  Theobalt, and Rhaleb Zayer.
\newblock Physics informed neural fields for smoke reconstruction with sparse
  data.
\newblock \emph{ACM Transactions on Graphics}, 41\penalty0 (4):\penalty0
  119:1--119:15, 2022.

\bibitem[Courdurier et~al.(2015)Courdurier, Monard, Osses, and
  Romero]{courdurier2015}
Matias Courdurier, Fran{\c c}ois Monard, Axel Osses, and Francisco Romero.
\newblock Simultaneous source and attenuation reconstruction in {SPECT} using
  ballistic and single scattering data, 2015.
\newblock arXiv:1503.07201.

\bibitem[de~Hoop and Ilmavirta(2017)]{dehoop2017}
Maarten~V. de Hoop and Joonas Ilmavirta.
\newblock Abel transforms with low regularity with applications to {X}-ray
  tomography on spherically symmetric manifolds, 2017.
\newblock arXiv:1702.07625.

\bibitem[Eckert et~al.(2019)Eckert, Um, and Thuerey]{eckert2019}
Marie-Lena Eckert, Kiwon Um, and Nils Thuerey.
\newblock {ScalarFlow}: A large-scale volumetric data set of real-world scalar
  transport flows for computer animation and machine learning.
\newblock \emph{ACM Transactions on Graphics}, 38\penalty0 (6):\penalty0 239,
  2019.

\bibitem[Erd{\'e}lyi(2014)]{erdelyi2014}
Tam{\'a}s Erd{\'e}lyi.
\newblock Basic polynomial inequalities on intervals and circular arcs.
\newblock \emph{Constructive Approximation}, 39:\penalty0 367--384, 2014.

\bibitem[Fang et~al.(2023)Fang, Song, Li, and Bo]{fang2023}
Qihang Fang, Yafei Song, Keqiang Li, and Liefeng Bo.
\newblock Reducing shape-radiance ambiguity in radiance fields with a
  closed-form color estimation method.
\newblock In \emph{Advances in Neural Information Processing Systems}, 2023.

\bibitem[Filmus et~al.(2026)Filmus, Moran, and
  Nesterova]{filmus2026linearqueries}
Yuval Filmus, Shay Moran, and Elizaveta Nesterova.
\newblock Optimal reconstruction from linear queries, 2026.
\newblock Accepted to COLT 2026.

\bibitem[Fridovich-Keil et~al.(2022)Fridovich-Keil, Yu, Tancik, Chen, Recht,
  and Kanazawa]{plenoxels2022}
Sara Fridovich-Keil, Alex Yu, Matthew Tancik, Qinhong Chen, Benjamin Recht, and
  Angjoo Kanazawa.
\newblock Plenoxels: Radiance fields without neural networks.
\newblock In \emph{CVPR}, pages 5501--5510, 2022.

\bibitem[Goli et~al.(2024)Goli, Reading, Sell{\'a}n, Jacobson, and
  Tagliasacchi]{goli2024}
Lily Goli, Cody Reading, Silvia Sell{\'a}n, Alec Jacobson, and Andrea
  Tagliasacchi.
\newblock {BayesRays}: Uncertainty quantification for neural radiance fields.
\newblock In \emph{CVPR}, 2024.

\bibitem[Golub and Pereyra(1972)]{golubpereyra1972}
Gene~H. Golub and Victor Pereyra.
\newblock The differentiation of pseudoinverses and nonlinear least squares
  problems whose variables separate.
\newblock Technical Report STAN-CS-72-261, Computer Science Department,
  Stanford University, 1972.

\bibitem[Golub and Pereyra(2003)]{golubpereyra2003}
Gene~H. Golub and Victor Pereyra.
\newblock Separable nonlinear least squares: the variable projection method and
  its applications.
\newblock \emph{Inverse Problems}, 19\penalty0 (2):\penalty0 R1--R26, 2003.

\bibitem[{GRAPHDECO}(2023)]{gsimplementation}
{GRAPHDECO}.
\newblock Official implementation of {3D Gaussian Splatting}: Training
  schedule, 2023.

\bibitem[Heil(2017)]{heil2017}
Konstantin Heil.
\newblock \emph{Killing and Conformal Killing Tensors}.
\newblock PhD thesis, Universit{\"a}t Stuttgart, 2017.

\bibitem[Holman and Richardson(2023)]{holman2023}
Sean Holman and Philip Richardson.
\newblock Simultaneous recovery of attenuation and source density in {SPECT}.
\newblock \emph{Inverse Problems and Imaging}, 17\penalty0 (4):\penalty0
  817--840, 2023.

\bibitem[Jackson(1930)]{jackson1930}
Dunham Jackson.
\newblock \emph{The Theory of Approximation}.
\newblock American Mathematical Society, 1930.

\bibitem[Keil and Jansson(2006)]{keiljansson2006}
Christian Keil and Christian Jansson.
\newblock Computational experience with rigorous error bounds for the netlib
  linear programming library, 2006.
\newblock Preprint.

\bibitem[Kerbl et~al.(2023)Kerbl, Kopanas, Leimk{\"u}hler, and
  Drettakis]{kerbl2023}
Bernhard Kerbl, Georgios Kopanas, Thomas Leimk{\"u}hler, and George Drettakis.
\newblock {3D} gaussian splatting for real-time radiance field rendering.
\newblock \emph{ACM Transactions on Graphics}, 42\penalty0 (4), 2023.

\bibitem[Mildenhall et~al.(2020)Mildenhall, Srinivasan, Tancik, Barron,
  Ramamoorthi, and Ng]{mildenhall2020}
Ben Mildenhall, Pratul~P. Srinivasan, Matthew Tancik, Jonathan~T. Barron, Ravi
  Ramamoorthi, and Ren Ng.
\newblock {NeRF}: Representing scenes as neural radiance fields for view
  synthesis.
\newblock In \emph{ECCV}, pages 405--421, 2020.

\bibitem[Monard(2018)]{monard2018}
Fran{\c c}ois Monard.
\newblock Efficient tensor tomography in fan-beam coordinates. {II}: Attenuated
  transforms.
\newblock \emph{Inverse Problems and Imaging}, 12\penalty0 (2):\penalty0
  433--460, 2018.

\bibitem[Ndiaye et~al.(2017)Ndiaye, Fercoq, Gramfort, and Salmon]{ndiaye2017}
Eugene Ndiaye, Olivier Fercoq, Alexandre Gramfort, and Joseph Salmon.
\newblock Gap safe screening rules for sparsity enforcing penalties.
\newblock \emph{Journal of Machine Learning Research}, 18\penalty0
  (128):\penalty0 1--33, 2017.

\bibitem[Paternain and Zhou(2016)]{pz2016}
Gabriel~P. Paternain and Hanming Zhou.
\newblock Invariant distributions and the geodesic ray transform.
\newblock \emph{Analysis \& PDE}, 9\penalty0 (8):\penalty0 1903--1930, 2016.

\bibitem[Paternain et~al.(2012)Paternain, Salo, and Uhlmann]{psu2012}
Gabriel~P. Paternain, Mikko Salo, and Gunther Uhlmann.
\newblock The attenuated ray transform for connections and {Higgs} fields.
\newblock \emph{Geometric and Functional Analysis}, 22:\penalty0 1460--1489,
  2012.

\bibitem[Pharr et~al.(2023)Pharr, Jakob, and Humphreys]{pbr2023}
Matt Pharr, Wenzel Jakob, and Greg Humphreys.
\newblock \emph{Physically Based Rendering: From Theory to Implementation}.
\newblock MIT Press, 4 edition, 2023.

\bibitem[Polyanskiy and Wu(2024)]{polyanskiywu2024draft}
Yury Polyanskiy and Yihong Wu.
\newblock Information theory: From coding to learning, 2024.
\newblock Author manuscript, Chapter 31.

\bibitem[Poole et~al.(2023)Poole, Jain, Barron, and Mildenhall]{poole2023}
Ben Poole, Ajay Jain, Jonathan~T. Barron, and Ben Mildenhall.
\newblock {DreamFusion}: Text-to-3d using 2d diffusion.
\newblock In \emph{International Conference on Learning Representations}, 2023.

\bibitem[Quinto(1983)]{quinto1983}
Eric~Todd Quinto.
\newblock The invertibility of rotation invariant {Radon} transforms.
\newblock \emph{Journal of Mathematical Analysis and Applications}, 91\penalty0
  (2):\penalty0 510--522, 1983.

\bibitem[Rump(2010)]{Rump2010}
Siegfried~M. Rump.
\newblock Verification methods: Rigorous results using floating-point
  arithmetic.
\newblock \emph{Acta Numerica}, 19:\penalty0 287--449, 2010.

\bibitem[Sharafutdinov(1994)]{sharafutdinov1994}
Vladimir~A. Sharafutdinov.
\newblock \emph{Integral Geometry of Tensor Fields}.
\newblock VSP, Utrecht, 1994.

\bibitem[Smith and Blinn(1996)]{smithblinn1996}
Alvy~Ray Smith and James~F. Blinn.
\newblock Blue screen matting.
\newblock In \emph{Proceedings of SIGGRAPH}, pages 259--268, 1996.

\bibitem[Stefanov(2014)]{stefanov2014}
Plamen Stefanov.
\newblock The identification problem for the attenuated {X}-ray transform.
\newblock \emph{American Journal of Mathematics}, 136\penalty0 (5):\penalty0
  1215--1247, 2014.

\bibitem[Trefethen and Weideman(2014)]{trefethen2014}
Lloyd~N. Trefethen and J.~A.~C. Weideman.
\newblock The exponentially convergent trapezoidal rule.
\newblock \emph{SIAM Review}, 56\penalty0 (3):\penalty0 385--458, 2014.

\bibitem[Verbin et~al.(2022)Verbin, Hedman, Mildenhall, Zickler, Barron, and
  Srinivasan]{verbin2022}
Dor Verbin, Peter Hedman, Ben Mildenhall, Todd Zickler, Jonathan~T. Barron, and
  Pratul~P. Srinivasan.
\newblock {Ref-NeRF}: Structured view-dependent appearance for neural radiance
  fields.
\newblock In \emph{CVPR}, pages 5491--5500, 2022.

\bibitem[Wang and Miller(2026)]{wang2026}
Songhe Wang and David Miller.
\newblock When {3D} {Gaussian} splatting recovers real surfaces.
\newblock In \emph{ECCV}, pages 621--639, 2026.

\bibitem[Weiss and Westermann(2021)]{weiss2021}
Sebastian Weiss and R{\"u}diger Westermann.
\newblock Analytic ray splitting for controlled precision {DVR}.
\newblock In \emph{EuroVis 2021---Short Papers}, pages 31--35, 2021.

\bibitem[Yang et~al.(2011)Yang, Cao, and Xiong]{yang2011}
Ruyue Yang, Feilong Cao, and Jingyi Xiong.
\newblock The strong converse inequality for de la {Vall\'ee Poussin} means on
  the sphere, 2011.

\bibitem[Zhang et~al.(2020)Zhang, Riegler, Snavely, and Koltun]{zhang2020}
Kai Zhang, Gernot Riegler, Noah Snavely, and Vladlen Koltun.
\newblock {NeRF++}: Analyzing and improving neural radiance fields, 2020.
\newblock arXiv:2010.07492.

\bibitem[Zhu et~al.(2023)Zhu, Yang, Wang, Zheng, and Guibas]{zhu2023}
Bingfan Zhu, Yanchao Yang, Xulong Wang, Youyi Zheng, and Leonidas Guibas.
\newblock {VDN-NeRF}: Resolving shape-radiance ambiguity via view-dependence
  normalization.
\newblock In \emph{CVPR}, pages 35--45, 2023.


\bibitem[Bourgeois(2013)]{bourgeois2013}
Laurent Bourgeois. A remark on Lipschitz stability for inverse problems.
\emph{Comptes Rendus Mathematique}, 351(5--6):187--190, 2013.
\url{https://www.numdam.org/item/CRMATH_2013__351_5-6_187_0.pdf}.
\end{thebibliography}
\end{document}